\documentclass{article}

\usepackage{microtype}
\usepackage{graphicx}
\usepackage{subfigure}
\usepackage{booktabs} 

\newcommand{\theHalgorithm}{\arabic{algorithm}}
\usepackage{hyperref}
\usepackage[accepted]{icml2023}

\usepackage{amsmath}
\usepackage{amssymb}
\usepackage{mathtools}
\usepackage{amsthm}

\usepackage[capitalize,noabbrev]{cleveref}

\usepackage{mydefs}
\theoremstyle{plain}
\newtheorem{theorem}{Theorem}[section]
\newtheorem{proposition}[theorem]{Proposition}
\newtheorem{lemma}[theorem]{Lemma}
\newtheorem{corollary}[theorem]{Corollary}
\theoremstyle{definition}
\newtheorem{definition}[theorem]{Definition}
\newtheorem{assumption}[theorem]{Assumption}
\newtheorem{remark}[theorem]{Remark}

\usepackage[textsize=tiny]{todonotes}

\icmltitlerunning{\hfill Diffusion Models are Minimax Optimal Distribution Estimators \hfill \thepage}

\begin{document}

\twocolumn[
\icmltitle{Diffusion Models are Minimax Optimal Distribution Estimators 
}



\icmlsetsymbol{equal}{*}

\begin{icmlauthorlist}
\icmlauthor{Kazusato Oko}{UT,RIKEN}
\icmlauthor{Shunta Akiyama}{UT}
\icmlauthor{Taiji Suzuki}{UT,RIKEN}
\end{icmlauthorlist}

\icmlaffiliation{UT}{Department of Mathematical Informatics, the University of Tokyo, Tokyo, Japan}
\icmlaffiliation{RIKEN}{Center for Advanced Intelligence Project, RIKEN, Tokyo, Japan}

\icmlcorrespondingauthor{Kazusato Oko}{oko-kazusato@g.ecc.u-tokyo.ac.jp}

\icmlkeywords{Machine Learning, ICML}

\vskip 0.3in
]



\printAffiliationsAndNotice{}

\begin{abstract}
While efficient 
distribution learning is no doubt behind the groundbreaking success of diffusion modeling, 
its theoretical guarantees are quite limited.
In this paper, we provide the first rigorous analysis on approximation and generalization abilities of diffusion modeling for well-known function spaces.
The highlight of this paper is 
 that when the true density function belongs to the Besov space 
 and the empirical score matching loss is properly minimized, 
the generated data distribution achieves the nearly minimax optimal estimation rates in the total variation distance and in the Wasserstein distance of order one. 
Furthermore, we extend our theory to demonstrate how diffusion models adapt to low-dimensional data distributions.
We expect these results advance theoretical understandings of diffusion modeling and its ability to generate verisimilar outputs.
\end{abstract}

\section{Introduction}
Diffusion modeling, also called score-based generative modeling \citep{sohl2015deep,song2019generative,song2020score,ho2020denoising,vahdat2021score} has achieved state-of-the-art performance 
in image \citep{song2020score,dhariwal2021diffusion}, video \citep{ho2022video}, and audio \citep{chen2020wavegrad,kong2020diffwave}.

Borrowing explanation from the unifying framework of \citet{song2020score}, diffusion modeling first gradually adds noise to the data distribution, 
and transforms the distribution to a predefined noise distribution.
This time evolution, called the forward process, can be formulated as a stochastic differential equation (SDE) that is data independent.
On the other hand, we can consider the time-reversal of the SDE, and by following this so-called backward process, one can generate data from noise.
Importantly, the drift term of the backward process is dependent on the data distribution, specifically on the gradient of the logarithmic density (score) at each time of the forward process.

%

In practice, however, we have only access to the true distribution through 
a finite number of sample.
For this reason, 
the score of the diffusion process from the empirical distribution 
is utilized instead
\citep{vincent2011connection,sohl2015deep,song2019generative}.
Moreover, for computational efficiency, the empirical score is further replaced by a neural network (score network) that is
 close to the empirical score in terms of some 
loss function using score matching techniques \citep{hyvarinen2005estimation,vincent2011connection}.
In this way, diffusion modeling implicitly learns the true distribution via learning of the empirical score.

Then the following natural question immediately arises:
%
\\ \noindent
\textit{Is diffusion modeling is a good distribution estimator?
 In other words, how can the estimation error of the generated data distribution  
be explicitly bounded by the number of the training data and in a data structure dependent way?}

\paragraph{On the effect of score approximation errors}
Existing literature has analyzed the estimation error with either of the two assumptions on the accuracy of score approximation.
(i) One popular assumption is that the error of the loss function in score matching is sufficiently small, which was first used by \citet{song2020score} to bound the Kullback–Leibler (KL) divergence for continuous-time dynamics via Girsanov theorem.
Recently, the polynomial bound has appeared in discrete-time, meaning that the polynomial order of the error in score estimate at each step and number of steps suffice to obtain the final estimation error in the total variation (TV) distance \citep{leeconvergence}.
\citet{leeconvergence} assumed the smoothness and log-Sobolev inequality (LSI) for the true density, and \citet{chen2022sampling} and \citet{lee2022convergence} eliminated the LSI but still with the smoothness.
Also, following \citet{song2020score}, \citet{pidstrigach2022scorebased} considered the true distribution on a manifold.
(ii) Another assumption is to bound the difference between the score and the network at each time and point.
\citet{de2021diffusion} (also with dissipativily) and \citet{de2022convergence} (under the manifold hypothesis) 
 derived non-polynomial bounds
in TV and in the Wasserstein distance of order one ($W_1$), respectively. 

\vspace{-1.5mm}
\paragraph{Generalization error analyses}
However, most of the literature assumes availability of the true score, and thus whether the score is appropriately approximated with a finite number of sample has been unaddressed, and therefore a doubt in reality of the above assumptions undermines the value of the resulting estimation error bounds.
As the only exception, \citet{de2022convergence} derived the $n^{-1/d}$ bound in $W_1$ for $n$ data and a $d$-dimensional distribution. 
However, in their analysis, the neural network is assumed to almost perfectly fit the empirical score and the estimation bound depends on the convergence rate of the empirical distribution to the true one \citep{weed2019sharp}.
Because of the same lower bound for the convergence of empirical measures \citep{dudley1969speed}, their $n^{-1/d}$ bound is  essentially unimprovable with any structural assumption on the data distribution. Therefore, it is impossible to extend their result to 
formal density estimation problems, where the faster convergence rates depending on the smoothness of the true density are expected.
We also mention generalization error analysis mainly on each one discretized step by \citet{block2020generative}, but they do not explicitly state the final estimation error and their intermediate bounds depend on the unknown Rademacher complexity which should be sufficiently large so that the hypothesis class well approximates the true score.

Thus, the fundamental question on the performance of diffusion models as a distribution learner largely remains open.
\subsection{Our contributions}
In this paper, we establish a statistical learning theory for diffusion modeling.
The convergence rate of the estimation error is derived assuming that the true density belongs to well-known function spaces and deep neural network is employed as an estimator.
Surprisingly, we find that diffusion modeling can achieve the nearly minimax estimation rates.
The contributions of this paper are detailed as follows:
\vspace{-1.5mm}
\begin{itemize}
    \item[(i)] 
    We give the explicit form of approximation of the score 
    with a neural network and derive the error bound in $L^2(p_t)$ at each $t$, where the initial density is supported in $[-1,1]^d$, in the Besov space $B_{p,q}^s([-1,1]^d)$, and smooth in the boundary.
\vspace{-.8mm}  
     \item[(ii)] We convert the approximation error analysis into the estimation error bounds. 
    We derive the bound of $n^{-\frac{s}{d+2s}}$ in TV. Moreover, the rate of $n^{-\frac{s+1-\delta}{d+2s}}$ in $W_1$ is derived for an arbitrary fixed $\delta>0$ under the modified score matching, via careful discussion of stochastic calculus.
    As a result, the obtained estimation rates are nearly minimax optimal, theoretically proving the success of diffusion models.
\vspace{-.8mm}  
    \item[(iii)] By extending our theory, we also demonstrate that the diffusion models avoid the curse of dimensionality under the manifold hypothesis, considering when the true data is distributed over the low-dimensional plane.
    This is a special case of \citet{de2022convergence} but our bound is by far tight in this case.
\end{itemize}
\subsection{Other related works}
\vspace{-0.5mm}
Recently, minimax estimation rates in the Wasserstein distance have been investigated by several works
(empirical distribution \citep{weed2019sharp,singh2018minimax,lei2020convergence}; smooth density
\citep{liang2017well,singh2018nonparametric,schreuder2021statistical}); Besov space \citep{niles2022minimax}).
\citet{niles2022minimax} utilized the wavelet basis 
for the Besov space, 
while
\citet{liang2017well} used 
 neural networks as an estimator motivated by Generative Adversarial Networks (GAN) \citep{goodfellow2020generative}.

We would like to emphasize that our work is not replacement of wavelet expansion of \citet{niles2022minimax} with neural networks.
In diffusion modeling, we first minimize the squared-error-like 
 score matching loss, and then consider the estimation error. 
This makes existing sharp bounds in $W_1$ unavailable. 
Contrary to the analysis of GAN, where the minimax problem of the final goal directly relates to $W_1$, analysis of diffusion models requires conversion of the score approximation error to the estimation error. 

What we are built on is rather the theory of function estimation with deep neural networks in $L^p$ norms \citep{barron1993universal,yarotsky2017error,petersen2018optimal,suzuki2018adaptivity,schmidt2020nonparametric,hayakawa2020minimax}.
Our approximation result can be seen as an extension of the B-spline basis expansion used in \citet{suzuki2018adaptivity}.
On the other hand, our generalization bound relies on \citet{schmidt2020nonparametric,hayakawa2020minimax}. 



\section{Preliminaries}
\paragraph{Diffusion modeling}
We basically follow the notation of \citet{de2022convergence}. $(B_t)_{[0,\overline{T}]}$ and $\beta_t\colon [0,\overline{T}]\to \R_+$ denote $d$-dimensional Brownian motion and a weighting function.
We use $p_t$ for the distribution of $X_t$, and therefore $p_0$ is the data distribution.
As a forward process $(X_t)_{[0,\overline{T}]}$ in $\R^d$, we consider the following Ornstein–Ulhenbeck (OU) process:
\begin{align}\label{eq:Ornstein-Uhlenbeck}
    \mathrm{d}X_t = - \beta_t X_t \mathrm{d}t + \sqrt{2\beta_t}\mathrm{d}B_t
    ,\quad X_0 \sim p_0.
\end{align}
Then we have that $X_t|X_0 \sim \mathcal{N}(m_t X_0, \sigma_t)$, where $m_t = \exp(-\int_0^t\beta_s\mathrm{d}s), \sigma_t^2 =1-\exp(-2\int_0^t\beta_s\mathrm{d}s)$.
Note that $1-m_t \simeq t \land 1$ and $\sigma_t \simeq \sqrt{t}\land 1$.
Under mild assumptions on $p_0$ \citep{haussmann1986time}, valid for our setting,
 the backward process $(Y_t)_{[0,T]}$ with $Y_t = X_{\overline{T}-t}$ satisfies
\begin{align}\hspace{-.6mm} 
 &   \mathrm{d}Y_t\hspace{-.6mm} =\hspace{-.6mm} \beta_{\overline{T}-t} (Y_t \hspace{-.6mm}+\hspace{-.6mm} 2\nabla \log p_{\overline{T}-t}(Y_t)\hspace{-.4mm})\mathrm{d}t\hspace{-.9mm}+\hspace{-.9mm}\hspace{-.6mm}\sqrt{2\beta_{\overline{T}-t}} \mathrm{d}B_t,\\
 & Y_0 \sim p_{\overline{T}}.
\end{align}
$\nabla \log p_t(x)$ is called the score, which is replaced by the score network $\hat{s}(x,t)$ trained with finite sample.
Also, because $p_t$ approaches $\mathcal{N}(0,I_d)$, we take $\overline{T}=\tilde{\Ord}(1)$ and replace the initial noise distribution of $Y_0$ by $\mathcal{N}(0,I_d)$.
Then the modified backward process $(\hat{Y}_t)_{[0,\overline{T}]}$ is defined as 
\begin{align}
  \hspace{-.6mm}  & \mathrm{d}\hat{Y}_t\hspace{-.6mm} =\hspace{-.6mm} \beta_{\overline{T}-t} (\hat{Y}_t \hspace{-.6mm}+\hspace{-.6mm} 2\hat{s}(\hat{Y}_t,\overline{T}-t)\hspace{-.4mm})\mathrm{d}t\hspace{-.9mm}+\hspace{-.9mm}\hspace{-.6mm}\sqrt{2\beta_{\overline{T}-t}} \mathrm{d}B_t
    ,\\ & \hat{Y}_0 \hspace{-.2mm}\sim\hspace{-.3mm} \mathcal{N}(0,I_d).
\end{align}
\paragraph{Score matching}
The score network is ideally selected from the hypothesis $\mathcal{S}$
to minimize the \textit{denoising score matching loss}
\begin{align}
\hspace{-1.5mm}\mathbb{E}_{t}\hspace{-.6mm}\left[\lambda(t)\hspace{-1mm}\left[\mathbb{E}_{x_0}\hspace{-.3mm}\left[\mathbb{E}_{x_t|x_0}[\|s(x_t,t) - \nabla \log p_t(x_t|x_0)\|^2]\right]\hspace{-.3mm}\right]\hspace{-.3mm}\right]\hspace{-1.0mm},\hspace{-1.4mm}
   \label{eq:ProblemSetting-ScoreMatching-1}
\end{align}
where $t\sim {\rm Unif}[0,\overline{T}], x_0 \sim p_0, x_t|x_0 \sim p_t(x_t|x_0)$ and $\lambda$ is a weighting function.
Training with finite data $\{x_{0,i}\}_{i=1}^n\ (x_{0,i} \overset{\rm i.i.d.}{\sim} p_0)$ selects $\hat{s}$ to minimize the following loss, which replaces $\mathbb{E}_{x_0}$ by the sample mean:
\begin{align}
\hspace{-2mm} \frac1n \sum_{i=1}^n \hspace{-1mm}\underset{\substack{t\sim {\rm Unif}[\underline{T},\overline{T}]\\x_{t}\sim p_t(x_t|x_{0,i})}}{\mathbb{E}}\hspace{-2mm}[\lambda(t)\|s(x_{t},t) - \nabla \log p_t(x_{t}|x_{0,i})\|^2].\hspace{-1mm}
\label{eq:ExpiricalLoss}
\end{align}
Here $p_t(x_t|x_{0,i})$ corresponds to $\mathcal{N}(m_t X_{0,i}, \sigma_t)$, and this empirical loss can be evaluated with an arbitrary accuracy.
We clip the integral interval by $\underline{T}>0$  
 because generally the score blows up as $t\to 0$ and \eqref{eq:ProblemSetting-ScoreMatching-1} gets $\infty$ for any neural network.
We let $\lambda(t) \equiv 1$ when there is no other remark.

We remark that the expectations with respect to $t$ and $x_t$ can be replaced with finite sample of $t$ and $x_t$, as will be detailed in \cref{subsection:Generalization-ScoreMatchingRemark}.
However, we then inevitably need polynomial number of sample $(t,x_t)$ for each $x_{0,i}$, or an artifactual modification on the distribution of $t$, mainly due to the unboundedness of the score.

\paragraph{Class of neural networks}
As usual in approximation with neural networks \citep{yarotsky2017error,liang2017well}, the hypothesis $\mathcal{S}$ set in score matching is a class of deep neural network with the ReLU activation $\ReLU(x)=\max\{0,x\}$ (operated element-wise for a vector) \citep{nair2010rectified,glorot2011deep} with a sparsity constraint (on the number of non-zero parameters).
The score network is a function from $(x,t)\in \R^d\times \R_+ $ to $y\in\R^d$.
\begin{definition}
    A class of neural networks $\Phi(L,W,S,B)$ with height $L$, width $W$, sparsity constraint $S$, and norm constraint $B$ is defined as
$\Phi(L,W,S,B):=\{(A^{(L)}\ReLU(\cdot) + b^{(L)}) \circ \cdots
     \circ(A^{(1)}x + b^{(1)})|\ A^{(i)}\in \R^{W_i\times W_{i+1}}
     , b^{(i)}\in \R^{W_{i+1}},
     \sum_{i=1}^l (\|A^{(i)}\|_0+\|b^{(i)}\|_0) 
     \leq S, \max_{i}\|A^{(i)}\|_\infty\hspace{-.8mm} \lor \|b^{(i)}\|_\infty \leq B
     \}.
    $
\end{definition}
We remark that our results for Fully-connected Neural Network (FNN) 
is easily translated into other architectures.
For example, variants of U-Net \citep{ronneberger2015u} prevalent in practice \citep{song2019generative,ho2020denoising,ramesh2022hierarchical} are a kind of Convolutional Neural Network (CNN) and we can utilize rich literature on converting the approximation results for FNN into those for CNN \citep{oono2019approximation,zhou2020universality,petersen2020equivalence}. 

\paragraph{Density estimation in the Besov space}
As a class of the true density, the Besov space is introduced via the modulus of smoothness.
We assume that $\Omega$ be a cube in $\R^d$.
\begin{definition}
\label{def:besov}
    For a function $f\in L^p(\Omega)$ for some $p\in (0,\infty]$, the $r$-th modulus of smoothness of $f$ is defined by
    \begin{align}
      &\quad  w_{r,p}(f,t) = \sup_{\|h\|_2\leq t}\|\Delta_h^r(f)\|_p,\quad\text{where }\ \Delta_h^r(f)(x)
        \\ &\hspace{-2mm} 
        = \begin{cases}
        \sum_{j=0}^r {r \choose j}(-1)^{r-j} f(x+jh) \hspace{-2mm} & (\text{if }x+jh\in \Omega\ \text{for all $j$})
        \\
        0 & (\text{otherwise}).
        \end{cases}
    \end{align}
\end{definition}
\begin{definition}[Besov space $B_{p,q}^s(\Omega)$]
    For $0<p,q\leq \infty, s>0, r:=\lfloor s\rfloor+1$, let the seminorm $|\cdot|_{B_{p,q}^s}$ be
    \begin{align}
        |f|_{B_{p,q}^s} = \begin{cases}\left(\int_0^\infty (t^{-s}w_{r,p}(f,t))^q \frac{dt}{t}\right)^\frac1q &(q<\infty ),
     \\   \sup_{t>0}t^{-s}w_{r,p}(f,t) &(q=\infty) .
        \end{cases}
    \end{align}
    The norm of the Besov space $B_{p,q}^s$ is defined by $\|f\|_{B_{p,q}^s} = \|f\|_p+|f|_{B_{p,q}^s}$, and we have $B_{p,q}^s=\{f\in L^p(\Omega)|\ \|f\|_{B_{p,q}^s}<\infty\}.$
\end{definition}
Considering the Besov space, many well-known function classes can be discussed in a unified manner.
Let us take 
several examples.
For $\alpha\in \Z_+^d$, let $\pd^\alpha = \frac{\pd^{|\alpha|}f}{\pd^{\alpha_1}_{x_1}\cdots \pd^{\alpha_d}_{x_d}}(x)$.
The 
\\
H\"{o}lder 
space for $s\in \R_{>0} \setminus\Z_{+}$ is a set of $\lfloor s\rfloor$ times
differentiable
 functions
$\mathcal{C}^s(\Omega)=\{f\colon \Omega\to \R|\ \|f\|_{\mathcal{C}^s}:=\max_{|\alpha|\leq s}\|\pd^\alpha f\|_\infty+ \max_{|\alpha|=\lfloor s\rfloor}\sup_{x,y\in \Omega}\frac{\|\pd^\alpha f(x)-\pd^\alpha f(y)\|}{\|x-y\|^{s-\lfloor s\rfloor}}$ $<\infty\}$
for $s\in \R_{>0} \setminus\Z_{+}$.
The Sobolev space for $s\in \N, 1\leq p\leq \infty$ is a set of $s$ times differentiable
 functions $W^s_p(\Omega) :=\{f\colon \Omega\to \R|\ \|f\|_{W^s_p}:= (\sum_{|\alpha|\leq s} \|\pd^\alpha f\|_p^p)^\frac1p<\infty\}$.
Then the following relationships are due to \citet{amann1983monographs}:
\begin{itemize}
    \item For $s\in \N$, $B_{p,1}^s (\Omega)\hookrightarrow W_p^s (\Omega) \hookrightarrow B_{p,\infty}^s (\Omega)$.
    \item $B_{2,2}^s(\Omega)=W_2^s(\Omega)$.
    \item For $s\in \R_{>0} \setminus\Z_{+}$, $\mathcal{C}^s(\Omega) = B^s_{\infty,\infty}(\Omega)$.
\end{itemize}
If $s>d/p$, $B_{p,q}^s(\Omega)$ is continuously embedded in the set of the continuous functions.
Otherwise, the elements in the space is no longer continuous.
Our result is valid for $B_{p,q}^s(\Omega)$ with $s>d(1/p-1/2)_+$, and thus can include not continuous functions, unlike existing bounds assuming Lipschitzness \citep{leeconvergence,lee2022convergence,chen2022sampling}.

In this problem settings, we evaluate how close the distribution of $\hat{Y}_{\overline{T}-\underline{T}}$ can be to the true distribution $p_0$. 
As a performance measure of the distribution estimator, we employ both the total variation distance $({\rm TV})$ and the Wasserstein distance of order one $(W_1)$.
In \cref{section:LowerDimensionality}, where the data is assumed to lie in a low dimensional manifold, we focus on the Wasserstein distance. This is because the generated distribution is never absolutely continuous with respect to the true distribution, and thus the robustness of the Wasserstein distance to small parallel shift of the distribution is essential to yield a non-trivial bound not $\infty$.


\subsection{Assumptions}
Here we formally state our minimal assumptions.
Let $d$ be a dimenision of the space, $n$ be a number of sample, and $0<p,q\leq \infty, s>0$ with $s>(1/p-1/2)_+$
be parameters of the Besov space.
Our main assumption is as follows.
\begin{assumption}\label{assumption:InBesov}
    The true density $p_0$ is supported on $[-1,1]^d$, upper and lower bounded by $C_f$ and $C_f^{-1}$ on the support, respectively.
    Also, $p_0$, when limited to $[-1,1]^d$, belongs to $U(B_{p,q}^s([-1,1]^d);C)$ for some constant $C$.
\end{assumption}
$U(\cdot;C)$ denotes the ball of radius $C$, sometimes written as $U(\cdot)$ by omitting a constant $C$.
We additionally make two technical assumptions.
One is the smoothness of $\beta_t$.
\begin{assumption}\label{assumption:SmoothBeta}
    $\beta_\cdot \colon [0,\overline{T}]\to \R_+\ (t \mapsto \beta_t)$ satisfies $0<\betalow \leq \beta_\cdot \leq \betahigh$ and $\beta_\cdot\in U(\mathcal{C}^\infty([0,\overline{T}]);1)$ as a function of $t \in [0,\bar{T}]$.
\end{assumption}
The other is the smoothness of the true density $p_0$ on the boundary region.
Let $a_0$ be a sufficiently small value defined later, for example, $a_0\approx n^{-\frac{1}{d+2s}}$ in \cref{Theorem:Generalization}.
\begin{assumption}\label{assumption:BoundarySmoothness}
    $p_0$, when limited to $[-1,1]^d\setminus [-1+a_0,1-a_0]^d$, belongs to $U(\mathcal{C}^\infty([-1,1]^d\setminus [-1+a_0,1-a_0]^d))$. 
\end{assumption}
This is to construct the score network in the region where $p_t$ is not lower bounded.
This is necessarily because in density estimation lower boundedness is typically assumed
\citep{tsybakov2009introduction} and without lower boundedness
 the minimax optimal rates sometimes get worse than otherwise \citep{niles2022minimax}.
This assumption can be replaced by sufficiently slow decay of the density, such as LSI used in \citet{leeconvergence}.
We also note that this modification does not harm the minimax rate.



\section{Approximation of the true score}\label{section:Approximation}
In this section, we consider approximating the true score $\nabla \log p_t$ via a deep neural network and derive the approximation error bound.
Throughout this section, we fix $\delta>0$ arbitrarily and take $N\gg 1$ as a parameter that determines the size of the network.
We assume \cref{assumption:BoundarySmoothness} with $a_0 = N^{-\frac{1-\delta}{d}}$ and take $\underline{T}={\rm poly}(N^{-1})$, and $\overline{T}\simeq\log N$.
The main contribution of this section is the following.
\begin{theorem}\label{theorem:Approximation}
    There exists a neural network $\NetworkScoreC\in \Phi(L,W,S,B)$ that satisfies, for all $t\in [\underline{T},\overline{T}]$,
    \begin{align}
    \int_x p_t(x)  \|\NetworkScoreC(x,t) - s(x,t)\|^2 \dx \lesssim \frac{N^{-\frac{2s}{d}}\log(N)}{\sigma_t^2}.
    \end{align} 
       Here, $L,W,S$ and $B$ are evaluated as $L = \Ord (\log^4 N),\| W\|_\infty = \Ord (N\log^6N),S = \Ord (N\log^8N), $ and $B = \exp(\Ord(\log^4 N )).$
    Moreover, we can take $\NetworkScoreC$ satisfying $\|\NetworkScoreC(\cdot,t)\|_\infty = \Ord(\sigma_t^{-1}\log^\frac12 N)$.
\end{theorem}
The formal proof can be found in \cref{section:Appendix-Approximation}.

\subsection{Proof overview}
In order to obtain this result, the approximation should be constructed in the following ways.
(i) It should reflect the structure of $p_0(x)$, especially the fact of $p_0(x) \in U(B_{p,q}^s)$.
(ii) It should give a good approximation of the score over all $t\in [\underline{T},\overline{T}]$.
To address these issues, we construct a novel basis decomposition in the space of $\R^d \times [\underline{T},\overline{T}]$, specifically designed for the score approximation.
Moreover, as usual in approximation theory \citep{yarotsky2017error,schmidt2020nonparametric}, each basis can be realized by a neural network very efficiently, meaning that a polylogarithmic-sized network suffices with respect to the permissible error.


\paragraph{Approximation via the diffused B-spline Basis}
We consider the approximation for $t\ll 1$.
First remind the B-spline basis decomposition of the Besov functions \citep{devore1988interpolation,suzuki2018adaptivity}.
Let $\mathcal{N}(x)=1\ (x\in [0,1]), 0\ (\text{otherwise}).$ The \textit{cardinal B-spline of order $l$} is defined by 
$\mathcal{N}_l (x) = \underbrace{\mathcal{N}*\mathcal{N}*\cdots*\mathcal{N}}_{l+1\text{ times convolution}}(x)$,
 where $(f*g)(x) = \int f(x-t)g(t)\dt$.
Then, the \textit{tensor product B-spline basis} in $\R^d$ is defined for $k\in \N^d$ and $j\in \Z^d$ as $M_{k,j}^d(x)=\prod_{i=1}^d \mathcal{N}(2^{k_i}x-j_i)$.
It is known that a function $f$ in the Besov space is approximated by a super-position of $M_{k,j}^d(x)$ as $f_N=\sum_{(k,j)}\alpha_{(k,j)}M_{k,j}^d(x)$.
\begin{lemma}[Informal version of \cref{Lemma:SuzukiBesov}; \citet{suzuki2018adaptivity}]\label{lemma:SuzukiBesov-Main}
    For any $p_0\in U(B_{p,q}^s)$, there exists a super-position $f_N$ of $N$ tensor-product B-spline bases satisfying
    \begin{align}
        \|p_0 - f_N\|_{L^2} \lesssim N^{-s/d}\|f\|_{B_{p,q}^s}. 
    \end{align}
\end{lemma}
Inspired by this, we introduce our basis decomposition.
Because of $X_t|X_0\sim \mathcal{N}(m_t X_0, \sigma_t)$,
we can write $p_t$ as
\begin{align}
    p_t(x)& = \int p_0(y) \underbrace{\frac{1}{\sigma^d(2\pi)^\frac{d}{2}}\exp\left(-\frac{\|x-m_ty\|^2}{2\sigma_t^2}\right)}_{=:K_t(x|y)} \dy. 
\end{align}
Because the transition kernel $K_t(x|y)$ linearly applies to $p_0$ and $p_0$ is approximated by $f_N=\sum_{(k,j)}\alpha_{(k,j)}M_{k,j}^d(x)$, we come up with the following approximation of $p_t$:
\begin{align}
    p_t(x) &\approx \sum_{(k,j)}\alpha_{(k,j)}\underbrace{\int M^d_{k,j}(y) K(x|y)\dy}_{=:E_{k,j}(x,t)}.
\end{align}
Moreover, $E_{k,j}$ is further decomposed as
\begin{align}
    & E_{k,j}(x,t) \\& = \prod_{i=1}^d \underbrace{\int\frac{\mathcal{N}(2^{k_i}x_i-j_i)}{\sigma_t\sqrt{2\pi}}\exp(-\frac{(x_i-m_ty_i)^2}{2\sigma_t^2})\dx_i}_{=:\mathcal{D}_{k,j}(x_i,t)}.
\end{align}
We name $\mathcal{D}_{k,j}$ as the \textit{diffused B-spline basis} and $E_{k,j}$ as the \textit{tensor product diffused B-spline basis}.
We show that there exists a neural network that approximates $\mathcal{D}_{k,j}$ and $E_{k,j}$ very efficiently.
Our construction then goes as follows.
We construct networks approximating $m_t $ and $\sigma_t $. 
\begin{lemma}[See also \cref{Lemma:MandSigma}]\label{lemma:SigmaM}
    Under \cref{assumption:BoundarySmoothness}, there exists neural networks $\NetworkMA(t), \NetworkSigmaA(t)\in \Phi(L,W,B,S)$ that approximates $m_t$ and $\sigma_t$ up to $\epsM$ for all $t \geq 0$, where
        $L = \Ord(\log^2 (\eps^{-1})), \|W\|_\infty =\Ord(\log^3(\eps^{-1})),S = \Ord(\log^4 (\eps^{-1}))$, and $B = \exp(\Ord(\log^2 (\eps^{-1})))$.
\end{lemma}
Next we clip the integral interval of $\mathcal{D}_{k,j}$ and approximate the integrand by a rational function of $(x,m_t,\sigma_t)$. Then the following is obtained as an informal version of \cref{Lemma:DiffusionBasis2}.
\begin{lemma}\label{lemma:DiffusedBesov-Main}
    For $\eps>0$, there exists a neural 
    network 
    $\phi_{\rm TDB}\colon \R^d \times \R_+\to \R^d$ that satisfies
        $
        \|\phi_{\rm TDB}(x,t) - E_{k,j}(x,t)\|_\infty \leq \eps.
    $ Here, $\phi_{\rm TDB}\in\Phi(L,W,S,B)$ with $L=\Ord(\log^4(\epsA^{-1})), \|W\|_\infty = \Ord(\log^6 (\epsA^{-1})), S=\Ord(\log^8 (\epsA^{-1})), B=\Ord(\exp(\Ord(\log^2(\epsA^{-1}))))$.
\end{lemma}
Here $\phi_{\rm TDB}$ approximates $E_{k,j}(x,t)$ given $(x,m_t,\sigma_t)$.
 Then we use $\phi_{\rm TDB}(x,\NetworkMA(t),\NetworkSigmaA(t))$ as the approximation of $E_{k,j}(x,t)$, and $p_t(x)$ is finally approximated.
Similar approximation can also be made for $\nabla p_t(x)$, and the score is finally approximated together with $\nabla \log p_t(x) = \frac{\nabla p_t(x)}{p_t(x)}$ and we obtain the bound 
as in \cref{theorem:Approximation}.

We remark that the network size given above is
slightly larger than that for the B-spline basis (\citet{suzuki2018adaptivity}) because approximating integrals and exponential functions (\cref{subsection:Preparation-Exponential}) and rational functions (\cref{subsection:Preparation-Multiplicative}) is more difficult than realizing the B-spline basis via polynomials.

We also remark that, in this construction, the approximation error for $\nabla p_t(x)$ is amplified in the area where $p_t(x)\ll 1$.
This is why we need the higher-order smoothness of $p_0$ in the area with distance less than $\tilde{\Ord}(\sqrt{t})$ from the edge of the support (\cref{assumption:BoundarySmoothness}).
This approach is used during $t\in [\underline{T},3N^{-\frac{2-\delta}{d}}]$, and it suffices to set $a_0$ to $a_0=N^{-\frac{1-\delta}{d}}$.


\paragraph{Utilizing the smoothness induced by the noise}
The above approach enables approximation of the score in $t\ll 1$, when the score is highly non-smooth, by using the structure of $p_0$.
On the other hand, after a certain period of time, the shape of $p_t$ gets almost like a Gaussian, very smooth and easy to be approximated.
This paragraph extends the previous approach and gives an alternative approximation based on 
the smoothness induced by the noise, yielding a tighter bound.

We begin with evaluating the derivatives of $p_t$ w.r.t. $t$.
\begin{lemma}
        For any $k\in \Z_+$, there exists a constant $C_{\rm a}$ depending only on $k$, $d$, and $C_f$ such that
    \begin{align}
        \left|\pd_{x_{i_1}}\pd_{x_{i_2}}\cdots \pd_{x_{i_k}} p_t(x)\right| \leq \frac{C_{\rm a}}{\sigma_t^k}.
    \end{align}
\end{lemma}
We have that $\|p_{t_*}\|_{W_p^k}=\Ord({t_*}^{-\frac{k}{2}})$ for $t_*>0$ from this, and that $W_p^k\hookrightarrow B_{p,\infty}^k$.
For $t>t_*$, consider $p_{t}$ as the diffused distribution from $p_{t_*}$, instead of $p_0$.
We can show that $\nabla \log p_{t}$ can be approximated with a neural network with the size $N'$, with an $L^2$ error of $\Ord\left(\frac{{N'}^{-2k/d}}{\sigma_t^2} \cdot t_*^{-k}\right)$.
If $N'$ and $k$ are sufficiently large, this is tighter than the previous bound of $\frac{N^{-\frac{2s}{d}}}{\sigma_t^{2}}$.
This argument is formalized as follows. In \cref{section:Appendix-Approximation}, this is presented as \cref{Lemma:ScoreFunc-2}.
\begin{lemma}\label{lemma:ApproximationSmoothArea}
    Let $N\gg 1$ and $N'\geq t_*^{-d/2}N^{\delta/2}$.
    Suppose $t_*\geq N^{-(2-\delta)/d}$.
    Then there exists a neural network $\NetworkScoreC'\in \Phi(L,W,S,B)$ that satisfies
    \begin{align}
    \int_x p_{t}(x)  \|\NetworkScoreC'(x,t) - s(x,t)\|^2 \dx \lesssim \frac{N^{-\frac{2(s+1)}{d}}}{\sigma_t^2}
    \end{align}
    for $t\in [2t _*,\overline{T}]$.
    Specifically, $L = \Ord (\log^4 (N)),\| W\|_\infty = \Ord (N),S = \Ord (N')$, and $ B = \exp(\Ord(\log^4 N ))$.
\end{lemma}
    Setting $t_*=N^{-\frac{2-\delta}{d}}$ and $N'=N$ in this lemma, we obtain the bound in \cref{theorem:Approximation} after $t\gtrsim t_*$, without \cref{assumption:BoundarySmoothness}.
     Moreover, further exploiting this lemma later plays an important role for achieving the minimax optimal estimation rate in the $W_1$ distance. 

\section{Generalization of the score network}\label{section:Generalization}
This section converts \cref{theorem:Approximation} into the generalization bound of the score network. 
We assume $n\gg 1$ and \cref{assumption:BoundarySmoothness}
with $a_0=n^{-\frac{1-\delta}{d+2s}}$, and take $N=n^{-d/(d+2s)}$, $\underline{T}={\rm poly}(N^{-1})={\rm poly}(n^{-1})$, and $\overline{T}\simeq \log N\simeq \log n$.
The formal proofs are found in \cref{section:Appendix-Generalization}.
We begin with the following fact (\cref{Lemma:VincentEquivalence}; \citet{vincent2011connection}).
\begin{lemma}
    The following holds for all $s(x,t)$ and $t>0$:
       \begin{align}
        &  \int_x\int_y  \|s(x,t)-\nabla \log p_t(x|y) \|^2p_t(x|y)p_0(x) \dy\dx
\\ &=  \int_x \|s(x,t) - \nabla \log p_t(x)\|^2 p_t(x) \dx + C_t
        .
    \end{align} 
\end{lemma}
Here $C_t$ is a constant depending on $p_t$.
According to this, minimizing \eqref{eq:ProblemSetting-ScoreMatching-1} is equivalent to minimizing the difference between the network and the score in $L^2(p_t)$.

Let us define
\begin{align}
\ell_s(x)\!=\!\!\int_{t=\underline{T}}^{\overline{T}}\!\int \! \|s(x_t,t)\!-\!\!\nabla \log p_t(x_t|x)\|^2 p_t(x_t|x)\dx_t\dt,
\end{align}
so that the expected loss \eqref{eq:ProblemSetting-ScoreMatching-1} and the empirical loss are written as $\mathbb{E}_{x\sim p_0}[\ell(x)]$ and $\frac{1}{n}\sum_{i=1}^n\hat{\ell}(x_i)$, respectively.
For the hypothesis $\mathcal{S}$ which we specify later, we define $\mathcal{L}=\{\ell_s|\ s\in \mathcal{S}\}$.
Define the empirical loss minimizer $\hat{s} \in \mathrm{arg}\min_{s\in \mathcal{S}} \frac{1}{n}\sum_i \ell_s(x_{0,i})$. This section evaluates the difference between the empirical loss and \eqref{eq:ProblemSetting-ScoreMatching-1} for $\hat{s}$.
To evaluate the difference, we need to bound (i) $\|\ell\|_\infty$ uniformly over $\mathcal{L}$ and (ii) the covering number of $\mathcal{L}$.

\paragraph{(i) Bounding sup-norm}
According to \cref{theorem:Approximation}, $\hat{s}(x,t)$ can be taken so that $\|\hat{s}(\cdot,t)\|_\infty \lesssim \frac{\log^\frac12 N}{\sigma_t}$.
Thus we limit $\Phi(L,W,S,B)$ of \cref{theorem:Approximation} into 
\begin{align}
     \mathcal{S}:= \{\phi \in \Phi(L,W,S,B)|\ \|\phi(\cdot,t)\|_\infty \lesssim \frac{\log^\frac12 n}{\sigma_t}\}.
\end{align}
Then \cref{subsection:Appendix-Generalization-Prepare-1} shows that, 
\begin{align}
  \sup_{s\in \Phi'} \sup_{x_0\in [-1,1]^d}\ell_s(x_0)  
   \lesssim \log^2 n.
\end{align}

\paragraph{(ii) Covering number evaluation}
By Lemma 3 of \citet{suzuki2018adaptivity} and the fact that  $\|\ell_s\|_\infty$ is bounded by $\|s\|_\infty$ up to $\mathrm{poly}(n)$, we obtain the following.
\begin{lemma}\label{Lemma:CoveringNumber-Main}
    The covering number of $\mathcal{L}$ is evaluated by
    \begin{align}
        \log \mathcal{N}(\mathcal{L}, \|\cdot\|_{L^\infty ([-1,1]^d)}, \delta) \lesssim SL\log(\delta^{-1}L\|W\|_\infty Bn).
    \end{align}
\end{lemma}
The proof is found in \cref{subsection:Appendix-CoveringNumber}.
Applying this to \cref{theorem:Approximation}, the covering number is bounded by $\log \mathcal{N} \lesssim N(\log^{16}N+\log^{12}N \log \delta^{-1})$.

According to the above discussion, we finally obtain the generalization bound.
The next bound is an extension of \citet{schmidt2020nonparametric,hayakawa2020minimax}.
While they considered the minimizer of the mean squared-loss, we consider the minimizer of the mean of $\ell(x_i)$.
\begin{theorem}\label{Theorem:Generalization}
The minimizer of the empirical score selected from $\Phi'$ satisfies that
 \begin{align}\label{eq:GeneralizationError}
    &   \mathbb{E}_{x_i} \left[\int_x \int_{t=\tA}^{\tB}\|\hat{s}(x,t) - \nabla \log p_t(x)\|^2 p_t(x) \dt\dx\right]
    \\&\lesssim \inf_{s\in \Phi'}\int_x \int_{\tA}^{\tB} \|s(x,t) - \nabla \log p_t(x)\|_2^2 p_t(x) \dx \dt\\ &\quad +
    \frac{\sup_{s\in \mathcal{S}}\|\ell_s\|_\infty\log \mathcal{N}}{n}+\delta.
    \end{align}
\end{theorem}
Applying $\sup_{\ell\in \Phi'}\|\ell\|_\infty\lesssim \log^2 n$ and $\log \mathcal{N} \lesssim N(\log^{16}N+\log^{12}N \log \delta^{-1})$ with $N=\delta=n^{-d/(2s+d)}$ yields that
\begin{align}\label{eq:Main-Generalization}
    {\rm \eqref{eq:GeneralizationError}} \lesssim n^{-\frac{2s}{d+2s}}\log^{18}n.
\end{align}

\subsection{Sampling $t$ and $x_t$ instead of taking expectation}\label{subsection:Generalization-ScoreMatchingRemark}
Since our main interest lies in the sample complexity, and for simple presentation, we have considered the situation where $\ell(x)$ can be exactly evaluated.
However, in usual implementation \citet{sohl2015deep,song2019generative}, two expectations in \eqref{eq:ExpiricalLoss} with respect to $t$ and $x_t$ are also replaced by sampling for computational efficiency
Here we also introduce two ways to replace the expectation by finite sample of $t$ and $x_t$.
\paragraph{Approximation via polynomial-size sample}
Let us sample $(i_j,t_{j},x_j)$ from $i_j \sim \rm{Unif}(\{1,2,\cdots,n\})$, $t_j\sim \rm{Unif}(\underline{T},\overline{T})$, and $x_{j}\sim p_{t_j}(x_j|x_{0,i})$.
Then we let $\hat{s}$ as 
\begin{align}
  \underset{s\in \mathcal{S}}{\mathrm{argmin}} \frac{1}{M}\sum_{j=1}^M \|s(x_j,t_j)-\nabla\log p_{t_j}(x_{j}|x_{0,i_j})\|^2
\end{align}
and evaluate the difference between
\begin{align}\label{eq:Main-Generalization-polysample}
    \frac1n\sum_{i=1}^n\ell_{\hat{s}}(x_i) 
-
    \underset{s\in \mathcal{S}}{\mathrm{argmin}} \frac1n\sum_{i=1}^n\ell_s(x_i) .
\end{align}
The complete proof and formal statement can be found in \cref{theorem:ApproxViaSample-1} of \cref{subsection:ApproxViaSample}, and here we provide the proof sketch.
We first show that $\|s(x_j,t_j)-\nabla\log p_{t_j}(x_{j}|x_{0,i_j})\|$ is sub-Gaussian (\cref{Lemma:Appendix-Generalization-SubGaussian}).
Here, we simply interpret this as
$\|s(x_j,t_j)-\nabla\log p_{t_j}(x_{j}|x_{0,i_j})\|=\tilde{\Ord}(t_j^{-\frac12})\lesssim \tilde{\Ord}(\underline{T}^{-\frac12})$ with high probability to proceed.
Then, by a similar argument that derived \cref{Theorem:Generalization}, we can bound \eqref{eq:Main-Generalization-polysample} by $\tilde{O}(\frac{\underline{T}^{-1}\cdot \log \mathcal{N}}{M})$.
Here, $\mathcal{N}$ satisfies $\log \mathcal{N} \lesssim n^{\frac{d}{2s+d}}\log^8 n$.
In order to make \eqref{eq:Main-Generalization-polysample} as small as \eqref{eq:Main-Generalization}, we need to take $M\gtrsim n \cdot \underline{T}^{-1}$.
Thus, for each $x_{0,i}$, $\Ord(\underline{T}^{-1})={\rm poly}(n^{-1})$ sample of $(t_j,x_j|x_{0,i})$ should be considered. 
We remark that the reason why we need polynomial-size sample is mainly due to the scale of $ \|s(x_j,t_j)-\nabla \log p_{t_j}(x_{j}|x_{0,i_j})\|^2$.

\paragraph{Modifying the distribution of $t$}
One may think whether it is possible to consider only one path for each sample $x_{0,i}$.
Here, the main problem is that the variance of $ \|s(x_j,t_j)-\nabla\log p_{t_j}(x_{j}|x_{0,i_j})\|^2$ can grow to infinity as $t_j$ approaches to 0. 
To address this issue, we sample $t_j$ from $\mu(t) \propto \frac{\mathbbm{1}[\underline{T}\leq t\leq \overline{T}]}{t}$ and modify $\lambda(t)$ as $\lambda(t)=\frac{t\log \overline{T}/\underline{T}}{\overline{T}-\underline{T}}$, while $i_j,x_j$ are sampled as previously.
Then, we have that
\begin{align}
& \mathbb{E}_{i_j,t_j,x_j}\hspace{-1mm}\left[\lambda(t_j)\|s(x_j,t_j)\hspace{-.8mm}-\hspace{-.8mm}\nabla\log p_{t_j}(x_{j}|x_{0,i})\|^2\right] \\
& 
=\hspace{-.8mm}\frac1n\sum_{i=1}^n\ell(x_i),
\end{align}
and that
$
    \lambda(t_j)\|s(x_{t_j},t_j)-\nabla \log p_{t_i}(x_{t_j}|x_{0,i})\|^2 = \tilde{\Ord}(1)
$ holds with high probability (because $\|s(x_j,t_j)-\nabla\log p_{t_j}(x_{j}|x_{0,i_j})\|^3=\tilde{\Ord}(t_j^{-1})$ and that $\lambda(t_j)\lesssim 1/t_j$).
In this way of sampling, we let $\hat{s}$ as
\begin{align}
  \underset{s\in \mathcal{S}}{\mathrm{argmin}} \frac{1}{M}\sum_{j=1}^M \lambda(t_j)\|s(x_j,t_j)-\nabla\log p_{t_j}(x_{j}|x_{0,i_j})\|^2
\end{align}
and evaluate the difference \eqref{eq:Main-Generalization-polysample}.
Finally, using a similar argument for \cref{Theorem:Generalization}, we again obtain that \eqref{eq:Main-Generalization-polysample} is bounded by $\tilde{O}(\frac{\log \mathcal{N}}{M})\lesssim \tilde{O}(\frac{n^{\frac{d}{2s+d}}}{M})$.
Taking $M=n$ suffices to make this difference as small as \eqref{eq:Main-Generalization}.

\section{Estimation error analysis}\label{section:Main-Estimation}
This section finally evaluates the goodness of diffusion modeling as a density estimator.
As a small modification, if $\|\hat{Y}_{\overline{T}-\underline{T}}\|_\infty \geq 2$, then we reset it to $\hat{Y}_{\overline{T}-\underline{T}}=0$.
This does not increase the estimation error because $\|X_0\|_\infty \leq 1\ \text{a.s.}$.
We introduce $(\bar{Y}_t)_{t=0}^{\overline{T}-\underline{T}}$, that replaces $\hat{Y}_0 \sim \mathcal{N}(0,I_d)$ in the definition of $(\hat{Y}_t)_{t=0}^{\overline{T}-\underline{T}}$ by $\bar{Y}_0 \sim p_t$.
\subsection{Estimation rates in $\rm TV$}
First, we consider the bound in the total variation distance in the same manner as \citet{song2019generative,chen2022sampling}.
Formal proofs are found in \cref{subsection:Appendix-Generalization-W1}.
The estimation error in TV is decomposed as
\begin{align}
   & \hspace{-1mm}\mathbb{E}[{\rm TV(X_0, \hat{Y}_{\overline{T}-\underline{T}})}] \lesssim \mathbb{E}[{\rm TV(X_0, X_{\underline{T}})}] \\ & +
   \mathbb{E}[{\rm TV}(X_{\overline{T}},\mathcal{N}(0,I_d)) ] + 
\mathbb{E}[{\rm TV}(\bar{Y}_{\overline{T}-\underline{T}},Y_{\overline{T}-\underline{T}})]. 
   \label{eq:EstimationTV-1}
\end{align}
The first term comes from truncation of the backward process and is bounded by $\sqrt{\underline{T}}n^{\Ord(1)}$ according to \cref{lemma:TotalVariation-Init}. 
The second term corresponds to truncation of the forward process or the difference between $\hat{Y}_{\overline{T}-\underline{T}}$ and $\bar{Y}_{\overline{T}-\underline{T}}$,  and is bounded by $\exp(-\overline{T})$ due to  \cref{lemma:TVBound-NoiseExp}.
For the final term, Girsanov's theorem with some modification (\cref{Proposition:Girsanov}) yields that
\begin{align}\label{eq:Estimation-1}
   & \mathbb{E}_{\{x_{0,i}\}_{i=1}^n}\sqrt{\left[\int_{t=\underline{T}}^{\overline{T}}\mathbb{E}_{x\sim p_t}\left[\|\hat{s}(x,t)\hspace{-1mm}-\hspace{-1mm}\nabla \log p_t(x)\|^2\dt\right]\right] }. 
\end{align}
The convexity of $\sqrt{}$ and the generalization bound of the score network \eqref{eq:Main-Generalization} yields $\mathrm{ \eqref{eq:Estimation-1}}
\lesssim n^{-\frac{s}{d+2s}}\log^\frac{5d+8s}{2d}n$.
Now, we formalize our estimation error bound.
\begin{theorem}\label{theorem:GeneralizationL1}
    Let $\underline{T}=n^{-\Ord(1)}$ and $\overline{T} = \frac{s\log n}{\betalow(d+2s)}$. Then,
\begin{align}
    \mathbb{E}[{\rm TV}(X_0, \hat{Y}_{\overline{T}-\underline{T}})] \lesssim n^{-s/(2s+d)}\log^\frac{5d+8s}{2d}n.
\end{align}
\end{theorem}
On the other hand, we can show that the estimation problem in the Besov space has the following lower bound. The proof is found in \cref{proposition:LowerL1}.
\begin{proposition}
For $0<p,q\leq \infty$, $s>0$, and $s>\max\{d(\frac1p - \frac12), 0\}$, we have that
    \begin{align}
        \inf_{\hat{\mu}} \sup_{p\in B_{p,q}^s} \mathbb{E}[{\rm TV}(\hat{\mu},p)] \gtrsim n^{-s/(2s+d)},
    \end{align}
    where $\hat{\mu}$ runs over all estimators based on $n$ observations.
\end{proposition}
Therefore, we have proved that diffusion modeling achieves the minimax estimation rate for the Besov space $B_{p,q}^s$  in the total variation distance up to the logarithmic factor.
\subsection{Estimation rates in $W_1$}\label{subsection:Main-Estimation-W1}
We also consider the estimation rate in $W_1$.
Because both $X_0$ and $\hat{Y}_{\overline{T}-\underline{T}}$ have bounded supports, \cref{theorem:GeneralizationL1} directly yields the convergence rate of $n^{-s/(2s+d)}\log^\frac{5d+8s}{2d}n.$
However, it is known from \citet{niles2022minimax} that the minimax estimation rate in $W_1$ is faster than this.
\begin{proposition}[\citet{niles2022minimax}]
    Let $p,q\geq 1$, $s> 0$ and $d\geq 2$.
        \begin{align}
        \int_{\hat{\mu}} \sup_{p\in B_{p,q}^s} \mathbb{E}[W_1(\hat{\mu},p)] \gtrsim n^{-(s+1)/(2s+d)},
    \end{align}
    where $\hat{\mu}$ runs over all estimators based on $n$ observations. 
    Moreover, if $1\leq p<\infty$, $1\leq q \leq \infty$, $s>0$, and $d\geq 3$, there exists an estimator $\hat{\mu}_*$ that achieves this minimax rate.
\end{proposition}
Then are diffusion models sub-optimal in this case?
In the following, we show the surprising fact that diffusion modeling also achieves this nearly minimax optimal rate, if some modification applied.
\begin{theorem}\label{theorem:GeneralizationW1}
For any fixed $\delta>0$, we can train the score network with $n (\gg 1)$ sample and with that we have
\begin{align}
    \mathbb{E}[W_1(X_0, \hat{Y}_{\overline{T}-\underline{T}})] \lesssim n^{-\frac{(s+1-\delta)}{d+2s}}.
\end{align}
\end{theorem}
The proofs are found in \cref{subsection:Appendix-Generalization-W1}.
\paragraph{Switching score networks}
We now sketch our strategy.
First, let us carefully consider where we lose the estimation rate, going back to the approximation error analysis \cref{section:Approximation}.
Although we used \cref{theorem:Approximation} for all $\underline{T}\leq t\leq \overline{T}$, \cref{lemma:ApproximationSmoothArea} tells us that if $t\gtrsim N^{-\frac{2-\delta}{d}}\simeq n^{-\frac{2-\delta}{2s+d}}$, we can make the approximation error smaller than $\frac{N^{-\frac{2(s+1)}{d}}}{\sigma_t^{-2}}=\frac{n^{-\frac{2(s+1)}{d+2s}}}{\sigma_t^{-2}}$ with a smaller network of size $N' \leq N$.
This means that we have used a sub-optimal network for $t\gtrsim n^{-\frac{2-\delta}{d+2s}}$ in terms of both approximation and generalization errors.

Based on this discussion, we divide the time into $t_0=\underline{T}<t_1=2n^{-\frac{2-\delta}{d+2s}}<\cdots<t_{K_*}=\overline{T}-\underline{T}$ with $t_{i+1}/t_{i}=\text{const.}\leq 2\ (i\geq 1)$.
The number of intervals amounts to $K_*=\Ord(\log n)$.
We consider to train a tailored network for each time interval $[t_i,t_{i+1}]$ and to switch them for different intervals.
\cref{lemma:ApproximationSmoothArea} yields that for $i\geq 1$ these exists a network $s_i\in \Phi(L_i,W_i,S_i,W_i)$ such that
\begin{align}
    \mathbb{E}_{x\sim p_t}[\|s_i(x,t)\hspace{-1mm}-\hspace{-1mm}\nabla \log p_t(x)\|^2]
    \!\lesssim\! \frac{n^{-\frac{2(s+1)}{d+2s}}}{\sigma_t^{2}}\ (t\in [t_{i},t_{i+1}]),
\end{align}
with $L = \Ord (\log^4 (N)),\| W\|_\infty = \Ord (N),S = \Ord (t_i^{-d/2}N^{\delta/2})$, and $B = \exp(\Ord(\log^4 N ))$.
Therefore, we choose a sequence of score networks $\hat{s}_i$ so that $\hat{s}_i$ minimizes the score matching loss restricted to $[t_i,t_{i+1}]$:
\begin{align}
\ \frac1n \sum_{j=1}^n \underset{\substack{t\sim {\rm Unif}[t_i,t_{j+1}]\\x_{t}\sim p_t(x_t|x_{0,j})}}{\mathbb{E}}[\|s(x_{t},t) - \nabla \log p_t(x_{t}|x_{0,j})\|^2].
\end{align}
Similarly to \cref{Theorem:Generalization}, \cref{Theorem:ExtendedSchmidt} yields that the following generalization error bound for $i\geq 1$:
\begin{align}\label{eq:Estimation-3}
   & \mathbb{E}_{\{x_{0,j}\}_{i=j}^n}\left[\int_{t=t_i}^{t_{i+1}}\mathbb{E}_x\left[\|\hat{s}_i(x,t)\hspace{-1mm}-\hspace{-1mm}\nabla \log p_t(x)\|^2\dt\right]\right]
   \\ & 
   \label{eq:Estimation-4}
   \leq \left(n^{-\frac{2(s+1)}{d+2s}} + \frac{t_i^{-d/2}n^{\frac{\delta d}{(d+2s)}}\log^{10} n }{n}\right)\cdot \underbrace{\tilde{\Ord}(t_i/\sigma_{t_i}^2)}_{=\tilde{\Ord}(1)}.
\end{align}
For $t\lesssim n^{-\frac{2-\delta}{d+2s}}$, we use a network trained via the score matching loss restricted to $[t_i,t_{i+1}]$. 
Thus, \eqref{eq:Estimation-3} for $i=0$ is bounded by $\tilde{\Ord}(n^{-\frac{2s}{d+2s}})$ similarly to \eqref{eq:Main-Generalization}.

One may think that the above improvement would be useless because the error caused at $t\leq n^{-\frac{2-\delta}{d+2s}}$ has the $n^{-2s/(d+2s)}$ rate and dominates the estimation error.
However, another important observation is that the Wasserstain distance is a transportation distance.
The score estimation error at time closer to $t=0$ less contributes to the estimation error, because the distance how much each path evolves is small from that time.
As we will see, the idea of improving accuracy for large $t$ indeed yields the minimax optimal rate in $W_1$.

To utilize this observation, let us consider a sequence of stochastic processes.
Let $(Y_t)_{[0,\overline{T}]}=(\bar{Y}_t^{(0)})_{[0,\overline{T}]}$, and for $i\geq 1$, let $(\bar{Y}^{(i)})_{[0,\overline{T}]}$ be a stochastic process which uses the true score during $[0,\overline{T}-t_i]$ and the estimated score $\hat{s}$ during $[\overline{T}-t_i,\overline{T}-\underline{T}]$, and $\bar{Y}^{(i)}_0\sim p_{\overline{T}}$.
Then, we have that
\begin{align}\label{eq:Estimation-2}
     &\hspace{-1mm}\mathbb{E}[{W_1(X_0, \hat{Y}_{\overline{T}-\underline{T}})}] \leq\mathbb{E}[{W_1(X_0, X_{\underline{T}})}] \\ & +
   \mathbb{E}[W_1(\hat{Y}_{\overline{T}-\underline{T}},\bar{Y}_{\overline{T}-\underline{T}})) ] + 
\mathbb{E}[{W_1}(\bar{Y}_{\overline{T}-\underline{T}},Y_{\overline{T}-\underline{T}})].
\end{align}
The first term is bounded by $\sqrt{\underline{T}}$ due to \eqref{eq:Appendix-Discretization-4} and the second term is bounded by $\exp(-\overline{T})$ due to \cref{lemma:Appendix-Estimation-W1-2}.
The last term $\mathbb{E}[{W_1}(\bar{Y}_{\overline{T}-\underline{T}},Y_{\overline{T}-\underline{T}})]$ is upper bounded by
$
    \sum_{i=1}^{K_*}\mathbb{E}[W_1(\bar{Y}_{\overline{T}-\underline{T}}^{(i-1)},\bar{Y}_{\overline{T}-\underline{T}}^{(i)})].
$
Then, we use the following lemma, an informal version of \cref{lemma:W1BoundCrucial}.
\begin{lemma}\label{Lemma:W1Bound-Main}
    For $i=1,2,\cdots,K_*$, we have that
    \begin{align}
      &  W_1(\hat{Y}_{\overline{T}-\underline{T}}^{(i-1)},\hat{Y}_{\overline{T}-\underline{T}}^{(i)}) \leq \tilde{\Ord}(1) \cdot \\ & 
      \sqrt{t_{i-1} \mathbb{E}_{\{x_{0,i}\}_{i=1}^n}\!\!\left[\int_{t=t_{i-1}}^{t_{i}}\!\!\!\mathbb{E}_x\left[\|\hat{s}(x,t)\hspace{-1mm}-\hspace{-1mm}\nabla \log p_t(x)\|^2\dt\right]\right]}.
    \end{align}
\end{lemma}
RHS is decomposed to the two factors: the score matching loss during $[t_{i-1},t_{i}]$ and $\sqrt{t_i}$.
The latter corresponds to how much $Y_t$ moves from $t=\overline{T}-t_i$ to $\overline{T}-\underline{T}$.
This bound represents that, as $t_i\to 0$, while score matching gets more difficult, its contribution to the $W_1$ error is reduced.
The formal proof requires construction of a path-wise transportation map; see the proof for  \cref{lemma:W1BoundCrucial}.

Putting it all together, we finally yields \cref{theorem:GeneralizationW1}, the nearly minimax optimal rate in $W_1$. Specifically, if we ignore logarithmic factors, \eqref{eq:Estimation-2} is bounded by
\begin{align}
    &\sqrt{\underline{T}}+\exp(-\overline{T})+\sqrt{t_0}n^{-\frac{2s}{d+2s}}\\ &+
    \sum_{i=2}^{K_*}\sqrt{t_i}\sqrt{ n^{-\frac{2(s+1)}{d+2s}} + \frac{t_i^{-d/2}n^{\frac{\delta d}{2(d+2s)}}}{n}}
    \lesssim n^{-\frac{s+1-\delta}{d+2s}},
\end{align}
where we set $\underline{T}=n^{-\frac{2(s+1)}{d+2s}}$ and $ \overline{T}=\frac{(s+1)\log n}{\betalow(d+2s)}.$

\begin{remark}
    Although we used differently optimized multiple networks, it is also possible that such modification is implicitly made in reality.
    The first evidence is \textit{implicit reguralization}, where sparsify of the solution is induced by learning procedures \citep{gunasekar2017implicit,arora2019implicit,soudry2018implicit}.
    When the sub-networks for differnt time intervals are learned in parallel via the score matching at once \eqref{eq:ProblemSetting-ScoreMatching-1}, these theory suggests the good score network is obtained without explicit regularization like our switching procedure.
    Another evidence is that in practice the weight function $\lambda(t)$ sometimes increases as $t$ gets large \citep{song2019generative,song2020score}, suggesting that the quality of the score network at larger $t$ is more emphasized.
\end{remark}

\subsection{Discussion on the discretization error}\label{subsection:Discretization}
    Although the continuous time SDE is mainly focused on for simple presentation, we can also take the discretization error into consideration.
    We here only provide the summary, and the details are presented in \cref{subsection:Appendix-Discretization}.
    Let $t_0=\underline{T}<t_1<\cdots<t_{K_*}=\overline{T}$ be the time steps with $\eta\equiv t_{k+1}-t_{k}$.
    We train the score network as the minimizer of
    \begin{align}
 \sum_{i=1}^n \sum_{k=0}^{K-1}\eta\mathbb{E}[\|s(x_{t_k},{t_k}) - \nabla \log p_{\overline{T}-t_k}(x_{t_k}|x_{0,i})\|^2].
\end{align}
Here the expectation is taken with respect to $x_{\overline{T}-t_k}\sim p_{\overline{T}-t_k}(x_{\overline{T}-t_k}|x_{0,i})$.
  Then consider the following process $(Y^{\rm d}_t)_{t=0}^{\eta K}$ with $Y^{\rm d}_0 \sim\mathcal{N}(0,I_d)$: for $t\in [\overline{T}-t_i,\overline{T}-t_{i+1}]$,
    \begin{align}
    \mathrm{d}Y^{\rm d}_t\hspace{-.2mm} =\hspace{-.2mm} \beta_t (Y^{\rm d}_{t} + 2\hat{s}(Y^{\rm d}_{\overline{T}-t_i},\overline{T}-t_i)\hspace{-.4mm})\dt +\hspace{-.6mm}\beta_{\overline{T}-t}\mathrm{d}B_t
    \end{align}
     This is just replacement of the drift term at $t$ by that at the last discretized step, and we can obtain $\bar{Y}_{\eta (k+1)}$ from $\bar{Y}_{\eta k}$ as easy as the classical Euler-Maruyama discretization because $\bar{Y}_{\eta (k+1)}$ conditioned on $\bar{Y}_{\eta k}$ is a Gaussian.
    This is also adopted in \citet{de2022convergence,chen2022sampling}. However, \citet{de2022convergence} requires $\eta_i \leq \exp(-n^{\Ord(1)})$ and \citet{chen2022sampling} assumes Lipschitzness of the score, which does not necessarily hold in our setting.

    We can show the following discretization error bound:
    \begin{theorem}
     Let $\underline{T}=n^{-\Ord(1)}$ and $\overline{T} = \frac{s\log n}{2s+d}$.
     Then,
\begin{align}
    \mathbb{E}[{\rm TV}(X_0, Y^{\rm d}_{\overline{T}-\underline{T}})] \lesssim \tilde{\Ord}\left(\eta^2 \underline{T}^{-3}+n^{-\frac{s}{d+2s}}\right).
\end{align}
\end{theorem}
Thus, taking $\eta=\underline{T}^{-1.5}n^{-s/(2s+d)}={\rm poly}(n^{-1})$ suffices to ignore the discretization error.

\section{Error analysis with intrinsic dimensionality}\label{section:LowerDimensionality}
Although the obtained rates in \cref{section:Main-Estimation} are minimax optimal, they still suffer from the \textit{curse of dimensionality}: the exponent of the convergence rates depend on $d$.
In statistics, one approach to avoid this curse of dimensionality is to assume mixed or anisotropic smoothness \citep{Ibragimov1984MoreOT,meier2009high,suzuki2018adaptivity,suzuki2021deep}, and our theory directly applies to them.
On the other hand, the \textit{manifold hypothesis}, that the distributions of real-world data lie in low dimensional manifolds, has been proposed \citep{tenenbaum2000global,fefferman2016testing}, and this is another assumption that can avoid the curse of dimensionality: convergence rates dependent not on the dimension $d$ of the space itself but on the manifold's dimension $d'$ can be derived \citet{schmidt2019deep,nakada2020adaptive}.

As for the diffusion models, despite its statistical importance, none of the literature has shown that diffusion models can ease the curse of dimensionality; in the first place, the density estimation problem itself has never been considered.

We introduce several recent works that investigated the convergence of diffusion modeling under the manifold hypothesis.
\citet{pidstrigach2022scorebased} discussed the effects of the score approximation, but their bounds are not quantitative and does not consider the estimation rate. 
\citet{de2022convergence} considered the estimation rates, but the approximation error should be exponentially small with respect to the desired estimation rate.
\citet{batzolis2022your} experimentally showed that diffusion modeling learns the dimension of the underlying manifold and the dimension of the manifold can be estimated from the trained diffusion models.

From now, we define the specific class of density function with intrinsic dimensionality and show the estimation rate.

Let $d'\le d$ be an integer and  $A\in\mathbb{R}^{d\times d'}$ be a matrix made of orthogonal column vectors with the norm one.
We consider the $d'$-dimensional subspace $V\coloneqq \{y\in \mathbb{R}^d\mid \exists x \in\mathbb{R}^{d'}\text{ s.t. }y=Ax\}$ where the true density has its support, i.e., $d'$ represents the intrinsic dimensionality.
Together with \cref{assumption:SmoothBeta}, we assume the followings.
\begin{assumption}
    The true density $p_0$ is a probability measure that is absolutely continuous with respect to the Lebesgue measure restricted on the sub-space $V$.
    Its probability density function as a function on the canonical coordinate system of the subspace $V$ is denoted by $q$.
\end{assumption}
\begin{assumption}
    $q$ is upper and lower bounded by $C_f$ and $C_f^{-1}$, respectively. 
    Moreover, $q$ belongs to $U(B_{p,q}^s;[-1,1]^{d'})$.
\end{assumption}
\begin{assumption}
     $q$ belongs to $U(\mathcal{C}^\infty([-1,1]^{d'}\setminus [-1+a_0,1-a_0]^{d'}))$ with $a_0 = n^{-\frac{1-\delta}{d'}}$.
\end{assumption}

We now state our result as follows:
\begin{theorem}\label{theorem:GeneralizationL1lowdim}
    For any fixed $\delta>0$, we can train the score network with $n (\gg 1)$ sample so that 
\begin{align}
    \mathbb{E}[W_1(X_0, \hat{Y}_{\overline{T}-\underline{T}})] \lesssim n^{-\frac{(s+1-\delta)}{d'+2s}}.
\end{align}
\end{theorem}
\cref{section:Appendix-IntrinsicDimensionality} provides the complete proof.
Contrary to \cref{theorem:GeneralizationL1}, the upper bound here depends on $d'$ (not on $d$). 
Thus, we can conclude that the diffusion models can avoid the curse of dimensionality. 


\section{Conclusion}
This paper analyzed diffusion modeling as a distribution learner from the viewpoint of statistical learning theory and derived several estimation rates.
When the true density belongs to the Besov space and deep neural networks are appropriately minimized, diffusion modeling can achieve nearly minimax optimal estimation rates in $\rm TV$ and $W_1$.

To approximate the score, the novel basis is introduced, which we call the diffused B-spline basis.
The bound in $W_1$ is derived by carefully balancing the difficulty in score matching and how much the error in score matching at each time affects the $W_1$ distance.
We also demonstrated that diffusion models can avoid the curse of dimensionality under the manifold hypothesis.

\section*{Acknowledgements}
KO was partially supported by Fujitsu Ltd. 
SA was partially supported by JSPS KAKENHI (22J13388).
TS was partially supported by JSPS KAKENHI (20H00576) and JST CREST.



\bibliography{example_paper}

\begin{thebibliography}{63}
\providecommand{\natexlab}[1]{#1}
\providecommand{\url}[1]{\texttt{#1}}
\expandafter\ifx\csname urlstyle\endcsname\relax
  \providecommand{\doi}[1]{doi: #1}\else
  \providecommand{\doi}{doi: \begingroup \urlstyle{rm}\Url}\fi

\bibitem[Amann et~al.(1983)Amann, Bourguignon, Grove, Lions, Araki, Brezzi,
  Chang, Hitchin, Hofer, Kn{\"o}rrer, et~al.]{amann1983monographs}
Amann, H., Bourguignon, J., Grove, K., Lions, P., Araki, H., Brezzi, F., Chang,
  K., Hitchin, N., Hofer, H., Kn{\"o}rrer, H., et~al.
\newblock Monographs in mathematics vol. 99.
\newblock 1983.

\bibitem[Arora et~al.(2019)Arora, Cohen, Hu, and Luo]{arora2019implicit}
Arora, S., Cohen, N., Hu, W., and Luo, Y.
\newblock Implicit regularization in deep matrix factorization.
\newblock \emph{Advances in Neural Information Processing Systems}, 32, 2019.

\bibitem[Bakry et~al.(2014)Bakry, Gentil, Ledoux, et~al.]{bakry2014analysis}
Bakry, D., Gentil, I., Ledoux, M., et~al.
\newblock \emph{Analysis and geometry of Markov diffusion operators}, volume
  103.
\newblock Springer, 2014.

\bibitem[Barron(1993)]{barron1993universal}
Barron, A.~R.
\newblock Universal approximation bounds for superpositions of a sigmoidal
  function.
\newblock \emph{IEEE Transactions on Information theory}, 39\penalty0
  (3):\penalty0 930--945, 1993.

\bibitem[Batzolis et~al.(2022)Batzolis, Stanczuk, and
  Sch{\"o}nlieb]{batzolis2022your}
Batzolis, G., Stanczuk, J., and Sch{\"o}nlieb, C.-B.
\newblock Your diffusion model secretly knows the dimension of the data
  manifold.
\newblock \emph{arXiv preprint arXiv:2212.12611}, 2022.

\bibitem[Block et~al.(2020)Block, Mroueh, and Rakhlin]{block2020generative}
Block, A., Mroueh, Y., and Rakhlin, A.
\newblock Generative modeling with denoising auto-encoders and langevin
  sampling.
\newblock \emph{arXiv preprint arXiv:2002.00107}, 2020.

\bibitem[Boull{\'e} et~al.(2020)Boull{\'e}, Nakatsukasa, and
  Townsend]{boulle2020rational}
Boull{\'e}, N., Nakatsukasa, Y., and Townsend, A.
\newblock Rational neural networks.
\newblock \emph{Advances in Neural Information Processing Systems},
  33:\penalty0 14243--14253, 2020.

\bibitem[Chang et~al.(2011)Chang, Cosman, and Milstein]{chang2011chernoff}
Chang, S.-H., Cosman, P.~C., and Milstein, L.~B.
\newblock Chernoff-type bounds for the gaussian error function.
\newblock \emph{IEEE Transactions on Communications}, 59\penalty0
  (11):\penalty0 2939--2944, 2011.

\bibitem[Chen et~al.(2020)Chen, Zhang, Zen, Weiss, Norouzi, and
  Chan]{chen2020wavegrad}
Chen, N., Zhang, Y., Zen, H., Weiss, R.~J., Norouzi, M., and Chan, W.
\newblock Wavegrad: Estimating gradients for waveform generation.
\newblock In \emph{International Conference on Learning Representations}, 2020.

\bibitem[Chen et~al.(2022)Chen, Chewi, Li, Li, Salim, and
  Zhang]{chen2022sampling}
Chen, S., Chewi, S., Li, J., Li, Y., Salim, A., and Zhang, A.~R.
\newblock Sampling is as easy as learning the score: theory for diffusion
  models with minimal data assumptions.
\newblock \emph{arXiv preprint arXiv:2209.11215}, 2022.

\bibitem[De~Bortoli(2022)]{de2022convergence}
De~Bortoli, V.
\newblock Convergence of denoising diffusion models under the manifold
  hypothesis.
\newblock \emph{Transactions on Machine Learning Research}, 2022.
\newblock URL \url{https://openreview.net/forum?id=MhK5aXo3gB}.

\bibitem[De~Bortoli et~al.(2021)De~Bortoli, Thornton, Heng, and
  Doucet]{de2021diffusion}
De~Bortoli, V., Thornton, J., Heng, J., and Doucet, A.
\newblock Diffusion schr{\"o}dinger bridge with applications to score-based
  generative modeling.
\newblock \emph{Advances in Neural Information Processing Systems},
  34:\penalty0 17695--17709, 2021.

\bibitem[DeVore \& Popov(1988)DeVore and Popov]{devore1988interpolation}
DeVore, R.~A. and Popov, V.~A.
\newblock Interpolation of {Besov} spaces.
\newblock \emph{Transactions of the American Mathematical Society},
  305\penalty0 (1):\penalty0 397--414, 1988.

\bibitem[Dhariwal \& Nichol(2021)Dhariwal and Nichol]{dhariwal2021diffusion}
Dhariwal, P. and Nichol, A.
\newblock Diffusion models beat gans on image synthesis.
\newblock \emph{Advances in Neural Information Processing Systems},
  34:\penalty0 8780--8794, 2021.

\bibitem[Dudley(1969)]{dudley1969speed}
Dudley, R.~M.
\newblock The speed of mean glivenko-cantelli convergence.
\newblock \emph{The Annals of Mathematical Statistics}, 40\penalty0
  (1):\penalty0 40--50, 1969.

\bibitem[Fefferman et~al.(2016)Fefferman, Mitter, and
  Narayanan]{fefferman2016testing}
Fefferman, C., Mitter, S., and Narayanan, H.
\newblock Testing the manifold hypothesis.
\newblock \emph{Journal of the American Mathematical Society}, 29\penalty0
  (4):\penalty0 983--1049, 2016.

\bibitem[Glorot et~al.(2011)Glorot, Bordes, and Bengio]{glorot2011deep}
Glorot, X., Bordes, A., and Bengio, Y.
\newblock Deep sparse rectifier neural networks.
\newblock In \emph{Proceedings of the fourteenth international conference on
  artificial intelligence and statistics}, pp.\  315--323. JMLR Workshop and
  Conference Proceedings, 2011.

\bibitem[Goodfellow et~al.(2020)Goodfellow, Pouget-Abadie, Mirza, Xu,
  Warde-Farley, Ozair, Courville, and Bengio]{goodfellow2020generative}
Goodfellow, I., Pouget-Abadie, J., Mirza, M., Xu, B., Warde-Farley, D., Ozair,
  S., Courville, A., and Bengio, Y.
\newblock Generative adversarial networks.
\newblock \emph{Communications of the ACM}, 63\penalty0 (11):\penalty0
  139--144, 2020.

\bibitem[Gunasekar et~al.(2017)Gunasekar, Woodworth, Bhojanapalli, Neyshabur,
  and Srebro]{gunasekar2017implicit}
Gunasekar, S., Woodworth, B.~E., Bhojanapalli, S., Neyshabur, B., and Srebro,
  N.
\newblock Implicit regularization in matrix factorization.
\newblock \emph{Advances in Neural Information Processing Systems}, 30, 2017.

\bibitem[Haussmann \& Pardoux(1986)Haussmann and Pardoux]{haussmann1986time}
Haussmann, U.~G. and Pardoux, E.
\newblock {Time Reversal of Diffusions}.
\newblock \emph{The Annals of Probability}, 14\penalty0 (4):\penalty0
  1188--1205, 1986.
\newblock \doi{10.1214/aop/1176992362}.
\newblock URL \url{https://doi.org/10.1214/aop/1176992362}.

\bibitem[Hayakawa \& Suzuki(2020)Hayakawa and Suzuki]{hayakawa2020minimax}
Hayakawa, S. and Suzuki, T.
\newblock On the minimax optimality and superiority of deep neural network
  learning over sparse parameter spaces.
\newblock \emph{Neural Networks}, 123:\penalty0 343--361, 2020.

\bibitem[Ho et~al.(2020)Ho, Jain, and Abbeel]{ho2020denoising}
Ho, J., Jain, A., and Abbeel, P.
\newblock Denoising diffusion probabilistic models.
\newblock \emph{Advances in Neural Information Processing Systems},
  33:\penalty0 6840--6851, 2020.

\bibitem[Ho et~al.(2022)Ho, Salimans, Gritsenko, Chan, Norouzi, and
  Fleet]{ho2022video}
Ho, J., Salimans, T., Gritsenko, A., Chan, W., Norouzi, M., and Fleet, D.~J.
\newblock Video diffusion models.
\newblock \emph{arXiv:2204.03458}, 2022.

\bibitem[Hyv{\"a}rinen \& Dayan(2005)Hyv{\"a}rinen and
  Dayan]{hyvarinen2005estimation}
Hyv{\"a}rinen, A. and Dayan, P.
\newblock Estimation of non-normalized statistical models by score matching.
\newblock \emph{Journal of Machine Learning Research}, 6\penalty0 (4), 2005.

\bibitem[Ibragimov \& Khas'minskii(1984)Ibragimov and
  Khas'minskii]{Ibragimov1984MoreOT}
Ibragimov, I.~A. and Khas'minskii, R.~Z.
\newblock More on the estimation of distribution densities.
\newblock \emph{Journal of Soviet Mathematics}, 25:\penalty0 1155--1165, 1984.

\bibitem[Karatzas et~al.(1991)Karatzas, Karatzas, Shreve, and
  Shreve]{karatzas1991brownian}
Karatzas, I., Karatzas, I., Shreve, S., and Shreve, S.~E.
\newblock \emph{Brownian motion and stochastic calculus}, volume 113.
\newblock Springer Science \& Business Media, 1991.

\bibitem[Kong et~al.(2020)Kong, Ping, Huang, Zhao, and
  Catanzaro]{kong2020diffwave}
Kong, Z., Ping, W., Huang, J., Zhao, K., and Catanzaro, B.
\newblock Diffwave: A versatile diffusion model for audio synthesis.
\newblock In \emph{International Conference on Learning Representations}, 2020.

\bibitem[Lee et~al.(2022{\natexlab{a}})Lee, Lu, and Tan]{lee2022convergence}
Lee, H., Lu, J., and Tan, Y.
\newblock Convergence of score-based generative modeling for general data
  distributions.
\newblock In \emph{NeurIPS 2022 Workshop on Score-Based Methods},
  2022{\natexlab{a}}.

\bibitem[Lee et~al.(2022{\natexlab{b}})Lee, Lu, and Tan]{leeconvergence}
Lee, H., Lu, J., and Tan, Y.
\newblock Convergence for score-based generative modeling with polynomial
  complexity.
\newblock In \emph{Advances in Neural Information Processing Systems},
  2022{\natexlab{b}}.

\bibitem[Lei(2020)]{lei2020convergence}
Lei, J.
\newblock Convergence and concentration of empirical measures under wasserstein
  distance in unbounded functional spaces.
\newblock \emph{Bernoulli}, 26\penalty0 (1):\penalty0 767--798, 2020.

\bibitem[Liang(2017)]{liang2017well}
Liang, T.
\newblock How well can generative adversarial networks learn densities: A
  nonparametric view.
\newblock \emph{arXiv preprint arXiv:1712.08244}, 2017.

\bibitem[Meier et~al.(2009)Meier, Van~de Geer, and B{\"u}hlmann]{meier2009high}
Meier, L., Van~de Geer, S., and B{\"u}hlmann, P.
\newblock High-dimensional additive modeling.
\newblock 2009.

\bibitem[Mhaskar \& Micchelli(1992)Mhaskar and
  Micchelli]{mhaskar1992approximation}
Mhaskar, H.~N. and Micchelli, C.~A.
\newblock Approximation by superposition of sigmoidal and radial basis
  functions.
\newblock \emph{Advances in Applied mathematics}, 13\penalty0 (3):\penalty0
  350--373, 1992.

\bibitem[Nair \& Hinton(2010)Nair and Hinton]{nair2010rectified}
Nair, V. and Hinton, G.~E.
\newblock Rectified linear units improve restricted boltzmann machines.
\newblock In \emph{Icml}, 2010.

\bibitem[Nakada \& Imaizumi(2020)Nakada and Imaizumi]{nakada2020adaptive}
Nakada, R. and Imaizumi, M.
\newblock Adaptive approximation and generalization of deep neural network with
  intrinsic dimensionality.
\newblock \emph{J. Mach. Learn. Res.}, 21\penalty0 (174):\penalty0 1--38, 2020.

\bibitem[Niles-Weed \& Berthet(2022)Niles-Weed and Berthet]{niles2022minimax}
Niles-Weed, J. and Berthet, Q.
\newblock Minimax estimation of smooth densities in {W}asserstein distance.
\newblock \emph{The Annals of Statistics}, 50\penalty0 (3):\penalty0
  1519--1540, 2022.

\bibitem[Oono \& Suzuki(2019)Oono and Suzuki]{oono2019approximation}
Oono, K. and Suzuki, T.
\newblock Approximation and non-parametric estimation of resnet-type
  convolutional neural networks.
\newblock In \emph{International conference on machine learning}, pp.\
  4922--4931. PMLR, 2019.

\bibitem[Petersen \& Voigtlaender(2018)Petersen and
  Voigtlaender]{petersen2018optimal}
Petersen, P. and Voigtlaender, F.
\newblock Optimal approximation of piecewise smooth functions using deep relu
  neural networks.
\newblock \emph{Neural Networks}, 108:\penalty0 296--330, 2018.

\bibitem[Petersen \& Voigtlaender(2020)Petersen and
  Voigtlaender]{petersen2020equivalence}
Petersen, P. and Voigtlaender, F.
\newblock Equivalence of approximation by convolutional neural networks and
  fully-connected networks.
\newblock \emph{Proceedings of the American Mathematical Society}, 148\penalty0
  (4):\penalty0 1567--1581, 2020.

\bibitem[Pidstrigach(2022)]{pidstrigach2022scorebased}
Pidstrigach, J.
\newblock Score-based generative models detect manifolds.
\newblock In Oh, A.~H., Agarwal, A., Belgrave, D., and Cho, K. (eds.),
  \emph{Advances in Neural Information Processing Systems}, 2022.
\newblock URL \url{https://openreview.net/forum?id=AiNrnIrDfD9}.

\bibitem[Ramesh et~al.(2022)Ramesh, Dhariwal, Nichol, Chu, and
  Chen]{ramesh2022hierarchical}
Ramesh, A., Dhariwal, P., Nichol, A., Chu, C., and Chen, M.
\newblock Hierarchical text-conditional image generation with clip latents.
\newblock \emph{arXiv preprint arXiv:2204.06125}, 2022.

\bibitem[Ronneberger et~al.(2015)Ronneberger, Fischer, and
  Brox]{ronneberger2015u}
Ronneberger, O., Fischer, P., and Brox, T.
\newblock U-net: Convolutional networks for biomedical image segmentation.
\newblock In \emph{International Conference on Medical image computing and
  computer-assisted intervention}, pp.\  234--241. Springer, 2015.

\bibitem[Schmidt-Hieber(2019)]{schmidt2019deep}
Schmidt-Hieber, J.
\newblock Deep relu network approximation of functions on a manifold.
\newblock \emph{arXiv preprint arXiv:1908.00695}, 2019.

\bibitem[Schmidt-Hieber(2020)]{schmidt2020nonparametric}
Schmidt-Hieber, J.
\newblock Nonparametric regression using deep neural networks with relu
  activation function.
\newblock \emph{The Annals of Statistics}, 48\penalty0 (4):\penalty0
  1875--1897, 2020.

\bibitem[Schreuder et~al.(2021)Schreuder, Brunel, and
  Dalalyan]{schreuder2021statistical}
Schreuder, N., Brunel, V.-E., and Dalalyan, A.
\newblock Statistical guarantees for generative models without domination.
\newblock In \emph{Algorithmic Learning Theory}, pp.\  1051--1071. PMLR, 2021.

\bibitem[Singh \& P{\'o}czos(2018)Singh and P{\'o}czos]{singh2018minimax}
Singh, S. and P{\'o}czos, B.
\newblock Minimax distribution estimation in {W}asserstein distance.
\newblock \emph{arXiv preprint arXiv:1802.08855}, 2018.

\bibitem[Singh et~al.(2018)Singh, Uppal, Li, Li, Zaheer, and
  P{\'o}czos]{singh2018nonparametric}
Singh, S., Uppal, A., Li, B., Li, C.-L., Zaheer, M., and P{\'o}czos, B.
\newblock Nonparametric density estimation under adversarial losses.
\newblock \emph{Advances in Neural Information Processing Systems}, 31, 2018.

\bibitem[Sohl-Dickstein et~al.(2015)Sohl-Dickstein, Weiss, Maheswaranathan, and
  Ganguli]{sohl2015deep}
Sohl-Dickstein, J., Weiss, E., Maheswaranathan, N., and Ganguli, S.
\newblock Deep unsupervised learning using nonequilibrium thermodynamics.
\newblock In \emph{International Conference on Machine Learning}, pp.\
  2256--2265. PMLR, 2015.

\bibitem[Song \& Ermon(2019)Song and Ermon]{song2019generative}
Song, Y. and Ermon, S.
\newblock Generative modeling by estimating gradients of the data distribution.
\newblock \emph{Advances in Neural Information Processing Systems}, 32, 2019.

\bibitem[Song et~al.(2020)Song, Sohl-Dickstein, Kingma, Kumar, Ermon, and
  Poole]{song2020score}
Song, Y., Sohl-Dickstein, J., Kingma, D.~P., Kumar, A., Ermon, S., and Poole,
  B.
\newblock Score-based generative modeling through stochastic differential
  equations.
\newblock In \emph{International Conference on Learning Representations}, 2020.

\bibitem[Soudry et~al.(2018)Soudry, Hoffer, Nacson, Gunasekar, and
  Srebro]{soudry2018implicit}
Soudry, D., Hoffer, E., Nacson, M.~S., Gunasekar, S., and Srebro, N.
\newblock The implicit bias of gradient descent on separable data.
\newblock \emph{The Journal of Machine Learning Research}, 19\penalty0
  (1):\penalty0 2822--2878, 2018.

\bibitem[Suzuki(2018)]{suzuki2018adaptivity}
Suzuki, T.
\newblock Adaptivity of deep relu network for learning in {Besov} and mixed
  smooth {Besov} spaces: optimal rate and curse of dimensionality.
\newblock In \emph{International Conference on Learning Representations}, 2018.

\bibitem[Suzuki \& Nitanda(2021)Suzuki and Nitanda]{suzuki2021deep}
Suzuki, T. and Nitanda, A.
\newblock Deep learning is adaptive to intrinsic dimensionality of model
  smoothness in anisotropic {Besov} space.
\newblock \emph{Advances in Neural Information Processing Systems},
  34:\penalty0 3609--3621, 2021.

\bibitem[Telgarsky(2017)]{telgarsky2017neural}
Telgarsky, M.
\newblock Neural networks and rational functions.
\newblock In \emph{International Conference on Machine Learning}, pp.\
  3387--3393. PMLR, 2017.

\bibitem[Tenenbaum et~al.(2000)Tenenbaum, Silva, and
  Langford]{tenenbaum2000global}
Tenenbaum, J.~B., Silva, V.~d., and Langford, J.~C.
\newblock A global geometric framework for nonlinear dimensionality reduction.
\newblock \emph{science}, 290\penalty0 (5500):\penalty0 2319--2323, 2000.

\bibitem[Triebel(2011)]{triebel2011entropy}
Triebel, H.
\newblock Entropy numbers in function spaces with mixed integrability.
\newblock \emph{Revista matem{\'a}tica complutense}, 24\penalty0 (1):\penalty0
  169--188, 2011.

\bibitem[Tsybakov(2009)]{tsybakov2009introduction}
Tsybakov, A.~B.
\newblock \emph{Introduction to Nonparametric Estimation}.
\newblock Springer series in statistics. Springer, 2009.
\newblock ISBN 978-0-387-79051-0.
\newblock \doi{10.1007/b13794}.
\newblock URL \url{https://doi.org/10.1007/b13794}.

\bibitem[Vahdat et~al.(2021)Vahdat, Kreis, and Kautz]{vahdat2021score}
Vahdat, A., Kreis, K., and Kautz, J.
\newblock Score-based generative modeling in latent space.
\newblock \emph{Advances in Neural Information Processing Systems},
  34:\penalty0 11287--11302, 2021.

\bibitem[Vincent(2011)]{vincent2011connection}
Vincent, P.
\newblock A connection between score matching and denoising autoencoders.
\newblock \emph{Neural computation}, 23\penalty0 (7):\penalty0 1661--1674,
  2011.

\bibitem[Weed \& Bach(2019)Weed and Bach]{weed2019sharp}
Weed, J. and Bach, F.
\newblock Sharp asymptotic and finite-sample rates of convergence of empirical
  measures in wasserstein distance.
\newblock \emph{Bernoulli}, 25\penalty0 (4A):\penalty0 2620--2648, 2019.

\bibitem[Yang \& Barron(1999)Yang and Barron]{yang1999information}
Yang, Y. and Barron, A.
\newblock Information-theoretic determination of minimax rates of convergence.
\newblock \emph{Annals of Statistics}, pp.\  1564--1599, 1999.

\bibitem[Yarotsky(2017)]{yarotsky2017error}
Yarotsky, D.
\newblock Error bounds for approximations with deep relu networks.
\newblock \emph{Neural Networks}, 94:\penalty0 103--114, 2017.

\bibitem[Zhou(2020)]{zhou2020universality}
Zhou, D.-X.
\newblock Universality of deep convolutional neural networks.
\newblock \emph{Applied and computational harmonic analysis}, 48\penalty0
  (2):\penalty0 787--794, 2020.

\end{thebibliography}
\bibliographystyle{icml2023}

\newpage
\appendix
\onecolumn

\section{Several high-probability bounds on the backward paths}\label{section:Appendix-HPB}
One of the difficulties in the analysis is the unboundedness of the space and the value of the score.
This subsection aims to provide several treatments for such issues.
These inequalities allow us to focus on the score approximation within the bounded region.
We note that, however, some of the following bounds still depend on the time $t$, and therefore the level of difficulty for approximation and estimation of the score differs with respect to $t$.

In the following, we define several constants $C_{{\rm a},i}$.
Other than in this section, we simply denote them as $C_{{\rm a}}$ for simplicity.
\subsection{Bounds on $\|Y_t\|$ and $\|\Delta Y_t\|$ with high probability}
We first provide several high-probability bounds, which guarantee that most of the paths travel within some bounded region.
\begin{lemma}[Bounds on $\|Y_t\|$ and $\|\Delta Y_t\|$ with high probability]\label{lemma:Appendix-HPB-1}
    There exists a constant $C_{{\rm a},1}$ such that
    \begin{align}
       \mathbb{P}\left[
       \|Y_t\|_\infty \leq m_{\overline{T}-t}+C_{{\rm a},1}\sigma_{\overline{T}-t}\sqrt{\log (\eps^{-1}\underline{T}^{-1}\overline{T})}
    \ \text{for all $t\in [0,\overline{T}-\underline{T}]$}\right] \geq 1-\eps.
    \end{align}
    Moreover, for an arbitrarily fixed $0<\tau\leq 1$,
    \begin{align}
       \mathbb{P}\left[
       \|Y_t-Y_{t+\tau}\|_\infty \leq C_{{\rm a},1}\sqrt{\tau\log (\eps^{-1}\tau^{-1}\overline{T})}
    \ \text{for all $t\in [0,\overline{T}-\tau]$}\right] \geq 1-\eps.
    \end{align}
\end{lemma}
\begin{proof}
Remind that $Y_t=X_{\overline{T}-t}$. Thus we discuss bounding $X_{t}$ in the following.

We begin with the first assertion.
Let $t_1,t_2,\cdots,t_K$ be time steps satisfying $\underline{T}=t_1<t_2<\cdots<t_K=\overline{T}$ with $t_i-t_{i-1}=\Delta t$ that is some scaler value specified later. 
We first show the following for some constant $C_1$:
    \begin{align}\label{eq:Appendix_PathUniformBound-1}
     \mathbb{P}\left[
       \|X_t\|_\infty \leq m_{t}+C_1\sigma_t\sqrt{\log \eps^{-1}}
    \ \text{for all $t=t_i\ (i=1,2,\cdots,K)$}\right] \geq 1-\eps K
    .
    \end{align}
Remind that $X_{t}|X_0$ follows $\mathcal{N}(m_{t}X_0,\sigma_t^2)$ and $\|X_0\|_\infty \leq 1$.
\cref{Lemma:GaussianBound} yields that
\begin{align}
  \mathbb{P}\left[
       \|X\|_\infty \leq m_{t}+C_1\sigma_{t}\sqrt{\log \eps^{-1}}
\ \text{for some fixed $t=t_i$}\right]\geq 1- \eps,
\end{align}
which immediately yields \eqref{eq:Appendix_PathUniformBound-1}.

Then we consider how far each particle $X_t$ moves from $t=t_{i-1}$ to $t_{i}$.
Equivalently, we consider $X_t$ and decompose it into
\begin{align}\label{eq:Appendix_PathUniformBound-3}
    X_{t} = 
    \exp\left(-\int_{s=t_{i-1}}^{t_{i}}\beta_s \ds\right)X_{t_{i-1}}+B_{1-\exp(-2\int_{s=t_{i-1}}^{t_{i}}\beta_s \ds)},
\end{align}
where $B_s$ denotes a $d$-dimensional Brownian motion.
This is obtained by considering the Ornstein-Uhlenbeck process starting from $t=t_{i-1}$.
By \cref{Lemma:HittingTime}, with probability at least $\eps$, the following holds uniformly over $t\in [t_{i-1},t_{i}]$: 
\begin{align}
    \|X_{t}\|_\infty &\leq  \exp\left(-\int_{s=t_{i-1}}^{t_{i}}\beta_s \ds\right)\|X_{t_{i-1}}\|_\infty
    +
    \sqrt{1-\exp(-2\int_{s=t_{i-1}}^{t_{i}}\beta_s \ds)} \cdot 2\sqrt{\betahigh 2\log d\eps^{-1}}
    \\
    &\leq  \exp\left(-\int_{s=t_{i-1}}^{t_{i}}\beta_s \ds\right)\|X_{t_{i-1}}\|_\infty
    +
    \sqrt{2\betalow \Delta t} \cdot 2\sqrt{\betahigh 2\log d\eps^{-1}}.
\end{align}
If $ \|X_{t_{i-1}}\|_\infty \leq m_{t_{i-1}}+C_1\sigma_{t_{i-1}}\sqrt{\log \eps^{-1}}$, this is further bounded by
\begin{align}
    \|X_{t}\|_\infty &\leq m_{t_{i-1}}+C_1\sigma_{t_{i-1}}\sqrt{\log \eps^{-1}}
    +
    \sqrt{\Delta t} \cdot 4\sqrt{\betahigh \betalow \log d\eps^{-1}}.
\end{align}
Because we can check that $\sigma_t \simeq \sqrt{t} \land 1\geq \sqrt{\underline{T}}$ holds, if we take $\Delta \leq \underline{T}$, then we have that
\begin{align}
    C_1\sigma_{t_{i-1}}\sqrt{\log \eps^{-1}}
    +
    \sqrt{\Delta t} \cdot 4\sqrt{\betahigh \betalow \log d\eps^{-1}}
    \lesssim
    C_2\sigma_{t_{i-1}}\sqrt{\log \eps^{-1}}
    \label{eq:Appendix_PathUniformBound-2}
\end{align}
for all $t\in [t_{i-1},t_{i}]$, with some constant $C_2$.

Therefore, with probability $1-2K\eps$ we have \eqref{eq:Appendix_PathUniformBound-1}, and \eqref{eq:Appendix_PathUniformBound-2} for all $i$.
We need to take $K = \Ord(\overline{T}/\underline{T})$ to satisfy $\Delta \leq \underline{T}$.
We reset $\frac{\eps}{K}$ as a new $\eps$ and adjust $C_2$ accordingly.
Now the first assertion is proved.

Next, we consider the second assertion.
Let us consider a different time discretization $t_0=0,t_1=\tau,t_2=2\tau,\cdots,t_K=K\tau$ with $K=\min\{i\in \N|K\tau\geq \overline{T}\}$.
Then, from the first argument, we have that
$
    \|X_t\|_\infty \leq m_{t}+C_2\sigma_{t}\sqrt{\log (\eps^{-1}\tau^{-1}\overline{T})}
$
holds with probability at least $1-\eps$, for all $t=t_0,t_1,\cdots,t_K$.
We condition the event conditioned by this.
By \eqref{eq:Appendix_PathUniformBound-3}, we have that, for $t\geq t_{i-1}$,
\begin{align}
    X_{t}-X_{t_{i-1}} = 
  \left[  \exp\left(-\int_{s=t_{i-1}}^{t_{i}}\beta_s \ds\right)-1\right]X_{t_{i-1}}+B_{1-\exp(-2\int_{s=t_{i-1}}^{t_{i}}\beta_s \ds)},
\end{align}
which yields that
\begin{align}
   \| X_{t}-X_{t_{i-1}}\|_\infty & \leq 
  \left|  \exp\left(-\int_{s=t_{i-1}}^{t_{i}}\beta_s \ds\right)-1\right|\|X_{t_{i-1}}\|_\infty+\left\|B_{1-\exp(-2\int_{s=t_{i-1}}^{t_{i}}\beta_s \ds)}\right\|_\infty
 \\ & \leq
 \tau \betahigh (m_{t_{i-1}}+C_2\sigma_{t_{i-1}}\sqrt{\log (\eps^{-1}\tau^{-1}\overline{T})})+\left\|B_{1-\exp(-2\int_{s=t_{i-1}}^{t_{i}}\beta_s \ds)}\right\|_\infty
\end{align}
We bound the last term over $t\in [t_{i-1},t_{i}]$.
With probability at least $1-\frac{\eps}{K}$, that is bounded by $\sqrt{2\betalow \tau} \cdot 2\sqrt{\betahigh 2\log dK\eps^{-1}}$ according to \cref{Lemma:HittingTime}.
To summarize, with probability at least $1-2\eps$, 
\begin{align}
    \sup_{t\in [t_{i-1},t_i]} \| X_{t}-X_{t_{i-1}}\|_\infty
     \leq \tau \betahigh (m_{t_{i-1}}+C_2\sigma_{t_{i-1}}\sqrt{\log (\eps^{-1}\tau^{-1}\overline{T})})
     +
     \sqrt{2\betalow \tau} \cdot 2\sqrt{\betahigh 2\log dK\eps^{-1}}
\end{align}
holds for all $i=0,1,\cdots,K-1$. RHS is bounded by $ C_3\sqrt{\tau\log \eps^{-1}\tau^{-1}\overline{T}}$ with some sufficiently large constant $C_3$.

Then, for any $t$, there exists $i$ such that $t\leq t_i \leq t+\tau$.
Thus, with probability $1-2\eps$, $\| X_{t}-X_{t+\tau}\|_\infty\leq \| X_{t}-X_{t_{i-1}}\|_\infty+\| X_{t_i}-X_{t_{i-1}}\|_\infty+\| X_{t+\tau}-X_{t_{i}}\|_\infty$ is bounded by $ 3C_3\sqrt{\tau\log \eps^{-1}\tau^{-1}\overline{T}}$ for all $t$.
Setting $2\eps$ to $\eps$ yields the second assertion.
\end{proof}
\subsection{Bounds on $p_t(x)$}
We then give upper and lower bounds on $p_{t}(x)$.
\begin{lemma}[Upper and lower bounds on the density $p_t(x)$]\label{Lemma:LowerandUpperBounds}
    The following upper and lower bounds on $p_t(x)$ holds for a constant $\ConstDifBoundK$ depending on 
    $C_f$ and $d$: 
    \begin{align}
       \ConstDifBoundK^{-1}\exp\left(-\frac{d(\|x\|_\infty - m_t)_+^2}{\sigma_t^2}\right) \leq  p_t(x) \leq \ConstDifBoundK\exp\left(-\frac{(\|x\|_\infty - m_t)_+^2}{2\sigma_t^2}\right).\quad (\text{for all $t$.})
    \end{align}
\end{lemma}
\begin{proof}
    We first consider the case when $x\in [-m_t,m_t]^d$.
    The upper bound is relatively easy. 
    $f(y) \leq C_f\mathbbm{1}[y\in [-1,1]^d]$ means
    \begin{align}\label{eq:Appendix-HPB-2-Bound-1}
        p_t(x) =  
        \int \frac{1}{\sigma_t^{d}(2\pi)^\frac{d}{2}}f(y)\exp\left(-\frac{\|x-m_t y\|^2}{2\sigma_t^2}\right) \dy
        \leq \int \frac{C_f\mathbbm{1}[y\in [-1,1]^d]}{\sigma_t^{d}(2\pi)^\frac{d}{2}}\exp\left(-\frac{\|x-m_t y\|^2}{2\sigma_t^2}\right) \dy = \frac{2^dC_f}{\sigma_t^d(2\pi)^\frac{d}{2}} .
    \end{align}
    At the same time, we have that
    \begin{align}\label{eq:Appendix-HPB-2-Bound-2}
        p_t(x) 
        \leq \int \frac{C_f}{\sigma_t^{d}(2\pi)^\frac{d}{2}}\exp\left(-\frac{\|x-m_t y\|^2}{2\sigma_t^2}\right) \dy = \frac{C_f}{m^d_t}.
    \end{align}
    Thus, according to \eqref{eq:Appendix-HPB-2-Bound-1} and \eqref{eq:Appendix-HPB-2-Bound-2}, $p_t(x)$ is bounded by $\min\left\{\frac{2^dC_f}{\sigma_t^d(2\pi)^\frac{d}{2}} ,\frac{C_f}{m_t^d}\right\}$. This is further bounded by a constant that depends only on $C_f$ and $d$, because $m_t^2 + \sigma_t^2=1$ holds for all $t$.
    
    The lower bound can be understood as follows.
    We have
    \begin{align}
        p_t(x) &=  
        \int \frac{C_f^{-1}}{\sigma_t^{d}(2\pi)^\frac{d}{2}}f(y)\exp\left(-\frac{\|x-m_t y\|^2}{2\sigma_t^2}\right) \dy
        \\ & \geq
        \frac{1}{(2\pi)^\frac{d}{2}} \int f(x/m_t-\sigma_t y)\exp\left(-\frac{\|m_ty\|^2}{2}\right) \dy
        \quad (\text{by letting $(x-m_ty)/\sigma_t \mapsto m_ty$}).
        \label{eq:LemmaApproxScoreA-1}
    \end{align}
    Since $x\in [-m_t,m_t]^d$, we have $x/m_t\in [-1,1]^d$.
    Thus, $|\{y\in [-1,1]^d|\ x/m_t-\sigma_t y \in [-1,1]\}|\geq 1$.
    Moreover, $\exp\left(-\frac{\|m_ty\|^2}{2}\right) \geq \exp(-d^2/2)$ in $y\in [-1,1]^d$.
    Therefore, the integral \eqref{eq:LemmaApproxScoreA-1} is lower bounded by $\exp(-d^2/2)$.

    We then consider the case when $x \notin [-m_t,m_t]^d$.
    For such $x$, let $r = (\|x\|_\infty - m_t)/\sigma_t$ and choose $i^* $ from $ \{1,2,\cdots,d\}$ such that $|x_{i^*}| = \|x\|_\infty= m_t + r/\sigma_t$ holds.
    Then, we have the upper bound of $p_t(x)$ as
    \begin{align}
         p_t(x) &= 
        \int \frac{1}{\sigma_t^{d}(2\pi)^\frac{d}{2}}f(y)\exp\left(-\frac{\|x-m_t y\|^2}{2\sigma_t^2}\right) \dy
        \\ & 
         \leq \ConstDensityBoundE \prod_{i=1}^d  \int\frac{\mathbbm{1}[-1\leq y_i\leq 1]}{\sigma_t(2\pi)^\frac{1}{2}}\exp\left(-\frac{(x_i-m_t y_i)^2}{2\sigma_t^2}\right) \dy_i
           \\ & \label{eq:Appendix_ProbabilityBound-1}
         \lesssim C_f \int_{y_{i^*} \in [-1,1]} \frac{1}{\sigma_t(2\pi)^\frac{1}{2}}\exp\left(-\frac{(x_{i^*}-m_t y_{i^*})^2}{2\sigma_t^2}\right) \dy
         \\ & \left(\text{because $\int\frac{\mathbbm{1}[-1\leq y_i\leq 1]}{\sigma_t(2\pi)^\frac{1}{2}}\exp\left(-\frac{(x_i-m_t y_i)^2}{2\sigma_t^2}\right) \dy_i$ for $i\ne i^*$ is bounded by $\Ord(1)$, as $p_t(x)$ for $x\in [-m_t,m_t]^d$.}\right)
         \\ &
         \leq \frac{C_f}{m_t} \int_{a=r/\sqrt{2}}^\infty \frac{1}{\sqrt{\pi}}\exp\left(-a^2\right) \mathrm{d}a \quad\quad\quad\quad\quad\quad\quad\quad\quad\quad\quad (\text{by }a=x_{i*}-m_ty_{i_*} /\sqrt{2}\sigma_t)
           \\ &
         \leq \frac{\ConstDensityBoundE}{m_t} \exp\left(-r^2/2\right)  = \frac{\ConstDensityBoundE}{m_t} \exp\left(-\frac{(\|x\|_\infty - m_t)^2}{2\sigma_t^2}\right)
    \end{align}
    where we used $\int_z^\infty e^{-a^2} \mathrm{d}a \leq e^{-z^2} $ (see, e.g. \citet{chang2011chernoff}) for the last inequality.
    Also, \eqref{eq:Appendix_ProbabilityBound-1} is alternatively bounded by $\frac{2C_f}{\sigma_t(2\pi)^\frac12}\exp\left(-\frac{(\|x\|_\infty - m_t)^2}{2\sigma_t^2}\right)$.
    Because $m_t^2 + \sigma_t^2=1$ means that $\min\{m_t,\sigma_t\}\gtrsim 1$, it holds that $p_t(x) \lesssim C_f\exp\left(-\frac{(\|x\|_\infty - m_t)^2}{2\sigma_t^2}\right).$
    
        On the other hand, 
    \begin{align}
         p_t(x) &= 
        \int \frac{1}{\sigma_t^{d}(2\pi)^\frac{d}{2}}f(y)\exp\left(-\frac{\|x-m_t y\|^2}{2\sigma_t^2}\right) \dy
        \\ & 
         \geq \ConstDensityBoundE^{-1} \prod_{i=1}^d  \underbrace{\int_{y_i \in [-1,1]} \frac{1}{\sigma_t(2\pi)^\frac{1}{2}}\exp\left(-\frac{(x_i-m_t y_i)^2}{2\sigma_t^2}\right) \dy}_{\rm (a)}
           \\ & 
         = \ConstDensityBoundE^{-1} \left(\int_{y_{i^*} \in [-1,1]} \frac{1}{\sigma_t(2\pi)^\frac{1}{2}}\exp\left(-\frac{(x_{i^*}-m_t y_{i^*})^2}{2\sigma_t^2}\right) \dy\right)^d \quad (\text{because (a) is minimized when $i=i_*$})
           \\ &    \geq \frac{\ConstDensityBoundE^{-1}}{m_t^d} \left(\int_{a=r/\sqrt{2}}^{r/\sqrt{2}+\sqrt{2}m_t/\sigma_t} \frac{1}{\sqrt{\pi}}\exp\left(-a^2\right) \dy\right)^d  \quad (\text{by }(x_{i^*}-m_t y_{i^*})/\sqrt{2}\sigma_t)
           \\             &
         \geq \frac{\ConstDensityBoundE^{-1}}{m_t^d} \left(\int_{a=r/\sqrt{2}}^{r/\sqrt{2}+\sqrt{2}m_t} \frac{1}{\sqrt{\pi}}\exp\left(-a^2\right) \dy\right)^d 
         \end{align}
         \begin{align}
       \hspace{13mm}  &
         \geq \frac{\ConstDensityBoundE^{-1}}{m_t^d} \left( \frac{\sqrt{2}m_t}{\sqrt{\pi}}\exp\left(-(r/\sqrt{2}+\sqrt{2}m_t)^2\right) \right)^d  
         \\ & \quad\quad \quad \quad \quad \quad   (\text{by lower bounding $\exp(-a^2)$ in the integral interval and just multiplying the width of the interval})
         \\ &
         \geq \frac{\ConstDensityBoundE^{-1}}{m_t^d} \left( \frac{\sqrt{2}m_t}{\sqrt{\pi}}\exp\left(-r^2 - 4\right) \mathrm{d}a\right)^d
         \\ &
         \geq \frac{\ConstDensityBoundE^{-1}2^{d/2}}{e^{4d}\pi^{d/2}} \exp\left(-dr^2\right) ,
           \end{align}
  which gives the lower bound on $p_t(x)$.
\end{proof}
\subsection{Bounds on the derivatives of $p_t(x)$ and the score}
This subsection evaluates the derivatives of $p_t(x)$ and the score.
On the one hand, straightforward argument yields that the derivatives of $p_t(x)$ is bounded by $\pd^k p_t(x) = \Ord(1/\sigma_t^k) = \Ord(t^{-k/2})$.
On the other hand, as for the score, $\sup_{x\in \R^d}\|\nabla \log p_t(x)\|=\infty$ holds in general, which prevents us to construct an approximation of the score with neural networks.
This is because $\nabla \log p_t(x)=\frac{\nabla p_t(x)}{p_t(x)}$ and $p_t(x)$ can be arbitrarily small as $\|x\|\to \infty$.
Nevertheless, using \cref{Lemma:LowerandUpperBounds}, we can show the bounds on the score dependent on $x$ and $t$, in the next \cref{Lemma:Smooth1}.
In \cref{Lemma:Decay1}, \cref{Lemma:Smooth1} is used to show that the decay of $p_t$ is so fast that the approximation error in the region with small $p_t(x)$ (that can be $\gg 1$ in some $x$) does not much affects the $L^2(p_t)$ approximation error bound;
We can show that $\|\nabla \log p_t(x)\| = \tilde{\Ord}(1/\sigma_t) = \tilde{\Ord}(1 \lor 1/\sqrt{t})$ with high probability (when $x\sim p_t$).

\begin{lemma}[Boundedness of derivatives]\label{Lemma:Smooth1}
    For $k\in \Z_+$, there exists a constant $\ConstDifBoundC$ depending only on $k$, $d$, and $C_f$ such that
    \begin{align}\label{eq:LemmaSmooth1-1}
        |\pd_{x_{i_1}}\pd_{x_{i_2}}\cdots \pd_{x_{i_k}} p_t(x)| \leq \frac{\ConstDifBoundC}{\sigma_t^{k}}.
    \end{align}
    Moreover, we have that
    \begin{align}\label{eq:LemmaSmooth1-2}
        \| \nabla \log p_t(x)\|
        \leq 
         \frac{\ConstDifBoundC}{\sigma_t} \cdot \left(\frac{(\|x\|_\infty - m_t)_+}{\sigma_t}\lor 1\right)
        ,
    \end{align}
    and that for $i\in \{1,2,\cdots,d\}$,
    \begin{align}\label{eq:LemmaSmooth1-4}
        \|\pd_{x_{i}} \nabla \log p_t(x)\|
        \leq 
          \frac{\ConstDifBoundC}{\sigma_t^2} \left(\frac{(\|x\|_\infty - m_t)_+^2}{\sigma_t^2}\lor 1\right)
        .
    \end{align}
    and that
    \begin{align}\label{eq:LemmaSmooth1-10}
        \|\pd_{t} \nabla \log p_t(x)\|
        \leq 
        \frac{\ConstDifBoundC}{\sigma_t^{3}}\left[|\pd_t \sigma_t |+ |\pd_t m_t |\right]\left(\frac{(\|x\|_\infty - m_t)_+^2}{\sigma_t^2}\lor 1\right)^\frac32
        .
    \end{align}
\end{lemma}
\begin{proof}
First, we consider \eqref{eq:LemmaSmooth1-1}.
    Let 
    $
     g_1(x) = p_t(x) = 
        \int \frac{1}{\sigma_t^{d}(2\pi)^\frac{d}{2}}f(y)\exp\left(-\frac{\|x-m_t y\|^2}{2\sigma_t^2}\right) \dy
    .$
For $s\in \Z_+^d$, we abbreviate the notation as 
    $
        g_1^{(s)}(x) = \pd_{x_1}^{s_1} \pd_{x_2}^{s_2} \cdots \pd_{x_d}^{s_d} g_1 (x).
    $
For $s\in \Z_+^d$, we define $B_{s} = \{s'\in \Z_+^d|{s'}_i \leq s_i\ (i=1,\cdots,d)\}$ and a constant
    $c_s$ such that $\pd_{x_1}^{s_1} \pd_{x_2}^{s_2} \cdots \pd_{x_d}^{s_d} e^{-\|x\|^2/2} = \sum_{s'\in B_{s}}c_{s'} x_1^{s_1'}x_2^{s_2'}\cdots x_d^{s_d'}e^{-\|x\|^2/2}$ holds.
    Then, because of $\pd_{x_i} = \frac{1}{\sigma}\pd_{\frac{x_i}{\sigma}}$, we can write $g_1^{(s)}(x)$ as
    \begin{align}\label{eq:LemmaSmooth1-7}
        g_1^{(s)}(x) = \frac{\sum_{s'\in B_{s}}c_{s'}}{\sigma_t^{\sum_{i=1}^d s_i}} \underbrace{\int \prod_{i=1}^d \left(\frac{x_i-my_i}{\sigma_t}\right)^{s_i'}\frac{1}{\sigma_t^{d}(2\pi)^\frac{d}{2}}f(y)\exp\left(-\frac{\|x-m_t y\|^2}{2\sigma_t^2}\right) \dy}_{\rm (a)}.
    \end{align}
    Note that $\max_{s\colon \sum s_i \leq k} \{\sum_{s'\in B_{s}}c_{s'}\}$ is bounded by a constant that only depends on $k$.
    Thus we focus on the evaluation of $\mathrm{(a)}$.
    When $t\leq 1$, $\rm (a)$ in \eqref{eq:LemmaSmooth1-7} can be bounded by $\Ord(1/m_t^d)\simeq \Ord(1)$ (we hide dependency on $\sum_{i=1}^d s_i' \leq k$ and $C_f$).
    This is because $m_t\simeq 1$ and $f(y)\leq C_f$.
     On the other hand, when $t\geq 1$, $\sigma_t \gtrsim 1$ holds, we can bound $\rm (a)$ by $\Ord(1)$ by noting that $f(y)\ne 0$ only for $y\in [-1,1]^d$.
    Now, the first statement \eqref{eq:LemmaSmooth1-1} has been proven.
 
    We then consider $\nabla \log p_t(x)$ and its derivatives.
    We can focus on $[\nabla \log p_t(x)]_1$, and all the other coordinates of the score are bounded in the same way.
    Let 
    $
    g_2(x) = \sigma_t[\nabla p_t(x)]_1= - \int \frac{x_1-m_t y_1}{\sigma_t^{d+1}(2\pi)^\frac{d}{2}}f(y)\exp\left(-\frac{\|x-m_t y\|^2}{2\sigma_t^2}\right) \dy
    ,
    $
     and define  $g_2^{(s)}$ in the same way as that for  $g_1^{(s)}$.
     
    We can see that 
    \begin{align}\label{eq:LemmaSmoothness-6}
      [ \nabla \log p_t(x)]_1
        =
        \frac{1}{\sigma_t}\cdot \frac{g_2(x)}{g_1(x)},\quad 
        [ \pd_{x_{i}}\nabla \log p_t(x)]_1
        =
        \frac{1}{\sigma_t}\cdot \frac{\pd_{x_{i}} g_2(x)}{g_1(x)} -\frac{1}{\sigma_t}\cdot \frac{ g_2(x)(\pd_{x_{i}} g_1(x))}{g_1^2(x)}
       .
    \end{align}
    Moreover, 
    \begin{align}
   \frac{g_2(x)}{g_1(x)}& = \frac{- \int \frac{x_1-m_t y_1}{\sigma_t^{d+1}(2\pi)^\frac{d}{2}}f(y)\exp\left(-\frac{\|x-m_t y\|^2}{2\sigma_t^2}\right) \dy}{\int \frac{1}{\sigma_t^{d}(2\pi)^\frac{d}{2}}f(y)\exp\left(-\frac{\|x-m_t y\|^2}{2\sigma_t^2}\right) \dy},\label{eq:g2g1}
   \\ 
   \frac{\pd_{x_i} g_1(x)}{g_1(x)}& = \frac{1}{\sigma_t}\cdot \frac{- \int \frac{x_i-m_ty_i}{\sigma_t^{d+1}(2\pi)^\frac{d}{2}}f(y)\exp\left(-\frac{\|x-m_t y\|^2}{2\sigma_t^2}\right) \dy}{\int \frac{1}{\sigma_t^{d}(2\pi)^\frac{d}{2}}f(y)\exp\left(-\frac{\|x-m_t y\|^2}{2\sigma_t^2}\right) \dy},\label{eq:pg1g1}
   \\ 
   \frac{\pd_{x_i} g_2(x)}{g_1(x)}& = -\frac{1}{\sigma_t}\cdot \frac{\int \frac{\mathbbm{1}[i=1]-\frac{x_1-m_ty_1}{\sigma_t}\frac{x_i-m_ty_i}{\sigma_t}}{\sigma_t^{d}(2\pi)^\frac{d}{2}}f(y)\exp\left(-\frac{\|x-m_t y\|^2}{2\sigma_t^2}\right) \dy}{\int \frac{1}{\sigma_t^{d}(2\pi)^\frac{d}{2}}f(y)\exp\left(-\frac{\|x-m_t y\|^2}{2\sigma_t^2}\right) \dy}.\label{eq:pg2g1}
    \end{align}

    In order to bound them, we consider the following quantity with $\sum_{i=1}^d s_i\leq 2$.
    Also, let $\eps$ be a scaler value specified later, with which we assume $p_t(x)\geq \eps$ holds for the moment.
  \begin{align}\label{eq:Target-1}
    \frac{
      \int \prod_{i=1}^d \left(\frac{x_{i}-m_ty_{i}}{\sigma_t}\right)^{s_i}\frac{1}{\sigma_t^{d}(2\pi)^\frac{d}{2}}f(y)\exp\left(-\frac{\|x-m_t y\|^2}{2\sigma_t^2}\right) \dy
        }{
 \int\frac{1}{\sigma_t^{d}(2\pi)^\frac{d}{2}}f(y)\exp\left(-\frac{\|x-m_t y\|^2}{\sigma_t^2}\right) \dy        
        }
  \end{align}
  According to \cref{Lemma:ClipInt}, we have that 
    \begin{align}
        &\left|
\int_{A^x} \prod_{i=1}^d \left(\frac{x_i-m_ty_i}{\sigma_t}\right)^{s_i}\frac{1}{\sigma_t^{d}(2\pi)^\frac{d}{2}}f(y)\exp\left(-\frac{\|x-m y\|^2}{2\sigma_t^2}\right) \dy
\right. \\ & \quad\quad\quad\quad\quad\quad\quad\quad\quad\quad\quad\quad\quad\quad\quad\quad \left.
-
 \int_{\R^d} \prod_{i=1}^d \left(\frac{x_i-m_ty_i}{\sigma_t}\right)^{s_i}\frac{1}{\sigma_t^{d}(2\pi)^\frac{d}{2}}f(y)\exp\left(-\frac{\|x-m y\|^2}{2\sigma_t^2}\right) \dy
        \right|
    \leq \frac{\epsC}{2}.
    \end{align}
        where $A^x = \prod_{i=1}^d a^x_i $ with $ a^x_i =  [\frac{x_1}{m_t} - \frac{\sigma_tC_{\mathrm{f}}}{m_t}\sqrt{\log 2\epsC^{-1}}, \frac{x_1}{m_t} + \frac{\sigma_tC_{\mathrm{f}}}{m_t}\sqrt{\log 2\epsC^{-1}}]$.
        Note that $C_{\mathrm{f}}$ only depends on $\sum_{i=1}^d s_i$, $d$, and $C_f$.

        Therefore, when $p_t(x)=g_1(x) \geq \epsC$, 
    \begin{align}{\rm \eqref{eq:Target-1}} 
       &\leq
        \frac{2\int\prod_{i=1}^d \left(\frac{x_{i}-m_ty_{i}}{\sigma_t}\right)^{s_i}\frac{1}{\sigma_t^{d}(2\pi)^\frac{d}{2}}f(y)\exp\left(-\frac{\|x-m_t y\|^2}{2\sigma_t^2}\right) \dy
        }{
 \int_{A^x} \frac{1}{\sigma_t^{d}(2\pi)^\frac{d}{2}}f(y)\exp\left(-\frac{\|x-m_t y\|^2}{\sigma_t^2}\right) \dy        
        }
        \end{align}
        \begin{align}\hspace{7mm} & \leq 
        \frac{2\int_{A^x}\prod_{i=1}^d \left(\frac{x_{i}-m_ty_{i}}{\sigma_t}\right)^{s_i}\frac{1}{\sigma_t^{d}(2\pi)^\frac{d}{2}}f(y)\exp\left(-\frac{\|x-m_t y\|^2}{2\sigma_t^2}\right) \dy
        }{
 \int_{A^x} \frac{1}{\sigma_t^{d}(2\pi)^\frac{d}{2}}f(y)\exp\left(-\frac{\|x-m_t y\|^2}{\sigma_t^2}\right) \dy        
        }
        + \frac{2\cdot\frac{\epsC}{2}}{\epsC}
\\ & \quad\quad\quad\quad\quad\quad\quad\quad\quad\quad\quad\quad\left(\text{note that the denominator is larger than $\eps$}\right)
        \\ & \leq
        2\max_{y\in A_x}\left[\prod_{i=1}^d\left(\frac{{x}_{i}-m_t{y}_{i}}{\sigma_t}\right)^{s_i}\right]
        +1
         \\ & \leq 
        2\left(C_{\mathrm{f}}^2\log \epsC^{-1}\right)^{(\sum_{i=1}^d s_i)/2}  +1
        \label{eq:LemmaSmoothness-5}
         .
    \end{align}
    Applying this bound to \eqref{eq:g2g1}, \eqref{eq:pg1g1}, and \eqref{eq:pg2g1}, $\frac{g_2(x)}{g_1(x)}, \frac{\pd_{x_i}g_1(x)}{g_1(x)}$, and $\frac{\pd_{x_i}g_2(x)}{g_1(x)}$ are bounded by
    \begin{align}
\log^{1/2} \epsC^{-1},  \frac{ \log^{1/2} \epsC^{-1}}{\sigma_t}, \text{ and }\frac{\log \epsC^{-1}}{\sigma_t},
    \end{align}
    up to constant factors, respectively.
    Finally, we apply this to \eqref{eq:LemmaSmoothness-6} and obtain that
    \begin{align}
        \| \nabla \log p_t(x)\|
        \lesssim 
        \frac{\log^{1/2} \epsC^{-1}}{\sigma_t} 
       \text{\ and},
        \|\pd_{x_i} \nabla \log p_t(x)\|
        \lesssim 
        \frac{\log \epsC^{-1}}{\sigma^{2}_t} 
        .
    \end{align}
    Now we replace $\eps$ with a specific value. 
    Remember that $\eps$ should satisfy $\eps\leq p_t(x)$.
    According to \cref{Lemma:LowerandUpperBounds}, we have $\ConstDifBoundK^{-1}\exp\left(-\frac{d(\|x\|_\infty - m_t)_+^2}{\sigma_t^2}\right) \leq  p_t(x)$, which yields that
        \begin{align}
       & \|\nabla \log p_t(x)\|
        \leq 
        \frac{\ConstDifBoundC}{\sigma_t} \cdot \frac{(\|x\|_\infty - m_t)_+}{\sigma_t}\lor 1 \text{, and}\quad 
        \|\pd_{x_{i}} \nabla \log p_t(x)\|
        \leq 
        \frac{\ConstDifBoundC}{\sigma_t^2} \left(\frac{(\|x\|_\infty - m_t)_+^2}{\sigma_t^2}\lor 1\right),
    \end{align}
    with $\ConstDifBoundC$ depending on $k$, $d$ and $C_f$. Thus, we obtain \eqref{eq:LemmaSmooth1-2} and \eqref{eq:LemmaSmooth1-4}.

    Finally, we consider $\pd_t \nabla \log p_t(x).$
    \begin{align}
      \pd_t &\nabla \log p_t(x) = \pd_t\left(\frac{1}{\sigma_t}\cdot \frac{g_2(x)}{g_1(x)}\right)
      =\left(\pd_t \frac{1}{\sigma_t}\right)\frac{g_2(x)}{g_1(x)}  -\frac{1}{\sigma_t}\cdot\frac{(\pd_t g_1(x))}{g_1(x)}\cdot \frac{g_2(x)}{g_1(x)}+ \frac{1}{\sigma_t}\cdot\frac{ \pd_t g_2(x)}{g_1(x)} 
    \\ &= \frac{(-\pd_t\sigma_t)}{\sigma_t}\nabla \log p_t(x)
    \\ & \quad - \frac{1}{\sigma_t}\cdot \frac{ \int \frac{-d(\pd_t\sigma_t)\sigma_t^{-1} + \|x-m_t y\|^2(\pd_t \sigma_t)\sigma_t^{-3} -(\pd_t m_t)y^\top (m_ty-x)\sigma_t^{-2} }{\sigma_t^{d}(2\pi)^\frac{d}{2}}f(y)\exp\left(-\frac{\|x-m_t y\|^2}{2\sigma_t^2}\right)\dy}{\int \frac{1}{\sigma_t^{d}(2\pi)^\frac{d}{2}}f(y)\exp\left(-\frac{\|x-m_t y\|^2}{2\sigma_t^2}\right)\dy}\cdot \nabla \log p_t(x).
   \\  &\quad + \frac{1}{\sigma_t}\cdot \frac{\int \frac{ (\pd_t m_t) y_1 + (x_1-m_ty_1)((d+1)(\pd_t \sigma_t)\sigma_t^{-1} - \|x-m_t y\|^2(\nabla_t \sigma_t)\sigma_t^{-3} + (\pd_t m_t)y^\top (m_ty-x)\sigma_t^{-2})}{\sigma_t^{d+1}(2\pi)^\frac{d}{2}}f(y)\exp\left(-\frac{\|x-m_t y\|^2}{2\sigma_t^2}\right) \dy}{\int \frac{1}{\sigma_t^{d}(2\pi)^\frac{d}{2}}f(y)\exp\left(-\frac{\|x-m_t y\|^2}{2\sigma_t^2}\right) \dy}
    \end{align}
    By carefully decomposing this into the sum of \eqref{eq:Target-1}, and then applying \eqref{eq:LemmaSmoothness-5} and \cref{Lemma:LowerandUpperBounds}, we have the final bound \eqref{eq:LemmaSmooth1-10}.
\end{proof}

Now, based on \cref{Lemma:Smooth1}  we show that we only need to approximate $\nabla \log p_t(x)$ on some bounded region and on $x$ where $p_t(x)$ is not too small.
\begin{lemma}[Error bounds due to clipping operations]\label{Lemma:Decay1}
    Let $t\geq \underline{T}$.
    There exists a constant $C_{{\rm a},4}$ depending on $d$ and $\ConstDensityBoundE$, we have
    \begin{align}\label{eq:Appendix-HPB-Clip-1}
    &  \int_{\|x\|_\infty \geq m_t + C_{{\rm a},4}\sigma_t\sqrt{\log \eps^{-1}\underline{T}^{-1}}}  p_t(x)\|\nabla \log p_t(x) \|^2 \dx \leq  
  \eps,
  \\ & \label{eq:Appendix-HPB-Clip-2}  \int_{\|x\|_\infty \geq m_t + C_{{\rm a},4}\sigma_t\sqrt{\log \eps^{-1}\underline{T}^{-1}}}  p_t(x)\dx \leq \eps
    \end{align}
    for all $t\geq \underline{T}$.

    Moreover, there exists a constant $\ConstDifBoundG$ depending on $d$ and $\ConstDensityBoundE$ and, 
    for $x$ such that $\|x\|_\infty \leq m_t + C_{{\rm a},4}\sigma_t\sqrt{\log \eps^{-1}}$, we have
    \begin{align} \label{eq:Appendix-HPB-Clip-3}  
       \|\nabla \log p_t(x) \|  \leq 
        \frac{C_{{\rm a},5}}{\sigma_t} \sqrt{\log \eps^{-1}}
       .
    \end{align}
    Therefore, 
    \begin{align}\label{eq:Appendix-HPB-Clip-4}  
        &\int_{\|x\|_\infty \leq m_t + C_{{\rm a},4}\sigma_t\sqrt{\log \eps^{-1}\underline{T}^{-1}}} p_t(x) \mathbbm{1}[p_t(x) \leq \eps]\| \nabla \log p_t(x)\|^2 \dx \leq \frac{\ConstDifBoundG\eps}{\sigma_t^2} \cdot \log^\frac{d+2}{2} (\eps^{-1}\underline{T}^{-1}),
        \\ & \int_{\|x\|_\infty \leq m_t + C_{{\rm a},4}\sigma_t\sqrt{\log \eps^{-1}\underline{T}^{-1}}} p_t(x) \mathbbm{1}[p_t(x) \leq \eps] \dx \leq \ConstDifBoundG\eps\cdot\log^\frac{d}{2} (\eps^{-1}\underline{T}^{-1}).
        \label{eq:Appendix-HPB-Clip-5}  
    \end{align}
\end{lemma}
\begin{proof}
    According to \cref{Lemma:LowerandUpperBounds} and \cref{Lemma:Smooth1},
    \begin{align}
       p_t(x)\| \nabla \log p_t(x)\|^2 &\leq \ConstDifBoundK \exp\left(-\frac{(\|x\|_\infty - m_t)_+^2}{2\sigma_t^2}\right) \cdot    \frac{\ConstDifBoundC^2}{\sigma_t^2}\frac{(\|x\|_\infty - m_t)_+^2}{\sigma_t^2}
    \\ &   \leq \frac{\ConstDifBoundK\ConstDifBoundC^2}{\sigma_t^2} \exp\left(-\frac{r^2}{2}\right)r^2,
    \end{align}
    where we let $r:= (\|x\|_\infty - m_t)_+/\sigma_t$.
    Then, 
 \begin{align}
    &  \int_{\|x\|_\infty\geq m_t+ C_{{\rm a},4}\sigma_t \sqrt{\log \eps^{-1}}} p_t(x)\| \nabla \log p_t(x)\|^2 \dx\\ &
      \leq \int_{C_{{\rm a},4}\sqrt{\log \eps^{-1}}}^\infty  \frac{\ConstDifBoundK\ConstDifBoundC^2}{\sigma_t}\exp\left(-\frac{r^2}{2}\right)r^2(d-1)(\sigma_t r+m_t)^{d-1} \mathrm{d}r
      \\ & \lesssim \frac{1}{\sigma_t}\eps\log^{d/2}\eps^{-1} .
    \end{align}
    We can make sure the final inequality by integration by parts.
    Because $\sigma_t \gtrsim \sqrt{\underline{T}}$, if we take $\eps'= \sqrt{\underline{T}} \cdot \eps^2$ then we have that $\frac{1}{\sigma_t}\eps'\log^{d/2}((\eps')^{-1}) \lesssim \eps$.
    Therefore, replacing $\eps$ with $\eps'$ and adjusting $C_{{\rm a},4}$ yield the bound \eqref{eq:Appendix-HPB-Clip-1}.

    In the same way, 
    \begin{align}
      \int_{\|x\|_\infty\geq m + C_{{\rm a},4}\sigma_t \sqrt{\log \eps^{-1}}} p_t(x) \dx &
      \leq \int_{C_{{\rm a},4}\sqrt{\log \eps^{-1}}}^\infty \ConstDifBoundK\sigma_t\exp\left(-\frac{r^2}{2}\right)(d-1)(\sigma_t r+m)^{d-1} \mathrm{d}r
      \\ & \lesssim \sigma_t\eps\log^{(d-2)/2}\eps^{-1},
    \end{align}
    which yields \eqref{eq:Appendix-HPB-Clip-2}.
    
    We then consider the second part of the lemma.
    Eq. \eqref{eq:Appendix-HPB-Clip-2} is a direct corollary of \cref{Lemma:Smooth1}:
    for $x$ with $\|x\|_\infty\leq m_t + C_{{\rm a},5}\sigma_t \sqrt{\log \eps^{-1}}$
     \begin{align}
        \|\nabla \log p_t(x)\|
        \leq 
            \frac{\ConstDifBoundC}{\sigma_t} \cdot C_{{\rm a},4}\sqrt{\log \eps^{-1}} \leq \frac{C_{{\rm a},5}}{\sigma_t}\sqrt{\log \eps^{-1}}. \quad (\text{by taking $C_{{\rm a},5}$ larger than $C_{{\rm a},3}C_{{\rm a},4}.$})
    \end{align}
    Using this, we have
    \begin{align}
        \int_{\|x\|_\infty \leq m_t +  C_{{\rm a},4}\sigma_t\sqrt{\log \eps^{-1}}} p_t(x) \mathbbm{1}[p_t(x) \leq \eps]\| \nabla \log p_t(x)\|^2 \dx \lesssim \eps  \cdot \frac{C_{{\rm a},4}^2}{\sigma_t^2} \log \eps^{-1}\cdot (m_t + C_{{\rm a},5}\sigma_t\sqrt{\log \eps^{-1}})^d.
    \end{align}
    Adjusting $C_{{\rm a},4},C_{{\rm a},5}$ and resetting $\eps$ yields \eqref{eq:Appendix-HPB-Clip-4}.
    Eq. \eqref{eq:Appendix-HPB-Clip-5} follows in the same way.
\end{proof}
\section{Approximation of the score function}\label{section:Appendix-Approximation}

In this section, we analyze approximation error for the (ideal) score matching loss minimization.
We construct a neural network that approximates $\nabla \log p_t(x)$ and bound the approximation error at each time $t$.
Throughout this section, we take a sufficiently large $N$ as a parameter that determines the size of the neural network, and $\underline{T}={\rm poly}(N^{-1})$ and $\overline{T}=\Ord(\log N)$.

\subsection{Approximation of $m_t$ and $\sigma_t$}
We begin with construction of sub-networks that approximate $m_t$ and $\sigma_t$.
In addition to the true data distribution $p_0(x)$, the score $\nabla \log p_t(x)$ also depends on $m_t$ and $\sigma_t$.
Indeed, in our construction, each diffused B-spline basis is approximated as a rational function of $x$, $m_t$ and $\sigma_t$.
Here, $m_t$ and $\sigma_t$ are as important as $x$, because we use exponentiation of $m_t$ and $\sigma_t$,  
as well as that of $x$, while exact values of $m_t$ and $\sigma_t$ are unavailable.
In other words, because approximation errors of $m_t$ and $\sigma_t$ are amplified via such exponentiation, approximating $m_t$ and $\sigma_t$ with high accuracy is necessary for obtaining tight bounds.
Therefore, in this subsection, we construct sub-networks for efficient approximation of $m_t$ and $\sigma_t$.
The following is the formal version of \cref{lemma:SigmaM}.
\begin{lemma}\label{Lemma:MandSigma}
    Let $0<\eps<\frac12$.
    Then, there exists a neural network $\NetworkMA(t)\in \Phi(L,W,B,S)$ that approximates $m_t$ for all $t \geq 0$, within the additive error of $\epsM$, where
        $L = \Ord(\log^2 \eps^{-1}), \|W\|_\infty =\Ord(\log \eps^{-1}),S = \Ord(\log^2 \eps^{-1})$, and $B = \exp(\Ord(\log^2 \eps^{-1}))$.

    Also, there exists a neural network $\NetworkSigmaA(t)\in \Phi(L,W,B,S)$ that approximates $\sigma_t$ for all $t \geq \eps$, within the additive error of $\epsM$, where
       $L\leq \Ord(\log^2 \epsD^{-1}), \|W\|_{\infty} = \Ord(\log^3 \epsD^{-1}), S = \Ord(\log^4 \epsD^{-1})$, and $B = \exp(\Ord(\log^2 \eps^{-1}))$.
\end{lemma}
\begin{proof}
    First we consider $m_t = \exp(-\int_0^t \beta_s \ds)$.
    Since $\beta \geq \betalow$, $\int_0^t \beta_s \ds \geq \log 4\eps^{-1}$ for all $t \geq A: = \log 4\eps^{-1}/\betalow$.
    We limit ourselves within $[0, A]$.
    Then, from \cref{assumption:SmoothBeta}, we can expand $\beta_s$ as $\beta_s = \sum_{i=0}^{k-1} \frac{\beta^{(i)}}{i!}s^i + \frac{\beta^{(k)}}{k!}(\theta s)^k$ with $|\beta^{(i)}|\leq 1$ and $0<\theta<1$, and therefore we obtain that
    \begin{align}
        \left|\int_0^t \beta_s \ds - \int_0^t \sum_{i=1}^{k-1}\frac{\beta^{(i)}}{i!}s^i \ds\right|
        \leq  \frac{|\beta^{(k)}|A^{k+1}}{(k+1)!} \leq \frac{A^{k+1}}{(k+1)!}.
    \end{align}
    We take $k = \max\{2eA,\lceil\log_2 4\eps^{-1}\rceil\} - 1$ so that we have $\frac{A^{k+1}}{(k+1)!}\leq \left(\frac{eA}{k+1}\right)^{k+1} \leq \frac{\eps}{4}$.
    $\int_0^t \sum_{i=1}^{k-1}\frac{\beta^{(i)}}{i!}s^i = \sum_{i=1}^{k-1}\frac{\beta^{(i)}}{(i+1)!}t^{i+1} $ can be realized with an additive error up to $\frac{\eps}{4}$ by the neural network with $L = \Ord(A^2 + \log^2 \eps^{-1}) =\Ord(\log^2 \eps^{-1}), \|W\|_\infty =\Ord(A+\log \eps^{-1})=\Ord(\log \eps^{-1}),S = \Ord(A^2 + \log^2 \eps^{-1})=\Ord(\log^2 \eps^{-1}), B = \exp(\log^2 \Ord(A + \log \eps^{-1}))=\Ord(\log^2 \eps^{-1})$, using \cref{Lemma:BaseNN02,Lemma:ParallelNetwork}.
    From the definition of $A$, we can easily check that $e^{-A}\leq \frac{\eps}{4}$ holds.
    We clip the input with $[0,A]$ to obtain the neural network $\phi_1$, which approximates $\int_0^t \beta_s \ds$ with an additive error of $\frac{\eps}{4} + \frac{\eps}{4} = \frac{\eps}{2}$ for $x\in [0,A]$, and satisfies $|\phi_1(x)| = |\phi_1(A)|$ for all $x\geq A$.

    Then we apply \cref{Lemma:TaylorExp2} with $\eps = \frac{\eps}{4}$.
    Then we obtain the neural network $\NetworkMA$ of the desired size, which approximates $m_t = \exp(-\int_0^t \beta_s \ds)$ with an additive error of $\frac{\eps}{2} + \frac{\eps}{4} = \frac{3\eps}{4}$ for $x\in [0,A]$ and $|\NetworkMA(x) - e^{-x}| \leq |\NetworkMA(x) - \NetworkMA(A)| + |\NetworkMA(A) - e^{-A}| + |e^{-A}- e^{-x}| \leq 0 + \frac{3\eps}{4}+ \frac{\eps}{4} = \eps$ for $x\geq A$.

    Similarly, we can approximate $\sigma^2 = 1 - \exp(-2\int_0^t \beta_s \ds)$ with an additive error of $\Ord(\eps^{1.5})$ using a neural network with $L = \Ord(\log^2 \eps^{-1}), \|W\|_\infty =\Ord(\log \eps^{-1}),S = \Ord(\log^2 \eps^{-1}), B = \exp(\Ord(\log^2 \eps^{-1}))$.
    Since $t \geq \eps$, we have $\sigma^2_t = 1- \exp(-2\int_0^t \beta_s \ds) \geq c \eps$ for some constant $c$ depending on $\betalow$. 
    Then, we apply \cref{Lemma:ApproxRoot} with $\eps =  c \eps$ and finally obtain a neural network $\NetworkSigmaA(t)$ that approximates $\sigma_t$ with an additive error of $c\eps+\frac{\eps^{1.5}}{\sqrt{c\eps}} = \Ord(\eps)$, with $L= \Ord(\log^2 \epsD^{-1}), \|W\|_{\infty} = \Ord(\log^3 \epsD^{-1}), S = \Ord(\log^4 \epsD^{-1})$, and $B = \exp(\Ord(\log^2 \eps^{-1}))$.
    Adjusting hidden constants can make the approximation error smaller than $\eps$, and concludes the proof.
\end{proof}

\subsection{Approximation via the diffused B-spline basis}\label{subsection:Approximation-BSpline}

This subsection introduces the approximation via the \textit{diffused B-spline basis} and the \textit{tensor-product diffused B-spline basis}, which enable us to approximate the score $\nabla \log p_t(x)$ in the space of $\R^d \times [\underline{T},\overline{T}]$.
Although we consider the function approximation in a $(d+1)$-dimensional space, the obtained rate (\cref{theorem:Approximation}) is the typical one for a $d$-dimensional space.
This is because our basis decomposition can reflect the structure of $p_0$ for $t>0$.
Before beginning the formal proof, we provide extended proof outline about the approximation via the diffusion B-spline basis and tensor-product diffused B-spline basis, which is more detailed than that in \cref{section:Approximation}.

Remind that the cardinal B-spline basis of order $l$ can be written as
\begin{align}
    \mathcal{N}_m(x) = \frac{1}{l!} \mathbbm{1}[0\leq x \leq l+1]\sum_{l'=0}^{l}(-1)^j {}_{l+1} \mathrm{C}_{l'} (x-l')_+^l 
\end{align}
(see Eq. (4.28) of \citet{mhaskar1992approximation} for example) and the function in the Besov space can be approximated by a sum of $M_{k,j}^d(x)$
\begin{align}
    M_{k,j}^d(x) = \prod_{i=1}^d \mathcal{N}_m(2^{k_i}x_i - j_i)
\end{align}
where $k \in \Z_+^d$ and $j\in \Z^d$.

Therefore, the denominator and numerator of the score
\begin{align}\nabla \log p_t(x) = \frac{\nabla p_t(x)}{p_t(x)} =-\frac{1}{\sigma_t}\cdot  \frac{ \int \frac{x-m_t y}{\sigma_t^{d+1}(2\pi)^\frac{d}{2}}f(y)\exp\left(-\frac{\|x-m_t y\|^2}{2\sigma_t^2}\right) \dy}{\int \frac{1}{\sigma_t^{d}(2\pi)^\frac{d}{2}}f(y)\exp\left(-\frac{\|x-m_t y\|^2}{2\sigma_t^2}\right) \dy}
\end{align}
are decomposed into the sum of
\begin{align}\label{eq:Diffused-B-splineBasis-1}
  E_{k,j}^{(1)}(x,t) :=  \int \frac{1}{\sigma_t^{d}(2\pi)^\frac{d}{2}}\mathbbm{1}[\|y\|_\infty\leq \ConstDifBoundO]M_{k,j}^d(y)\exp\left(-\frac{\|x-m_t y\|^2}{2\sigma_t^2}\right) \dy
\end{align}
and
\begin{align}\label{eq:Diffused-B-splineBasis-2}
  E_{k,j}^{(2)}(x,t) := \int \frac{x-m_t y}{\sigma^{d+1}(2\pi)^\frac{d}{2}}\mathbbm{1}[\|y\|_\infty\leq \ConstDifBoundO]M_{k,j}^d(y)\exp\left(-\frac{\|x-m_t y\|^2}{2\sigma_t^2}\right) \dy,
\end{align}
respectively.
This corresponds to what we called the tensor-product diffused B-spline basis in \cref{section:Approximation}.
Here $E_{k,j}^{(1)}(x,t)$ is the same as $E_{k,j}(x,t)$ in \cref{section:Approximation}, except for the term of $\mathbbm{1}[\|y\|_\infty\leq \ConstDifBoundO]$.
Note that $\ConstDifBoundO$ be a scaler value adjusted later.
We then approximate each of the denominator and numerator of $\nabla \log p_t(x)$ combining sub-networks that approximates each $E_{k,j}^{(1)}(x,t)$ or $E_{k,j}^{(2)}(x,t)$. 

Here we briefly remark why $\mathbbm{1}[\|y\|_\infty\leq \ConstDifBoundO]$ appears. 
Let us assume $\ConstDifBoundO=1$ and approximate $p_t(x)$ based on basis decomposition of $p_0(x)$, although later we need to consider other situations.
If we use basis decomposition as $p_0(x) \approx f_N(x) = \sum M_{k,j}^d(x)$, existing results such as \cref{Lemma:SuzukiBesov} only assure that the approximation is valid within $[-1,1]^d$ and do not guarantee anything outside the region.
This might harm the approximation accuracy when we integrate the approximation of $p_t(x)$ over all $\R^d$.
Therefore, we need to force $f_N(x)=0$ if $\|x\|_\infty>1$ by the indicator function.

From now, we realize the (modified) tensor-product diffused B-spline basis with neural networks.
We take $E_{k,j}^{(1)}$ as an example, and the procedures for $E_{k,j}^{(2)}$ is essentially the same.
Remind that in \cref{section:Approximation} we decomposed $E_{k,j}$ into the product of the diffused B-spline basis:
\begin{align}
  \mathcal{D}_{k,j}(x_i,t)= \int\frac{\mathcal{N}(2^{k}x_i-j_i)}{\sigma_t\sqrt{2\pi}}\exp\left(-\frac{(x_i-m_ty_i)^2}{2\sigma_t^2}\right)\dx_i.
\end{align}
Although the way we proceed is essentially the same as that in \cref{section:Approximation},
here, more formally, we first truncate the integral intervals.
We clip the integral interval as
\begin{align}
    E_{k,j}^{(1)}(x,t) & \fallingdotseq \int_{y\in A^{x,t}} \frac{1}{\sigma_t^{d}(2\pi)^\frac{d}{2}}\mathbbm{1}[\|y\|_\infty\leq \ConstDifBoundO]M_{k,j}^d(y)\exp\left(-\frac{\|x-m_t y\|^2}{2\sigma_t^2}\right) \dy
    \\ & = \prod_{i=1}^d \left(\sum_{l'=0}^{l+1}\frac{(-1)^{l'}{}_{l+1} \mathrm{C}_{l'}}{l!} \int_{y_i \in a^x_{i}} \frac{1}{\sigma_t(2\pi)^\frac{1}{2}}\mathbbm{1}[|y_i|\leq \ConstDifBoundO]\mathbbm{1}[0\leq 2^{k_i}y_i-j_i \leq l+1]\right. \\ &  \quad\quad\quad\quad\quad\quad\quad\quad\quad\quad\quad\quad\quad\quad\quad\quad\quad\quad\quad\quad\quad\quad\times (2^{k}y_i-l'-j_i)_+^l  \exp\left(-\frac{(x_i-m_t y_i)^2}{2\sigma_t^2}\right)\dy_i \biggr) 
    ,\label{eq:B3FirstDiscussion-1}
\end{align}
where $A^{x,t} = \prod_{i=1}^d a^{x,t}_{i} $ with $ a^{x,t}_{i} =  [\frac{x_{i}}{m_t} - \frac{\sigma_t\ConstDifBoundP}{m_t}\sqrt{\log \epsK^{-1}}, \frac{x_{i}}{m_t} + \frac{\sigma_t\ConstDifBoundP}{m_t}\sqrt{\log \epsK^{-1}}]$, $\ConstDifBoundP = \Ord(1)$, and $0<\epsK<1$.
This clipping causes the error at most $\Ord(\eps)$ 
 according to \cref{Lemma:ClipInt} and the observation $\mathbbm{1}[\|y\|_\infty \leq \ConstDifBoundO]M_{k,j}^d(y) \leq \left((l+1)^{l+1}2^{l+1}\right)^d$.
In summary, owing to the fact that $M_{k,j}^d(x)$ is a product of univariate functions of $x_i\ (i=1,2,\cdots,d)$, the integral over $\R^d$ is now decomposed into the integral with respect to only one variable over the bounded region, which is a truncated version of the diffused B-spline basis $ \mathcal{D}_{k,j}$ introduced in \cref{section:Approximation}.

We now begin the formal proof with the following lemma.
We approximate 
\begin{align}\label{eq:jlowjhighbound-3}
    \int_{y_i \in a^{x,t}_{i}} \frac{1}{\sigma_t(2\pi)^\frac{1}{2}}\mathbbm{1}[|y_i|\leq \ConstDifBoundO]\mathbbm{1}[0\leq 2^{k}y_i-j_i \leq l+1](2^{k_i}y_i-l'-j_i)_+^l\exp\left(-\frac{(x_i-m_t y_i)^2}{2\sigma_t^2}\right)\dy_i
\end{align}
(remind \eqref{eq:B3FirstDiscussion-1}).
Note that
$\mathbbm{1}[|y_i|\leq \ConstDifBoundO]\mathbbm{1}[0\leq 2^{k}y_i-j_i \leq l+1]\equiv 0$ or $=\mathbbm{1}[a\leq 2^{k}y_i\leq b]$ holds with $a,b$ satisfying 
\begin{align}
    -C2^{k}-l\leq \min_i j_i \leq j_i \leq a<b\leq j_i + l+1\leq \max_i j_i + l+1 \leq C2^{k} + l+1,
    \label{eq:jlowjhighbound}
\end{align}
if we assume ${\rm supp}(p_0)=[-C,C]^d$ (see \cref{Lemma:SuzukiBesov}).
Based on \eqref{eq:jlowjhighbound}, \eqref{eq:jlowjhighbound-3} (if $\mathbbm{1}[|y_i|\leq \ConstDifBoundO]\mathbbm{1}[0\leq 2^{k}y_i-j_i \leq l+1](2^{k}y_i-l'-j_i)_+^l\not\equiv 0$) can alternatively written as
\begin{align}\label{eq:jlowjhighbound-2}
&\int_{y_i \in a^{x,t}_{i}} \frac{1}{\sigma_t(2\pi)^\frac{1}{2}} \mathbbm{1}[\jlow\leq 2^k y\leq \jhigh](2^ky_i-j')^l\exp\left(-\frac{(x_i-m_t y_i)^2}{2\sigma_t^2}\right)\dy_i
   , \\
    &\text{with }\jlow,\jhigh,j'\in \R,
    \quad  \jhigh-l-1\leq j'\leq \jlow\leq\jhigh,\quad  -C2^{k}-l\leq j',\jlow,\jhigh\leq C2^{k}+l+1.
\end{align}
In the following lemma, we consider the approximation of \eqref{eq:jlowjhighbound-2}.
We omit the subscript $i$ for the coordinates, for simple presentation.
Also, $j'$ in \eqref{eq:jlowjhighbound-2} is denoted by $j$, because $j\in \R^d$ will not be used in the following lemma.

\begin{lemma}[Approximation of the diffused B-spline basis]\label{Lemma:DiffusionBasis1}
    Let $j,k,l \in \Z, \jlow,\jhigh\in \R$ satisfy 
    $ \jhigh-l-1\leq j\leq \jlow\leq\jhigh,\ -C2^{k}-l\leq j,\jlow,\jhigh\leq C2^{k}+l+1$, 
    and $k,l\geq 0$.
    Assume that 
    $|\sigma'-\sigma_t|,|m'-m_t|\leq \epsSensitivity$, and take $\epsA$ from $0<\epsA<\frac12$ and $C>0$ arbitrarily.
    Then, there exists a neural network $\NetworkSplineA^{j,\jhigh,\jlow,k}\in \Phi(L,W,S,B)$
         with 
    \begin{align}
         L &= \Ord (\log^4 \epsA^{-1}+ \log^2 C+k),
        \\\|W\|_\infty &= \Ord (\log^6 \epsA^{-1} ),
      \\  S& = \Ord (\log^8 \epsA^{-1}+ \log^2 C+k), 
      \\  B &= \Ord(C^l2^{kl}) + \log^{\Ord(\log \epsA^{-1})}  \epsA^{-1}.
    \end{align}
    such that
    \begin{align}
   & \left|
    \NetworkSplineA^{j,\jhigh,\jlow,k}(x,\sigma',m') - 
        \int_{-\frac{\sigma_t \ConstDifBoundP}{m_t}\sqrt{\log \epsK^{-1}}+\frac{x}{m_t}}^{\frac{\sigma_t \ConstDifBoundP}{m_t}\sqrt{\log \epsK^{-1}}+\frac{x}{m_t}}\frac{1}{\sqrt{2\pi}\sigma_t}\mathbbm{1}[\jlow \leq 2^k y \leq \jhigh](2^k y-j)^l\exp\left(-\frac{(x-m_ty)^2}{2\sigma_t^2}\right)\dy 
        \right|
       \\& \leq 
          \tilde{\Ord}(\epsA) +
     \epsSensitivity C^{4l} 2^{k(4l+1)}\log^{\Ord(\log \epsA^{-1})} \epsA^{-1} 
        .
    \end{align}
    holds for all $x$ in $-C\leq x\leq C$ and for all $t\geq \eps$.

    Also, with the same conditions, there exists a neural network $ \NetworkSplineB^{j,\jhigh,\jlow,k}\in \Phi(L,W,S,B)$
         with the same bounds on $L,\|W\|_\infty, S,B$ as above
    such that
    \begin{align}
   & \left|
    \NetworkSplineB^{j,\jhigh,\jlow,k}(x,\sigma',m') - 
        \int_{-\frac{\sigma_t \ConstDifBoundP}{m_t}\sqrt{\log \epsK^{-1}}+\frac{x}{m_t}}^{\frac{\sigma_t \ConstDifBoundP}{m_t}\sqrt{\log \epsK^{-1}}+\frac{x}{m_t}}\frac{[x-m_ty]_i}{\sqrt{2\pi}\sigma_t^2}\mathbbm{1}[\jlow \leq 2^k y \leq \jhigh](2^k y-j)^l\exp\left(-\frac{(x-m_ty)^2}{2\sigma_t^2}\right)\dy 
        \right|
       \\& \leq 
          \tilde{\Ord}(\epsA) +
     \epsSensitivity C^{4l} 2^{k(4l+1)}\log^{\Ord(\log \epsA^{-1})} \epsA^{-1} 
        .
    \end{align}
    holds for all $x$ in $-C\leq x\leq C$ and for all $t\geq \eps$.

    Furthermore, we can take these networks so that
    $\|\NetworkSplineA^{j,\jhigh,\jlow,k}\|_\infty,\ \| \NetworkSplineB^{j,\jhigh,\jlow,k}\|_\infty = \Ord(1)$ hold.
\end{lemma}
\begin{proof}
    Here we only consider $\NetworkSplineA^{j,\jhigh,\jlow,k}$, because the assertion for $\NetworkSplineB^{j,\jhigh,\jlow,k}$ essentially follows the argument for $\NetworkSplineA^{j,\jhigh,\jlow,k}$.
    
    First, we approximate the exponential function within the closed interval, using polynomials of degree at most $\Ord (\log \epsA^{-1})$.
    Note that $\mathbbm{1}[\jlow \leq 2^k y \leq \jhigh](2^k y-j)^l$ is bounded by $(l+1)^l$, from the assumption of $\jhigh-l-1\leq j\leq \jlow\leq\jhigh$.
    Therefore, according to \cref{Lemma:TaylorExp}, there exists $S = \Ord(\log \epsA^{-1})$ and we have that
    \begin{align} \label{eq:LemmaBasis01-10}
        \left|\exp\left(-\frac{(x-m_ty)^2}{2\sigma_t^2}\right) -   
        \sum_{s=0}^{S-1} \frac{(-1)^s}{s!}\frac{(x-m_ty)^{2s}}{2^s \sigma_t^{2s}} \right|
        \leq 
        \epsA^2
    \end{align}
    for all $y\in [-\frac{\sigma_t \ConstDifBoundP}{m_t}\sqrt{\log \epsA^{-1}}+x, \frac{\sigma_t \ConstDifBoundP}{m_t}\sqrt{\log \epsA^{-1}}+x]$. 
    Then, we have that
    \begin{align}
    &    \left|
 \int_{-\frac{\sigma_t \ConstDifBoundP}{m_t}\sqrt{\log \epsK^{-1}}+\frac{x}{m_t}}^{\frac{\sigma_t \ConstDifBoundP}{m_t}\sqrt{\log \epsK^{-1}}+\frac{x}{m_t}}\frac{1}{\sqrt{2\pi}\sigma_t}\mathbbm{1}[\jlow \leq 2^k y \leq \jhigh](2^k y-j)^l\exp\left(-\frac{(x-m_ty)^2}{2\sigma_t^2}\right)\dy
 \right.
 \\ &\quad \left.-
  \int_{-\frac{\sigma_t \ConstDifBoundP}{m_t}\sqrt{\log \epsK^{-1}}+\frac{x}{m_t}}^{\frac{\sigma_t \ConstDifBoundP}{m_t}\sqrt{\log \epsK^{-1}}+\frac{x}{m_t}}\frac{1}{\sqrt{2\pi}\sigma_t}\mathbbm{1}[\jlow \leq 2^k y \leq \jhigh](2^k y-j)^l\left(\sum_{s=0}^{S-1} \frac{(-1)^s}{s!}\frac{(x-m_ty)^{2s}}{2^s \sigma_t^{2s}} \right)\dy
        \right|
     \\ &   \leq
    \max\left\{\frac{2\sigma_t \ConstDifBoundP}{m_t}\sqrt{\log \epsK^{-1}} , (l+1) \right\}
       \cdot \frac{1}{\sqrt{2\pi}\sigma_t^2}(l+1)^l \cdot \eps
       \lesssim \eps\log^\frac12 \eps^{-1}.
    \end{align}
    Here, $\frac{2\sigma_t \ConstDifBoundP}{m_t}\sqrt{\log \epsK^{-1}}$ comes from the length of the integral interval and $l+1$ comes from the interval where $\mathbbm{1}[\jlow \leq 2^k y \leq \jhigh]=1$ holds.

    
    Now all we need is to approximate the integral of polynomials over the closed interval:
    \begin{align}
    &
        \sum_{s=0}^{S-1}\int_{-\frac{\sigma_t \ConstDifBoundP}{m_t}\sqrt{\log \epsK^{-1}}+\frac{x}{m_t}}^{\frac{\sigma_t \ConstDifBoundP}{m_t}\sqrt{\log \epsK^{-1}}+\frac{x}{m_t}}
        \frac{1}{\sqrt{2\pi}\sigma_t}\mathbbm{1}[\jlow \leq 2^k y \leq \jhigh](2^k y-j)^l \cdot \frac{(-1)^s}{s!}\frac{(x-m_ty)^{2s}}{2^s \sigma_t^{2s}}\dy
    \\ & =
        \sum_{s=0}^{S-1}
        \sum_{l'=0}^l
        \frac{-(-1)^{s+l}}{\sqrt{2\pi}m_t^{l+1} s! 2^s }
        \left[
        {}_{l} \mathrm{C}_{l'} (2^k \sigma_t)^{l'} (jm_t-2^k x)^{l-l'} 
        \int_{-\ConstDifBoundP\sqrt{\log \epsA^{-1}}}^{\ConstDifBoundP\sqrt{\log \epsA^{-1}}}
         \mathbbm{1}\left[\frac{x-m_t2^{-k}\jhigh }{\sigma_t} \leq y \leq\frac{x-m_t2^{-k}\jlow }{\sigma_t} \right] y^{l'+2s} \dy
        \right]
    \\ & 
        \hspace{130mm} \left(\text{by resetting }y\leftarrow\frac{x-m_ty}{\sigma_t} \right)
    \\ & =
    \sum_{s=0}^{S-1}
        \sum_{l'=0}^l
        \frac{-(-1)^{s+l}{}_{l} \mathrm{C}_{l'} 2^{kl'} \sigma^{l'} (jm_t-2^k x)^{l-l'}}{\sqrt{2\pi}m_t^{l+1} s! 2^s (l'+2s+1)} 
        \Biggl[
        \left(\min\left\{\ConstDifBoundP\sqrt{\log  \epsA^{-1}},\max\left\{\frac{x-m_t2^{-k}\jlow }{\sigma_t} , - \ConstDifBoundP\sqrt{\log  \epsA^{-1}}\right\}\right\}\right)^{l'+2s+1}
        \\ & \hspace{55mm}
        -
        \left(\min\left\{\ConstDifBoundP\sqrt{\log  \epsA^{-1}},\max\left\{\frac{x-m_t2^{-k}\jhigh }{\sigma_t} , - \ConstDifBoundP\sqrt{\log  \epsA^{-1}}\right\}\right\}\right)^{l'+2s+1} 
        \Biggr]    \label{eq:LemmaBasis01-01}
        .
    \end{align}   

    We decompose \eqref{eq:LemmaBasis01-01} into the following sub-modules for convenience.
    We let
    \begin{align}
        f_1^{l',s}(x, \sigma,m) &= (\min\{\ConstDifBoundP\log^\frac12(\epsA^{-1}),\max\{\frac{x-m2^{-k}\jlow  }{\sigma} , - \ConstDifBoundP\log^\frac12(\epsA^{-1})\}\})^{l'+2s+1} ,
    \\    f_2^{l',s}(x,\sigma,m) &= (\min\{\ConstDifBoundP\log^\frac12(\epsA^{-1}),\max\{\frac{x-m2^{-k}\jhigh }{\sigma} , - \ConstDifBoundP\log^\frac12(\epsA^{-1})\}\})^{l'+2s+1} ,
    \\ f_3^{l',s}(x, \sigma,m) & = f_1^{l',s}(x, \sigma,m) - f_2^{l',s}(x, \sigma,m)
    \\f_4^{l'}(x,m) & = (jm-2^k x)^{l-l'},
    \\ f_5^{l'}(\sigma) &= \sigma^{l'},
    \\ f_6(m) &= m^{-(l+1)},
   \\ f_7^{l',s}(x, \sigma,m) & =        
   f_3^{l',s}(x, \sigma,m)f_4^{l'}(x,m)f_5^{l'}(\sigma)f_6(m)
    .
    \end{align}
    They also depends on $j,\jlow ,\jhigh,k$, and $l$, but we omit the dependency on these variables for simple presentation.
    We take some $\epsI>0$, which is adjusted at the final part of the proof.

    We first consider approximation of $f_1^{l',s}(x, \sigma,m)$.
    We realize this as 
    \begin{align}
       & f_1^{l',s}(x, \sigma,m) \fallingdotseq \phi_1^{l',s}(x, \sigma,m) \\& := \NetworkMultiB(\cdot; l'+2s+1)\circ \NetworkClipA(\cdot ;- \ConstDifBoundP\log^\frac12  (\epsA^{-1}), - \ConstDifBoundP\log^\frac12  (\epsA^{-1})) \circ (\NetworkMultiB(x-m2^{-k}\jlow ,\NetworkInvA(\sigma)))
        .
    \end{align}
    by setting $\epsD = \min\{\sigma_{\eps},\epsI\}$ in \cref{Corollary:ApproxInv} for $\NetworkInvA$, $\epsB =\epsI, \ConstMultA = \max\{2C+l+1, \sigma_{\eps}^{-1}\}\geq \max\{|x|+m2^{-k}\jlow, \sigma_{\eps}^{-1}\}$ in \cref{Lemma:BaseNN02} for the first  $\NetworkMultiB$, $a =  -\ConstDifBoundP\log^\frac12  (\epsA^{-1}), b= \ConstDifBoundP\log^\frac12  (\epsA^{-1})$ in \cref{Lemma:ClippingFunc} for $\NetworkClipA$, and $\epsB = \epsI, \ConstMultA = \ConstDifBoundP\log^\frac12  (2\epsA^{-1})$ in \cref{Lemma:BaseNN02} for the second $\NetworkMultiB$.
    Note that $\sigma_{\eps}\simeq \sqrt{\eps}$.
    Then, using  \cref{Lemma:BaseNN02,Lemma:ConcateNetwork,Lemma:ClippingFunc,Lemma:ApproxInv} the size of the network is at most
    \begin{align}
    \begin{array}{l}
        L = \Ord (\log^2\epsI^{-1}+ \log^2 \epsA^{-1}+ \log^2 C),
        \\\|W\|_\infty = \Ord (\log^3 \epsI^{-1} + \log^3 \epsA^{-1} ),
      \\  S = \Ord (\log^4\epsI^{-1}+  \log^4 \epsA^{-1}+ \log^2 C), 
      \\  B = \Ord(\epsI^{-2}+ C^2) + \log^{\Ord(\log \epsA^{-1})}  \epsA^{-1}.
      \end{array}
      \label{eq:LemmaBasis01-02}
    \end{align}
    Approximation error between $f_1^{l',s}(x, \sigma_t,m_t) $ and $ \phi_1^{l',s}(x, \sigma',m')$ is bounded by
    \begin{align}
    &
       \epsI + \Ord(\log \epsA^{-1})(\ConstDifBoundP\log^\frac12  \epsA^{-1})^{\Ord(\log \epsA^{-1})}
      \cdot (\epsI + \max\{C+l+2, \sigma_\eps^{-1}\}^2\cdot(\epsI + \epsSensitivity(\epsI^{-2}+\epsA^{-2})))
       \\ & =(\epsI+ \epsSensitivity ) \left(\log^{\Ord(\log \epsA^{-1})} \epsA^{-1} +C^2\right)
       .
    \end{align}
    $f_2^{l',s}(x, \sigma_t,m_t)$ is also approximated in the same way, and therefore aggregating $f_1^{l',s}(x, \sigma_t,m_t)$ and $f_2^{l',s}(x, \sigma_t,m_t)$ (by using \cref{Lemma:ParallelNetwork}) yields that $f_3^{l',s}(x, \sigma_t,m_t)$ is approximated by $\phi_3^{l',s}(x,\sigma',m')$ with the error up to an additive error of 
    $(\epsI+ \epsSensitivity ) \left(\log^{\Ord(\log \epsA^{-1})} \epsA^{-1} +C^2\right)$
     using a neural network with the same size as that of \eqref{eq:LemmaBasis01-02}.

    Next, we consider $f_4^{l'}(x,m_t)$.
    Since $2^k x = \Ord(C 2^k)$ and $|jm_t-jm'|\leq \Ord(C2^k\epsSensitivity) $, we approximate $f_4^{l'}(x,m_t)$ with a neural network $\phi_4^{l'}(x,m') \in \Phi(L,W,S,B)$, where $L,\|W\|_\infty, S,B$ are evaluated by \cref{Lemma:BaseNN02,Lemma:ConcateNetwork} (setting $\epsB =\epsI, \ConstMultA = \Ord(C 2^k)$) as
    \begin{align}
      L = \Ord (\log \epsI^{-1}+ k\log C)
      ,\quad W = \Ord (1)
     ,\quad S = \Ord (\log\epsI^{-1} + k\log C), 
     \quad B = \Ord( C^l 2^{kl})
      .
      \end{align}
      Approximation error between $f_4^{l'}(x,m_t)$ and $\phi_4^{l'}(x,m')$ is bounded as 
     $
          \epsI + \Ord( C^{l} 2^{kl}) \epsSensitivity
      $, 
      using \cref{Lemma:BaseNN02}.
      
      The arguments for $f_5^{l'}(\sigma)$ and $f_6(m)$ are just setting appropriate parameters in \cref{Lemma:BaseNN02} and \cref{Corollary:ApproxInv}, respectively.
      For $f_5^{l'}(\sigma_t)$, there exists a neural network $\phi^{l'}_5(\sigma')$ with $L = \Ord(\log \epsI^{-1}), \|W\|_\infty = 48l , S = \Ord(\log \epsI^{-1} ), B=1$ and the approximation error between $f_5^{l'}(\sigma)$ and $\phi_5^{l'}(\sigma')$ is bounded by $\epsI + l\epsSensitivity$, by setting $d=l'(\leq l), \epsB = \epsI$ in \cref{Lemma:BaseNN02}.
        For $f_6(m_t)$, there exists a neural network $\phi_6(m')$ with $L = \Ord(\log^2 \epsI^{-1} + \log^2 m_\eps^{-1}), \|W\|_\infty = \Ord(\log^3 \epsI^{-1} + \log^3 m_\eps^{-1}), S = \Ord(\log^4 \epsI^{-1} + \log^4 m_\eps^{-1}), B=\Ord(\epsI^{-l-1} + \mlow^{-l-1})$ and the approximation error between $f_6(m_t)$ and $\phi_6(m')$ is bounded by $\epsI + (l+1)\epsI^{-l-2}\epsSensitivity + (l+1)m_\eps^{-l-2}\epsSensitivity$, by setting $d=l+1, \epsB = \min\{\epsI,m_\eps\}$ in \cref{Corollary:ApproxInv}.
        Note that $m_\eps \gtrsim 1$.

    Therefore, \cref{Lemma:BaseNN02} with $\epsB=\epsI$ yields that
      there exists a neural network $\phi_7^{l',s}(x,m,\sigma)$ such that 
      \begin{align}
        L &= \Ord (\log^2\epsI^{-1}+ \log^2 \epsA^{-1}+ \log^2 C+k),
        \\\|W\|_\infty &= \Ord (\log^3 \epsI^{-1} + \log^3 \epsA^{-1} ),
      \\  S& = \Ord (\log^4\epsI^{-1}+  \log^4 \epsA^{-1}+ \log^2 C+k), 
      \\  B &= \Ord(\epsI^{-2}+ C^2) + \log^{\Ord(\log \epsA^{-1})}  \epsA^{-1}+ C^l2^{kl}.
      \end{align}
      where approximation error between $f_7^{l',s}(x,m_t,\sigma_t)$ and $\phi_7^{l',s}(x,m',\sigma')$ is bounded as
      \begin{align}
          \left|f_7^{l',s}(x, \sigma,m) - \phi_7^{l',s}(x,m',\sigma')\right|
          \leq
          (\epsI + \epsSensitivity(\epsI^{-l-2}+C^{4l} 2^{4kl}))\log^{\Ord(\log \epsA^{-1})} \epsA^{-1} 
          .
      \end{align}

    Finally, we sum up $\phi_7^{l',s}(x,m',\sigma')$ multiplied $\frac{-(-1)^{s+l}{}_{l} \mathrm{C}_{l'} 2^{kl'}}{\sqrt{2\pi}s! 2^s (l'+2s+2)}$ over $(l',s)$, according to \eqref{eq:LemmaBasis01-01} and using \cref{Lemma:ParallelNetwork}.
    Here, the coefficient is bounded by $2^{(k+1)l}$ and the total number of possible combinations $(l',s)$ is bounded by $\Ord(lS) = \Ord(\log\epsA^{-1})$.
    Then,
    approximation error for \eqref{eq:LemmaBasis01-01} is bounded as
    \begin{align}
     2^{(k+1)l}(\epsI + \epsSensitivity(\epsI^{-l-2}+C^{4l} 2^{4kl}))\log^{\Ord(\log \epsA^{-1})} \epsA^{-1} 
    .
    \end{align}
    In order to bound the terms related to $\epsI$ by $\Ord(\epsA)$, we take $\epsI = \Ord(2^{-(k+1)l}\log^{-\Ord(\log \epsA^{-1})} \epsA^{-1} )$.
    Then, the total approximation error is bounded by 
    $\tilde{\Ord}(\epsA) +
     \epsSensitivity C^{4l} 2^{k(4l+1)}\log^{\Ord(\log \epsA^{-1})} \epsA^{-1} 
    $
    and this is achieved by a neural network with 
    \begin{align}
         L &= \Ord (\log^4 \epsA^{-1}+ \log^2 C+k),
        \\\|W\|_\infty &= \Ord (\log^6 \epsA^{-1} ),
      \\  S& = \Ord (\log^8 \epsA^{-1}+ \log^2 C+k), 
      \\  B &= \Ord(C^l2^{kl}) + \log^{\Ord(\log \epsA^{-1})}  \epsA^{-1}.
    \end{align}

    Finally, because 
    \begin{align}
&\left|\int_{-\frac{\sigma_t \ConstDifBoundE}{m_t}\sqrt{\log \epsK^{-1}}+\frac{x}{m}}^{\frac{\sigma_t \ConstDifBoundP}{m_t}\sqrt{\log \epsK^{-1}}+\frac{x}{m_t}}\frac{1}{\sqrt{2\pi}\sigma_t}\mathbbm{1}[\jlow \leq 2^k y \leq \jhigh](2^k y-j)^l\exp\left(-\frac{(x-m_ty)^2}{2\sigma_t^2}\right)\dy\right|
\\ &\leq 
       \int\frac{1}{\sqrt{2\pi}\sigma_t}\mathbbm{1}[\jlow \leq 2^k y \leq \jhigh](l+1)^l\exp\left(-\frac{(x-m_ty)^2}{2\sigma_t^2}\right)\dy
       \lesssim C_f,
    \end{align}
    we can clip $\NetworkSplineA^{j,\jhigh,\jlow,k}$ so that it is bounded by $\Ord(1)$.
\end{proof}

   We now approximate the (modified) tensor product diffused B-spline basis.
   The following is the formal version of \cref{lemma:DiffusedBesov-Main}.
   Without the term of $\mathbbm{1}[\|y\|_\infty\leq \ConstDifBoundO]$, the statement matches that of \cref{lemma:DiffusedBesov-Main}.
   This network $\NetworkSplineC$ corresponds to $\phi_{\rm TDB}$ in \cref{lemma:DiffusedBesov-Main}.
\begin{lemma}[Approximation of the tensor-product diffused B-spline bases]\label{Lemma:DiffusionBasis2}
    Let $k\in \Z_+, j\in \Z^d, l\in \Z_+$ with 
  $-C2^{k}-l\leq j_i\leq C2^{k}\ (i=1,2,\cdots,d)$, $\epsJ\ (0<\epsJ<\frac12)$ and $C>0$.
    There exists a neural network $\NetworkSplineC(x,t)\in \Phi(L,W,S,B)$ with 
          \begin{align}
           L &= \Ord (\log^4 \epsA^{-1}+ \log^2 C+k^2),
        \\\|W\|_\infty &= \Ord (\log^6 \epsA^{-1} + \log^3 C+k^3),
      \\  S& = \Ord (\log^8 \epsA^{-1}+ \log^4 C+k^4), 
      \\  B &= \exp\left(\log^4 \epsA^{-1}+\log C + k\right),
    \end{align}
  such that
    \begin{align}
   & \left|
    \NetworkSplineC^{k,j}(x,t) - 
    \int_{\R^d} \frac{1}{\sigma_t^{d}(2\pi)^\frac{d}{2}}\mathbbm{1}[\|y\|_\infty\leq \ConstDifBoundO]M_{k,j}^d(y)\exp\left(-\frac{\|x-m_t y\|^2}{2\sigma_t^2}\right) \dy
        \right|
          \leq \eps 
    \end{align}
    holds for all $x\in [-C,C]^d$. 

    Also, with the same conditions, there exists a neural network $\NetworkSplineD\in \Phi(L,W,S,B)$
    with the same bounds on $L,\|W\|_\infty, S,B$ as above
    such that
    \begin{align}
   & \left\|
    \NetworkSplineD^{k,j}(x,\sigma',m') - 
       \int_{\R^d} \frac{x-m_ty}{\sigma_t^{d+1}(2\pi)^\frac{d}{2}}\mathbbm{1}[\|y\|_\infty\leq \ConstDifBoundO]M_{k,j}^d(y)\exp\left(-\frac{\|x-m_t y\|^2}{2\sigma_t^2}\right) \dy
        \right\|\leq \eps
        .
    \end{align}
    holds for all $x\in [-C,C]^d$.

    Furthermore, we can choose these networks so that
    $\|\NetworkSplineC^{k,j}\|_\infty,\ \|\NetworkSplineD^{k,j}\|_\infty = \Ord(1)$ hold. 
\end{lemma}
\begin{proof}
    Here we only prove the first part, because the second part follows in the same way.
    We assume $|\sigma'-\sigma_t|,|m'-m_t|\leq \epsSensitivity$.

  From the discussion \eqref{eq:B3FirstDiscussion-1}, we approximate 
\begin{align} 
    & \prod_{i=1}^d \left(\sum_{l'=0}^{l+1}\frac{(-1)^{l'}{}_{l+1} \mathrm{C}_{l'}}{l!} \int_{y_i \in a^x_{i}} \frac{1}{\sigma(2\pi)^\frac{1}{2}}\mathbbm{1}[|y_i|\leq \ConstDifBoundO]\mathbbm{1}[0\leq 2^{k}y_i-j_i \leq l+1]\right. \\ &  \quad\quad\quad\quad\quad\quad\quad\quad\quad\quad\quad\quad\quad\quad\quad\quad\quad\quad\quad\quad\quad\quad\times (2^{k_i}y_i-l'-j_i)_+^l \exp\left(-\frac{(x_i-m y_i)^2}{2\sigma^2}\right)  \dy_i\biggr)
    ,\label{eq:B3FirstDiscussion-1-revisit}
\end{align}
which
is equal to $D_{k,j}^d(x)$ within an additive error of $\Ord(\eps)$, so we approximate \eqref{eq:B3FirstDiscussion-1-revisit}.
Here $ a^x_{i} =  [\frac{x_{i}}{m_t} - \frac{\sigma_t\ConstDifBoundP}{m_t}\sqrt{\log \epsK^{-1}}, \frac{x_{i}}{m_t} + \frac{\sigma_t\ConstDifBoundP}{m_t}\sqrt{\log \epsK^{-1}}]$.

We let $f_{i}(y_{i};j_{i},k,l'):= \mathbbm{1}[|y_i|\leq \ConstDifBoundO]\mathbbm{1}[0\leq 2^{k}y_i-j_i \leq l+1](2^{k}y_i-l'-j_i)_+^l \exp\left(-\frac{(x_i-m_t y_i)^2}{2\sigma_t^2}\right)  \dy_i$.
First, $\sum_{l'=0}^{l+1}\frac{(-1)^{l'}{}_{l+1}\mathrm{C}_{l'}}{l!}  f_{i}(y_{i};j_{i},k,l')$ is approximated by $\sum_{l'=0}^{l+1}\frac{(-1)^{l'}{}_{l+1}\mathrm{C}_{l'}}{l!}  \NetworkSplineA^{j_i-l',\jhigh_{l'},\jlow_{l'},k}(y_i,\sigma',m') $ (see \cref{Lemma:ParallelNetwork} for aggregation of the networks).
Here, $\jhigh_{l'}$ and $\jlow_{l'}$ are defined so that $\mathbbm{1}[\jlow_{l'} \leq 2^k y \leq \jhigh_{l'}] = \mathbbm{1}[|y_i|\leq \ConstDifBoundO]\mathbbm{1}[0\leq 2^{k}y_i-j_i \leq l+1]$ holds.

Now we multiply $\sum_{l'=0}^{l+1}\frac{(-1)^{l'}{}_{l+1}\mathrm{C}_{l'}}{l!}  \NetworkSplineA^{j_i,\jhigh_{l'},\jlow_{l'},k}(y_i,\sigma',m') $ over $i=1,2,\cdots,d$ using $\NetworkMultiB$ to obtain the desired network $\NetworkSplineC^{k,j}$. According to \cref{Lemma:DiffusionBasis1} with $\epsA = \epsK$ and \cref{Lemma:BaseNN02} with $\epsB = \epsK$ and $\ConstMultA = \Ord(1)$ (because $\|\NetworkSplineA^{j_i,\jhigh_{l'},\jlow_{l'},k}\|_\infty = \Ord(1)$), there exists a neural network $\phi_1(x,m',\sigma')\in \Phi(L,W,S,B)$ with 
      \begin{align}
         L &= \Ord (\log^4 \epsA^{-1}+ \log^2 C+k),
        \\\|W\|_\infty &= \Ord (\log^6 \epsA^{-1} ),
      \\  S& = \Ord (\log^8 \epsA^{-1}+ \log^2 C+k), 
      \\  B &= \Ord(C^l2^{kl}) + \log^{\Ord(\log \epsA^{-1})}  \epsA^{-1}
    \end{align}
    and we can bound the approximation error between $\phi_1(x,m',\sigma')$ and \eqref{eq:B3FirstDiscussion-1-revisit} with
    \begin{align}
     \tilde{\Ord}(\epsA) +
     \epsSensitivity C^{4l} 2^{k(4l+1)}\log^{\Ord(\log \epsA^{-1})} \epsA^{-1} 
        .\label{eq:FinalErrorbefore-1}
    \end{align}
    
    Now, we consider $ \NetworkSplineC=\phi_1(x,\NetworkMA(t),\NetworkSigmaA(t))$.
    We apply \cref{Lemma:MandSigma} with $ \eps = C^{-4l} 2^{-k(4l+1)}\log^{-\Ord(\log \epsA^{-1})} \epsA^{-1}$, so that $     \epsSensitivity$ gets small enough and \eqref{eq:FinalErrorbefore-1} is bounded by $   \tilde{\Ord}(\epsA)$.
    Then, the size of $\NetworkSplineC$ is bounded by
   \begin{align}
         L &= \Ord (\log^4 \epsA^{-1}+ \log^2 C+k^2),
        \\\|W\|_\infty &= \Ord (\log^6 \epsA^{-1} + \log^3 C+k^3),
      \\  S& = \Ord (\log^8 \epsA^{-1}+ \log^4 C+k^4), 
      \\  B &= \exp\left(\log^4 \epsA^{-1}+\log C + k\right) .
    \end{align}
    Now, adjusting $\eps$ to replace $\tilde{\Ord}(\epsA)$ by $\eps$ yields the first assertion.

   We can make $\|\NetworkSplineC^{k,j}\|_\infty$ hold, because $\int_{\R^d} \frac{1}{\sigma_t^{d}(2\pi)^\frac{d}{2}}\mathbbm{1}[\|y\|_\infty\leq \ConstDifBoundO]M_{k,j}^d(y)\exp\left(-\frac{\|x-m_t y\|^2}{2\sigma_t^2}\right) \dy=\Ord(1)$.
\end{proof}

\subsection{Approximation error bound: based on $p_0$}

Now we put it all together and derive \cref{theorem:Approximation}.
Throughout this and the next subsections, we take $N\gg 1$, $T_1=\underline{T}=\mathrm{poly}(N^{-1})$ and $T_5=\overline{T} = \Ord(\log N)$.
Moreover, we let $T_2=N^{-(2-\delta)/d}$, $T_3 = 2T_2$, $T_4 = 3T_2$.
This subsection considers the approximation for $t\in [T_1, T_4]$.

We begin with the following lemma, which gives the basis decompositon of the Besov functions.
\begin{lemma}[Basis decomposition]\label{Lemma:BesovConstruction}
    Under $N \gg 1$, Assumptions \ref{assumption:InBesov}, \ref{assumption:SmoothBeta}, \ref{assumption:BoundarySmoothness} with $a_0=N^{-(1-\delta)/d}$,
    there exists $f_N$ that satisfies
    \begin{align}
      &  \|p_0-f_N\|_{L^2([-1,1]^d)} \lesssim N^{-s/d}
       , \\ & \|p_0-f_N\|_{L^2([-1,1]^d \setminus [-1+N^{-(1-\delta)/d}, 1-N^{-(1-\delta)/d}]^d)} \lesssim N^{-(3s+2)/d},
    \end{align}
    and $f_N(x)=0$ for all $x$ with $\|x\|_\infty\geq 1$, and has the following form:
    \begin{align}
        f_N(x) = \sum_{i=1}^{N} \alpha_{i} \mathbbm{1}[\|x\|_\infty \leq 1] M_{k,j_i}^d(x)+\sum_{i=N+1}^{3N} \alpha_{i} \mathbbm{1}[\|x\|_\infty \leq 1-N^{-(1-\delta)/d}  ] M_{k,j_i}^d(x)
        ,\label{eq:BesovConstruction}
    \end{align}
    where $-2^{(k)_m}-l \leq (j_i)_m \leq 2^{(k)_m}\ (i=1,2,\cdots,N,\ m=1,2,\cdots,d)$, $|k| \leq K^* = (\Ord(1)+\log N) \nu^{-1} + \Ord(d^{-1}\log N)$ for $\delta = d(1/p - 1/r)_+$ and $\nu = (2s-\delta)/(2\delta)$.
    Moreover, $|\alpha_{i}| \lesssim N^{(\nu^{-1} + d^{-1})(d/p - s)_+}$.
\end{lemma}
\begin{proof}
    Because $p_0 \in \mathcal{C}^{3s+2}([-1,1]^d \setminus [-1+N^{-(1-\delta)/d}, 1-N^{-(1-\delta)/d}]^d)$, according to \cref{Lemma:SuzukiBesov}, we have $f_1$ such that
    \begin{align}
    \|p_0-f_1\|_{L^2([-1,1]^d \setminus [-1+N^{-(1-\delta)/d}, 1-N^{-(1-\delta)/d}]^d)} \lesssim N^{-(3s+2)/d}.
    \end{align}
    and has the following form:
    \begin{align}
        f_1(x) = \sum_{i=1}^N \alpha_{i} M_{k,j_i}^d(x),
    \end{align}
        where $-2^{(k)_m}-l \leq (j_i)_m \leq 2^{(k)_m}\ (i=1,2,\cdots,N,\ m=1,2,\cdots,d)$, $|k| \leq K^* = (\Ord(1)+\log N) \nu^{-1} + \Ord(d^{-1}\log N)$ for $\delta = d(1/p - 1/r)_+$ and $\nu = (2s-\delta)/(2\delta)$.
    Moreover, $|\alpha_{1,i}| \lesssim N^{(\nu^{-1} + d^{-1})(d/p - 2s)_+}$.

    Next let us approximate $f$ in $[-1,1]^d$. Because $\|p_0\|_{B^s_{p,q}} \lesssim 1$, we have $f_2$ such that
     \begin{align}
    \|p_0-f_2\|_{L^2([-1,1]^d )} \lesssim N^{-s/d}.
    \end{align}
    and has the following form:
    \begin{align}
        f_2(x) = \sum_{i=N+1}^{2N} \alpha_{i} M_{k,j_i}^d(x),
    \end{align}
        where $-2^{(k)_j}-l \leq (j_i)_j \leq 2^{(k)_j}\ (i=1,2,\cdots,N,\ j=1,2,\cdots,d)$, $|k| \leq K^* = (\Ord(1)+\log N) \nu^{-1} + \Ord(d^{-1}\log N)$ for $\delta = d(1/p - 1/r)_+$ and $\nu = (s-\delta)/(2\delta)$.
    Moreover, $|\alpha_{2,i}| \lesssim N^{(\nu^{-1} + d^{-1})(d/p - s)}$.

    Therefore, 
    \begin{align}
      & \mathbbm{1}[\|x\|_\infty \leq 1] f_1(x) -\mathbbm{1}[\|x\|_\infty \leq 1-N^{-(1-\delta)/d}  ] f_1(x)+\mathbbm{1}[\|x\|_\infty \leq 1-N^{-(1-\delta)/d}  ] f_2 (x) \\ & = \sum_{i=1}^N \alpha_{i} M_{k_i,j_i}^d(x) -
      \sum_{i=1}^{N} \alpha_i\mathbbm{1}[\|x\|_\infty \leq 1-N^{-(1-\delta)/d}  ] M_{k_i,j_i}^d(x)
      + \sum_{i=N+1}^{2N}\alpha_i \mathbbm{1}[\|x\|_\infty \leq 1-N^{-(1-\delta)/d}  ] M_{k_i,j_i}^d(x)
    \end{align}
    holds and reindexing the bases gives the result.
\end{proof}

The following lemma gives neural network that approximates $\nabla \log p_t(x)$ in $[T_1, T_4]$.

\begin{lemma}[Approximation of score function for $T_1\leq t\leq T_4$]\label{Lemma:ScoreFunc-1}
    There exists a neural network $\NetworkScoreA\in \Phi(L,W,S,B)$ that satisfies
    \begin{align}\label{eq;Appendix-Approx-1-goal}
      \int p_t(x)  \|\NetworkScoreA(x,t) - \nabla \log p_t(x)\|^2 \dx\dt \lesssim \frac{ N^{-2s/d} \log N }{\sigma_t^2}
    \end{align}
    Here, $L, \|W\|_\infty, S, B$ is evaluated as
    \begin{align}
        L = \Ord (\log^4 N),
      \quad  \| W\|_\infty = \Ord (N\log^6 N),
    \quad    S = \Ord (N\log^8N ), 
    \quad  \text{and } B = \exp(\Ord(\log^4 N )).
    \end{align}
\end{lemma}
\begin{proof}
Before we proceed to the main part of the proof, we limit the discussion into the bounded region. 
According to \cref{Lemma:Decay1}, we have that
 \begin{align}
   \int_{\|x\|_\infty \geq m_t + \Ord(1)\sigma_t\sqrt{\log N}} p_t(x)\|s(x,t) - \nabla \log p_t(x)\|^2 \dx
  \lesssim \frac{\underline{T}}{N^{(2s+1)/d}}\left(1+\|s(\cdot,t)\|_\infty^2\right)
  \label{eq:B4FirstDiscussion-1},
    \end{align}
    with a sifficiently large hidden constant in $\Ord(1)$.
    Because $\|\nabla \log p_t(x)\|$ is bounded with $\frac{\log^\frac12 N}{\sigma_t}$ in $\|x\|_\infty \geq m_t + \Ord(1)\sigma_t\sqrt{\log N}$ due to \cref{Lemma:Smooth1}, $s$ can be taken so that $\|s(\cdot,t)\|_\infty \lesssim \frac{\log^\frac12 N}{\sigma_t}$ and therefore \eqref{eq:B4FirstDiscussion-1} is bounded by $\frac{\underline{T}}{N^{(2s+1)}} \cdot \frac{\log N}{\underline{T}} = N^{-(2s+1)/d}\log N $, which is smaller than the upper bound of \eqref{eq;Appendix-Approx-1-goal}.
Thus, we can focus on the approximation of the score $\nabla \log p_t(x)$ within $\|x\|_\infty \leq m_t + \Ord(1)\sigma_t\sqrt{\log N}=\Ord(1)$.
    Moreover, we can also exclude the case where $p_t(x)\leq N^{-(2s+1)/d}$, because \cref{Lemma:Decay1} can bound the error
    \begin{align}
        \int_{\|x\|_\infty \leq m_t + \Ord(1)\sigma_t\sqrt{\log N}} p_t(x) \mathbbm{1}[p_t(x) \leq \eps]\| s(x,t) -\nabla \log p_t(x)\|^2 \dx& \lesssim \frac{\eps}{\sigma_t^2} \log^\frac{d+2}{2} (\eps^{-1}\underline{T}^{-1}) +\eps \| s(x,t)\|
        \\ & \lesssim \frac{\eps}{\sigma_t^2}\log^\frac{d+2}{2} (\eps^{-1}\underline{T}^{-1}) + \frac{\eps}{\sigma_t^2}\log N,
        \label{eq:B4FirstDiscussion-11}
    \end{align}
    and setting $\eps = N^{-(2s+1)/d}$ makes \eqref{eq:B4FirstDiscussion-11} smaller than the bound \eqref{eq;Appendix-Approx-1-goal}.

    Thus, in the following, we consider $x$ such that $\|x\|_\infty \leq m_t + \Ord(1)\sigma_t\sqrt{\log N}=\Ord(1)$ and $p_t(x)\geq N^{-(2s+1)/d}$ holds. In this case, we have $\|\nabla \log p_t(x)\|\lesssim \frac{\log^\frac12 N}{\sigma_t}$.

    The construction is straightforward. Based on \eqref{eq:BesovConstruction} of \cref{Lemma:BesovConstruction}, 
    we let
    \begin{align}
      p_t(x)&=   \int \frac{1}{\sigma_t^{d}(2\pi)^\frac{d}{2}}p_0(y)\exp\left(-\frac{\|x-m_t y\|^2}{2\sigma_t^2}\right) \dy
         \fallingdotseq
          \int \frac{1}{\sigma_t^{d}(2\pi)^\frac{d}{2}}f_N(y)\exp\left(-\frac{\|x-m_t y\|^2}{2\sigma_t^2}\right) \dy
       \\&  = \sum_{i=1}^N \alpha_i E_{k_i,j_i}^{(1)}(x,t) =:\tilde{f}_1(x,t),
       \\  f_1(x,t)&:=\tilde{f}_1(x,t)\lor N^{-(2s+1)/d},
    \end{align}
    and
     \begin{align}
      \sigma_t\nabla p_t(x)&= \int \frac{x-m_ty}{\sigma_t^{d+1}(2\pi)^\frac{d}{2}}p_0(y)\exp\left(-\frac{\|x-m_t y\|^2}{2\sigma_t^2}\right) \dy
         \fallingdotseq
          \int \frac{x-m_ty}{\sigma_t^{d+1}(2\pi)^\frac{d}{2}}f_N(y)\exp\left(-\frac{\|x-m_t y\|^2}{2\sigma_t^2}\right) \dy
       \\ & = \sum_{i=1}^N \alpha_i E_{k_i,j_i}^{(2)}(x,t)=:f_2(x,t),
       \\ f_3(x,t)&:=\frac{f_2(x,t)}{f_1(x,t)}\mathbbm{1}\left[\left\|\frac{f_2(x,t)}{f_1(x,t)}\right\|\lesssim \frac{\log^\frac12 N}{\sigma_t}\right]
    \end{align} 
    so that $\alpha_i$, $E_{k_i,j_i}^{(1)}(x,t)$ and $E_{k_i,j_i}^{(2)}(x,t)$ correspond to the basis decomposition in \cref{Lemma:BesovConstruction}.
    Thus, $|\alpha_{i}| \lesssim N^{(\nu^{-1} + d^{-1})(d/p - s)_+}$ and $|k_i| = \Ord(\log N)$.
    We remark that $C_{\mathrm{b},1}$ is set to be $1$ or $1-N^{-(1-\delta)/d}$ in \eqref{eq:Diffused-B-splineBasis-1} and \eqref{eq:Diffused-B-splineBasis-2}.
    We approximate $E_{k,j_i}^{(1)}$ and $E_{k,j_i}^{(2)}$ by $\NetworkSplineC^{k_i,j_i}$ and $\NetworkSplineD^{k_i,j_i}$ in \cref{Lemma:DiffusionBasis2}, by setting $\eps=\eps_1$ and $C=m_t + \Ord(1)\sigma_t\sqrt{\log N}=\Ord(1)$ (because $\sigma_t \leq \sigma_{T_2}\lesssim \log^{-\frac{1}{2}} N$), where $\eps_1=\mathrm{poly}(N^{-1})$ is a scaler value adjusted below.
    Then we sum up these sub-networks using \cref{Lemma:ParallelNetwork} and obtain  neural networks $\NetworkBesovA(x,t)$ and $\NetworkBesovB(x,t)$ that approximate $f_1(x,t)$ and $f_2(x,t)$, respectively.

 Because we can decompose the error as
    \begin{align}
   &     \int_{\|x\|_\infty \leq m_t + \Ord(1)\sigma_t\sqrt{\log N}} p_t(x)\mathbbm{1}[p_t(x)\geq N^{-\frac{2s+1}{d}}] \|s(x,t) - \nabla \log p_t(x)\|^2 \dx \label{eq:Appendix-Approx-3-target}
          \\ &   \lesssim
           \int_{\|x\|_\infty \leq m_t + \Ord(1)\sigma_t\sqrt{\log N}}\mathbbm{1}[p_t(x)\geq N^{-\frac{2s+1}{d}}] p_t(x) \left\| \NetworkScoreA(x,t) -\frac{ f_3(x,t)}{\sigma_t}\right\|^2\dx \label{eq:Appendix-Approx-3-target-1}
      \\ & \quad      +
           \int_{\|x\|_\infty \leq m_t + \Ord(1)\sigma_t\sqrt{\log N}}\mathbbm{1}[p_t(x)\geq N^{-\frac{2s+1}{d}}] p_t(x) \left\| \frac{ f_3(x,t)}{\sigma_t }-\nabla \log p_t(x)\right\|^2\dx, \label{eq:Appendix-Approx-3-target-2}
    \end{align}
       we consider the approximation of $\frac{ f_3(x,t)}{\sigma_t}$ for the moment, instead of $\nabla\log p_t(x)=\frac{\nabla p_t(x,t)}{f_1(x,t)}$, and bound \eqref{eq:Appendix-Approx-3-target-1}.
 From the construction of the networks, we have the following bounds:
 \begin{align}
     |f_1(x,t) - \NetworkBesovA(x,t)|& \lesssim N \cdot \max|\alpha_i|\cdot \eps_1,\label{eq:Appendix-Approx-3-4}
  \\   \|f_2(x,t) - \NetworkBesovB(x,t)\|& \lesssim N \cdot \max|\alpha_i|\cdot \eps_1.\label{eq:Appendix-Approx-3-5}
 \end{align}
 for all $x$ with $\|x\|_\infty \leq m_t + \Ord(1)\sigma_t\sqrt{\log N}=\Ord(1)$.
    Note that $\max|\alpha_i|$ is bounded by $N^{(\nu^{-1} + d^{-1})(d/p - s)_+}$. Thus, we take $\eps_1 \lesssim N^{-1} \cdot N^{-(\nu^{-1} + d^{-1})(d/p - s)_+}\cdot N^{-\frac{9s+3}{d}}$ so that \eqref{eq:Appendix-Approx-3-4} and \eqref{eq:Appendix-Approx-3-5} are bounded by $N^{-\frac{9s+3}{d}}$ in \cref{Lemma:BaseNN02}.

        Then we define $\NetworkBesovC$ as
    \begin{align}
        [\NetworkBesovC(x,t)]_i:= \NetworkClipA(\NetworkMultiB(\NetworkInvA(\NetworkClipA(\NetworkBesovA(x,t);N^{-(2s+1)/d},\Ord(1)))), [\NetworkBesovB(x,t)]_i);-\Ord(\log^\frac12 N),\Ord(\log^\frac12 N)).
    \end{align}
    to approximate $\sigma_t \nabla \log p_t(x)$.
    Here we used the boundedness of $p_t(x)$ with $[N^{-(2s+1)/d},\Ord(1)]$ to clip $\NetworkBesovA(x,t)$ and the boundedness of $\sigma_t\nabla\log p_t(x)$ with $[-\Ord(\log^\frac12 N), \Ord(\log^\frac12 N)]$ to clip the whole output.
    For $\NetworkInvA$ we let $\eps=N^{-(3s+1)/d}$ in \cref{Lemma:ApproxInv} and for $\NetworkMultiB$ we let $\eps=N^{-s/d}$ and $C=N^{(2s+1)/d}$.
    Then,
    \begin{align}
     & \|\NetworkBesovC(x,t)-f_3(x,t)\|= \left\| \NetworkBesovC(x,t) - \frac{f_2(x,t)}{f_1(x,t)}\mathbbm{1}\left[\left\|\frac{f_2(x,t)}{f_1(x,t)}\right\|\lesssim \frac{\log^\frac12 N}{\sigma_t}\right]\right\|
     \\ &  \lesssim N^{-s/d} + N^{(2s+1)/d} \cdot (N^{-(3s+1)/d}+N^{2(3s+1)/d}|f_1(x,t) - \NetworkBesovA(x,t)|+\|f_2(x,t) - \NetworkBesovB(x,t)\|)
    \\ & \label{eq:Appendix-Approx-3-6} \lesssim N^{-s/d} + N^{(8s+3)/d}|f_1(x,t) - \NetworkBesovA(x,t)| + N^{(2s+1)/d}\|f_2(x,t) - \NetworkBesovB(x,t)\|.
    \end{align}
    Applying \eqref{eq:Appendix-Approx-3-4}$\leq N^{-\frac{9s+3}{d}}$ and \eqref{eq:Appendix-Approx-3-5}$\leq N^{-\frac{9s+3}{d}}$ yields that \eqref{eq:Appendix-Approx-3-6}$\leq N^{-\frac{s}{d}}$.
    
    Finally, we let
  \begin{align}
        \NetworkScoreA(x,t) := \NetworkMultiB(\NetworkBesovC(x,t), \NetworkSigmaA(t)).
    \end{align}
    By setting $\eps=N^{-s/d}$ and $C\simeq \max\{\log^\frac12 N, \sigma_{\underline{T}}\}\lesssim \mathrm{poly}(N)$ in \cref{Lemma:BaseNN02} for $\NetworkMultiB$ and $\eps=N^{-s/d}/\mathrm{poly}(N)$ in \cref{Lemma:MandSigma} for $\NetworkSigmaA$.
    Then,
    \begin{align}
        \left\| \NetworkScoreA(x,t) -\frac{ f_3(x,t)}{\sigma_t}\right\|
        \lesssim N^{-s/d} +\mathrm{poly}(N) \cdot N^{-s/d}/\mathrm{poly}(N) \lesssim N^{-s/d},
    \end{align}
    which yields
    \begin{align}
        \eqref{eq:Appendix-Approx-3-target-1}=\int_{\|x\|_\infty \leq m_t + \Ord(1)\sigma_t\sqrt{\log N}}\mathbbm{1}[p_t(x)\geq N^{-\frac{2s+1}{d}}] p_t(x) \left\| \NetworkScoreA(x,t) -\frac{ f_3(x,t)}{\sigma_t}\right\|^2\dx
        \lesssim N^{-2s/d}.
    \end{align}
    The structure of $\NetworkBesovC$ and $ \NetworkScoreA$ are evaluated as
       \begin{align}
         L = \Ord (\log^4 N),
       \ \|W\|_\infty = \Ord (N\log^6N),
      \  S = \Ord (N\log^8 N), \text{ and }
      \  B = \exp\left(\log^4 N\right) .
    \end{align}
    Here we used $|k_i| = \Ord(\log N)$ and $C=\Ord(1)$.
    

    We move to the error analysis between $\frac{ f_3(x,t)}{\sigma_t}$ and $\nabla\log p_t(x)$ to bound \eqref{eq:Appendix-Approx-3-target-2}.
    Remind that we consider $x$ such that $\|x\|_\infty \leq m_t + \Ord(1)\sigma_t\sqrt{\log N}=\Ord(1)$ and $p_t(x)\geq N^{-(2s+1)/d}$ holds. In this case, we have $\|\nabla \log p_t(x)\|\lesssim \frac{\log^\frac12 N}{\sigma_t}$.
    First, we consider the case $x\in [-m_t,m_t]^d$.
    Since $p_t(x)$ is lower bounded by $C_{\mathrm{a}}^{-1}$ according to \cref{Lemma:LowerandUpperBounds}, as long as $|f_1(x,t) - p_t(x)| \leq C_{\mathrm{a}}^{-1}/2$, we can say that the approximation error is bounded by $\lesssim |f_1(x,t) - p_t(x)|+\|f_2(x,t) - \sigma_t \nabla p_t(x)\|$.
    On the other hand, if $|f_1(x,t) - p_t(x)| \geq C_{\mathrm{a}}^{-1}/2$, we no longer have such bound, but this time we can use the fact that $\frac{f_2(x,t)}{f_1(x,t)}$ and $\sigma_t\frac{\sigma_t\nabla p_t(x)}{p_t(x)}$ is bounded by $\log^\frac12 N$.
    Therefore, when $x\in [-m_t,m_t]^d$, we can bound the approximation error as
    \begin{align}
       \left\|  f_3(x,t)-\sigma_t\frac{\nabla p_t(x)}{p_t(x)}\right\|
    \leq 
      \left\|  \frac{f_2(x,t)}{f_1(x,t)}-\sigma_t\frac{\nabla p_t(x)}{p_t(x)}\right\|
      \lesssim \log^\frac12 N (|f_1(x,t)-p_t(x)|+\|f_2(x,t)-\sigma_t \nabla p_t(x)\| ).
    \end{align}
    Next, we consider the case when $x\in [- m_t - \Ord(1)\sigma_t\sqrt{\log N}, m_t + \Ord(1)\sigma_t\sqrt{\log N}]^d \setminus [-m_t,m_t]^d$.
    Then,  we have that
    \begin{align} \left\| f_3(x,t)-\sigma_t\frac{\nabla p_t(x)}{p_t(x)}\right\|
    \leq 
         \left\|  \frac{ f_2(x,t)}{ f_1(x,t)}-\sigma_t\frac{\nabla p_t(x)}{p_t(x)}\right\|
         & \lesssim
         \frac{\| f_2(x,t)- \sigma_t\nabla p_t(x)\|}{f_1(x,t)}
         +
         \|\sigma_t\nabla p_t(x)\|\left|\frac{1}{ f_1(x,t)}-\frac{1}{p_t(x)}\right|.
         \label{eq:Appendix-Approx-2-1}
    \end{align}
    The first term is bounded by $N^{(2s+1)/d}\|f_2(x,t)(x,t) - \sigma_t\nabla p_t(x)\|$ because we focus on the case $p_t(x)\geq N^{-(2s+1)/d}$.
    For the second term, because $\|\nabla \log p_t(x)\|=\left\|\sigma_t\frac{\nabla p_t(x)}{p_t(x)}\right\|\lesssim \frac{\log^\frac12}{\sigma_t}$, we have $\|\sigma_t\nabla p_t(x)\|\lesssim p_t(x)\log^\frac12 N$.
    By using this, we can bound the second term as
    \begin{align}
        \|\sigma_t\nabla p_t(x)\|\left|\frac{1}{f_1(x,t)}-\frac{1}{p_t(x)}\right|
       & \lesssim \log^\frac12 N p_t(x) \left|\frac{1}{f_1(x,t)}-\frac{1}{p_t(x)}\right|
       \\ & \lesssim \log^\frac12 N  \frac{\left|p_t(x)-f_1(x,t)\right|}{f_1(x,t)}
       \\ & \lesssim N^{\frac{2s+1}{d}}\log^\frac12 N \left|p_t(x)-f_1(x,t)\right|,
    \end{align}
    where we used $f_1(x,t)\geq N^{-(2s+1)/d}$.
    Thus, for $x\in [- m_t - \Ord(1)\sigma_t\sqrt{\log N}, m_t + \Ord(1)\sigma_t\sqrt{\log N}]^d \setminus [-m_t,m_t]^d$ and $p_t(x)\geq N^{-\frac{2s+1}{d}}$, \eqref{eq:Appendix-Approx-2-1} is bounded by 
     \begin{align}
      \left\|  \NetworkBesovC(x,t)-\frac{\sigma_t\nabla p_t(x)}{p_t(x)}\right\|
      \lesssim N^{\frac{2s+1}{d}}\log^\frac12 N 
      (|\NetworkBesovA(x,t)-p_t(x)|+\|\NetworkBesovB(x,t)-\sigma_t \nabla p_t(x)\| ).
    \end{align}
    Therefore, we have that
    \begin{align}\label{eq:Appenxid-Approx-2-4}
  &  \left\| \frac{f_2(x,t)}{\sigma_tf_1(x,t)}-\frac{\nabla p_t(x)}{p_t(x)}\right\|
   \\ & 
        \lesssim    \begin{cases}
      \log^\frac12 N (|f_1(x,t)-p_t(x)|+\|f_2(x,t)-\sigma_t \nabla p_t(x)\|)/\sigma_t\quad (\|x\|_\infty \leq m_t)
     \\ 
  N^{\frac{2s+1}{d}}\log^\frac12 N (|f_1(x,t)-p_t(x)|+\|f_2(x,t)-\sigma_t \nabla p_t(x)\|)/\sigma_t \quad  \\ \quad \quad\quad\quad\quad\quad\quad\quad\quad\quad\quad\quad\quad\quad\quad\quad\quad\quad\quad  (x\in [- m_t - \Ord(1)\sigma_t\sqrt{\log N}, m_t + \Ord(1)\sigma_t\sqrt{\log N}]^d \setminus [-m_t,m_t]^d).
        \end{cases} \label{eq:Appenxid-Approx-2-7}
    \end{align}
    
    We consider the $L^2(p_t)$ loss of \eqref{eq:Appenxid-Approx-2-7}.
    First, we consider the case of $\|x\|_\infty \leq m_t$.
    \begin{align}
      &   \int_{\|x\|_\infty\leq m_t } p_t(x) \left\| \frac{f_2(x,t)}{\sigma_tf_1(x,t)}-\frac{\nabla p_t(x)}{p_t(x)}\right\|^2\dx
    \\  & \lesssim  \int_{\|x\|_\infty\leq m_t}(|f_1(x,t)-p_t(x)|^2+\|f_2(x,t)-\sigma_t \nabla p_t(x)\|^2)\log N/\sigma_t^2\dx \ (\text{we used\eqref{eq:Appenxid-Approx-2-7} and $p_t(x)=\Ord(1)$ by \cref{Lemma:LowerandUpperBounds}.})
   \\   &  \lesssim 
    \int_{\|x\|_\infty\leq m_t } \left(\left| \int \frac{1}{\sigma_t^{d}(2\pi)^\frac{d}{2}}p_0(y)\exp\left(-\frac{\|x-m_t y\|^2}{2\sigma_t^2}\right) \dy- \int \frac{1}{\sigma_t^{d}(2\pi)^\frac{d}{2}}f_N(y)\exp\left(-\frac{\|x-m_t y\|^2}{2\sigma_t^2}\right) \dy\right|^2\right.
    \\ & \quad \left.+\left\| \int \frac{x-m_ty}{\sigma_t^{d+1}(2\pi)^\frac{d}{2}}p_0(y)\exp\left(-\frac{\|x-m_t y\|^2}{2\sigma_t^2}\right) \dy- \int \frac{x-m_ty}{\sigma_t^{d+1}(2\pi)^\frac{d}{2}}p_0(y)\exp\left(-\frac{\|x-m_t y\|^2}{2\sigma_t^2}\right) \dy\right\|^2\right)\log N/\sigma_t^2\dx 
    \\ & \lesssim \log N/\sigma_t^2 \cdot \int_{\|x\|_\infty\leq m_t } \int \frac{1}{\sigma_t^{d}(2\pi)^\frac{d}{2}}\exp\left(-\frac{\|x-m_t y\|^2}{2\sigma_t^2}\right)|p_0(y)-f_N(y)|^2\dy\dx 
      \end{align}
      \begin{align} & \quad + \log N/\sigma_t^2 \cdot \int_{\|x\|_\infty\leq m_t } \int\frac{|x-m_ty|}{\sigma_t^{d+1}(2\pi)^\frac{d}{2}}\exp\left(-\frac{\|x-m_t y\|^2}{2\sigma_t^2}\right)|p_0(y)-f_N(y)|^2\dy\dx
  \\  & = \log N/\sigma_t^2 \cdot \int \int_{\|x\|_\infty\leq m_t } \frac{1}{\sigma_t^{d}(2\pi)^\frac{d}{2}}\exp\left(-\frac{\|x-m_t y\|^2}{2\sigma_t^2}\right)|p_0(y)-f_N(y)|^2\dx\dy
   \\ & \quad + \log N/\sigma_t^2 \cdot \int\int_{\|x\|_\infty\leq m_t }  \frac{|x-m_ty|}{\sigma_t^{d+1}(2\pi)^\frac{d}{2}}\exp\left(-\frac{\|x-m_t y\|^2}{2\sigma_t^2}\right)|p_0(y)-f_N(y)|^2\dx\dy
   \\ & \lesssim  \log N/\sigma_t^2 \cdot \int |p_0(y)-f_N(y)|^2\dy + \log N/\sigma_t^2 \cdot \int |p_0(y)-f_N(y)|^2\dy
   \lesssim \log N/\sigma_t^2 \cdot N^{-2s/d}.\label{eq:Appenxid-Approx-2-9}
    \end{align}
    For the third inequality, we used Jensen's inequality.
    For the second last inequality, we used the construction of $f_N$ and \cref{Lemma:BesovConstruction}.

    We then consider the case of $x\in [- m_t - \Ord(1)\sigma_t\sqrt{\log N}, m_t + \Ord(1)\sigma_t\sqrt{\log N}]^d \setminus [-m_t,m_t]^d$.
    Most of the part is the same as previously.
   {\small
    \begin{align} 
         &   \int_{m_t\leq \|x\|_\infty\leq m_t + \Ord(1)\sigma_t\sqrt{\log N}} p_t(x) \mathbbm{1}[p_t(x)\geq N^{-\frac{2s+1}{d}}] \left\| \frac{f_2(x,t)}{\sigma_tf_1(x,t)}-\frac{\nabla p_t(x)}{p_t(x)}\right\|^2\dx
    \\  & \lesssim  \int_{m_t\leq\|x\|_\infty\leq m_t + \Ord(1)\sigma_t\sqrt{\log N}}(|f_1(x,t)-p_t(x)|^2+\|f_2(x,t)-\sigma_t \nabla p_t(x)\|^2)N^{\frac{4s+2}{d}} \log N/\sigma_t^2\dx
   \\   &  \lesssim 
\int_{m_t\leq\|x\|_\infty\leq m_t + \Ord(1)\sigma_t\sqrt{\log N}}\left(\left| \int \frac{1}{\sigma_t^{d}(2\pi)^\frac{d}{2}}p_0(y)\exp\left(-\frac{\|x-m_t y\|^2}{2\sigma_t^2}\right) \dy- \int \frac{1}{\sigma_t^{d}(2\pi)^\frac{d}{2}}f_N(y)\exp\left(-\frac{\|x-m_t y\|^2}{2\sigma_t^2}\right) \dy\right|^2\right.
    \\ & \quad \left.+\left\| \int \frac{x-m_ty}{\sigma_t^{d+1}(2\pi)^\frac{d}{2}}p_0(y)\exp\left(-\frac{\|x-m_t y\|^2}{2\sigma_t^2}\right) \dy- \int \frac{x-m_ty}{\sigma_t^{d+1}(2\pi)^\frac{d}{2}}p_0(y)\exp\left(-\frac{\|x-m_t y\|^2}{2\sigma_t^2}\right) \dy\right\|^2\right)N^{\frac{4s+2}{d}} \log N/\sigma_t^2\dx 
    \\ & \lesssim N^{\frac{4s+2}{d}} \log N/\sigma_t^2 \cdot  \int_{m_t\leq\|x\|_\infty\leq m_t + \Ord(1)\sigma_t\sqrt{\log N}} \int \frac{1}{\sigma_t^{d}(2\pi)^\frac{d}{2}}\exp\left(-\frac{\|x-m_t y\|^2}{2\sigma_t^2}\right)|p_0(y)-f_N(y)|^2\dy\dx 
   \\ & \quad + N^{\frac{4s+2}{d}} \log N/\sigma_t^2 \cdot  \int_{m_t\leq\|x\|_\infty\leq m_t + \Ord(1)\sigma_t\sqrt{\log N}} \int \frac{|x-m_ty|^2}{\sigma_t^{d+2}(2\pi)^\frac{d}{2}}\exp\left(-\frac{\|x-m_t y\|^2}{2\sigma_t^2}\right)|p_0(y)-f_N(y)|^2\dy\dx
  \end{align}
      \begin{align} 
   &\hspace{-5mm} \lesssim\left[ \int_{m_t\leq\|x\|_\infty\leq m_t + \Ord(1)\sigma_t\sqrt{\log N}}\left[\int_{\|\frac{x}{m_t}-y\|_\infty \leq \Ord(1)\sigma_t \sqrt{\log N}}\frac{1}{\sigma_t^{d}(2\pi)^\frac{d}{2}}\exp\left(-\frac{\|x-m_t y\|^2}{2\sigma_t^2}\right)|p_0(y)-f_N(y)|^2\dy+N^{-\frac{6s+2}{d}}\right]\dx\right. 
     \\  &\left.\hspace{-5mm}+\int_{m_t\leq\|x\|_\infty\leq m_t + \Ord(1)\sigma_t\sqrt{\log N}} \left[  \int_{\|\frac{x}{m_t}-y\|_\infty \leq \Ord(1)\sigma_t \sqrt{\log N}}\frac{|x-m_ty|^2}{\sigma_t^{d+2}(2\pi)^\frac{d}{2}}\exp\left(-\frac{\|x-m_t y\|^2}{2\sigma_t^2}\right)|p_0(y)-f_N(y)|^2\dy+N^{-\frac{6s+2}{d}}\right]\dx \right]
   \\ & \quad \cdot  N^{\frac{4s+2}{d}}  \log N/\sigma_t^2 \quad (\text{we used \cref{Lemma:ClipInt}.})
    \\ & \hspace{-8mm}\lesssim N^{\frac{4s+2}{d}}\log N/\sigma_t^2\cdot \left[ \int_{m_t\leq\|x\|_\infty\leq m_t + \Ord(1)\sigma_t\sqrt{\log N}}\left[\int_{\|\frac{x}{m_t}-y\|_\infty \leq \Ord(1)\sigma_t \sqrt{\log N}}\frac{1}{\sigma_t^{d}(2\pi)^\frac{d}{2}}\exp\left(-\frac{\|x-m_t y\|^2}{2\sigma_t^2}\right)|p_0(y)-f_N(y)|^2\dy\right]\dx\right. 
   \\   &\left.+\int_{m_t\leq\|x\|_\infty\leq m_t + \Ord(1)\sigma_t\sqrt{\log N}} \left[  \int_{\|\frac{x}{m_t}-y\|_\infty \leq \Ord(1)\sigma_t \sqrt{\log N}}\frac{\log N}{\sigma_t^{d}(2\pi)^\frac{d}{2}}\exp\left(-\frac{\|x-m_t y\|^2}{2\sigma_t^2}\right)|p_0(y)-f_N(y)|^2\dy\right]\dx +N^{-\frac{6s+2}{d}}\right]
   \\ & \lesssim N^{\frac{4s+2}{d}}\log^2 N/\sigma_t^2 \int_{m_t\leq\|x\|_\infty\leq m_t + \Ord(1)\sigma_t\sqrt{\log N}}\int_{\|\frac{x}{m_t}-y\|_\infty \leq \Ord(1)\sigma_t \sqrt{\log N}}\frac{1}{\sigma_t^{d}(2\pi)^\frac{d}{2}}\exp\left(-\frac{\|x-m_t y\|^2}{2\sigma_t^2}\right)|p_0(y)-f_N(y)|^2\dx\dy \\ & \quad +N^{-\frac{2s}{d}}\log N/\sigma_t^2\label{eq:Appenxid-Approx-2-8}
    \end{align}}
        For the third inequality, we used Jensen's inequality.
        Here, we note that if $(x,y)$ satisfies $m_t\leq\|x\|_\infty\leq m_t + \Ord(1)\sigma_t\sqrt{\log N}=\Ord(1)$ and $\|\frac{x}{m_t}-y\|_\infty \leq \Ord(1)\sigma_t \sqrt{\log N}$, then we have that $1-\Ord(1)\sigma_t \sqrt{\log N}\leq \|y\|_\infty \leq 1+\Ord(1)\frac{\sigma_t}{m_t}\sqrt{\log N}$ and that $1-\Ord(1)\sqrt{t}\leq \|y\|_\infty \leq 1+\Ord(1)\sqrt{t}$.
        Because we are considering the time $t\leq T_4 = 3N^{-\frac{2-\delta}/d}$, $\Ord(1)\sqrt{t}\lesssim N^{-\frac{1-\delta}{d}}$ holds for sufficiently large $N$.
        Therefore, \eqref{eq:Appenxid-Approx-2-8} is further bounded by
        \begin{align}
          &  \eqref{eq:Appenxid-Approx-2-8}\\ &\lesssim 
         N^{\frac{4s+2}{d}}\log^2 N/\sigma_t^2 \int_x\int_{1-N^{-\frac{1-\delta}{d}}\leq \|y\|_\infty \leq 1+N^{-\frac{1-\delta}/d}} \frac{1}{\sigma_t^{d}(2\pi)^\frac{d}{2}}\exp\left(-\frac{\|x-m_t y\|^2}{2\sigma_t^2}\right)|p_0(y)-f_N(y)|^2\dx\dy
       \\ & \quad     +N^{-\frac{2s}{d}}\log N/\sigma_t^2
           \\ & =
             N^{\frac{4s+2}{d}}   \log^2 N/\sigma_t^2 \int_{1-N^{-\frac{1-\delta}{d}}\leq \|y\|_\infty \leq 1+N^{-\frac{1-\delta}/d}}\int_x \frac{1}{\sigma_t^{d}(2\pi)^\frac{d}{2}}\exp\left(-\frac{\|x-m_t y\|^2}{2\sigma_t^2}\right)|p_0(y)-f_N(y)|^2\dy\dx
      \\ & \quad         +N^{-\frac{2s}{d}}\log N/\sigma_t^2
       \\   &  \lesssim N^{\frac{4s+2}{d}}\log^2 N/\sigma_t^2\cdot N^{-\frac{6s+4}{d}} +N^{-\frac{2s}{d}}\log N/\sigma_t^2 \lesssim N^{-\frac{2s}{d}}\log N/\sigma_t^2,\label{eq:Appenxid-Approx-2-10}
        \end{align}
        where we used the construction of $f_N$ and \cref{Lemma:BesovConstruction} for the second last inequality.
    Now we successfully bounded \eqref{eq:Appendix-Approx-3-target-2} and the conclusion follows.
    \end{proof}

\subsection{Approximation error bound: using the induced smoothness}
We then consider the approximation for $t \gtrsim T_2= N^{-(2-\delta)/d}$.
This can be proved by considering diffusion process starting at $t=t_*>0$.
We begin with the following lemma.
\begin{lemma}[Basis decomposition of $p_t$ at $t=t_*$]\label{Lemma:BesovConstruction-2}
    If $N,N'\gg 1$ and $N' \geq t_*^{-\frac{d}{2}}N^{\frac{\delta}{2}}$,
    there exists $f_{N'}$ such that
    \begin{align}
      &  \|p_{t_*}-f_{N'}\|_{L^2(\R^d)} \lesssim N^{-(3s+5)/d},
    \end{align}
    $f_{N'}(x)=0$ for $x$ with $\|x\|_\infty \gtrsim \Ord(\sqrt{\log N})$, and has the following form:
    \begin{align}
        f_N(x) = \sum_{i=1}^{N'} \mathbbm{1}[\|x\|_\infty \lesssim \Ord(\sqrt{\log N})] M_{k_i,j_i}^d(x),
    \end{align}
    where $-\sqrt{\log N}2^{(k_i)_m}-l\lesssim (j_i)_l\lesssim \sqrt{\log N} 2^{(k_i)_l}\ (i=1,2,\cdots,N,\ m=1,2,\cdots,d)$, $\|k_i\|_\infty \leq K=\Ord(d^{-1}\log N)$ and $|\alpha_{i}| \lesssim N^{\frac{(3s+6)(2-\delta)}{\delta }}$.
\end{lemma}

\begin{proof}
 Let $\alpha =\frac{2(3s+6)}{\delta} + 1$.
 According to \cref{Lemma:Smooth1}, for any $x$, we have
  \begin{align}
        \|\pd_{x_{i_1}}\pd_{x_{i_2}}\cdots \pd_{x_{i_k}} p_{T_2}(x)\|
        \leq 
        \frac{C_{\mathrm{a}}}{\sigma_{t_*}^k}.
    \end{align}
    Because all derivatives up to order $\alpha$ is bounded by $\sigma_{t_*}^{-\alpha}\lesssim t_*^{-\frac{\alpha}{2}} \lor 1$, $\frac{ p_{t_*}(x)}{t_*^{-\frac{\alpha}{2}}\lor a}$ belongs to $W^\alpha_\infty$ and its norm in $W^\alpha_\infty$ is bounded by a constant depending on $\alpha$, and hence to $B_{\infty, \infty}^{\alpha}$.
    Therefore, according to \cref{Lemma:SuzukiBesov}, there exists a basis decomposition with the order of the B-spline basis $l=\alpha +2$:
     \begin{align}
        f_{N'}(x) =( t_*^{-\frac{\alpha}{2}} \lor 1) \sum_{i=1}^N \alpha_{i} M_{k_i,j_i}^d (x).
    \end{align}
    such that 
    \begin{align}
      &  \|p_{t_*} - f_{N'}\|_{L^2([-\Ord(\sqrt{\log N}),\Ord(\sqrt{\log N})]^d)} \lesssim (\sqrt{\log N})^\alpha {N'}^{-\alpha/d} t_*^{-\frac{\alpha}{2}}
       \\ & =(\sqrt{\log N})^\alpha N^{\alpha \delta /2d} =(\sqrt{\log N})^\alpha N^{-(3s+6)/d} \lesssim N^{-(3s+5)/d},
    \end{align}
    where $-\sqrt{\log N} 2^{(k_i)_m}-l\lesssim (j_i)_l\lesssim \sqrt{\log N} 2^{(k_i)_l}\ (i=1,2,\cdots,N,\ m=1,2,\cdots,d)$, $\|k_i\|_\infty \leq K=\Ord(d^{-1}\log N)$, and $|\alpha_{i}| \lesssim 1$.
    Also, \cref{Lemma:Decay1} with $\eps = N^{-\frac{6s+10}{d}}$ and  $m_{t_*}+\Ord(1)\sigma_{t_*}\sqrt{\log N} \lesssim \sqrt{\log N}$ guarantees that $ \|p_{T_2} - f_N\|_{L^2(\R^d \subseteq [-\Ord(\sqrt{\log N}),\Ord(\sqrt{\log N})]^d)} \lesssim N^{-(3s+5)/d}$.
    Therefore, by resetting $\alpha_i \leftarrow ( t_*^{-\frac{\alpha}{2}} \lor 1) \alpha_i$, the assertion holds. ($\alpha_i$ is then bounded by $T_2^{-\frac{\alpha}{2}}$.)
\end{proof}
\cref{Lemma:BesovConstruction-2} gives a concrete construction of the neural network for $T_3 \leq t \leq T_5$.
\begin{lemma}[Approximation of score function for $T_3\leq t\leq T_5$; \cref{lemma:ApproximationSmoothArea}]\label{Lemma:ScoreFunc-2}
   Let $N\gg 1$ and $N'\geq t_*^{-d/2}N^{\delta/2}$.
    Suppose $t_*\geq N^{-(2-\delta)/d}$.
    Then there exists a neural network $\NetworkScoreB\in \Phi(L,W,S,B)$ that satisfies
    \begin{align}
    \int_x p_{t}(x)  \|\NetworkScoreB(x,t) - s(x,t)\|^2 \dx \lesssim \frac{N^{-\frac{2(s+1)}{d}}}{\sigma_t^2}
    \end{align}
    for $t\in [2t_*,\overline{T}]$.
    Specifically, $L = \Ord (\log^4 (N)),\| W\|_\infty = \Ord (N),S = \Ord (N')$, and $ B = \exp(\Ord(\log^4 N ))$.
     Moreover, we can take $\NetworkScoreB$ satisfying $\|\NetworkScoreB\|_\infty = \Ord(\sigma_t^{-1}\log^\frac12 N)$.
\end{lemma}
\begin{proof}
    The proof is essentially the same as that of \cref{Lemma:ScoreFunc-1}.
    Here, the slight differences are that (i) $p_t$, $\NetworkBesovD$, and $f_1$ are lower bounded by $N^{-(2s+3)/d}$, not by $N^{-(2s+1)/d}$, that (ii) $L^2(p_t)$ error should be bounded by $\frac{N^{-\frac{2(s+1)}{d}}}{\sigma_t^2}$, not by $\frac{N^{-\frac{2s}{d}}}{\sigma_t^2}$, and that (iii) $p_{t_*}$ is supported on $\R^d$, not on $[-1,1]^d$.
    Bounding the difference between 
   Observe that $t_*\geq T_1=N^{-\frac{2-\delta}{d}}$ holds, which is necessary to apply the argument of \cref{Lemma:ScoreFunc-1}.

    Let us reset the time $t\leftarrow t-t_*$ in the following proof and consider the diffusion process from $p_0$ (in the new definition), for simplicity.
    We have $t\geq t_*\gtrsim \mathrm{poly}(N^{-1})$ in the new definition.
According to \cref{Lemma:Decay1}, we have that
 \begin{align}
   \int_{\|x\|_\infty \geq m_t + \Ord(1)\sigma_t\sqrt{\log N}} p_t(x)\|s(x,t) - \nabla \log p_t(x)\|^2 \dx
  \lesssim \frac{t_*}{N^{(2s+2)/d}}\left(1+\|s(\cdot,t)\|_\infty^2\right)
,\label{eq:Appendix-Approx-5-1-1}
    \end{align}
    with a sifficiently large hidden constant in $\Ord(1)$.
    We limit the domain of $x$ into $\|x\|_\infty \leq m_t + \Ord(1)\sigma_t\sqrt{\log N}=\Ord(\sqrt{\log N})$.
    In this region, \cref{Lemma:Smooth1} yields $\|\nabla \log p_t(x)\|\lesssim \frac{\sqrt{\log N}}{\sigma_t}$, and therefore we can take $s$ such that $\|s(\cdot,t)\|_\infty\leq \frac{\sqrt{\log N}}{\sigma_t} \lesssim \frac{\sqrt{\log N}}{\sqrt{t_*} \land 1}$ holds.
    Then, \eqref{eq:Appendix-Approx-5-1-1} is bounded by $N^{-2(s+1)/d}$.
    Moreover, 
\begin{align}
        \int_{\|x\|_\infty \leq m_t + \Ord(1)\sigma_t\sqrt{\log N}} p_t(x) \mathbbm{1}[p_t(x) \leq N^{-(2s+3)/d}]\| s(x,t) -\nabla \log p_t(x)\|^2 \dx& \lesssim \frac{\eps}{\sigma_t^2} \log^\frac{d+2}{2} (N) +\eps \| s(x,t)\|
        \\ & \hspace{-80mm}\lesssim \left(\frac{N^{-(2s+3)/d}}{\sigma_t^2}\log^\frac{d+2}{2} (N) + \frac{N^{-(2s+3)/d}}{\sigma_t^2}\log N\right)\log^\frac{d}{2}N
        \lesssim N^{-2(s+1)/d}.
    \end{align}
    This means that we only need to consider $x$ with $p_t(x) \geq N^{-(2s+3)/d}$.

    Using the basis decomposition in the previous lemma, we let
      \begin{align}
      p_t(x)&=   \int \frac{1}{\sigma_t^{d}(2\pi)^\frac{d}{2}}p_0(y)\exp\left(-\frac{\|x-m_t y\|^2}{2\sigma_t^2}\right) \dy
         \fallingdotseq
          \int \frac{1}{\sigma_t^{d}(2\pi)^\frac{d}{2}}f_N(y)\exp\left(-\frac{\|x-m_t y\|^2}{2\sigma_t^2}\right) \dy
       \\&  = \sum_{i=1}^{N'} \alpha_i E_{k_i,j_i}^{(1)}(x,t) =:\tilde{f}_1(x,t),
       \\  f_1(x,t)&:=\tilde{f}_1(x,t)\lor N^{-(2s+3)/d},
    \end{align}
    and
     \begin{align}
      \sigma_t\nabla p_t(x)&= \int \frac{x-m_ty}{\sigma_t^{d+1}(2\pi)^\frac{d}{2}}p_0(y)\exp\left(-\frac{\|x-m_t y\|^2}{2\sigma_t^2}\right) \dy
         \fallingdotseq
          \int \frac{x-m_ty}{\sigma_t^{d+1}(2\pi)^\frac{d}{2}}f_N(y)\exp\left(-\frac{\|x-m_t y\|^2}{2\sigma_t^2}\right) \dy
       \\ & = \sum_{i=1}^{N'} \alpha_i E_{k_i,j_i}^{(2)}(x,t)=:f_2(x,t),
       \\ f_3(x,t)&:=\frac{f_2(x,t)}{f_1(x,t)}\mathbbm{1}\left[\left\|\frac{f_2(x,t)}{f_1(x,t)}\right\|\lesssim \frac{\log^\frac12 N}{\sigma_t}\right]
    \end{align} 
    (exactly the same definitions as that in \cref{Lemma:ScoreFunc-1}, except for $f_1(x,t):=\tilde{f}_1(x,t)\lor N^{-(2s+3)/d}$).
   Then we approximate each $ \alpha_i E_{k_i,j_i}^{(1)}(x,t)$ and $\alpha_i E_{k_i,j_i}^{(2)}(x,t)$ using \cref{Lemma:DiffusionBasis2} with $\eps \lesssim {N'}^{-1} \cdot N^{\frac{(3s+6)(2-\delta)}{\delta }}\cdot N^{-\frac{9s+10}{d}}$ and $C=m_t+\Ord(1)\sigma_t\sqrt{\log N}=\Ord(\sqrt{\log N})$ and aggregate them by \cref{Lemma:ParallelNetwork} to obtain $\NetworkBesovD(x,t)$ and $\NetworkBesovE(x,t)$, that approximate $f_1$ and $f_2$, respectively, and satisfy
  \begin{align}
      |f_1(x,t) - \NetworkBesovD(x,t)| \lesssim N^{-\frac{9s+3}{d}}, \quad \|f_2(x,t) - \NetworkBesovE(x,t)\| \lesssim N^{-\frac{9s+10}{d}}.
  \end{align}
  for all $x$ with $\|x\|_\infty = \Ord(\sqrt{\log N}).$
Now, we define $\NetworkBesovC$ as
    \begin{align}
        [\NetworkBesovF(x,t)]_i:= \NetworkClipA(\NetworkMultiB(\NetworkInvA(\NetworkClipA(\NetworkBesovD(x,t);N^{-(2s+3)/d},\Ord(1)))), [\NetworkBesovE(x,t)]_i);-\Ord(\log^\frac12 N),\Ord(\log^\frac12 N)),
    \end{align}
    where we let $\eps=N^{-(3s+4)/d}$ in \cref{Lemma:ApproxInv} for $\NetworkInvA$ and we let $\eps=N^{-(s+1)/d}$ and $C=N^{(2s+3)/d}$ for $\NetworkMultiB$ in \cref{Lemma:BaseNN02}.
      Finally, we let
  \begin{align}
        \NetworkScoreB(x,t) := \NetworkMultiB(\NetworkBesovF(x,t), \NetworkSigmaA(t)).
    \end{align}
    where $\eps=N^{-(s+1)/d}$ and $C\simeq \max\{\log^\frac12 N, \sigma_{\underline{T}}\}\lesssim \mathrm{poly}(N)$ in \cref{Lemma:BaseNN02} for $\NetworkMultiB$ and $\eps=N^{-(s+1)/d}/\mathrm{poly}(N)$ in \cref{Lemma:MandSigma} for $\NetworkSigmaA$.
    In summary, we can check that
    \begin{align}
       \left\| \NetworkScoreB(x,t) - \frac{f_3(x,t)}{\sigma_t}\right\| \lesssim N^{-(s+1)/d}
    \end{align}
    holds for all $x$ with $\|x\|_\infty \lesssim \sqrt{\log N}$ and therefore
    \begin{align}\label{eq:Appendix-Approx-5-1}
      \int_{\|x\|_\infty \lesssim \sqrt{\log N}} p_t(x)\left\| \NetworkScoreB(x,t) - \frac{f_3(x,t)}{\sigma_t}\right\|^2 \lesssim N^{-(s+1)/d}.
    \end{align}
    Moreover, the size of $\NetworkScoreB$ is bounded by
    \begin{align}\label{eq:Appendix-Approx-5-3}
         L = \Ord (\log^4 N),
       \ \|W\|_\infty = \Ord (N'\log^6N) \lesssim \Ord(N),
      \  S = \Ord (N'\log^8 N), \text{ and }
      \  B = \exp\left(\log^4 N\right) .
    \end{align}
    
     Now, we consider the difference between $f_3(x,t)/\sigma_t$ and $\nabla \log p_t(x)$.
    Its $L^2$ error in $\|x\|_\infty \leq m_t + \Ord(1)\sigma_t\sqrt{\log N}$ is bounded as previously, and we finally get
    \begin{align}
         &   \int_{ \|x\|_\infty\leq m_t + \Ord(1)\sigma_t\sqrt{\log N}} \mathbbm{1}[p_t(x)\geq N^{-\frac{2s+3}{d}} ]p_t(x) \left\| \frac{f_3(x,t)}{\sigma_t}-\frac{\nabla p_t(x)}{p_t(x)}\right\|^2\dx
    \\  & \lesssim  N^{\frac{4s+6}{d}} \int_{\|x\|_\infty\leq m_t + \Ord(1)\sigma_t\sqrt{\log N}}(|f_1(x,t)-p_t(x)|^2+\|f_2(x,t)-\sigma_t \nabla p_t(x)\|^2)\log N/\sigma_t^2\dx
    \\ & \lesssim N^{\frac{4s+6}{d}}   \log N/\sigma_t^2  \int_{\|x\|_\infty\leq m_t + \Ord(1)\sigma_t\sqrt{\log N}} \left|\int_y \frac{1}{\sigma_t^{d}(2\pi)^\frac{d}{2}}\exp\left(-\frac{\|x-m_t y\|^2}{2\sigma_t^2}\right)(p_0(y)-f_N(y))\dy\right|^2\dx
    \\ & \quad + N^{\frac{4s+6}{d}} \log N/\sigma_t^2  \int_{\|x\|_\infty\leq m_t + \Ord(1)\sigma_t\sqrt{\log N}} \left|\int_y \frac{x-m_ty}{\sigma_t^{d}(2\pi)^\frac{d}{2}}\exp\left(-\frac{\|x-m_t y\|^2}{2\sigma_t^2}\right)(p_0(y)-f_N(y))\dy\right|^2\dx
      \\ & \lesssim N^{\frac{4s+6}{d}}   \log N/\sigma_t^2  \int_{\|x\|_\infty\leq m_t + \Ord(1)\sigma_t\sqrt{\log N}}\int_y \frac{1}{\sigma_t^{d}(2\pi)^\frac{d}{2}}\exp\left(-\frac{\|x-m_t y\|^2}{2\sigma_t^2}\right)|p_0(y)-f_N(y)|^2\dy\dx
    \\ & \quad + N^{\frac{4s+6}{d}} \log N/\sigma_t^2  \int_{\|x\|_\infty\leq m_t + \Ord(1)\sigma_t\sqrt{\log N}} \int_y \frac{|x-m_ty|}{\sigma_t^{d}(2\pi)^\frac{d}{2}}\exp\left(-\frac{\|x-m_t y\|^2}{2\sigma_t^2}\right)|p_0(y)-f_N(y)|^2\dy\dx
 \\ & \lesssim N^{\frac{4s+6}{d}}   \log N/\sigma_t^2 \int_y\int_x \frac{1}{\sigma_t^{d}(2\pi)^\frac{d}{2}}\exp\left(-\frac{\|x-m_t y\|^2}{2\sigma_t^2}\right)|p_0(y)-f_N(y)|^2\dx\dy
    \\ & \quad + N^{\frac{4s+6}{d}} \log N/\sigma_t^2 \int_y\int_x \frac{|x-m_ty|}{\sigma_t^{d}(2\pi)^\frac{d}{2}}\exp\left(-\frac{\|x-m_t y\|^2}{2\sigma_t^2}\right)|p_0(y)-f_N(y)|^2\dx\dy
    \\ & \lesssim N^{\frac{4s+6}{d}}  \log N/\sigma_t^2\int_y |p_0(y)-f_N(y)|^2\dy
    \lesssim N^{\frac{4s+6}{d}}  \log N/\sigma_t^2 \cdot N^{-\frac{6s+10}{d}} \lesssim N^{-\frac{2(s+1)}{d}} /\sigma_t^2.
    \label{eq:Appendix-Approx-5-2}
    \end{align}
    Here we used the result of the previous lemma for the last inequality.
    Eqs. \eqref{eq:Appendix-Approx-5-1} and \eqref{eq:Appendix-Approx-5-3}, \eqref{eq:Appendix-Approx-5-2} yield the conclusion.
\end{proof}

Combining \cref{Lemma:ScoreFunc-1,Lemma:ScoreFunc-2}, where we use \cref{Lemma:ScoreFunc-1} for $T_1 \leq t \leq T_4$ and \cref{Lemma:ScoreFunc-2} for $T_3\leq t \leq T_5$, we immediately obtain \cref{theorem:Approximation}.
\begin{proof}[Proof of \cref{theorem:Approximation}]
Note that we can set $N'=N$ and $t_* = N^{-(2-\delta)/d}$ in \cref{Lemma:ScoreFunc-2}.
According to \cref{Lemma:ScoreFunc-1,Lemma:ScoreFunc-2}, we have two neural networks $\NetworkScoreA(x,t)$ and $\NetworkScoreB(x,t)$, that approximate the score function in $[T_1,T_4]$ and $[T_3,T_5]$.
Therefore, letting $\overline{t}_1 =T_4$ and $\underline{t}_2=T_3$ in \cref{Lemma:SwitchingFunc}, 
    $\NetworkScoreC(x,t)=\NetworkSwitchA(t;\underline{t}_2,\overline{t}_1)\NetworkScoreA(x,t) + \NetworkSwitchB(t;\underline{t}_2,\overline{t}_1)\NetworkScoreB(x,t)$
    approximates the approximation error in $L^2(p_t)$ with an additive error of $\frac{N^{-2s/d}\log N}{\sigma_t^2}$.
    Realization of the multiplications ($\NetworkSwitchA\NetworkScoreA$ and $\NetworkSwitchB\NetworkScoreB$ and aggregation $\NetworkSwitchA\NetworkScoreA+\NetworkSwitchB\NetworkScoreB$ is trivial.
    Finally, according to \cref{Lemma:ScoreFunc-1,Lemma:ScoreFunc-2}, the size of the network is bounded by 
    \begin{align}
        L = \Ord (\log^4 (N)),\| W\|_\infty = \Ord (N\log^6 N),S = \Ord (N\log^8 N),\quad \text{ and }  B = \exp(\Ord(\log^4 N )),
    \end{align}
which concludes the proof.
\end{proof}
We also prepare an integral form of the approximation theorems.
\begin{theorem}[Approximation theorem]\label{Lemma:Approximation-Appendix}
Suppose Assumptions \ref{assumption:InBesov}, \ref{assumption:SmoothBeta}, \ref{assumption:BoundarySmoothness} with $a_0=N^{-(1-\delta)/d}$, $N \gg 1$, $\underline{T}=\mathrm{poly}(N^{-1})$, and $\overline{T}\simeq \log N$.
 Then there exists a neural network $\NetworkScoreC\in \Phi(L,W,S,B)$ that satisfies
    \begin{align}
      \int_{t=\underline{T}}^{\overline{T}} \int_x p_t(x)  \|\NetworkScoreC(x,t) - \nabla \log p_t(x)\|^2 \dx\dt \lesssim  N^{-2s/d} \log N (\log (\overline{T}/\underline{T}) + (\overline{T}-\underline{T})). 
    \end{align}
    Here, $L, \|W\|_\infty, S, B$ is evaluated as
    \begin{align}
        L = \Ord (\log^4 N),
      \quad  \| W\|_\infty = \Ord (N),
    \quad    S = \Ord (N), 
    \quad  \text{and } B = \exp(\Ord(\log^4 N )).
    \end{align}

    Moreover, suppose $N'\geq t_*^{-d/2}N^{\delta/2}$, $t_*\geq N^{-(2-\delta)/d}$, and $\underline{T}\geq 2t_*$, 
    then there exists a neural network $\NetworkScoreC\in \Phi(L,W,S,B)$ that satisfies
    \begin{align}
    \int_{t=\underline{T}}^{\overline{T}} \int_x p_t(x)  \|\NetworkScoreC(x,t) - \nabla \log p_t(x)\|^2 \dx\dt  \lesssim N^{-\frac{2(s+1)}{d}}(\log (\overline{T}/\underline{T}) + (\overline{T}-\underline{T})).
    \end{align}
    Specifically, $L = \Ord (\log^4 (N)),\| W\|_\infty = \Ord (N),S = \Ord (N')$, and $ B = \exp(\Ord(\log^4 N ))$.
\end{theorem}
\begin{proof}
    We only show the first part; the second part comes from \cref{Lemma:ScoreFunc-2} in the same way.
    According to \cref{theorem:Approximation}, there exists a network $\NetworkScoreC$ with the desired size that satisfies
     \begin{align}
    \int_x p_t(x)  \|\NetworkScoreC(x,t) - s(x,t)\|^2 \dx \lesssim \frac{N^{-\frac{2s}{d}}\log(N)}{\sigma_t^2}.
    \end{align}
    Note that $\sigma_t \gtrsim t \land 1$.
    Therefore, 
    \begin{align}
        \int_{t=\underline{T}}^{\overline{T}}\frac{N^{-\frac{2s}{d}}\log(N)}{\sigma_t^2}\dt
        \lesssim 
         \int_{t=\underline{T}}^{\overline{T}}N^{-\frac{2s}{d}}\log(N) (1\lor 1/t)\dt
         \leq N^{-\frac{2s}{d}}\log(N)(\log (\overline{T}/\underline{T}) + (\overline{T}-\underline{T})), 
    \end{align}
    which gives the first part of the theorem.
\end{proof}

\section{Generalization of the score network}\label{section:Appendix-Generalization}
Now we consider the generalization error.
As in \cref{section:Generalization}, we first consider the sup-norm of $\ell$ and evaluate the covering number.
\subsection{Bounding sup-norm}\label{subsection:Appendix-Generalization-Prepare-1}
\begin{lemma}\label{lemma:Appendix-Generalization-Prepare-1}
     Suppose that $\|s(\cdot,t)\|_\infty = \Ord(\sigma_t^{-1}\log^\frac12 n)$, $\underline{T}=\mathrm{poly}(n^{-1})$ and $\overline{T}\simeq \log n$.
     Then, we have that
     \begin{align}
         \int_{t=\underline{T}}^{\overline{T}}\int_{x_t}\|s(x_t,t)-\nabla \log p_t(x_t|x_0)\|^2 p_t(x_t|x_0)\dx_t\dt\lesssim \log^2 n.
     \end{align}
\end{lemma}
\begin{proof}
   The evaluation is mostly straightforward.
    \begin{align}
     &    \int_{t=\underline{T}}^{\overline{T}}\int_{x_t}\|s(x_t,t)-\nabla \log p_t(x_t|x_0)\|^2 p_t(x_t|x_0)\dx_t\dt
      \\   & \leq 2 \int_{t=\underline{T}}^{\overline{T}}\int_{x_t}\|s(x_t,t)\|^2
        p_t(x_t|x_0) \dx\dt
       + 2 \int_{t=\underline{T}}^{\overline{T}}\int_{x_t}\|\log p_t(x_t|x_0)\|^2
         p_t(x_t|x_0) \dx_t\dt
              \\   & \lesssim
               \int_{t=\underline{T}}^{\overline{T}} \frac{\log n}{\sigma_t^2}\dt 
               +\int_{t=\underline{T}}^{\overline{T}}\frac{1}{\sigma_t^2}\dt
                 \\   & \lesssim   \int_{t=\underline{T}}^{\overline{T}} \frac{\log n}{t\land 1}\dt 
                 \leq (\log n )\cdot (\log \underline{T}^{-1} + \overline{T}) \lesssim \log^2 n
    \end{align}
    For the evaluation of $\int_{x_t}\|\log p_t(x_t|x_0)\|^2
         p_t(x_t|x_0) \dx_t$, we used the fact that $p_t(x_t|x_0)$ is the density function of $\mathcal{N}(m_tx_0,\sigma_t^2)$.
         Also, we used that $\underline{T}=\mathrm{poly}(n^{-1})$ and $\overline{T}\simeq \log n$ for the last inequality.
\end{proof}
\subsection{Covering number evaluation}\label{subsection:Appendix-CoveringNumber}
\begin{lemma}[Covering number of $\mathcal{L}$]
For a neural network $s\cdot \R^{d}\times \R\to \R^d$, we define $\ell\cdot \R^d\to \R$ as
  \begin{align}\ell_s(x)=\int_{t=\underline{T}}^{\overline{T}}\int_{x_t}\|s(x_t,t)-\nabla \log p_t(x_t|x)\|^2 p_t(x_t|x)\dx\dt.
\end{align}
    For the hypothesis network class $\mathcal{S}\in \Phi(L,W,S,B)$, we define a function class $\mathcal{L}=\{\ell_s|\ s\in \mathcal{S}\}$.
    If the corresponding $s$ is obvious for some $\ell_s$, we sometimes abbreviate $\ell_s$ as $\ell$.
    
    Assume that $s(x,t)$ is bounded by $\|\|s(\cdot,t)\|_2\|_{L^\infty} = \Ord(\sigma_t^{-1}\log^\frac12 n)$ uniformly over all $s\in \mathcal{S}$ and $C\geq 1$.
    Then the covering number of $\mathcal{S}$ is evaluated by
    \begin{align}\label{eq:Appendix-Generalization-CoveringNumber-1}
         \log \mathcal{N}(\mathcal{S}, \|\|\cdot\|_2\|_{L^\infty ([-C,C]^{d+1})}, \delta) 
         \lesssim 2SL\log(\delta^{-1}L\|W\|_\infty(B\lor 1)C),
    \end{align}
    and based on this, the covering number of $\mathcal{L}$ is evaluated by
    \begin{align}\label{eq:Appendix-Generalization-CoveringNumber-2}
        \log \mathcal{N}(\mathcal{L}, \|\cdot\|_{L^\infty ([-1,1]^d)}, \delta) \lesssim SL\log(\delta^{-1}L\|W\|_\infty(B\lor 1) n)
    \end{align}
    when $\delta^{-1}, \underline{T}^{-1}, \overline{T}, N = {\rm poly}(n)$.
\end{lemma}
\begin{proof}
The first bound \eqref{eq:Appendix-Generalization-CoveringNumber-1} is directly obtained from \citet{suzuki2018adaptivity}, with a slight modification of the input region.
By following their proof, we can see that their $\delta$-net for the $L^\infty([0,1]^d)$-norm serves as the $C\delta$-net for the $L^\infty([-C,C]^d)$-norm. Therefore, we simply set $\delta\leftarrow C^{-1}\delta$ in their bound to obtain \eqref{eq:Appendix-Generalization-CoveringNumber-1}.

    We next consider \eqref{eq:Appendix-Generalization-CoveringNumber-2}. First we clip the integral interval in the definition of $\ell$.
    \begin{align}&\left|\ell_s(x)-\int_{t=\underline{T}}^{\overline{T}}\int_{\|x_t\|_\infty\leq \Ord(\sqrt{\log n})}\|s(x_t,t)-\nabla \log p_t(x_t|x)\|^2 p_t(x_t|x)\dx_t\dt\right|
    \\
   &\leq  \int_{t=\underline{T}}^{\overline{T}}
    \int_{\|x_t\|_\infty\geq \Ord(\sqrt{\log n})}\|s(x_t,t)-\nabla \log p_t(x_t|x)\|^2 p_t(x_t|x)\dx_t\dt
    \dt
    \\ & \leq
  \|\| s(\cdot,\cdot)\|_2\|_{L^\infty}^2 \int_{t=\underline{T}}^{\overline{T}} \int_{\|x_t\|_\infty\geq \Ord(\sqrt{\log n})}p_t(x_t|x)\dx_t\dt
  +
  \int_{t=\underline{T}}^{\overline{T}} \int_{\|x_t\|_\infty\geq \Ord(\sqrt{\log n})}\|\nabla \log p_t(x_t|x)\|^2p_t(x_t|x)\dx_t\dt.
  \label{eq:Appendix-Generalization-CoveringNumber-3}
\end{align}
Because $p_t(x_t|x)$ is the density function of $\mathcal{N}(m_tx|\sigma_t^2)$, we can show that $\int_{\|x_t\|_\infty\geq \Ord(\sqrt{\log n})}p_t(x_t|x)\dx_t$ and $\int_{\|x_t\|_\infty\geq \Ord(\sqrt{\log n})}\|\nabla \log p_t(x_t|x)\|^2p_t(x_t|x)\dx_t$ are bounded by $\frac{\delta}{3\overline{T}(\|\| s(\cdot,\cdot)\|_2\|_{L^\infty}^2\lor 1)}$ if $\delta^{-1}, \underline{T}^{-1}, \overline{T}, N = {\rm poly}(n)$ and the hidden constant in $\Ord(\sqrt{\log n})$ is sufficiently large (see \cref{Lemma:GaussianBound}).
Therefore, \eqref{eq:Appendix-Generalization-CoveringNumber-3} is bounded by 
\begin{align}
    \|\| s(\cdot,\cdot)\|_2\|_{L^\infty}(\overline{T}-\underline{T})\cdot \frac{\delta}{3\overline{T}\|\| s(\cdot,\cdot)\|_2\|_{L^\infty}} + (\overline{T}-\underline{T})\cdot \frac{\delta}{3\overline{T}} \leq \frac23 \delta.
    \label{eq:Appendix-Generalization-CoveringNumber-6}
\end{align}
    We then take $C={\rm poly}(n)\gtrsim\sqrt{\log n}$ and construct $\frac{\delta}{3}$-net for a set of
    \begin{align}\label{eq:Appendix-Generalization-CoveringNumber-7}
       \ell'(x):= \int_{t=\underline{T}}^{\overline{T}}\int_{\|x_t\|_\infty\leq C}\|s(x_t,t)-\nabla \log p_t(x_t|x)\|^2 p_t(x_t|x)\dx_t\dt
    \end{align}
    over all $s\in \mathcal{S}$.
    For this, we take $\frac{\delta}{n^{\Ord(1)}}$-net of $\mathcal{S}$ with the $L^\infty ([-C,C]^{d+1})$-norm.
    According to \eqref{eq:Appendix-Generalization-CoveringNumber-1}, the covering number is evaluated as
       \begin{align}
         \log \mathcal{N}\left(\mathcal{S}, \|\|\cdot\|_2\|_{L^\infty ([-C,C]^{d+1})}, \frac{\delta}{n^{\Ord(1)}}\right) 
         \lesssim 2SL\log(\delta^{-1}L\|W\|_\infty(B\lor 1)n).
    \end{align}
    For different $s$ and $s'$, because $\|\nabla \log p_t(x_t|x)\|\lesssim \frac{C}{\sigma_t^2}$ for $\|x_t\|_\infty \leq C$, 
    we have that
    \begin{align} \label{eq:Appendix-Generalization-CoveringNumber-4}
     & |  \|s(x_t,t)-\nabla \log p_t(x_t|x)\|^2-\|s'(x_t,t)-\nabla \log p_t(x_t|x)\|^2|
    \\  & \leq
     (\|s(x_t,t)-\nabla \log p_t(x_t|x)\|+\|s'(x_t,t)-\nabla \log p_t(x_t|x)\|^2)
     |
     \|s(x_t,t)-\nabla \log p_t(x_t|x)\|-\|s'(x_t,t)-\nabla \log p_t(x_t|x)\||
     \\ & \leq
     (\|\| s(\cdot,\cdot)\|_2\|_{L^\infty}+\|\| s'(\cdot,\cdot)\|_2\|_{L^\infty} + 2C/\sigma_t^2)\cdot \frac{\delta}{n^{\Ord(1)}}.\label{eq:Appendix-Generalization-CoveringNumber-5}
    \end{align}
    By taking the hidden constant in $\frac{\delta}{n^{\Ord(1)}}$ sufficiently large, this is further bounded by $\frac{\delta}{3\overline{T}(2C)^d}$ when $C, \underline{T}^{-1}, \overline{T} = {\rm poly}(n)$.
    Integrating \eqref{eq:Appendix-Generalization-CoveringNumber-4} and \eqref{eq:Appendix-Generalization-CoveringNumber-5} over $ \int_{t=\underline{T}}^{\overline{T}}\int_{\|x_t\|_\infty\leq C}\dx_t\dt$ yields that this $\frac{\delta}{n^{\Ord(1)}}$-net of $\mathcal{S}$ actually gives the $\frac{\delta}{3}$-net for the set of \eqref{eq:Appendix-Generalization-CoveringNumber-7}; finally, we have obtained the $\delta$-net for $\mathcal{L}$ together with \eqref{eq:Appendix-Generalization-CoveringNumber-6}.
\end{proof}

\subsection{Generalization error bound on the score matching loss}
This subsection gives the complete proof of \cref{Theorem:Generalization}.
First, the following relationship is useful.
This shows the equivalence of explicit score matching and denoising score matching, and can be used to show that the minimizer of the empirical denoising score matching also approximately minimizes the explicit score matching loss.
\begin{lemma}[Equivalence of explicit score matching and denoising score matching (\citet{vincent2011connection})]\label{Lemma:VincentEquivalence}
    The following equality holds for all $s(x_t,t)$ and $t>0$:
       \begin{align}
         \int_{x_t} \|s(x_t,t) - \nabla \log p_t(x_t)\|^2 p_t(x_t) \dx_t
=  \int_{x_0}\int_{x_t}  \|s(x_t,t)-\nabla \log p_t(x_t|x_0) \|^2p_t(x_t|x_0)p_0(x_0) \dx_0\dx_0 + C
        ,
    \end{align} 
    where $C = \int_{x_t}\|\nabla \log p_t(x_t)\|^2 p_t(x_t) \dx_t - \int_{x_0}\int_{x_t}  \|\nabla \log p_t(x_t|x_0) \|^2p_t(x_t|x_0)p_0(x_0)\dx_t\dx_0$.
\end{lemma}
\begin{proof}
    The proof follows \citet{vincent2011connection}.
    \begin{align}
        & \int_{x_t} \|s(x_t,t) - \nabla \log p_t(x_t)\|^2 p_t(x_t) \dx_t
    \\ &  = - 2 \int_{x_t} p_t(x_t) s(x_t,t)^\top \nabla \log p_t(x_t) \dx +\int_{x_t} \|s(x_t,t)\|^2 p_t(x_t) \dx_t+ \int_{x_t}\|\nabla \log p_t(x_t)\|^2 p_t(x_t) \dx
      \\&  =   - 2 \int_{x_t} s(x_t,t)^\top \nabla p_t(x_t) \dx_t +\int_{x_t} \|s(x_t,t)\|^2 p_t(x_t) \dx_t+ \int_{x_t}\|\nabla \log p_t(x_t)\|^2 p_t(x_t) \dx 
        \\&  =   - 2 \int_{x_t} s(x_t,t)^\top \nabla \left(\int_{x_0} p_t(x_t|x_0) p_0(x_0) \dx_0\right) \dx_t +\int_{x_t} \|s(x_t,t)\|^2 p_t(x_t) \dx_t+ \int_{x_t}\|\nabla \log p_t({x_t})\|^2 p_t({x_t}) \dx_t   
        \\&  =   - 2 \int_{x_t} s(x_t,t)^\top \left(\int_{x_0} p_0(x_0) \nabla p_t(x_t|x_0) \dx_0\right) \dx_t +\int_{x_t} \|s(x_t,t)\|^2 p_t(x_t) \dx_t+ \int_{x_t}\|\nabla \log p_t({x_t})\|^2 p_t({x_t}) \dx_t     
        \\&  =   - 2 \int_{x_t}p_t(x_t|y)p_0(x_0) s(x_t,t)^\top \left(\int_{x_0}  \nabla \log p_t(x_t|x_0) \dx_0\right) \dx_t +\int_{x_t} \|s(x_t,t)\|^2 p_t(x_t) \dx_t+ \int_{x_t}\|\nabla \log p_t({x_t})\|^2 p_t({x_t}) \dx_t       
        \\&  =   - 2 \int_{x_0} \int_{x_t} p_t(x_t|x_0)p_0(x_0)  s(x_t,t)^\top  \nabla \log p_t(x_t|x_0) \dx_t\dx_0 +\int_{x_0}\int_{x_t}p_t(x_t|x_0)p_0(x_0) \|s(x_t,t)\|^2 \dx_t\dx_0
        \\ &\hspace{105mm}+ \int_{x_t}\|\nabla \log p_t(x_t)\|^2 p_t(x_t) \dx_t 
        \\&  =  \int_{x_0}\int_{x_t} p_t(x_t|x_0)p_0(x_0) \|s(x_t,t)-\nabla \log p_t(x_t|x_0) \|^2 \dx_t\dx_0+ \int_{x_t}\|\nabla \log p_t(x_t)\|^2 p_t(x_t) \dx_t \\ & \hspace{105mm}- \int_{x_0}\int_{x_t} p_t(x_t|x_0)p_0(x_0) \|\nabla \log p_t(x_t|x_0) \|^2\dx_t\dx_0
        ,
    \end{align}
    where we used $\nabla \log p_t(x_t) = (\nabla p_t(x_t)) / p_t(x_t)$ for the second, $p_t(x_t) = \int_{x_0} p_t(x_t|x_0) p_0(x_0) \dx_0$ for the third, $ \nabla \log p_t(x_t|x_0) = (\nabla p_t(x_t|x_0) )/p_t(x_t|x_0) $ for the fifth equalities.

\end{proof}
Now, we evaluate the generalization error and the following theorem is a formal version of \cref{Theorem:Generalization}. 
\begin{theorem}[Generalization error bound based on the covering number]\label{Theorem:ExtendedSchmidt}
Let $\hat{s}$ be the minimizer of 
\begin{align}\label{eq:TheoremExtendedSchmidt-1}
    \frac1n \sum_{i=1}^n \int_{t=\underline{T}}^{\overline{T}} \int_x\|s(x,t) - \nabla \log p_t(x|x_i)\|_2^2 p_t(x|x_{0,i}) \dx \dt,
\end{align}
taking values in $\mathcal{S}\subset L^2(\R^d \times [\tA,\tB])$.
For each $s\in \mathcal{S}$, let $\ell(x) = \int_{t=\tA}^{\tB} \int_x\|s(x,t) - \nabla \log p_t(y|x)\|_2^2 p_t(y|x) \dy \dt$ and $\mathcal{L}$ be a set of $\ell$ corresponding to each $s\in \mathcal{S}$.
Suppose every element $\ell\in \mathcal{L}$ satisfies $\|\ell\|_{L^\infty([-1,1]^d)}\leq C_{\ell}$ 
 for some fixed $0< C_{\ell}$.
For an arbitrary $\delta>0$, if $N := N(\mathcal{L}, \|\cdot\|_{L^\infty([-1,1]^d)},\delta)\geq 3$, then we have that
 \begin{align}\label{eq:TheoremExtendedSchmidt-2}
    &   \mathbb{E}_{\{x_i\}_{i=1}^n} \left[\int_x \int_{t=\tA}^{\tB}\|\hat{s}(x,t) - \nabla \log p_t(x)\|^2 p_t(x) \dt\dx\right] \\&\leq 2 \inf_{s\in \mathcal{S}}\int_x \int_{\tA}^{\tB} \|s(x,t) - \nabla \log p_t(x)\|_2^2 p_t(x) \dx \dt +\frac{2C_{\ell}}{n}\left(\frac{37}{9}\log N + 32\right) +3\delta.
    \end{align}
\end{theorem}
\begin{proof}
In the following proof, $x_{0,i}$ is denoted as $x_i$ for simplicity.
    \eqref{eq:TheoremExtendedSchmidt-1} is written as $ \frac1n\sum_{i=1}^n \ell(x_i)$.
    Also, with $s^\circ(x,t)=\nabla\log p_t(x)$, we write
    \begin{align}
        R(\hat{\ell}, \ell^\circ) & := \int_x\int_{t=\tA}^{\tB}  \|\hat{s}(x,t) - \nabla \log p_t(x)\|^2 p_t(x) \dt\dx
        \\ &= \int_x\int_{t=\tA}^{\tB}  \|\hat{s}(x,t) - \nabla \log p_t(x)\|^2 p_t(x) \dt\dx- \underbrace{\int_x\int_{t=\tA}^{\tB}    \|s^\circ(x,t) - \nabla \log p_t(x)\|^2 p_t(x) \dt\dx}_{=0}
   \\  &=
     \int_y\int_{t=\tA}^{\tB}\int_x \|s(x,t)-\nabla \log p_t(x|y) \|^2p_t(x|y)p_0(x) \dy\dt\dx+ C(\tB-\tA)
  \\& \hspace{45mm}   -
   \int_y\int_{t=\tA}^{\tB} \int_x  \|s^\circ(x,t)-\nabla \log p_t(x|y) \|^2p_t(x|y)p_0(x) \dy\dt\dx - C(\tB-\tA)
     \\&=\label{eq:TheoremExtendedSchmidt-3}
     \mathbb{E}_{\{x_i'\}_{i=1}^n} \left[ \frac1n\sum_{i=1}^n (\hat{\ell}(x_i')-\ell^\circ(x_i') ) \right]
    \end{align}
    with $\{x_i'\}_{i=1}^n$, that is an i.i.d. sample from $p_0$ and independent of $\{x_i\}_{i=1}^n$.
    For the second equality, we used \cref{Lemma:VincentEquivalence}.
    
    First, we evaluate the value of 
    \begin{align}
        D := \left| \mathbb{E}_{\{x_i\}_{i=1}^n}\left[\frac1n\sum_{i=1}^n (\hat{\ell}(x_i) - \ell^\circ(x_i))\right]
        - R(\hat{\ell}, \ell^\circ)
        \right|.
    \end{align}
    Using \eqref{eq:TheoremExtendedSchmidt-3}, we obtain
    \begin{align}
        D = \left| \mathbb{E}_{x_i,x_i'}\left[\frac1n\sum_{i=1}^n ((\hat{\ell}(x_i) - \ell^\circ(x_i))-(\hat{\ell}(x_i') - \ell^\circ(x_i'))) \right]\right| 
         \leq 
       \frac1n \mathbb{E}_{x_i,x_i'}\left[\left| \sum_{i=1}^n ((\hat{\ell}(x_i) - \ell^\circ(x_i))-(\hat{\ell}(x_i') - \ell^\circ(x_i'))) \right| \right]
       .
    \end{align}
    Let $\mathcal{L}_d = \{\ell_1,\ell_2,\cdots,\ell_N\}$ be a $\delta$-covering of $\mathcal{L}$ with the minimum cardinality in the $L^\infty([-1,1]^d)$ metric.
    From the assumption of $N(\mathcal{L}, \|\cdot\|_{\infty},\delta)\geq 3$, we have $\log N \geq 1$.
    We define $g_j(x,x')=(\ell_j(x) - \ell^\circ (x)) - (\ell_j(x') - \ell^\circ (x'))$ and a random variable $J$ taking values in $\{1,2,\cdots,N\}$ such that $\|\hat{\ell} - f_J\|_{\infty} \leq \delta$, so that we have
    \begin{align}
        D \leq \frac1n  \mathbb{E}_{x_i,x_i'}\left[\left| \sum_{i=1}^n g_J(x_i,x_i') \right| \right] 
            + \|(\hat{\ell}_j(x) - \ell_J (x)) - (\hat{\ell}_j(x') - \ell_J (x')\|_\infty
            \leq 
            \frac1n  \mathbb{E}_{x_i,x_i'}\left[\left| \sum_{i=1}^n g_J(x_i,x_i') \right| \right] + \delta
           .  \label{eq:TheoremExtendedSchmidt-4}
    \end{align}
    Then we define $r_j:=\max \{ A, \sqrt{\mathbb{E}_{x'}[\ell_j(x') - \ell^\circ(x')]}\}\ (j=1,2,\cdots,N)$ and a random variable
    \begin{align}
        G:= \max_{1\leq j\leq N}\left|\sum_{i=1}^n\frac{g_j(x_i,x_i')}{r_j}\right|,
    \end{align}
    where $A>0$ is a constant adjusted later. 
    Then we further evaluate \eqref{eq:TheoremExtendedSchmidt-4} as
    \begin{align}
        D \leq \frac{1}{n}\mathbb{E}_{x_i,x_i'}[r_J G]+\delta
        \leq
        \frac1n\sqrt{\mathbb{E}_{x_i,x_i'}[r_J^2]\mathbb{E}_{x_i,x_i'}[G^2]} + \delta
        \leq
        \frac12\mathbb{E}_{x_i,x_i'}[r_J^2] + \frac{1}{2n^2}\mathbb{E}_{x_i,x_i'}[G^2] + \delta
        \label{eq:TheoremExtendedSchmidt-5}
        ,
    \end{align}
    by the Cauthy-Schwarz inequality and the AM-GM inequality.
    The definition of $J$ yields that
    \begin{align}
        \mathbb{E}_{x_i,x_i'}[r_J^2]
        \leq
        A^2 + \mathbb{E}_{x'}[\ell_J(x') - \ell^\circ(x')]
        \leq
        A^2 + \mathbb{E}_{x'}[\hat{\ell}(x') - \ell^\circ(x')] + \delta
        = R(\hat{\ell}, \ell^\circ) + A^2 + \delta.\label{eq:TheoremExtendedSchmidt-6}
    \end{align}
    Because of the independence of $x_i$ and $x_i'$, we have that
    \begin{align}
    \mathbb{E}_{x_i,x_i'}\left[\left(\sum_{i=1}^n \frac{g_j(x_i,x_i')}{r_j}\right)^2\right]
   & \leq 
    \sum_{i=1}^n\mathbb{E}_{x_i,x_i'}\left[\left(\frac{g_j(x_i,x_i')}{r_j}\right)^2\right]
    \\ & =
      \sum_{i=1}^n\left(\mathbb{E}_{x_i,x_i'}\left[\frac{(\ell_j(x_i) - \ell^\circ (x_i))^2}{r_j^2}\right]
      + \mathbb{E}_{x_i,x_i'}\left[\frac{(\ell_j(x_i') - \ell^\circ (x_i'))^2}{r_j^2}\right]\right)
      \\ & \leq 2C_{\ell} n \label{eq:IndependenceYueni}
    \end{align}
    holds, where we used the fact that $g_j(x_i,x_i')$ is centered and $|\ell_j(x) - \ell^\circ(x)|$ is bounded by $C_{\ell}$.
    Also, $\frac{g_j(x_i,x_i')}{r_j}$ is bounded with $C_{\ell}/A$.
    Then, using Bernstein's inequality, we have that
    \begin{align}
        \mathbb{P}[G^2 \geq t] = \mathbb{P}[G \geq \sqrt{t}] \leq
        2N\exp\left( - \frac{t}{2C_{\ell}(2n+\frac{\sqrt{t}}{3A})}\right),
    \end{align}
    for any $t\geq 0$.
    This gives evaluation of $\mathbb{E}_{x_i,x_i'}[G^2]$.
    For any $t_0>0$, we have that
    \begin{align}
        \mathbb{E}_{x_i,x_i'}[G^2] &= \int_{0}^\infty \mathbb{P}[G^2\geq t] \dt 
        \\ & \leq t_0 + \int_{t_0}^\infty \mathbb{P}[G^2\geq t] \dt 
        \\ & \leq t_0 + 2N \int_{t_0}^\infty \exp\left(-\frac{t}{8C_{\ell} n}\right)\dt 
        + 2N \int_{t_0}^\infty \exp\left(-\frac{3A\sqrt{t}}{4C_{\ell}}\right)\dt .
    \end{align}
    These two integrals are computed as
    \begin{align}
        \int_{t_0}^\infty \exp\left(-\frac{t}{8C_{\ell}n}\right)\dt  &=
        \left[-8C_{\ell} n \exp\left(-\frac{t}{8C_{\ell}n}\right)\right]_{t_0}^\infty
        = 8C_{\ell} n \exp\left(-\frac{t_0}{8C_{\ell}n}\right)
        \\ 
        \int_{t_0}^\infty \exp\left(-\frac{3A\sqrt{t}}{4C_{\ell}}\right)\dt  &=
        \int_{t_0}^\infty \exp\left(-a\sqrt{t}\right)\dt  \quad\quad\quad\quad\quad\quad\quad\quad\quad\quad\quad\quad\quad\quad\quad\quad (a:= 3A/4C_{\ell})
        \\ &=
        \left[-\frac{2(a\sqrt{t}+1)}{a^2}\exp(-a\sqrt{t})\right]_{t_0}^\infty
     \\&   = \frac{8C_{\ell}\sqrt{t_0}}{3A}\exp\left(-\frac{3A\sqrt{t_0}}{4C_{\ell}}\right)
        + \frac{32C_{\ell}}{9A^2}\exp\left(-\frac{3A\sqrt{t_0}}{4C_{\ell}}\right).
    \end{align}
    We take $A=\sqrt{t_0}{6n}$ so that
    \begin{align}
         \mathbb{E}_{x_i,x_i'}[G^2] &\leq t_0 + 2N \left(8C_{\ell} n + 16 C_{\ell} n + \frac{128C_{\ell} n^2}{t_0}\right)\exp\left(-\frac{t_0}{8C_{\ell} n}\right)
         \\ & \leq
         t_0 + 16N \exp\left(-\frac{3A\sqrt{t_0}}{4C_{\ell}}\right) n (3 + 16n/t_0)\exp\left(-\frac{t_0}{8C_{\ell} n}\right)
    \end{align}
    holds. Furthermore, we take $t_0 = 8C_{\ell} n\log N$, and then it holds that
    \begin{align}
        \mathbb{E}_{x_i,x_i'}[G^2]
        \leq 18 C_{\ell} n \left(\log N + 6 + \frac{2}{C_{\ell}\log N}\right).
        \label{eq:TheoremExtendedSchmidt-7}
    \end{align}
    Now, we combine \eqref{eq:TheoremExtendedSchmidt-5}, \eqref{eq:TheoremExtendedSchmidt-6}, \eqref{eq:TheoremExtendedSchmidt-7}, and $A^2 = \frac{2C_{\ell}\log N}{9n}$ to obtain
    \begin{align}
        D & \leq \left(\frac12 R(\hat{\ell}, \ell^\circ) + \frac12 A^2 + \frac12\delta\right)
        + \frac{4C_{\ell} }{n}\left(\log N + 6 + \frac{2}{C_{\ell}\log N}\right)
        +\delta
        \\ & \leq
        \frac12 R(\hat{\ell}, \ell^\circ) + \frac{C_{\ell}}{n}\left(\frac{37}{9}\log N + 32\right) + \frac32 \delta
        ,
    \end{align}
    where we have used that $\log N \geq 1$.
    Therefore, we obtain 
    \begin{align}
        R(\hat{\ell}, \ell^\circ) \leq 2\mathbb{E}_{\{x_i\}_{i=1}^n}\left[\frac1n\sum_{i=1}^n (\hat{\ell}(x_i) - \ell^\circ(x_i))\right]+\frac{2C_{\ell}}{n}\left(\frac{37}{9}\log N + 32\right) +3\delta
        . \label{eq:TheoremExtendedSchmidt-8}
    \end{align}
    For any fixed $\ell\in \mathcal{L}$,
    \begin{align}
         \mathbb{E}_{\{x_i\}_{i=1}^n}\left[\frac1n\sum_{i=1}^n (\hat{\ell}(x_i) - \ell^\circ(x_i))\right]
         \leq
          \mathbb{E}_{\{x_i\}_{i=1}^n}\left[\frac1n\sum_{i=1}^n (\ell(x_i) - \ell^\circ(x_i))\right]
          = \mathbb{E}_{x} [\ell(x) - \ell^\circ(x)]
          .
    \end{align}
    RHS is minimized as $\inf_{\ell \in \mathcal{L}} \mathbb{E}_{x} [\ell(x) - \ell^\circ(x)]$.
    Finally, combining this with \eqref{eq:TheoremExtendedSchmidt-8}, we obtain
    \begin{align}
        R(\hat{\ell}, \ell^\circ) \leq 2\inf_{\ell \in \mathcal{L}} \mathbb{E}_{x} [\ell(x) - \ell^\circ(x)]+\frac{2C_{\ell}}{n}\left(\frac{37}{9}\log N + 32\right) +3\delta
        .
    \end{align}
    According to \cref{Lemma:VincentEquivalence}, we have
    \begin{align}
        R(\hat{\ell}, \ell^\circ) \leq 2 \inf_{s\in \mathcal{S}}\int_{\tA}^{\tB} \int_x \|s(x,t) - \nabla \log p_t(x)\|_2^2 p_t(x) \dx \dt +\frac{2C_{\ell}}{n}\left(\frac{37}{9}\log N + 32\right) +3\delta
        .
    \end{align}
\end{proof}

\subsection{Sampling $t$ and $x_t$ instead of taking expectation}\label{subsection:ApproxViaSample}
This section provides justification of two approaches presented in \cref{subsection:Generalization-ScoreMatchingRemark}.
We assume $\delta^{-1}, \underline{T}^{-1}, \overline{T}, N = {\rm poly}(n)$.
We first begin with the following lemma. This shows that $\|s(x_j,t_j)-\nabla p_{t_j}(x_j|x_{0,i_j})\|$ is sub-Gaussian.
\begin{lemma}\label{Lemma:Appendix-Generalization-SubGaussian}
    Let us sample $(i_j,t_{j},x_j)$ from $i_j \sim \rm{Unif}(\{1,2,\cdots,n\})$, $t_j\sim \rm{Unif}(\underline{T},\overline{T})$, and $x_{j}\sim p_{t_j}(x_j|x_{0,i_j})$.
    Then, we have that, for all $t>0$,
    \begin{align}
       \mathbb{P}\left[ \|s(x_j,t_j)-\nabla p_{t_j}(x_j|x_{0,i_j})\|\geq \sup_{(x,t)}\|s(x,t)\| + \frac{\sqrt{d}t}{\sigma_{\underline{T}}}
       \right]\leq 2 \exp\left(-t^2/2\right)
       .
    \end{align}
\end{lemma}
\begin{proof}
    First note that
    \begin{align}
    \|s(x_j,t_j)-\nabla p_{t_j}(x_j|x_{0,i_j})\|
    \leq 
     \|s(x_j,t_j)\|
     +
     \|\nabla p_{t_j}(x_j|x_{0,i_j})\|
     \leq
 \sup_{x,t}\|s(x,t)\|
     +
     \|\nabla p_{t_j}(x_j|x_{0,i_j})\|.
    \end{align}
    Because $\nabla p_{t_j}(x_j|x_{0,i_j}) = \frac{x_j-m_tx_{0,i_j}}{\sigma_t^2}$ and $x_{j}\sim p_{t_j}(x_j|x_{0,i_j})=\mathcal{N}\left(m_tx_{0,i_j},\sigma_t^2\right)$, we have that $[\nabla p_{t_j}(x_j|x_{0,i_j})]_i$ is sub-Gaussian with $\sigma_t^{-1}$.
    Thus, $\|\nabla p_{t_j}(x_j|x_{0,i_j})\|$ is sub-Gaussian with $\sqrt{d}\sigma_t^{-1}$.
    Now, applying $\sigma_t \geq \sigma_{\underline{T}}$, we have the assertion.
\end{proof}
Now, we give the following theorem for the first approach.
\begin{theorem}\label{theorem:ApproxViaSample-1}
Let us sample $(i_j,t_{j},x_j)$ from $i_j \sim \rm{Unif}(\{1,2,\cdots,n\})$, $t_j\sim \rm{Unif}(\underline{T},\overline{T})$, and $x_{j}\sim p_{t_j}(x_j|x_{0,i})$.
Let $s_1$ be the minimizer of 
\begin{align}
    \frac{1}{M}\sum_{j=1}^M \|s(x_j,t_j)-\nabla p_{t_j}(x_{j}|x_{0,i})\|^2
\end{align}
and $s_2$ be the minimizer of 
\begin{align}
    \frac1n\sum_{i=1}^n\ell(x_i) 
    =
\frac1n\sum_{i=1}^n\int_{t=\underline{T}}^{\overline{T}}\|s(x_t,t)-\nabla p_{t}(x_t|x_{0,i})\|^2p_t(x_t|x_{0,i})\dx_t\dt,
\end{align}
over $\mathcal{S}\subseteq \Phi(L,W,S,B)$, where $s\in \mathcal{S}$ satisfies $\|\|s(\cdot,t)\|_2\|_{L^\infty} = \Ord(\sigma_t^{-1}\log^\frac12 n)\lesssim \Ord(\sigma_{\underline{T}}^{-1}\log^\frac12 n)=:C_s$.
Then, we have that
\begin{align}
  \mathbb{E}_{\{(i_j,t_j,x_j)\}_{i=1}^n}  \left|
  \frac1n\sum_{i=1}^n\ell_1(x_i)- \frac1n\sum_{i=1}^n\ell_2(x_i)\right|
  \lesssim \frac{C_s^2+\sigma_{\underline{T}}^{-2}}{M} 2SL\log(\delta^{-1}L\|W\|_\infty(B\lor 1)(C_s))+\delta.
\end{align}
\end{theorem}
\begin{proof}
    We denote $(i_j,t_{j},x_j)=y_j$ for simplicity and $Y=\{(i_j,t_{j},x_j)\}_{j=1}^M=\{y_j\}_{j=1}^M$.
    Let $Y'=\{(i_{j}',t_{j}',x_{j}')\}_{j=1}^M=\{y_j'\}_{j=1}^M$ be a copy of $Y$, which is independent of $Y$.
    We write $\kappa(y_j)=\|s(x_j,t_j)-\nabla p_{t_j}(x_{j}|x_{0,i_j})\|^2$.
    Then, we have that
    \begin{align}\label{eq:Appendix-Generalization-Particle-1-3}
&  \mathbb{E}_{Y}\left|  \frac{1}{M}\sum_{j=1}^M \kappa_1(y_j)
-
\frac{1}{M}\sum_{j=1}^M \kappa_2(y_j)
-
    \frac1n\sum_{i=1}^n\ell_1(x_i)-\frac1n\sum_{i=1}^n\ell_2(x_i) \right|
 \\  & =
  \mathbb{E}_{Y} \left|   \frac{1}{M}\sum_{j=1}^M (\kappa_1(y_j)-\kappa_2(y_j))
-
\mathbb{E}_{Y'}\left[\frac{1}{M}\sum_{j=1}^M (\kappa_1(y_{j}')-\kappa_2(y_{j}'))\right]\right|
   \\ & 
    \leq \mathbb{E}_{Y,Y'}\left|\frac{1}{M}\sum_{j=1}^M ((\kappa_1(y_{j})-\kappa_2(y_{j}))-(\kappa_1(y_{j}')-\kappa_2(y_{j}')))\right|.
   \label{eq:Appendix-Generalization-Particle-1-1}
\end{align}
    Next, we let $C_s$ be the minimum integer that satisfies
      $C_s \geq \sup_{s\in \mathcal{C}}\sup_{x,t}\|s(x,t)\|$,  and 
    for $i=1,2,\cdots$, we define $\mathcal{E}_i$ as an event where $C_s+\frac{\sqrt{d}(i-1)}{\sigma_{\underline{T}}} \leq \sup_{s\in \mathcal{C}}\max_{j}\max\{\|s(x_j,t_j)-\nabla p_{t_j}(x_{j}|x_{0,i_j})\|,\|s(x_j',t_j')-\nabla p_{t_j'}(x_{j}'|x_{0,i_j'})\| \}<C_s+\frac{\sqrt{d}i}{\sigma_{\underline{T}}}$ holds. 
    For $i=0$, we define $\mathcal{E}_0$ as an event where $ \sup_{s\in \mathcal{S}}\max_{j}\max\{\|s(x_j,t_j)-\nabla p_{t_j}(x_{j}|x_{0,i_j})\|,\|s(x_j',t_j')-\nabla p_{t_j'}(x_{j}'|x_{0,i_j'})\| \}<C_s$ holds.
    We let
      $  a_i = \mathbb{P}\left[\mathcal{E}_i\right]$ and $\mathbb{E}_i$ be the expectation conditioned by the event $\mathcal{E}_i$.
      Then, \eqref{eq:Appendix-Generalization-Particle-1-1} is bounded by
      \begin{align}\label{eq:Appendix-Generalization-Particle-1-2}
      \mathbb{E}_{0}\left|\frac{1}{M}\sum_{j=1}^M ((\kappa_1(y_{j})-\kappa_2(y_{j}))-(\kappa_1(y_{j}')-\kappa_2(y_{j}')))\right|
   +  \sum_{i=1}^\infty a_i \mathbb{E}_{i}\left|   \frac{1}{M}\sum_{j=1}^M ((\kappa_1(y_{j})-\kappa_2(y_{j}))-(\kappa_1(y_{j}')-\kappa_2(y_{j}')))\right|.
      \end{align}
      We remark that $\frac{1}{M}\sum_{j=1}^M ((\kappa_1(y_{j})-\kappa_2(y_{j}))-(\kappa_1(y_{j}')-\kappa_2(y_{j}')))$ is bounded by $8C_s^2 + \frac{8di^2}{\sigma_t^2}$ for each $\mathbb{E}_i$.
      Here, $\kappa_1$ is the minimizer of $\frac{1}{M}\sum_{j=1}^M \kappa(y_{j})$ and $\kappa_2$ is the minimizer of $\mathbb{E}\left[\kappa(y)\right]$.
      Moreover, because $\|(x_j-x_{0,i_j})/\sigma_t\|=\|\nabla p_{t_j}(x_j|x_{0,i_j})\|\leq \|s(x_j,t_j)-\nabla p_{t_j}(x_j|x_{0,i_j})\|+\|s(x_j,t_j)\|$, 
      we have that $\|s(x_j,t_j)-\nabla p_{t_j}(x_j|x_{0,i_j})\|\leq C_s + \frac{\sqrt{d}i}{\sigma_{\underline{T}}}$ implies $\|x_j\| \leq 2C_s+\sqrt{d}i$.
      We apply the same argument as that in \cref{Theorem:ExtendedSchmidt} to obtain that
      \begin{align}
      &    \mathbb{E}_{i}\left|  \frac{1}{M}\sum_{j=1}^M \kappa_1(y_j)
-
\frac{1}{M}\sum_{j=1}^M \kappa_2(y_j)
-
    \frac1n\sum_{i=1}^n\ell_1(x_i)-\frac1n\sum_{i=1}^n\ell_2(x_i) \right| 
  \\ &  \lesssim 
    \frac{C_s^2+\sigma_{\underline{T}}^{-2}i^2}{M}\log \mathcal{N}(\mathcal{S},L^\infty([-(2C_s+\sqrt{d}i),2C_s+\sqrt{d}i]^{d+1}),\delta/(C_s + i\sigma_{\underline{T}}^{-1}))+\delta.
    \\ & \lesssim
    \frac{C_s^2+\sigma_{\underline{T}}^{-2}i^2}{M} 2SL\log(\delta^{-1}L\|W\|_\infty(B\lor 1)(C_s+i))+\delta.
      \end{align}
      We remark that, $y_j$ and $y_{j}'$ are not independent, when conditioned by $\mathcal{E}_i$. However, the similar argument still holds in \eqref{eq:IndependenceYueni}, where we used the independentness of $x_i$ and $x_i'$ in the original proof, because the symmetry of $y_j$ and $y_{j}'$ is not collapsed by taking the conditional expectation.
      Based on this, and $a_i\leq 2\exp(-(i-1)^2/2)\ (i\geq 1)$ due to \cref{Lemma:Appendix-Generalization-SubGaussian},  we evaluate \eqref{eq:Appendix-Generalization-Particle-1-2}
       as
       \begin{align}
        &   \mathrm{\eqref{eq:Appendix-Generalization-Particle-1-2}}
      \\    & \lesssim 
           \frac{C_s^2+\sigma_{\underline{T}}^{-2}}{M} SL\log(\delta^{-1}L\|W\|_\infty(B\lor 1)(C_s))+\delta
           +
           \sum_{i=1}^\infty a_i \left[\frac{C_s^2+\sigma_{\underline{T}}^{-2}i^2}{M} SL\log(\delta^{-1}L\|W\|_\infty(B\lor 1)(C_s+i))+\delta\right]
         \\  &
           \lesssim
            \frac{C_s^2+\sigma_{\underline{T}}^{-2}}{M} SL\log(\delta^{-1}L\|W\|_\infty(B\lor 1)(C_s))+\delta
       \\ &   +  \sum_{i=1}^\infty \exp\left(-\frac{(i-1)^2}{2}\right)
            \left[\frac{C_s^2+\sigma_{\underline{T}}^{-2}i^2}{M} 2SL\log(\delta^{-1}L\|W\|_\infty(B\lor 1)(C_s+i))+\delta\right]
            \\  & \lesssim
             \frac{C_s^2+\sigma_{\underline{T}}^{-2}}{M} SL\log(\delta^{-1}L\|W\|_\infty(B\lor 1)(C_s))+\delta.
       \end{align}
This bounds \eqref{eq:Appendix-Generalization-Particle-1-3}. Thus, we finally obtain that
    \begin{align}
     &    \mathbb{E}_{\{y_i\}_{i=1}^n} \left[ 
        \frac1n\sum_{i=1}^n\ell_1(x_i)-\frac1n\sum_{i=1}^n\ell_2(x_i) \right] 
    \\   & \leq
        \mathbb{E}_{\{y_i\}_{j=1}^M}\left[  \frac{1}{M}\sum_{j=1}^M \kappa_1(y_j)
    -
  \sum_{j=1}^M \kappa_2(y_j) \right] +\frac{C_s^2+\sigma_{\underline{T}}^{-2}}{M}SL\log(\delta^{-1}L\|W\|_\infty(B\lor 1)(C_s))+\delta
    \\ & \leq \frac{C_s^2+\sigma_{\underline{T}}^{-2}}{M} SL\log(\delta^{-1}L\|W\|_\infty(B\lor 1)(C_s))+\delta,
    \end{align}
    because $\kappa_1$ is the minimizer of $\frac{1}{M}\sum_{j=1}^M \kappa(y_j)$. Now, we obtain the assertion.
\end{proof}
\begin{remark}
    When $\|s(x,t)\|=\sqrt{\log N}/\sigma_t$ holds, $\underline{T}=\mathrm{poly}(N^{-1}), \overline{T}=\Ord(\log N)$, we have $\sup_{(x,t)}\|s(x,t)\|=C_s\lesssim \sqrt{\underline{T}^{-1}\log N}$.
    we set $N=n^\frac{d}{2s+d}$, $\delta = n^{-\frac{2s}{d+2s}}$ and use the network class in \cref{theorem:Approximation} to obtain that
    \begin{align}
         \mathbb{E}_{(i_j,t_j,x_j)}\left[\frac1n\sum_{i=1}^n \ell_1(x_i)\right]-\int_{\ell_s\colon s\in \mathcal{S}}\frac1n\sum_{i=1}^n \ell_s(x_i)
     &   \lesssim 
        \frac{C_s^2+\sigma_{\underline{T}}^{-2}}{M} 2SL\log(\delta^{-1}L\|W\|_\infty(B\lor 1)(C_s))+\delta
      \\ &  \lesssim \frac{\underline{T}^{-1}\log n+\underline{T}^{-1}}{M}n^{-\frac{d}{2s+d}}\log^{16}n \lesssim \frac{n^{-\frac{d}{2s+d}}\log^{17}n}{\underline{T}M}.
    \end{align}
\end{remark}
Next, we show the proof for the second approach. 
\begin{theorem}\label{theorem:ApproxViaSample-2}
    We sample $t_j$ from $\mu(t) \propto \frac{\mathbbm{1}[\underline{T}\leq t\leq \overline{T}]}{t}$ and modify $\lambda(t)$ as $\lambda(t)=\frac{t\log\overline{T}/\underline{T} }{\overline{T}-\underline{T}}$, while $i_j,x_j$ are sampled as $i_j \sim \rm{Unif}(\{1,2,\cdots,n\})$ and $x_{j}\sim p_{t_j}(x_j|x_{0,i})$.
    Then, the minimizer $s_1$ over $\mathcal{S}\subseteq \Phi(L,W,S,B)$ of 
    \begin{align}
         \frac{1}{M}\sum_{j=1}^M\lambda(t_j) \|s(x_j,t_j)-\nabla p_{t_j}(x_{j}|x_{0,i})\|^2
    \end{align}
    satisfies 
 \begin{align}
        \mathbb{E}_{(i_j,t_j,x_j)}\left[\frac1n\sum_{i=1}^n \ell_1(x_i)\right]-\int_{\ell_s\colon s\in \mathcal{S}}\frac1n\sum_{i=1}^n \ell_s(x_i)
        \lesssim  \frac{C_s^2+\overline{T}}{M}SL\log(\delta^{-1}L\|W\|_\infty(B\lor 1)(C_s))+\delta,
    \end{align}
    Here, $C_s=\sup_{t,x} \sqrt{\lambda(t)}\|s(x,t)\|$.
\end{theorem}
\begin{proof}
    We just replace $\|s(x_j,t_j)-\nabla p_{t_j}(x_{j}|x_{0,i})\|$ by $\sqrt{\lambda(t_j)}\|s(x_j,t_j)-\nabla p_{t_j}(x_{j}|x_{0,i})\|$ in the previous lemma.
    Similarly to \cref{Lemma:Appendix-Generalization-SubGaussian}, we have that, for all $t>0$, 
    \begin{align}
       \mathbb{P}\left[ \lambda^\frac12(t_j)\|s(x_j,t_j)-\nabla p_{t_j}(x_j|x_{0,i_j})\|\geq \sup_{(x,t)}\lambda^\frac12(t)\|s(x,t)\| + \frac{\sqrt{d}\lambda^\frac12(t_j)t}{\sigma_{t_j}}
       \right]\leq 2 \exp\left(-t^2/2\right)
       .
    \end{align}
    Then, we replace $\sup_{(x,t)}\|s(x,t)\|$ by $\sup_{(x,t)}\lambda^\frac12(t)\|s(x,t)\|$, and $\frac{\sqrt{d}}{\sigma_{\underline{T}}}$ by $\sup_t\frac{\sqrt{d}\lambda^\frac12(t)}{\sigma_{t}}$, respectively, to obtain that
    \begin{align}
&\mathbb{E}_{i_j,t_j,x_j}\mathbb{E}_{i_j',t_j',x_j'}\hspace{-.8mm}\left[\lambda(t_j)\|s_1(x_j,t_j)-\nabla p_{t_j}(x_{j}|x_{0,i_j})\|^2\right]  -
\inf_{s\in \mathcal{S}}\mathbb{E}_{i_j,t_j,x_j}\hspace{-.8mm}\left[\lambda(t_j)\|s(x_j,t_j)-\nabla p_{t_j}(x_{j}|x_{0,i_j})\|^2\right]
\\ & \lesssim \frac{C_s^2+\overline{T}}{M}SL\log(\delta^{-1}L\|W\|_\infty(B\lor 1)(C_s))+\delta,
\label{eq:Appendix-Genelization-Sample-2-1}
    \end{align}
    where $(i_j',t_j',x_j')$ are the independent copy of $(i_j,t_j,x_j)$.
    Note that
    \begin{align}
\mathbb{E}_{i_j,t_j,x_j}\hspace{-.8mm}\left[\lambda(t_j)\|s(x_j,t_j)-\nabla p_{t_j}(x_{j}|x_{0,i_j})\|^2\right]
=
\frac1n\sum_{i=1}^n \ell(x_i)
\label{eq:Appendix-Genelization-Sample-2-2}
    \end{align}
    for all (fixed) $s$.
    \eqref{eq:Appendix-Genelization-Sample-2-1} and \eqref{eq:Appendix-Genelization-Sample-2-2} yield that
    \begin{align}
        \mathbb{E}_{(i_j,t_j,x_j)}\left[\frac1n\sum_{i=1}^n \ell_1(x_i)\right]-\int_{\ell_s\colon s\in \mathcal{S}}\frac1n\sum_{i=1}^n \ell_s(x_i)
        \leq \frac{C_s^2+\overline{T}}{M}SL\log(\delta^{-1}L\|W\|_\infty(B\lor 1)(C_s))+\delta,
    \end{align}
    which concludes the proof.
\end{proof}
\begin{remark}
    When $\|s(x,t)\|=\sqrt{\log N}/\sigma_t$ holds, $\underline{T}=\mathrm{poly}(N^{-1}), \overline{T}=\Ord(\log N)$, we have $\sup_{(x,t)}\sqrt{\lambda(t)}\|s(x,t)\|=C_s\lesssim \sqrt{\log N}$.
    we set $N=n^\frac{d}{2s+d}$, $\delta = n^{-\frac{2s}{d+2s}}$ and use the network class in \cref{theorem:Approximation} to obtain that
    \begin{align}
         \mathbb{E}_{(i_j,t_j,x_j)}\left[\frac1n\sum_{i=1}^n \ell_1(x_i)\right]-\int_{\ell_s\colon s\in \mathcal{S}}\frac1n\sum_{i=1}^n \ell_s(x_i)
        \lesssim n^{-\frac{2s}{d+2s}}\log^{17} n.
    \end{align}
\end{remark}
\section{Estimation error analysis}\label{section:Estimation}
The following Girsanov theorem is useful when converting the error of the score matching to the estimation error.
\begin{proposition}[Girsanov's Theorem \citep{karatzas1991brownian}] \label{Proposition:Girsanov}
    Let $p_0$ be any probability distribution, and let $Z=(Z_t)_{t\in [0,T]}, Z'=(Z_t')_{t\in [0,T]}$ be two different processes satisfying
    \begin{align}
       & \mathrm{d}Z_t = b(Z_t,t)\dt + \sigma(t)\mathrm{d}B_t, \quad Z_0 \sim p_0,
        \\ &
        \mathrm{d}Z_t' = b'(Z_t',t)\dt + \sigma(t)\mathrm{d}B_t, \quad Z_0' \sim p_0.
    \end{align}
    We define the distributions of $Z_t$ and $Z_t'$ as $p_t$ and $p_t'$, and the path measures of $Z$ and $Z'$ as $\mathbb{P}$ and $\mathbb{P}'$, respectively.

    Suppose the following Novikov’s condition:
    \begin{align}\label{eq:Appendix-Estimation-Girsanov-1}
        \mathbb{E}_{\mathbb{P}}\left[\exp\left(\int_0^T \frac12 \int_x \sigma^{-2}(t)\|(b-b')(x,t)\|^2 \dx\dt \right)\right] < \infty.
    \end{align}
    Then, 
    the Radon-Nikodym derivative of $\mathbb{P}$ with respect to $\mathbb{P}'$ is
    \begin{align}
        \frac{\mathrm{d}\mathbb{P}}{\mathrm{d}\mathbb{P}'}(Z)
        =\exp\left\{- \frac12 \int_0^T \sigma(t)^{-2}\|(b-b')(Z_t,t)\|^2 \dt -  \int_0^T \sigma(t)^{-1}(b-b')(Z_t,t)\mathrm{d}B_t\right\}
        ,
    \end{align}
    and therefore we have that
    \begin{align}
       \mathrm{ KL}(p_T|p_T') \leq    \mathrm{ KL}(\mathbb{P}|\mathbb{P}')= \int_0^T \frac12 \int_x p_t(x) \sigma(t)^{-2} \|(b-b')(x,t)\|^2 \dx\dt 
        .
    \end{align}
    Moreover, \citet{chen2022sampling} showed that if $\int_x p_t(x) \sigma^{-2}(t) \|(b-b')(x,t)\|^2 \dx \leq C$ holds for some consant $C$ over all $t$, we have that
    \begin{align}
               \mathrm{ KL}(p_T|p_T') \leq \int_0^T \frac12 \int_x p_t(x) \sigma(t)^2 \|(b-b')(x,t)\|^2 \dx\dt ,
    \end{align}
    even if the Novikov’s condition \eqref{eq:Appendix-Estimation-Girsanov-1} is not satisfied.
\end{proposition}
\subsection{Estimation bounds in the $\rm TV$ distance}
We show the upper and lower estimation rates in the total variation distance in this subsection.
Let $\bar{Y}$ be $\hat{Y}$ with replacing $\hat{Y}_0 \sim \mathcal{N}(0,I_d)$ by $\bar{Y}_0 \sim p_t$.
First notice that
\begin{align}
\mathbb{E}[{\rm TV(X_0, \hat{Y}_{\overline{T}-\underline{T}})}] & \lesssim \mathbb{E}[{\rm TV(Y_{\overline{T}}, Y_{\overline{T}-\underline{T}})}] 
+\mathbb{E}[{\rm TV}(\bar{Y}_{\overline{T}-\underline{T}},\hat{Y}_{\overline{T}-\underline{T}})]
+
   \mathbb{E}[{\rm TV}(\bar{Y}_{\overline{T}-\underline{T}},Y_{\overline{T}-\underline{T}}) ]  
\\ & \lesssim {\rm TV(X_0, X_{\underline{T}})} +
  \mathbb{E}[{\rm TV}(X_{\overline{T}},\hat{Y}_0) ]
+ 
   \mathbb{E}[{\rm TV}(\bar{Y}_{\overline{T}-\underline{T}},Y_{\overline{T}-\underline{T}}) ] 
\\ & = {\rm TV(X_0, X_{\underline{T}})}  +
   \mathbb{E}[{\rm TV}(X_{\overline{T}},\mathcal{N}(0,I_d)) ]
   +
     \mathbb{E}[{\rm TV}(\bar{Y}_{\overline{T}-\underline{T}},Y_{\overline{T}-\underline{T}}) ]
   \label{eq:AppendixTVDecomposition}
\end{align}
Here, $\mathbb{E}[{\rm TV(Y_{\overline{T}}, Y_{\overline{T}-\underline{T}})}] ={\rm TV(X_0, X_{\underline{T}})} +
  \mathbb{E}[{\rm TV}(X_{\overline{T}},\hat{Y}_0) ]$ follows from the correspondence between the forward and backward processes, and $\mathbb{E}[{\rm TV}(\bar{Y}_{\overline{T}-\underline{T}},\hat{Y}_{\overline{T}-\underline{T}})]\leq \mathbb{E}[{\rm TV}(X_{\overline{T}},\hat{Y}_0) ]$ follows from the definitions of $\hat{Y}$ and $\bar{Y}$ (the only difference is the initial distribution.).
We then bound the three terms in \eqref{eq:AppendixTVDecomposition} in a row.
We begin with the first term.
\begin{theorem}\label{lemma:TotalVariation-Init}
    We have that
    \begin{align}
        {\rm TV(X_0, X_{\underline{T}})}\lesssim
     \sqrt{\underline{T}}n^{\Ord(1)} \lesssim  n^{-s/(d+2s)} 
    \end{align}
    for $\underline{T}=n^{-\Ord(1)}$. 
\end{theorem}
\begin{proof}
    We need to evaluate $\|p_0-p_{\underline{T}}\|_{L_1}$.
    Remember that $p_0$ is decomposed as
     \begin{align}
        f_N(x) = \sum_{i=1}^{N} \alpha_{i} \mathbbm{1}[\|x\|_\infty \leq 1] M_{k_i,j_i}^d(x)
    \end{align}
    in \cref{Lemma:SuzukiBesov}, where $\|k\|_\infty \leq K^* = (\Ord(1)+\log N) \nu^{-1} + \Ord(d^{-1}\log N)$ for $\delta = d(1/p - 1)_+$ and $\nu = (2s-\delta)/(2\delta)$, 
    and $\|p_0-f_N\|_{L^1([-1,1]^d)} \lesssim N^{-s/d}\simeq n^{-s/(2s+d)}$ holds.
    Because we take $N=n^{d/(2s+d)}=n^{\Ord(1)}$, we can say that
    each $M_{k_i,j_i}^d(x)$ is $n^{\Ord(1)}$-Lipschitz.
    Moreover,  $|\alpha_{i}| \lesssim N^{(\nu^{-1} + d^{-1})(d/p - s)} = n^{\Ord(1)}$.
    Therefore, $f_N$ is $n^{\Ord(1)}$-Lipschitz.
    
    We decompose $p_0$ as $p_0 = f_N + (p_0-f_N)$ using the above $f_N$.
    Then we have that
    \begin{align}
      &  \left| p_{\underline{T}}(x)
        -
      \int \frac{f_N(y)}{\sigma_{\underline{T}}^d(2\pi)^\frac{d}2} \exp\left(-\frac{\|x-m_{\underline{T}}y\|^2}{2\sigma_{\underline{T}}^2}\right)\dy
        \right| \label{eq:Appendix-Estimation-2-1}
    \\  &  =
         \left| 
      \int \frac{(p_0(y)-f_N(y))}{\sigma_{\underline{T}}^d(2\pi)^\frac{d}2} \exp\left(-\frac{\|x-m_{\underline{T}}y\|^2}{2\sigma_{\underline{T}}^2}\right)\dy
        \right|
        \\
        &\leq  \int \frac{\left| p_0(y)-f_N(y)    \right|}{\sigma_{\underline{T}}^d(2\pi)^\frac{d}2} \exp\left(-\frac{\|x-m_{\underline{T}}y\|^2}{2\sigma_{\underline{T}}^2}\right)\dy.
    \end{align}
    Integrating this over all $x$ yields that
    \begin{align}
        \int  \left| p_{\underline{T}}(x)
        -
      \int \frac{f_N(y)}{\sigma_{\underline{T}}^d(2\pi)^\frac{d}2} \exp\left(-\frac{\|x-m_{\underline{T}}y\|^2}{2\sigma_{\underline{T}}^2}\right)\dy
        \right| \dx 
     &   \leq
       \int \int \frac{\left| p_0(y)-f_N(y)    \right|}{\sigma_{\underline{T}}^d(2\pi)^\frac{d}2} \exp\left(-\frac{\|x-m_{\underline{T}}y\|^2}{2\sigma_{\underline{T}}^2}\right)\dy\dx
       \\ & 
       =  \int \left| p_0(y)-f_N(y)    \right| \int\frac{1}{\sigma_{\underline{T}}^d(2\pi)^\frac{d}2} \exp\left(-\frac{\|x-m_{\underline{T}}y\|^2}{2\sigma_{\underline{T}}^2}\right)\dx \dy
       \\ & \leq
        \int \left| p_0(y)-f_N(y)  \right|\dy
        =\|p_0-f_N\|_{L^1([-1,1]^d)}.
    \end{align}
    Thus, $\|p_0-p_{\underline{T}}\|_{L_1}$ is upper bounded by 
    \begin{align}\label{eq:EstimationL1Oozappa-4}
        \|p_0-f_N\|_{L^1([-1,1]^d)} + 
     \underbrace{\int \left|  f_N(x)-  \int \frac{f_N(y)}{\sigma_{\underline{T}}^d(2\pi)^\frac{d}2} \exp\left(-\frac{\|x-m_{\underline{T}}y\|^2}{2\sigma_{\underline{T}}^2}\right)\dy
        \right| \dx}_{\text{if $f_N$ is replaced by $p_0$, this is equal to $\|p_0-p_t\|_{L_1}$}} +  \underbrace{\|p_0-f_N\|_{L^1([-1,1]^d)}}_{\eqref{eq:Appendix-Estimation-2-1}}.
    \end{align}
    Because $\|p_0-f_N\|_{L^1([-1,1]^d)}$ is bounded by $n^{-s/(2s+d)}$, we focus on the second term.

    Note that at each $x$,
    \begin{align}\label{eq:EstimationL1Oozappa-1}
       \left|
       \int \frac{f_N(y)}{\sigma_{\underline{T}}^d(2\pi)^\frac{d}2} \exp\left(-\frac{\|x-m_{\underline{T}}y\|^2}{2\sigma_{\underline{T}}^2}\right)\dy
       -
       \int_{A^x} \frac{f_N(y)}{\sigma_{\underline{T}}^d(2\pi)^\frac{d}2} \exp\left(-\frac{\|x-m_{\underline{T}}y\|^2}{2\sigma_{\underline{T}}^2}\right)\dy
       \right|\lesssim n^{-s/(d+2s)},
    \end{align}
    where $A^x = \prod_{i=1}^d a^x_i $ with $ a^x_i =  [\frac{x_i}{m_{\underline{T}}} - \frac{\sigma_{\underline{T}}\Ord(1)}{m_{\underline{T}}}\sqrt{\log n}, \frac{x_i}{m_{\underline{T}}} + \frac{\sigma_{\underline{T}}\Ord(1)}{m_{\underline{T}}}\sqrt{\log n}]$, according to \cref{Lemma:ClipInt}.
    Because $\sigma_{\underline{T}} = \Ord(\sqrt{\underline{T}})$ and $m_{\underline{T}}=\Ord(1)$ for sufficiently small ${\underline{T}}$, the value of $p_{\underline{T}}(x)$ is almost determined by the value from points that is only $\Ord(\sqrt{\underline{T}\log n})$ away from $x$.
    Because of the Lipschitzness of $p_{0}$, for each $x\in [-m_{\underline{T}} - \Ord(\sqrt{\underline{T}\log n}),m_{\underline{T}}+\Ord(\sqrt{\underline{T}\log n})]^d$, 
    \begin{align}\label{eq:EstimationL1Oozappa-2}
    \left|
\int_{A^x} \frac{f_N(y)}{\sigma_{\underline{T}}^d(2\pi)^\frac{d}2} \exp\left(-\frac{\|x-m_{\underline{T}}y\|^2}{2\sigma_{\underline{T}}^2}\right)\dy
-
\int_{A^x} \frac{f_N(x)}{\sigma_{\underline{T}}^d(2\pi)^\frac{d}2} \exp\left(-\frac{\|x-m_{\underline{T}}y\|^2}{2\sigma_{\underline{T}}^2}\right)\dy
    \right|
        \leq n^{\Ord(1)} \cdot \sqrt{\underline{T}\log n}.
    \end{align}
    where we used the Lipshitzness of $f_N$.
    By taking $\underline{T}$ polynomially small w.r.t. $n$, we have that $\eqref{eq:EstimationL1Oozappa-2}\lesssim n^{-s/(d+2s)}$.
    Moreover, 
    \begin{align}
   & \left|
    \int_{A^x} \frac{f_N(x)}{\sigma_{\underline{T}}^d(2\pi)^\frac{d}2} \exp\left(-\frac{\|x-m_{\underline{T}}y\|^2}{2\sigma_{\underline{T}}^2}\right)\dy
    - f_N(x)
    \right|
  \\ &  =
    \left|
\int_{A^x} \frac{f_N(x)}{\sigma_{\underline{T}}^d(2\pi)^\frac{d}2} \exp\left(-\frac{\|x-m_{\underline{T}}y\|^2}{2\sigma_{\underline{T}}^2}\right)\dy
-
\int \frac{f_N(x)}{\sigma_{\underline{T}}^d(2\pi)^\frac{d}2} \exp\left(-\frac{\|x-m_{\underline{T}}y\|^2}{2\sigma_{\underline{T}}^2}\right)\dy
    \right| 
      \lesssim n^{-s/(d+2s)},
      \label{eq:EstimationL1Oozappa-3}
    \end{align}
    again with \cref{Lemma:ClipInt}.
    
    Therefore, combining  
    \eqref{eq:EstimationL1Oozappa-4}, \eqref{eq:EstimationL1Oozappa-1}, \eqref{eq:EstimationL1Oozappa-2}, and \eqref{eq:EstimationL1Oozappa-3},
    we obtain that
    \begin{align}
        \|p_0-p_{\underline{T}}\|_{L_1} \lesssim \sqrt{\underline{T}}n^{\Ord(1)}  \lesssim n^{-s/(d+2s)} .
    \end{align}
    for $\underline{T}=n^{-\Ord(1)}$. 
\end{proof}
We next consider the second term.
\begin{lemma}\label{lemma:TVBound-NoiseExp}
    We can bound ${\rm TV}(X_{\overline{T}},\mathcal{N}(0,I_d)) $ as follows.
    \begin{align}
  {\rm TV}(X_{\overline{T}},\mathcal{N}(0,I_d))  \lesssim \exp(-\betalow \overline{T}).
    \end{align}  
\end{lemma}
\begin{proof}
    Exponential convergence of the Ornstein–Ulhenbeck process \citep{bakry2014analysis} yields that
    \begin{align}
    {\rm TV}(X_{\overline{T}},\mathcal{N}(0,I_d)) \lesssim \sqrt{{\rm KL}(p_{\overline{T}}\| \mathcal{N}(0,I_d)) } \leq \exp(-\betalow \overline{T}) \sqrt{{\rm KL}(p_0\| \mathcal{N}(0,I_d)) }\lesssim \exp(-\betalow \overline{T}) ,
    \end{align}
    because $C_f^{-1}\leq p_0\leq C_f$ holds and the density of $\mathcal{N}(0,I_d)$ is lower bounded by $\gtrsim 1$ in ${\rm supp}(p_0)=[-1,1]^d$, which means that $ {\rm KL}(p_0\| \mathcal{N}(0,I_d)) =\Ord(1)$.
\end{proof}
The third term $\mathbb{E}[{\rm TV}(\bar{Y}_{\overline{T}-\underline{T}},Y_{\overline{T}-\underline{T}}) ] $ in \eqref{eq:AppendixTVDecomposition} is bounded by Girsanov's theorem \cref{Proposition:Girsanov} and \eqref{eq:Main-Generalization} from \cref{section:Generalization}:
\begin{align}
  \mathbb{E}_{\{x_{0,i}\}_{i=1}^n}  {\rm TV}(\bar{Y}_{\overline{T}-\underline{T}},Y_{\overline{T}-\underline{T}}) & \lesssim 
      \mathbb{E}_{\{x_{0,i}\}_{i=1}^n}   \sqrt{\int_{t=\underline{T}}^{\overline{T}}p_t(x)\beta_t^{-2}\|\hat{s}(x,t)-\nabla \log p_t(x)\|^2\dx\dt}
       \\ & \lesssim \sqrt{ \mathbb{E}_{\{x_{0,i}\}_{i=1}^n} \int_{t=\underline{T}}^{\overline{T}}p_t(x)\beta_t^{-2}\|\hat{s}(x,t)-\nabla \log p_t(x)\|^2\dx\dt}
          \\ & \lesssim \sqrt{  n^{-\frac{2s}{d+2s}}\log^{18}n}
          \\ & \lesssim  n^{-\frac{s}{d+2s}}\log^{9}n.
\end{align}

Therefore, all three terms in \eqref{eq:AppendixTVDecomposition} are bounded as above and \cref{theorem:GeneralizationL1} follows.
We also show the lower bound as follows.
\begin{proposition}\label{proposition:LowerL1}
    Assume that $0<p,q\leq \infty$, $s>0$, and 
    \begin{align}
      s>\left\{d\left(\frac1p - \frac12\right), d\left(\frac1p -1\right), 0\right\}
    \end{align}
    holds. Then, we have that
    \begin{align}
        \inf_{\hat{\mu}} \sup_{p\in B_{p,q}^s([-1,1]^d)} \mathbb{E}[{\rm TV}(\hat{\mu},p)] \gtrsim n^{-s/(d+2s)},
    \end{align}
    where the expectation is with respect to the sample, and the infimum is taken over all estimators based on $n$ observations.
\end{proposition}
\begin{proof}
    Theorem 10 of \citet{triebel2011entropy} showed that, for a bounded domain $\Omega \subset \R^d$,
    \begin{align}\label{eq:TriebelCoveringNumber}
        \log N (U(B_{p,q}^s(\Omega)), \|\cdot\|_r, \eps) \simeq \eps^{-d/s},
    \end{align}
    for $0<p,q\leq \infty, 1\leq r<\infty$, and $s>0$ that satisfy
    \begin{align}
        s > \max\left\{d\left(\frac1p - \frac1r\right), d\left(\frac1p -1\right), 0\right\}.
    \end{align}
    Although they considered all Besov functions that does not satisfy $\int f\mathrm{d}\mu = 1$, we can check by following their proof that bounding the functions does not harm the order of the entropy number. 
    Now we use Theorem 4 of \citet{yang1999information}.
    Note that the equivalence of the covering number and the entropy holds because $\|\cdot\|_r$ is a distance, and therefore \eqref{eq:TriebelCoveringNumber} is transferred to the entropy.
    The condition 2 of the theorem is checked directly from \eqref{eq:TriebelCoveringNumber}.
    Moreover, the condition 3 holds if we take $f_*(x) = 1/2^d\ (x\in [-1,1]^d), 0\ (\text{otherwise})$ for all $\alpha \in (0,1)$.
    Finally, if $s>\left\{d(\frac1p - \frac12), d(\frac1p -1), 0\right\}$, $ \log N (U(B_{p,q}^s(\Omega)), \|\cdot\|_2, \eps)\simeq \log N (U(B_{p,q}^s(\Omega)), \|\cdot\|_1, \eps) $ holds.
    Therefore, Theorem 4 (i) of \citet{yang1999information} is applied, and we get
    \begin{align}
        \min_{\hat{\mu}}\max_{p\in B_{p,q}^s}\mathbb{E}[\|\hat{\mu}-p\|_1] \simeq \eps_n,
    \end{align}
    where $\eps_n$ is chosen as $\log N (U(B_{p,q}^s(\Omega)), \|\cdot\|_r, \eps_n) = n\eps_n^2$ holds.
    Together with \eqref{eq:TriebelCoveringNumber}, we obtain the assertion.
\end{proof}

\subsection{Estimation rate in the $W_1$ distance}\label{subsection:Appendix-Generalization-W1}
Similarly to \eqref{eq:AppendixTVDecomposition}, we have the following decomposition:
\begin{align}
     \hspace{-1mm}\mathbb{E}[{W_1(X_0, \hat{Y}_{\overline{T}-\underline{T}})}] &\leq
     \mathbb{E}[{W_1(Y_{\overline{T}}, Y_{\overline{T}-\underline{T}})}] 
+\mathbb{E}[W_1(\bar{Y}_{\overline{T}-\underline{T}},\hat{Y}_{\overline{T}-\underline{T}})]
+
   \mathbb{E}[W_1(\bar{Y}_{\overline{T}-\underline{T}},Y_{\overline{T}-\underline{T}}) ]  
   \\ &
    \leq \mathbb{E}[{W_1(X_0, X_{\underline{T}})}] +
   \mathbb{E}[W_1(\bar{Y}_{\overline{T}-\underline{T}},\hat{Y}_{\overline{T}-\underline{T}}) ] + 
\mathbb{E}[{W_1}(\bar{Y}_{\overline{T}-\underline{T}},Y_{\overline{T}-\underline{T}})].
\label{eq:Appendix-Estimation-W1-1}
\end{align}
First, we bound the first term of \eqref{eq:Appendix-Estimation-W1-1}.
\begin{lemma}\label{lemma:Appendix-Estimation-W1-1}
    We can bound $W_1(X_0, X_{\underline{T}})$ as follows.
    \begin{align}
      W_1(X_0, X_{\underline{T}})\lesssim \sqrt{\underline{T}}
    \end{align}
\end{lemma}
\begin{proof}
    Let $X \sim p_0$ and $Z \sim N(0,I_d)$. Then,
    \begin{align}
   W_1(X_0, X_{\underline{T}}) & \leq \mathbb{E}[\|X-m_{T_1}X + \sigma_{T_1}Z\|]
       \leq (1-m_{\underline{T}}) \mathbb{E}[\|X\|] + \sigma_{\underline{T}}\mathbb{E}[\|Z\|]
      \\ &  \leq (1-m_{\underline{T}}) \sqrt{d} + \sigma_{\underline{T}}\sqrt{d}
        \lesssim \sqrt{\underline{T}}, 
    \end{align}
    which concludes the proof.
\end{proof}
Next, we bound the second term of \eqref{eq:Appendix-Estimation-W1-1}.
\begin{lemma}\label{lemma:Appendix-Estimation-W1-2}
    We can bound $\mathbb{E}[W_1(\bar{Y}_{\overline{T}-\underline{T}},\hat{Y}_{\overline{T}-\underline{T}}) ] $ as follows.
    \begin{align}\label{eq:Appendix-Discretization-10}
   \mathbb{E}[W_1(\bar{Y}_{\overline{T}-\underline{T}},\hat{Y}_{\overline{T}-\underline{T}}) ] \lesssim {\rm TV}(X_{\overline{T}},\hat{Y}_{0})
  \lesssim \exp(-\betalow\overline{T}).
    \end{align}  
\end{lemma}
\begin{proof}
    Exponential convergence of the Ornstein–Ulhenbeck process \citep{bakry2014analysis} yields that
    \begin{align}
       {\rm TV}(X_{\overline{T}},\hat{Y}_{0})= {\rm TV}(p_{\overline{T}}, \mathcal{N}(0,I_d)) \leq \sqrt{2{\rm KL}(p_{\overline{T}}\| \mathcal{N}(0,I_d)) } \leq 2\exp(-\overline{T}\betalow) \sqrt{{\rm KL}(p_0\| \mathcal{N}(0,I_d)) }\lesssim \exp(-\betalow\overline{T}) ,
    \end{align}
    because $C_f^{-1}\leq p_0\leq C_f$ holds and the density of $\mathcal{N}(0,I_d)$ is lower bounded by $\Ord(1)$ in ${\rm supp}(p_0)=[-1,1]^d$, which means $ {\rm KL}(p_0\| \mathcal{N}(0,I_d)) =\Ord(1)$.
    In addition because $\|\hat{Y}^{(k)}_{\overline{T}-\underline{T}}\|_\infty, \|\hat{Y}_{\overline{T}-\underline{T}}\|_\infty \leq 2=\Ord(1)$, and because the only difference between $\hat{Y}^{(k)}$ and $\hat{Y}$ is the initial distribution, we have $  W_1(\hat{Y}^{(k)}_{\overline{T}-\underline{T}}, \hat{Y}_{\overline{T}-\underline{T}})\lesssim {\rm TV}(X_{\overline{T}},\hat{Y}_{0})= {\rm TV}(p_{\overline{T}}, \mathcal{N}(0,I_d))$.
    Putting it all together, we obtain that
    \begin{align}
    W_1(\hat{Y}^{(k)}_{\overline{T}-\underline{T}}, \hat{Y}_{\overline{T}-\underline{T}})\lesssim {\rm TV}(X_{\overline{T}},\hat{Y}_{0})= {\rm TV}(p_{\overline{T}}, \mathcal{N}(0,I_d))\lesssim \exp(-\betalow \overline{T}),
    \end{align}
    which yields the assertion.
\end{proof}
Finally, we bound the third term of \eqref{eq:Appendix-Estimation-W1-1}.
As we saw in \cref{subsection:Main-Estimation-W1}, 
\begin{align}\label{eq:Appendix-Estimation-W1-sum}
\mathbb{E}[{W_1}(\bar{Y}_{\overline{T}-\underline{T}},Y_{\overline{T}-\underline{T}})]
\leq 
\sum_{i=1}^{K_*}\mathbb{E}[W_1(\bar{Y}_{\overline{T}-\underline{T}}^{(i-1)},\bar{Y}_{\overline{T}-\underline{T}}^{(i)})].
\end{align}
Remember the definition of a sequence of stochastic processes $\{(\hat{Y}_t^{(i)})_{t=0}^{\overline{T}-\underline{T}}\}_{i=0}^{K_*}.$
First, $\bar{Y}^{(0)}=(\bar{Y}^{(0)}_t)_{t\in [0,\overline{T}]}=Y = (Y_t)_{t\in [0,\overline{T}]}$ is defined as a process such that 
\begin{align}
    \mathrm{d}Y_t = \beta_{\overline{T} - t} (Y_t + 2\nabla \log p_t(Y_t,\overline{T}-t)) \dt + \sqrt{2\beta_{\overline{T} - t}}\mathrm{d}B_t\ (t\in [0,\overline{T}]),\quad Y^{(0)}_0 \sim p_{\overline{T}}.
\end{align}
Then, $Y_{\overline{T}-t} \sim p_t$ holds for all $t\in [0,\overline{T}]$.
Next, for $i=1,2,\cdots,K_*$, we let $\bar{Y}^{(i)} = (\bar{Y}^{(i)}_t)_{t\in [0,\overline{T}-\underline{T}]}$ to satisfy 
\begin{align}\label{eq:Estimation-Backward-2}
   \bar{Y}^{(i)}_0 \sim p_{\overline{T}},\quad
   \mathrm{d}\bar{Y}^{(i)}_t& = \beta_{\overline{T} - t} (\bar{Y}^{(i)}_t + 2\nabla\log p_t(\bar{Y}^{(i)}_t,\overline{T}-t) )\dt + \sqrt{2\beta_{\overline{T} - t}}\mathrm{d}B_t\ (t\in [0,\overline{T}-t_i])
    ,
    \\ \mathrm{d}\bar{Y}^{(i)}_t& = \beta_{\overline{T} - t} (\bar{Y}^{(i)}_t + 2\hat{s}(\bar{Y}^{(i)}_t,\overline{T}-t)) \dt + \sqrt{2\beta_{\overline{T} - t}}\mathrm{d}B_t\ (t\in [\overline{T}-t_i, \overline{T}-\underline{T}]).
\end{align}
Note that $t_0 = \underline{T}$, $t_1 = N^{-\frac{2-\delta}{d}}= n^{-\frac{2-\delta}{d+2s}},\ 1<\frac{t_{i+1}}{t_i} = \text{const.}\ \leq 2$, and $t_{K_*}=\overline{T}-\underline{T}$.
Then, $\bar{Y}^{(K_*)}=\bar{Y}$ holds.
Here $\bar{Y}^{(i)}_{\overline{T}-t} \sim p_t$ holds for all $t\in [0,\overline{T}-t_i]$, but after $t=\overline{T}-t_i$, the true score function is replaced by the estimated one.
If $\|\bar{Y}^{(i)}_{\overline{T}-\underline{T}}\|_\infty> 2$ in the original definition, we reset $\bar{Y}^{(i)}_{\overline{T}-\underline{T}}$ as $\bar{Y}^{(i)}_{\overline{T}-\underline{T}}:=0$.

Also, we introduce another stochastic process $ \bar{Y}^{(i)'}$.
We define $d+1$-dimensional set $A\subseteq \R^{d+1}$ as
\begin{align}
    A=\left\{(x,t)\in \R^d \times \R \left|\ \|x\|_\infty\leq  m_{t}+C_{{\rm a},1}\sigma_{t}\sqrt{\log (n)}, \ \underline{T}\leq t \leq \overline{T}
    \right.\right\}.
\end{align}
According to \cref{lemma:Appendix-HPB-1}, with probability at least $1-n^{-\Ord(1)}$, a path of the backward process $(Y_t)_{t=0}^{\overline{T}}$ satisfies $ (Y_t,\overline{T}-t) \in A$ for all $\underline{T}\leq t \leq \overline{T}$.
Based on this, for $i=0,1,\cdots,K_*-1$, $ \bar{Y}^{(i)'}$ is defined as
\begin{align}
   \bar{Y}^{(i)'}_0 &\sim p_{\overline{T}},
 \\  \mathrm{d}\bar{Y}^{(i)'}_t& = \beta_{\overline{T} - t} (\bar{Y}^{(i)'}_t + 2\nabla\log p_t(\bar{Y}^{(i)'}_t,\overline{T}-t) )\dt + \sqrt{2\beta_{\overline{T} - t}}\mathrm{d}B_t\ (t\in [0,\overline{T}-t_i])
    ,
       \\ \mathrm{d}\bar{Y}^{(i)'}_t& = \beta_{\overline{T} - t} \left(\bar{Y}^{(i)'}_t + 2\mathbbm{1}[(\bar{Y}^{(i)'}_s,\overline{T}-s)\notin A\text{ for some } s\leq t]\nabla \log p_t(\bar{Y}^{(i)'}_t) \right.\\&\quad  + \left.2\mathbbm{1}[(\bar{Y}^{(i)'}_s,\overline{T}-s)\in A \text{ for all } s\leq t] \hat{s}(\bar{Y}^{(i)'}_t,\overline{T}-t)\right) \dt  + \sqrt{2\beta_{\overline{T} - t}}\mathrm{d}B_t\ (t\in [\overline{T}-t_{i+1}, \overline{T}-t_{i}]),
    \\ \mathrm{d}\bar{Y}^{(i)'}_t& = \beta_{\overline{T} - t} (\bar{Y}^{(i)'}_t + 2\hat{s}(\bar{Y}^{(i)'}_t,\overline{T}-t)) \dt + \sqrt{2\beta_{\overline{T} - t}}\mathrm{d}B_t\ (t\in [\overline{T}-t_{i}, \overline{T}-\underline{T}]).
\end{align}
\begin{lemma}\label{lemma:W1BoundCrucial}
    Suppose that $\|\hat{s}(\cdot,t)\|_\infty\lesssim \frac{\log^\frac12 n}{\sqrt{t}\land 1}$ holds.
    Then, the following holds for all $i=1,2,\cdots,K_*$:
        \begin{align}\label{eq:W1BoundCrucial-1}
        W_1(\bar{Y}_{\overline{T}-\underline{T}}^{(i-1)},\bar{Y}_{\overline{T}-\underline{T}}^{(i)}) \lesssim \sqrt{t_{i}\log n}
      \sqrt{\mathbb{E}_{\{x_{0,i}\}_{i=1}^n}\left[\int_{t=t_{i-1}}^{t_{i}}\mathbb{E}_x\left[\|\hat{s}(x,t)\hspace{-1mm}-\hspace{-1mm}\nabla \log p_t(x)\|^2\dt\right]\right]}+n^{-\frac{s+1}{d+2s}}.
    \end{align}
    Therefore, we have that
    \begin{align}\label{eq:W1BoundCrucial-2}
      \mathbb{E}_{\{x_{0,i}\}_{i=1}^n}[  W_1(\bar{Y}_{\overline{T}-\underline{T}}^{(i-1)},\bar{Y}_{\overline{T}-\underline{T}}^{(i)}) ]\lesssim \sqrt{t_{i}\log n}
      \sqrt{\mathbb{E}_{\{x_{0,i}\}_{i=1}^n}\left[\int_{t=t_{i-1}}^{t_{i}}\mathbb{E}_x\left[\|\hat{s}(x,t)\hspace{-1mm}-\hspace{-1mm}\nabla \log p_t(x)\|^2\dt\right]\right]}+n^{-\frac{s+1}{d+2s}}.
    \end{align}
\end{lemma}
\begin{proof}

    We construct the transportation map between $\bar{Y}^{(i-1)}_{\overline{T}-\underline{T}}$ and $ \bar{Y}^{(i)}_{\overline{T}-\underline{T}}$.
    Our approach focuses on each path.
    
    Because the Novikov's condition is not satisfied for $\bar{Y}^{(i-1)}_{\overline{T}-\underline{T}}$ and $ \bar{Y}^{(i)}_{\overline{T}-\underline{T}}$, \cref{Proposition:Girsanov} cannot be used to consider the total variation distance between the two paths; \cref{Proposition:Girsanov} only gives ${\rm KL}(\bar{Y}^{(i-1)}_{\overline{T}-\underline{T}},\bar{Y}^{(i)}_{\overline{T}-\underline{T}})$, not ${\rm KL}(\bar{Y}^{(i-1)},\bar{Y}^{(i)}$, and this bound is insufficient for our discussion.
    Therefore, we first bound $\mathbb{E}[W_1(\bar{Y}^{(i-1)}_{\overline{T}-\underline{T}},\bar{Y}^{(i-1)'}_{\overline{T}-\underline{T}})]$.
    According to \cref{lemma:Appendix-HPB-1}, with probability at least $1-n^{-\Ord(1)}$, a path of the processes $(\bar{Y}^{(i-1)}_{t})_{t=0}^{\overline{T}}$ and $(\bar{Y}^{(i-1)'}_{t})_{t=0}^{\overline{T}}$
     satisfy $(\bar{Y}^{(i-1)}_{t},\overline{T}-t), (\bar{Y}^{(i-1)'}_{t},\overline{T}-t)\in A$ for all $0\leq t \leq \overline{T}-t_{i-1}$.
    Thus, $\mathbb{E}[{\rm TV}(\bar{Y}^{(i-1)}_{\overline{T}-\underline{T}},\bar{Y}^{(i-1)'}_{\overline{T}-\underline{T}})]$ is bounded by $n^{-\Ord(1)}$ (with a sufficiently large constant in $\Ord(1)$.).
    This implies $\mathbb{E}[W_1(\bar{Y}^{(i-1)}_{\overline{T}-\underline{T}},\bar{Y}^{(i-1)'}_{\overline{T}-\underline{T}})]\lesssim n^{-\Ord(1)}$, because $\bar{Y}^{(i-1)}_{\overline{T}-\underline{T}},\bar{Y}^{(i-1)'}_{\overline{T}-\underline{T}}= \Ord(1)\ (\text{a.s.})$. 

    We now discuss $\mathbb{E}[W_1(\bar{Y}^{(i-1)'}_{\overline{T}-\underline{T}},\bar{Y}^{(i)}_{\overline{T}-\underline{T}} )].$
    Let us write the path measures of $\bar{Y}^{(i-1)'}$ and $\bar{Y}^{(i)}$ be $\mathbb{P}$ and $\mathbb{P'}$, and take some path $p$ that is $y$ at $t=\overline{T}-\underline{T}$ and is $z$ at $t=\overline{T}-t_{i}$.
    If $\mathrm{d}\mathbb{P}[p] > \mathrm{d}\mathbb{P}'[p]$, then we move the mass of $\bar{Y}^{(i-1)'}_{\overline{T}-\underline{T}}=y$  that amounts to $\mathrm{d}\mathbb{P}[p] - \mathrm{d}\mathbb{P}'[p]$, to $z$, along the path $p$ by reversing the time until $t=\overline{T}-t_{i}$.
    Applying this to all paths $p$, then the total mass of $\bar{Y}^{(i-1)'}_{\overline{T}-\underline{T}}$ that is moved is at most
    \begin{align}
       & \frac12 \mathrm{TV}((\bar{Y}^{(i-1)'}), (\bar{Y}^{(i)})) \leq \frac12\sqrt{\int_{t=t_{i-1}}^{t_{i}} \int_x  p_t(x) \beta_t^{-2}\|\hat{s}(x,t)-\nabla \log p_t(x)\|^2 \dx \dt} 
      \label{eq:W1Bound-1-1}
      .
    \end{align}
        according to \cref{Proposition:Girsanov}. 
        Here we remark that the Novikov's condition certainly holds for this case.
  
  Until now, a part of the mass of $\hat{Y}^{(i-1)'}_{\overline{T}-\underline{T}}$ is moved along each corresponding path,
  but at this time no coupling measure has been constructed.
  To realize the coupling measure, we consider the same process for $\bar{Y}^{(i)}_{\overline{T}-\underline{T}}$.
  That is, for each path $p$ with $\bar{Y}^{(i)}_{\overline{T}-\underline{T}}=y$ and $\bar{Y}^{(i)}_{\overline{T}-t_i}=z$, if $\mathrm{d}\mathbb{P}[p] < \mathrm{d}\mathbb{P}'[p]$, then we move the mass of $\bar{Y}^{(i)}_{\overline{T}-\underline{T}}=y$, as much as $\mathrm{d}\mathbb{P}'[p] - \mathrm{d}\mathbb{P}[p]$, to $z$ along the path $p$.
    The total mass of $\bar{Y}^{(i)}_{\overline{T}-\underline{T}}$ affected is bounded by
    $\frac12 \mathrm{TV}((\bar{Y}^{(i-1)'}), (\bar{Y}^{(i)'}))$, which is bounded by \eqref{eq:W1Bound-1-1}.

    Now, we can see that, the same amount of mass is transported from both $\bar{Y}^{(i-1)'}_{\overline{T}-\underline{T}}$ and $\bar{Y}^{(i)}_{\overline{T}-\underline{T}}$ to $t=\overline{T}-t_i$.
    Thus, at each $z$, we can arbitrarily associate the mass from $\bar{Y}^{(i-1)'}_{\overline{T}-\underline{T}}$ to that from $\bar{Y}^{(i)}_{\overline{T}-\underline{T}}$.
    Using this, as much as  $\frac12 \mathrm{TV}((\bar{Y}^{(i-1)'}), (\bar{Y}^{(i)'})) $ of the mass is transported from $\bar{Y}^{(i-1)'}_{\overline{T}-\underline{T}}$  to $\bar{Y}^{(i)}_{\overline{T}-\underline{T}}$, by reversing the path to $t=\overline{T}-t_i$.

    Now our interest is how far each transport is required to move on average.
    First we consider when $t_i\lesssim 1$.
    
    First we bound $\|\bar{Y}^{(i)}_{\overline{T}-\underline{T}}-\bar{Y}^{(i)}_{\overline{T}-t_i}\|$.
    According to \cref{lemma:Appendix-HPB-1}, we have $\|\int_{\overline{T}-t_i}^{\overline{T}-\underline{T}}2\beta_{\overline{T}-t}\mathrm{d}B_t\|\lesssim \sqrt{t_i \log n}$ for all $t\in [\overline{T}-t_i,\overline{T}-\underline{T}]$, and $\bar{Y}^{(i)}_{\overline{T}-t_i} \lesssim m_{\overline{T}-t_i}+\sigma_{\overline{T}-t_i}\sqrt{\log n}\lesssim \sqrt{\log n}$  with probability $1-n^{-\Ord(1)}$.
   We  consider the event conditioned on them.
   Note that $\|s(x,t)\|\lesssim \frac{\sqrt{\log n}}{\sigma_t}\lesssim \frac{\sqrt{\log n}}{\sqrt{t}}$ holds.
    Then we have that, for all $\overline{T}-t_i \leq t \leq \overline{T}-\underline{T}$,
    \begin{align}
   \|\bar{Y}^{(i)}_t-\bar{Y}^{(i)}_{\overline{T}-t_i}\|& = \left\|\int_{\overline{T}-t_i}^{\overline{T}-\underline{T}}\beta_{\overline{T} - s} (\bar{Y}^{(i)}_s + 2\nabla\log p_t(\bar{Y}^{(i)}_s,\overline{T}-s) )\dt + \int_{\overline{T}-t_i}^{\overline{T}-\underline{T}}\sqrt{2\beta_{\overline{T} - s}}\mathrm{d}B_s\right\|
   \\ & \lesssim \betahigh \int_{\overline{T}-t_i}^{\overline{T}-\underline{T}}\|\bar{Y}^{(i)}_s\|\ds 
   +2\betahigh\int_{\overline{T}-t_i}^{\overline{T}-\underline{T}}\frac{\sqrt{\log n}}{\sqrt{s}}\ds + \sqrt{t_i \log n}, 
   \\ & \lesssim \betahigh \int_{\overline{T}-t_i}^{\overline{T}-\underline{T}}\|\bar{Y}^{(i)}_s\|\ds + \sqrt{t_i \log n} +\sqrt{t_i \log n}.
     \\ & \lesssim \int_{\overline{T}-t_i}^{\overline{T}-\underline{T}}\|\bar{Y}^{(i)}_s-\bar{Y}^{(i)}_{\overline{T}-t_i}\|\ds
     +\sqrt{t_i \log n}+t_i\|\bar{Y}^{(i)}_{\overline{T}-t_i}\|
      \\ & \lesssim    \int_{\overline{T}-t_i}^{\overline{T}-\underline{T}}\|\bar{Y}^{(i)}_s-\bar{Y}^{(i)}_{\overline{T}-t_i}\|\ds
     +\sqrt{t_i \log n}+t_i \sqrt{\log n}
    \end{align}
    Now we apply the Gronwall's inequality to obtain
    \begin{align}
         \|\bar{Y}^{(i)}_t-\bar{Y}^{(i)}_{\overline{T}-t_i}\| \lesssim e^{\betahigh t_i}\sqrt{t_i \log n}
        \lesssim \sqrt{t_i \log n}.
    \end{align}
 for all $\overline{T}-t_i \leq t \leq \overline{T}-\underline{T}$.
 Thus, with probability $1-n^{-\Ord(1)}$, $\|\bar{Y}^{(i)}_t-\bar{Y}^{(i)}_{\overline{T}-t_i}\|$ is bounded by $\sqrt{t_i \log n}$ up to a constant factor, over all $\overline{T}-t_i \leq t \leq \overline{T}-\underline{T}$.

Next we bound $\|\bar{Y}^{(i-1)'}_{\overline{T}-\underline{T}}-\bar{Y}^{(i-1)'}_{\overline{T}-t_i}\|$.
This is decomposed into 
\begin{align}
    \|\bar{Y}^{(i-1)'}_{\overline{T}-t_i}-\bar{Y}^{(i-1)'}_{\overline{T}-t_{i-1}}\|+
    \|\bar{Y}^{(i-1)'}_{\overline{T}-\underline{T}}-\bar{Y}^{(i-1)'}_{\overline{T}-t_{i-1}}\|.
\end{align}
The first term is bounded by $\lesssim \sqrt{t_i \log n}$ with probability at least $1-n^{-\Ord(1)}$.
This is because $\bar{Y}^{(i-1)'}_t \in A$ holds with probability $1-n^{-\Ord(1)}$ due to the first part of \cref{lemma:Appendix-HPB-1}, and for such paths the evolution of $\bar{Y}^{(i-1)'}_t$ is the same as that of $Y_t$, where  we apply the second part of \cref{lemma:Appendix-HPB-1}.
The second term is bounded by $\sqrt{t_{i-1} \log n}$ with probability $1-n^{-\Ord(1)}$, following the discussion on $\|\bar{Y}^{(i)}_t-\bar{Y}^{(i)}_{\overline{T}-t_i}\|$.
In summary, with probability $1-n^{-\Ord(1)}$ we can bound $\|\bar{Y}^{(i-1)'}_{\overline{T}-\underline{T}}-\bar{Y}^{(i-1)'}_{\overline{T}-t_i}\|$ by $\sqrt{t_{i-1} \log n}(\leq \sqrt{t_{i} \log n})$ up to a constant factor.

In summary, when $t_i\lesssim 1$, the transportation map moves at most $\Ord(\sqrt{t_{i} \log n})$ with probability $1-n^{-\Ord(1)}$.
Because the supports of $\bar{Y}^{(i-1)'}_{\overline{T}-\underline{T}}$ and $\bar{Y}^{(i)}_{\overline{T}-\underline{T}}$ are both bounded, for the mass moved more than $\sqrt{t_{i} \log n}$ affects the Wasserstein distance at most $n^{-\Ord(1)}$.
Therefore, we obtain the desired bound \eqref{eq:W1BoundCrucial-1} for $t_i\lesssim 1$.

For $t_i \gtrsim 1$, because the supports of $\bar{Y}^{(i-1)}_{\overline{T}-\underline{T}}$ and $\bar{Y}^{(i)}_{\overline{T}-\underline{T}}$ are both bounded, 
\begin{align}
     W_1(\bar{Y}_{\overline{T}-\underline{T}}^{(i-1)},\bar{Y}_{\overline{T}-\underline{T}}^{(i)}) \lesssim 
     {\rm TV}(\bar{Y}_{\overline{T}-\underline{T}}^{(i-1)},\bar{Y}_{\overline{T}-\underline{T}}^{(i)}) 
     \lesssim 
   \frac12\sqrt{\int_{t=t_{i-1}}^{t_{i}} \int_x  p_t(x) \beta_t^{-2}\|\hat{s}(x,t)-\nabla \log p_t(x)\|^2 \dx \dt}
\end{align}
 holds. Therefore we obtain \eqref{eq:W1BoundCrucial-1} as well.

From \eqref{eq:W1BoundCrucial-1}, \eqref{eq:W1BoundCrucial-2} is easily obtained by jensen's inequality.
\end{proof}

Also, we bound the generalization error of each network $s_i$.
\begin{lemma}\label{lemma:Appendix-Estimation-Prepare-3}
    For $1\leq i \leq K_*-1$, let $s_i$ be a network that is selected from $\Phi(L,W,S,B)$ with
    \begin{align}
        L=\Ord(\log^4 n), \ \|W\|_\infty = \Ord(n^{\frac{d}{d+2s}}), \ S=\Ord (t_i^{-d/2}n^{\frac{\delta d}{2(2s+d)}}), \ \text{ and }B=\exp(\Ord(\log^4 n)),
    \end{align}
    and $\|s_i(\cdot,t)\|_{L^\infty}\lesssim \frac{\log^\frac12 n}{\sigma_t}$.
    Then, we have that
\begin{align}
    \mathbb{E}_{\{x_{0,j}\}_{i=j}^n}\left[\int_{t=t_i}^{t_{i+1}}\mathbb{E}_x\left[\|\hat{s}_i(x,t)\hspace{-1mm}-\hspace{-1mm}\nabla \log p_t(x)\|^2\dt\right]\right]
   \lesssim n^{-\frac{2(s+1)}{d+2s}}\log n +\frac{t_i^{-d/2}n^{\frac{\delta d}{2(d+2s)}}\log^{10}n}{n}.
\end{align}
    Moreover, for $i=0$, let $s_0$ be a network that is selected from $\Phi(L,W,S,B)$ with
    \begin{align}
        L=\Ord(\log^4 n), \ \|W\|_\infty = \Ord(n^{\frac{d}{d+2s}}\log^6 n ), \ S=\Ord (n^{\frac{d}{2s+d}}\log^8 n), \ \text{ and }B=\exp(\Ord(\log^4 n)),
    \end{align}
    and $\|s_0(\cdot,t)\|_{L^\infty}\lesssim \frac{\log^\frac12 n}{\sigma_t}$.
    Then, we have that
\begin{align}
    \mathbb{E}_{\{x_{0,j}\}_{i=j}^n}\left[\int_{t=t_i}^{t_{i+1}}\mathbb{E}_x\left[\|\hat{s}_0(x,t)\hspace{-1mm}-\hspace{-1mm}\nabla \log p_t(x)\|^2\dt\right]\right]
   \lesssim n^{-\frac{2s}{d+2s}}\log^{18} n.
\end{align}
\end{lemma}
\begin{proof}
    First we consider the first part. We take $N=n^{d}{d+2s}$ and $t_*=t_i/2$ in \cref{lemma:ApproximationSmoothArea}. Note that $N$ and $t_*(\geq n^{\frac{2-\delta}{d+2s}})$ satisfies $t_*\geq N^{-(2-\delta)/d}$(, which is assumed in \cref{Lemma:Approximation-Appendix}).
    Then, there exists a neural network $\phi\in \Phi(L,W,S,B)$ that satisfies
    \begin{align}
    \int_{t=t_i}^{t_{i+1}}\int_x p_{t}(x)  \|\phi(x,t) - s(x,t)\|^2 \dx \dt \lesssim N^{-\frac{2(s+1)}{d}}\log n = N^{-\frac{2(s+1)}{d+2s}}\log n.
    \end{align}
    Specifically, $L = \Ord (\log^4 (n)),\| W\|_\infty = \Ord (n^{\frac{d}{d+2s}}),S = \Ord (t_i^{-d/2}n^{\frac{\delta d}{2(d+2s)}})$, and $ B = \exp(\Ord(\log^4 n ))$.
    Therefore, we apply \eqref{eq:TheoremExtendedSchmidt-1} by replacing $\underline{T}$ and $\overline{T}$ by $t_i$ and $t_{i+1}$, respectively, and with $\delta = n^{-\frac{2(s+1)}{d+2s}}$ to obtain the first assertion as
    \begin{align}
        \mathbb{E}_{\{x_{0,j}\}_{i=j}^n}\left[\int_{t=t_i}^{t_{i+1}}\mathbb{E}_x\left[\|\hat{s}_i(x,t)\hspace{-1mm}-\hspace{-1mm}\nabla \log p_t(x)\|^2\dt\right]\right]
     &   \lesssim N^{-\frac{2(s+1)}{d}}\log n + \frac{C_\ell}{n}\log \mathcal{N} + \delta
        \\ & \lesssim n^{-\frac{2(s+1)}{d+2s}}\log n + \frac{\log^2 n}{n}\left(t_i^{-d/2}n^{\frac{\delta d}{2(d+2s)}}\log^8\right)+ n^{-\frac{2(s+1)}{d+2s}}
\\ & \lesssim n^{-\frac{2(s+1)}{d+2s}}\log n +\frac{t_i^{-d/2}n^{\frac{\delta d}{2(d+2s)}}\log^{10}n}{n}.
    \end{align}

    For the second part, we simply follow the discussion that derived \eqref{eq:Main-Generalization}, by replacing $\overline{T}$ by $t_1(\overline{T})$, which does not increase the generalization error.
\end{proof}

\begin{proof}[Proof of  \cref{theorem:GeneralizationW1}]
We use the sequence of networks presented in \cref{lemma:Appendix-Estimation-Prepare-3}.
Specifically, we consider the following process.
\begin{align}\label{eq:Estimation-Backward-2}
   \hat{Y}^{(i)}_0 \sim \mathcal{N}(0,I),\quad
\mathrm{d}\hat{Y}^{(i)}_t& = \beta_{\overline{T} - t} (\hat{Y}^{(i)}_t + 2\hat{s}(\hat{Y}^{(i)}_t,\overline{T}-t)) \dt + \sqrt{2\beta_{\overline{T} - t}}\mathrm{d}B_t\ (t\in [\overline{T}-t_i, \overline{T}-t_{i+1}], i=0,1,\cdots,K_*),
\end{align}
and we modify $\hat{Y}^{(i)}_{\overline{T}-\underline{T}}$ to $0$ if $\|\hat{Y}^{(i)}_{\overline{T}-\underline{T}}\|_\infty >2$.

Finally, we sum up the errors for the above process.
Eq. \eqref{eq:Appendix-Estimation-W1-sum} is further bounded by
\begin{align}
 & \mathbb{E}[{W_1}(\bar{Y}_{\overline{T}-\underline{T}},Y_{\overline{T}-\underline{T}})]
\\ &\leq 
\sum_{i=1}^{K_*}\mathbb{E}[W_1(\bar{Y}_{\overline{T}-\underline{T}}^{(i-1)},\bar{Y}_{\overline{T}-\underline{T}}^{(i)})].\\
    & 
   \lesssim \sum_{i=1}^{K_*}\left[\sqrt{t_{i-1}\log n}
      \sqrt{\mathbb{E}_{\{x_{0,i}\}_{i=1}^n}\left[\int_{t=t_{i}}^{t_{i}}\mathbb{E}_x\left[\|\hat{s}(x,t)\hspace{-1mm}-\hspace{-1mm}\nabla \log p_t(x)\|^2\dt\right]\right]}+n^{-\frac{s+1}{d+2s}}\right]
      \quad (\text{by \cref{lemma:W1BoundCrucial}})
    \\ & \lesssim \sum_{i=2}^{K_*}\left[\sqrt{t_{i}\log n} \left(n^{-\frac{(s+1)}{d+2s}} \sqrt{\log n}+ \frac{t_i^{-d/4}n^{\frac{\delta d}{4(d+2s)}}\log^{5}n}{\sqrt{n}}\right)+n^{-\frac{(s+1)}{d+2s}}\right] \\ & \quad + \sqrt{t_{1}\log n}\left[n^{-\frac{s}{d+2s}}\log^{9} n+n^{-\frac{s}{d+2s}}\right] \quad (\text{by \cref{lemma:Appendix-Estimation-Prepare-3}})
    \\ & \lesssim  \left[\sqrt{t_1}n^{-\frac{s}{d+2s}}+\sqrt{t_1}\frac{t_1^{-d/4}n^{\frac{\delta d}{4(d+2s)}}}{\sqrt{n}}\right] \cdot \tilde{\Ord}(1)
    \\ & \quad (\text{because $K_*=\Ord(\log n)$ and $t_1\leq \cdots t_{K_*}=\Ord(\log N)$ with $1<t_{i+1}/t_i=\text{const.}\leq 2\ (i\geq 1)$.})
    \\ & =
    \left[(n^{-\frac{2-\delta}{d+2s}})^{\frac12}n^{-\frac{s}{d+2s}}+(n^{-\frac{2-\delta}{d+2s}})^{\frac12}\frac{(n^{-\frac{2-\delta}{d+2s}})^{-d/4}n^{\frac{\delta d}{4(d+2s)}}}{\sqrt{n}} \right] \cdot \tilde{\Ord}(1)
    \\ & \lesssim n^{-\frac{(s+1-\delta)}{d+2s}}.\label{eq:Appendix-Estimation-W1-103}
\end{align}
Therefore, by taking $\underline{T}\lesssim n^{-\frac{2(s+1)}{d+2s}}$ and $\overline{T}=\frac{(s+1)\log n}{\betalow (d+2s)}$, we obtain that
\begin{align}
    W_1(X_0,\hat{Y}_{\overline{T}-\underline{T}}) &\leq \mathbb{E}[{W_1(X_0, X_{\underline{T}})}] +
   \mathbb{E}[W_1(\bar{Y}_{\overline{T}-\underline{T}},\hat{Y}_{\overline{T}-\underline{T}}) ] + 
\mathbb{E}[{W_1}(\bar{Y}_{\overline{T}-\underline{T}},Y_{\overline{T}-\underline{T}})] \\ &\lesssim \sqrt{\underline{T}}+\exp(-\betalow\overline{T})+n^{-\frac{(s+1-\delta)}{d+2s}}\quad (\text{by \cref{lemma:Appendix-Estimation-W1-1,lemma:Appendix-Estimation-W1-2} and \eqref{eq:Appendix-Estimation-W1-103}})
\\ & \lesssim n^{-\frac{(s+1-\delta)}{d+2s}}+n^{-\frac{(s+1-\delta)}{d+2s}}+n^{-\frac{(s+1-\delta)}{d+2s}} \lesssim n^{-\frac{(s+1-\delta)}{d+2s}}
,
\quad 
\end{align}
which concludes the proof for \cref{theorem:GeneralizationW1}.
\end{proof}
\subsection{Discussion on the discretization error}\label{subsection:Appendix-Discretization}
As in \cref{subsection:Discretization}, 
$t_0=\underline{T}<t_1<\cdots<t_{K_*}=\overline{T}$ be the time steps with $t_{k+1}-t_k\equiv \eta\ll 1$.
Consider the following process $(Y^{\rm d}_t)_{t=0}^{\eta K}=(Y^{\rm d}_t)_{t=0}^{\overline{T}-\underline{T}}$ with $Y^{\rm d}_0 \sim\mathcal{N}(0,I_d)$:
    \begin{align}\label{eq:Appendix-Discretization-1}
    \mathrm{d}Y^{\rm d}_{t}= \beta_t (Y^{\rm d}_{t} + 2\hat{s}(Y^{\rm d}_{\overline{T}-t_i},\overline{T}-t_i))\dt+\sqrt{2\beta_{\overline{T}-t}}\mathrm{d}B_t
  \quad (t\in [\overline{T}-t_i,\ \overline{T}-t_{i-1}]).
    \end{align}
Here $\hat{s}$ is the score network obtained by the score matching:
\begin{align}\label{eq:Appendix-Discretization-4}
   \hat{s}\in  {\rm argmin}\ \frac1n \sum_{i=1}^n \sum_{k=1}^{K}\eta\mathbb{E}[\|s(x_{t_k},{t_k}) - \nabla \log p_{t_k}(x_{t_k}|x_{0,i})\|^2].
\end{align}
Here, each expectation is taken with respect to $x_{\overline{T}-t_k}\sim p_{\overline{T}-t_k}(x_{\overline{T}-t_k}|x_{0,i})$.
\begin{theorem}
     Let $\underline{T}=n^{-\Ord(1)}$, $\overline{T} = \frac{s\log n}{2s+d}$, and $\eta = {\rm poly}(n^{-1})$.
     Then,
\begin{align}
    \mathbb{E}[{\rm TV}(X_0, \bar{Y}_{\overline{T}-\underline{T}})] \lesssim n^{-\frac{2s}{d+2s}}\log^{18} n +\eta^2 \underline{T}^{-3}\log^3 n 
     +\eta \underline{T}^{-1}\log^3 n+\eta \log^4 n.
\end{align}
\end{theorem}
\begin{proof}
    We first show that the minimizer $\hat{s}$ over $\Phi'$ (given in \cref{section:Generalization}) of
    \begin{align}
      \hat{s}\in  {\rm argmin}\ \frac1n \sum_{i=1}^n\sum_{k=}^{K}\eta\mathbb{E}[\|s(x_{t_k},{t_k}) - \nabla \log p_{t_k}(x_{t_k}|x_{0,i})\|^2].
    \end{align}
    satisfies
    \begin{align}\label{eq:Appendix-Discretization-6}
   \mathbb{E}_{\{x_{0,i}\}_{i=1}^n}\left[ \sum_{k=1}^{K}\eta
       \mathbb{E}_{x_{t_k}\sim p_{t_k}}[\|\hat{s}(x_{t_k},{t_k}) - \nabla \log p_{t_k}(x_{t_k})\|^2]
    \right]
    \lesssim n^{-2s/(2s+d)}\log^{18} n.
    \end{align} 
    We take $N=n^{\frac{d}{d+2s}}$
    According to \cref{theorem:Approximation}, for $N\gg 1$, there exists a neural network $\NetworkScoreC$ with $L = \Ord (\log^4 N),\| W\|_\infty = \Ord (N\log^6N),S = \Ord (N\log^8N), $ and $B = \exp(\Ord(\log^4 N ))$ that satisfies
    \begin{align}\label{eq:Appendix-Discrete-1}
    \int_x p_t(x)  \|\NetworkScoreC(x,t) - s(x,t)\|^2 \dx \lesssim \frac{N^{-\frac{2s}{d}}\log(N)}{\sigma_t^2}.
    \end{align} 
      for all  $t\in [\underline{T},\overline{T}]$.
    By summing up this for all $t=t_k$, we have that
      \begin{align}\label{eq:Appendix-Discrete-2}
     &  \sum_{k=1}^{K}\eta
       \mathbb{E}_{x_{t_k}\sim p_{t_k}}[\|\NetworkScoreC(x_{t_k},t_k)-\nabla \log p_{\eta k}(X_{t_k})\|^2]
       \lesssim \sum_{k=1}^{K}\eta\frac{N^{-\frac{2s}{d}}\log(N)}{1\land t_k}
    \\ &   \leq  N^{-\frac{2s}{d}}\log(N)\left(\eta K+\eta\sum_{k=1}^K\frac{1}{t_k}\right)
       \lesssim  N^{-\frac{2s}{d}}\log(N)(\overline{T}+\log (\overline{T}/\underline{T})) \lesssim  N^{-\frac{2s}{d}}\log^2(N).
    \end{align}
    
    In order to convert this into the generalization bound, we need to evaluate the following two things.
    First, $\hat{s}$ can be taken so that
    \begin{align}
        \sup_x\|\NetworkScoreC(x,t) \| \dx \lesssim \frac{\log^\frac12(N)}{\sigma_t},
    \end{align}
    and therefore we clip $s$ as in \cref{section:Generalization}. Because such $s$ satisfies 
    \begin{align}
        \int_x p_t(x)  \|\NetworkScoreC(x,t) - \nabla \log p_t(x)\|^2 \dx \lesssim \frac{\log(N)}{\sigma_t^2},
    \end{align}
    we have that
    \begin{align}
     & \sum_{k=1}^{K}\eta
      \mathbb{E}_{x_{t_k}\sim p_{t_k}}[\|\NetworkScoreC(x_{t_k},t_k)-\nabla \log p_{t_k}(x_{t_k})\|^2]
      \leq C_{\ell} =\Ord( \log^2(n))
    \end{align}
    (follow the argument for \cref{lemma:Appendix-Generalization-Prepare-1} and how we derived \eqref{eq:Appendix-Discrete-2} from \eqref{eq:Appendix-Discrete-1}).
    Second, the covering number of the network class of $\ell(x)=\sum_{k=1}^{K}\eta\mathbb{E}[\|s(x_{t_k},{t_k}) - \nabla \log p_{t_k}(x_{t_k}|x)\|^2]$ over all $s$ with $\delta = n^{-\frac{2s}{d+2s}}$ is bounded by $n^{\frac{d}{d+2s}}\log^{16} n$, by following \cref{subsection:Appendix-CoveringNumber}. Thus, \cref{Theorem:ExtendedSchmidt} can be modified to this setting and we obtain that
    \begin{align}\label{eq:Appendix-Discretization-106}
   \mathbb{E}_{\{x_{0,i}\}_{i=1}^n}\left[ \sum_{k=1}^{K}\eta
       \mathbb{E}_{x_{t_k}\sim p_{t_k}}[\|s(x_{t_k},{t_k}) - \nabla \log p_{t_k}(x_{t_k})\|^2]
    \right]
    \lesssim n^{-s/(2s+d)}\log^2 n.
    \end{align} 
    holds. Therefore, following the discussion in \cref{section:Generalization}, we have that
    \begin{align}
       & 
   \mathbb{E}_{\{x_{0,i}\}_{i=1}^n}\left[ \sum_{k=1}^{K}\eta_k
       \mathbb{E}_{x_{t_k}\sim p_{t_k}}[\|s(x_{t_k},{t_k}) - \nabla \log p_{t_k}(x_{t_k})\|^2]
    \right]
   \\ \lesssim & \sum_{k=1}^{K}\eta
       \mathbb{E}_{x_{t_k}\sim p_{t_k}}[\|\NetworkScoreC(x_{t_k},t_k)-\nabla \log p_{\eta k}(X_{t_k})\|^2]
       + \frac{C_\ell}{n}\log \mathcal{N} + \delta
    \\ \lesssim & n^{\frac{d}{d+2s}} \log^2 n  +\frac{\log^2 n}{n}\cdot n^{\frac{d}{d+2s}}\log^{18} n + n^{-\frac{2s}{d+2s}} \lesssim n^{-\frac{2s}{d+2s}}\log^{18} n,
    \end{align}
    which proves \eqref{eq:Appendix-Discretization-6}.

    From now, we bound $\mathrm{TV}(Y_{0}, Y_{\overline{T}-\underline{T}}^{\mathrm{d}})$.
    We introduce the following processes.
    $\bar{Y}^{\mathrm{d}}=(\bar{Y}^{\mathrm{d}}_t)_{t=0}^{\overline{T}-\underline{T}}$ is defined in the same way as $Y^{\mathrm{d}}$, except for the initial distribution of $  \bar{Y}_0^{\mathrm{d}} \sim p_{\overline{T}}$.
    At $t=\overline{T}-\underline{T}$, if the $\int$-norm is more than $2$, then we reset it to $0$.
    $\bar{Y}=(\bar{Y}_t)_{t=0}^{\overline{T}-\underline{T}}$ is defined as $  \bar{Y}_0 \sim p_{\overline{T}}$, and
\begin{align}
  \bar{Y}_0 &\sim p_{\overline{T}},
 \\  \mathrm{d}\bar{Y}_t& = \beta_{\overline{T} - t} \left(Y_t + 2\mathbbm{1}[(\bar{Y}_s,\overline{T}-s)\notin A \text{ for some } s\leq t]\nabla \log p_t(\bar{Y}_t) \right.\\&\quad  + \left.2\mathbbm{1}[(Y_s,\overline{T}-s)\in A\text{ for all } s\leq t] \hat{s}(\bar{Y}_{\overline{T}-t_k},\overline{T}-t_k)\right) \dt  + \sqrt{2\beta_{\overline{T} - t}}\mathrm{d}B_t\ (t\in [\overline{T}-t_i,\ \overline{T}-t_{i-1}]).
\end{align}
At $t=\overline{T}-\underline{T}$, if the $\infty$-norm is more than $2$, then we reset it to $0$.
Here,  $A\subseteq \R^{d+1}$ is defined as
\begin{align}
    A=\left\{(x,t)\in \R^d \times \R \left|\ \|x\|_\infty\leq  m_{t}+C_{{\rm a}}\sigma_{t}\sqrt{\log (n)}, \ \underline{T}\leq t \leq \overline{T}
    \right.\right\}.
\end{align}
Then, we have that
\begin{align}
    \mathrm{TV}(Y_{\overline{T}}, Y_{\overline{T}-\underline{T}}^{\mathrm{d}})
   & \leq
    \mathrm{TV}(Y_{\overline{T}}, Y_{\overline{T}-\underline{T}}) + \mathrm{TV}(Y_{0}, \bar{Y}_{\overline{T}-\underline{T}})
    +\mathrm{TV}(\bar{Y}_{\overline{T}-\underline{T}},\bar{Y}_{\overline{T}-\underline{T}}^{\mathrm{d}})
    +\mathrm{TV}(\bar{Y}_{\overline{T}-\underline{T}}^{\mathrm{d}},\bar{Y}^{\mathrm{d}})\\
    &\leq
 \mathrm{TV}(X_{0}, X_{\underline{T}}) + \mathrm{TV}(Y_{0}, \bar{Y}_{\overline{T}-\underline{T}})
    +\mathrm{TV}(\bar{Y}_{\overline{T}-\underline{T}},\bar{Y}_{\overline{T}-\underline{T}}^{\mathrm{d}})
    +\mathrm{TV}(X_{\overline{T}},\mathcal{N}(0,I_d)).
\end{align}
The first term is bounded by $n^{-\frac{2s}{d+2s}}$, by setting $\underline{T}=n^{-\Ord(1)}$ in \cref{lemma:TotalVariation-Init}.
The second term is bounded by $n^{-\frac{2s}{d+2s}}$, by taking $C_{{\rm a}}$ sufficient large, according to \cref{lemma:Appendix-HPB-1}.
The forth term is bounded by $\exp(-\betalow \overline{T})$ by \cref{lemma:TVBound-NoiseExp}, and thus setting $\underline{T}=\Ord(\log n)$ yields $\exp(-\betalow \overline{T}) \lesssim  n^{-\frac{2s}{d+2s}}$.

Now, we bound the third term. \cref{Proposition:Girsanov} yields that
\begin{align}
   & \mathrm{TV}(\bar{Y}_{\overline{T}-\underline{T}},\bar{Y}_{\overline{T}-\underline{T}}^{\mathrm{d}})
  \\ &  \lesssim
    \sum_{k=1}^{K}\int_{t=\overline{T}-t_{k}}^{\overline{T}-t_{k-1}}
      \mathbb{E}_{\bar{Y}}[\mathbbm{1}[(\bar{Y}_s,\overline{T}-s)\in A \text{ for all } s\leq t]\|\hat{s}(\bar{Y}_{\overline{T}-t_k},\overline{T}-{t_k}) - \nabla \log p_{t}(\bar{Y}_{t})\|^2] \dt
      \\ & \leq
       \sum_{k=1}^{K}\int_{t=\overline{T}-t_{k}}^{\overline{T}-t_{k-1}}
     \mathbb{E}_{\bar{Y}}[\mathbbm{1}[(\bar{Y}_t,\overline{T}-t)\in A, (\bar{Y}_{\overline{T}-t_k},t_k)\in A]\|\hat{s}(\bar{Y}_{\overline{T}-t_k},\overline{T}-{t_k}) - \nabla \log p_{t}(\bar{Y}_{t})\|^2] \dt 
     \\ & \leq 
       \sum_{k=1}^{K}\int_{t=t_{k-1}}^{t_{k}}
     \mathbb{E}_{X}[\mathbbm{1}[(X_t,t)\in A, (X_{t_k},t_k)\in A]\|\hat{s}(X_{t_k},{t_k}) - \nabla \log p_{t}(X_t)\|^2] \dt 
       \\ & \lesssim\label{eq:Appendix-Discretization-201}
       \sum_{k=1}^{K}\int_{t=t_{k-1}}^{t_{k}}
     \mathbb{E}_{x_{t_k}\sim p_{t_k}}[\|\hat{s}(x_{t_k},{t_k}) - \nabla \log p_{t_k}(x_{t_k})\|^2] \dt 
        \\ & \quad +\label{eq:Appendix-Discretization-202}
       \sum_{k=1}^{K}\int_{t=t_{k-1}}^{t_{k}}
     \mathbb{E}_{X}[\mathbbm{1}[(X_t,t)\in A, (X_{t_k},t_k)\in A]\|\nabla \log p_{t}(X_t) - \nabla \log p_{t_k}(X_t)\|^2] \dt 
           \\ & \quad +\label{eq:Appendix-Discretization-203}
       \sum_{k=1}^{K}\int_{t=t_{k-1}}^{t_{k}}
     \mathbb{E}_{X}[\mathbbm{1}[(X_t,t)\in A, (X_{t_k},t_k)\in A]\|\nabla \log p_{t_k}(X_t) - \nabla \log p_{t_k}(X_{t_k})\|^2] \dt 
\end{align}
First, we consider \eqref{eq:Appendix-Discretization-202}. Because $(X_t,t)\in A$, $(\|X_t\|_\infty-m_{t})_+ \lesssim \sigma_{t}\sqrt{\log (n)}$. Over all $t\leq s\leq t_k$, $|\pd_s \sigma_s |\lesssim \frac{1}{\sqrt{t}}$, $|\pd_s m_s |\lesssim 1$, and
\begin{align}
    \|\pd_s \nabla \log p_s(x)\|\lesssim\frac{|\pd_s \sigma_s |+ |\pd_s m_s |}{\sigma_s^{3}}\left(\frac{(\|x\|_\infty - m_s)_+^2}{\sigma_s^2}\lor 1\right)^\frac32
 \lesssim\frac{|\pd_{t} \sigma_{t_k} |+ |\pd_{t} m_{t_k} |}{\sigma_{t_k}^{3}}\left(\frac{(\|x\|_\infty - m_{t_k})_+^2}{\sigma_{t_k}^2}\lor 1\right)^\frac32,
\end{align}
according to \cref{Lemma:Smooth1}.
Therefore, \eqref{eq:Appendix-Discretization-202} is bounded by
$ \sum_{k=1}^{K}\eta(\eta( t_k^{-2}\lor 1)\log^\frac32 n)^2 = \eta^2 (t_k^{-4}\lor 1)\log^3 n$.

Next, for \eqref{eq:Appendix-Discretization-203}, we first note that $\|X_t\|_\infty-m_{t_k},\|X_{t_k}\|_\infty-m_{t_k} \lesssim \sigma_{t_k}\sqrt{\log (n)}=\tilde{\Ord}(1)$.
Therefore, according to \cref{Lemma:Smooth1}, $ \|\pd_{x_i} \nabla \log p_{t_k}(x)\|$ is bounded by $\frac{1}{\sigma_{t_k}^2}\left(\frac{(\|X_{t_{k}}\|_\infty - m_{t_k})_+^2}{\sigma_{t_k}^2}\lor 1\right)\lesssim t_k^{-1}\log n$.
With probability at least $1-n^{-\Ord(1)}$, $\|X_t-X_{t_k}\|_\infty\lesssim \sqrt{\eta \log n}$, according to \cref{Lemma:HittingTime}.
Therefore, 
\begin{align}
    \eqref{eq:Appendix-Discretization-203}\lesssim \sum_{k=1}^{K}\eta(\sqrt{\eta \log n}\cdot (t_k^{-1}\lor 1)\log n)^2  +n^{-\Ord(1)} \cdot \tilde{\Ord}(1)\lesssim \sum_{k=1}^{K}\eta^2 (t_k^{-2}\lor 1) \log^3 n.
\end{align}

Finally, for \eqref{eq:Appendix-Discretization-203}, we apply \eqref{eq:Appendix-Discretization-6}.
Now, all three terms of \eqref{eq:Appendix-Discretization-201}, \eqref{eq:Appendix-Discretization-202}, and \eqref{eq:Appendix-Discretization-203} are bounded and we obtain that
\begin{align}
     \mathbb{E}_{\{x_{0,i}\}_{i=1}^n}\left[\mathrm{TV}(\bar{Y}_{\overline{T}-\underline{T}},\bar{Y}_{\overline{T}-\underline{T}}^{\mathrm{d}})\right] &\lesssim n^{-\frac{2s}{d+2s}}\log^{18} n + \sum_{k=1}^{K}(\eta^3 (t_k^{-4}\lor 1)\log^3 n + \eta^2 (t_k^{-2}\lor 1) \log^3 n)
     \\ & \lesssim
     n^{-\frac{2s}{d+2s}}\log^{18} n + \eta^2 \underline{T}^{-3}\log^3 n 
     +\eta \underline{T}^{-1}\log^3 n +\eta \overline{T}\log^3 n
     \\ & \lesssim
     n^{-\frac{2s}{d+2s}}\log^{18} n +\eta^2 \underline{T}^{-3}\log^3 n 
     +\eta \underline{T}^{-1}\log^3 n+\eta \log^4 n.
\end{align}
Therefore, by setting $\eta = \underline{T}^{-1.5}n^{-\frac{s}{d+2s}}$ yields the assersion.

\end{proof}

\section{Error analysis with intrinsic dimensionality}\label{section:Appendix-IntrinsicDimensionality}
\subsection{Brief proof overview}
The generalization error analysis of the score network and how much the score estimation error affects in the final estimation rate in \cref{theorem:GeneralizationL1lowdim} are derived by just replacing $d$ by $d'$ in the previous analysis. 
Therefore we focus on the approximation error bounds.
In order to obtain the counterparts of \cref{theorem:Approximation} and \cref{lemma:ApproximationSmoothArea}, we aim to decompose the score function into two parts: each of them is determined by the intrinsic structure components (in $V$) and other components (in $V^\perp$).
We use $z$ as a $d'$-dimensional vector corresponding to the canonical system of $V$.
The first observation to this goal is 
\begin{align}
   p_t(x) &= \int \frac{1}{\sigma_t^d(2\pi)^{\frac{d}{2}}}p_0(y)\exp\left(-\frac{\|x-m_ty\|^2}{2\sigma_t^2}\right)\mathrm{d}y\\
  &=\int_{V} \frac{1}{\sigma^d_t(2\pi)^{\frac{d}{2}}}q(z)\exp\left(-\frac{\|A^\top x-m_tz\|^2+\|(I_d-A^\top) x\|^2}{2\sigma_t^2}\right)\mathrm{d}z
 \\ & \quad (\text{$z$ is a $d'$-dimensional vector corresponding to the canonical system of $V$.})
  \\
 &=\underbrace{\int_{V} \frac{q(z)}{\sigma^{d'}_t(2\pi)^{\frac{d'}{2}}}\exp\left(-\frac{\|A^\top x-m_tz\|^2}{2\sigma_t^2}\right)\mathrm{d}z}_{p_t^{(1)}(x)}\cdot\underbrace{ \frac{1}{\sigma^{d-d'}_t(2\pi)^{\frac{d-d'}{2}}}\exp\left(-\frac{\|(I_d-A^\top) x\|^2}{2\sigma_t^2}\right)}_{p_t^{(2)}(x)}.
\end{align}
Here $p^{(1)}_t(x)$ and $p^{(2)}_t(x)$ can be seen as the density function with respect to the intrinsic components and remaining space.
Note that
\begin{align}\label{eq:Appendix-Intrinstic-1}
\nabla \log p_t(x)=\nabla\log(p_t^{(1)}(x)p_t^{(2)}(x)) = 
\nabla\log p_t^{(1)}(x)
+
\nabla\log p_t^{(2)}(x).
\end{align}
Due to this,
we only need to construct the neural networks approximating each term and concatenate them. 
In addition, $p^{(1)}_t(x)$ can be seen as the density at $A^\top x$, about the diffusion process on the $d'$-dimensional space, where the initial density is defined by $q$.
Thus we let 
\begin{align}
    q_t(z') = \int_{V} \frac{q(z)}{\sigma^{d'}_t(2\pi)^{\frac{d'}{2}}}\exp\left(-\frac{\|z'-m_tz\|^2}{2\sigma_t^2}\right)\mathrm{d}z
\end{align}
for $z'\in \R^{d'}$. Here $p_t^{(1)}(x)=q_t(A^\top x)$ holds.

\subsection{Proof of \cref{theorem:GeneralizationL1lowdim}}
We first consider the approximation of $p^{(1)}_t(x)$. We have the following counterpart of \cref{theorem:Approximation} and \cref{lemma:ApproximationSmoothArea}, where the only difference is that here $d$ is replaced by $d'$.
\begin{lemma}\label{lemma:GeneralizationL1lowdim-1}
    Let $N\gg 1$, $\underline{T}=\mathrm{poly}(N^{-1})$ and $\overline{T}=\Ord(\log N)$.
    Then there exists a neural network $\phi_{\mathrm{score}, 3}\in \Phi(L,W,S,B)$ that satisfies, for all $t\in [\underline{T},\overline{T}]$,
    \begin{align}\label{eq:GeneralizationL1lowdim-1-1}
    \int_{x\in \R^d} p_t(x)\| \nabla \log p_t^{(1)}(x) -\phi_{\mathrm{score}, 3}(A^\top x,t)\|^2 \dx \lesssim \frac{N^{-\frac{2s}{d'}}\log(N)}{\sigma_t^2}.
    \end{align} 
       Here, $L,W,S$ and $B$ are evaluated as $L = \Ord (\log^4 N),\| W\|_\infty = \Ord (N\log^6N),S = \Ord (N\log^8N), $ and $B = \exp(\Ord(\log^4 N )).$
       We can take $\phi_{\mathrm{score}, 3}$ satisfying $\|\phi_{\mathrm{score}, 3}(\cdot,t)\|_\infty = \Ord(\sigma_t^{-1}\log^\frac12 N)$.
    
     Moreover, let $N'\geq t_*^{-d'/2}N^{\delta/2}$ and $t_*\geq N^{-(2-\delta)/d'}$.
    Then there exists a neural network $\phi_{\mathrm{score}, 4}\in \Phi(L,W,S,B)$ that satisfies
    \begin{align}\label{eq:GeneralizationL1lowdim-1-2}
   \int_{x\in \R^d} p_t(x)\| \nabla \log p_t^{(1)}(x) -A \phi_{\mathrm{score}, 4}(A^\top x,t)\|^2 \dx \lesssim \frac{N^{-\frac{2(s+1)}{d'}}}{\sigma_t^2}
    \end{align}
    for $t\in [2t_*,\overline{T}]$.
    Specifically, $L = \Ord (\log^4 (N)),\| W\|_\infty = \Ord (N),S = \Ord (N')$, and $ B = \exp(\Ord(\log^4 N ))$.
We can take $\phi_{\mathrm{score}, 4}$ satisfying $\|\phi_{\mathrm{score}, 4}(\cdot,t)\|_\infty = \Ord(\sigma_t^{-1}\log^\frac12 N)$.
\end{lemma}
\begin{proof}
    Let $\NetworkScoreC\colon \R^{d'}\times \R_+\to \R^{d'}$ that approximates $q_t(z)$.
    Note that
    \begin{align}
        \nabla \log p_t^{(1)}(x) = A \nabla \log q_t(A^\top x)
    \end{align}
    and therefore
    \begin{align}
        \int_{x\in \R^d} p_t(x)\| \nabla \log p_t^{(1)}(x) -A \NetworkScoreC(A^\top x,t)\|^2 \dx
       & = 
        \int_{x\in \R^d} p_t^{(1)}(x)p_t^{(2)}(x)\| A \nabla \log p^{(1)}_t(A^\top x)-A\NetworkScoreC(A^\top x,t)\|^2 \mathrm{d}x
        \\ &= 
         \int_{x\in \R^d} q_t(A^\top x)\| A \nabla \log p^{(1)}_t(A^\top x)-A\NetworkScoreC(A^\top x,t)\|^2 \mathrm{d}x
     \\ &   =
        \int_{z\in\R^{d'}}q_t(z)\| \nabla \log q_t(z)-\NetworkScoreC(z,t)\|^2 \mathrm{d}z,
    \end{align}
    where we used the fact that $p_t^{(1)}$ and $p^{(2)}_t$ depend on $A^\top x$ and $(I-A^\top)x$, respectively, and $A^\top x$ and $(I-A^\top)x$ are orthogonal.
    Moreover, we used $\mathrm{det}(A^\top A)=1$ and orthogonality of the columns of $A$.
    Thus, we can translate \cref{theorem:Approximation} and \cref{lemma:ApproximationSmoothArea}.
\end{proof}

We next consider the approximation of $p^{(2)}_t(x)$.
As we did in \cref{section:Appendix-HPB}, we first show that it suffice to consider the approximation within the bounded region.
\begin{lemma}\label{lemma:Intrinsic-Auxiliary}
    For $\eps>0$, we define $B_{t,\eps}$ as
    \begin{align}
        B_{t,\eps} = \left\{
            x\in \mathbb{R}^d \left|\|(I_d-A^\top) x\|\leq C_{\mathrm{e}}\sigma_t \sqrt{\log \eps^{-1}}. 
            \right.
        \right\}
    \end{align}
    We sometimes abbreviate this as $B_\eps$.
    Then, we have that
    \begin{align}
        \int_{x\in \bar{B}_\eps}p_t(x)\left[1 \lor \|\nabla\log(p_t^{(2)}(x))\|^2\right] \dx \lesssim \eps.
    \end{align}
\end{lemma}
\begin{proof}
The the columns of $A$ are orthogonal.
$p_t^{(1)}$ and $p^{(2)}_t$ depend on $A^\top x$ and $(I-A^\top)x$, respectively, and $A^\top x$ and $(I-A^\top)x$ are orthogonal.
Thus, we have that
    \begin{align}\label{eq:Appendix-Intrinstic-2}
        \int_{x\in \bar{B}_{t,\eps}}p_t(x)\left[1 \lor \|\nabla\log(p_t(x))\|^2\right] \dx
     & =  \int_{x\in \bar{B}_{t,\eps}}p_t^{(1)}(x)p_t^{(2)}(x)\left[1 \lor \|\nabla\log(p_t(x))\|^2\right] \dx
     \\ & =\int_{x\in \bar{B}_{t,\eps}}p_t^{(2)}(x)\left[1 \lor \|\nabla\log(p_t(x))\|^2\right] \dx
     \\ & 
    = \int_{w\in \R^{d-d'}\colon\|w\|\geq C_{\mathrm{e}}\sigma_t \sqrt{\log \eps^{-1}}}\frac{1\lor \|w\|^2/\sigma_t^2}{\sigma^{d-d'}_t(2\pi)^{\frac{d-d'}{2}}}\exp\left(-\frac{\|w\|^2}{2\sigma_t^2}\right)\mathrm{d}w.
    \end{align}
   Applying \cref{Corollary:ApproxInv}, \eqref{eq:Appendix-Intrinstic-2} is bounded by $\eps$ with a sufficiently large constant $C_{\mathrm{e}}$.
\end{proof}

Now we only need consider the approximation of $\nabla\log p_t^{(2)}(x)$ within $B_{t,\eps}$.
\begin{lemma}\label{lemma:GeneralizationL1lowdim-2}
    Let $N\gg 1$, $\underline{T},\eps={\rm poly}(N^{-1})$ and $\overline{T}\simeq \log N$.
    There exists a neural network $\phi_{\mathrm{score}, 4}\in \Phi(L,W,S,B)$ such that
    \begin{align}\label{eq:GeneralizationL1lowdim-2-1}
      \sup_{t\in [\underline{T}, \overline{T}]} \int_{x} p_t(x)\|\nabla\log p_t^{(2)}(x)-\phi_{\mathrm{score}, 4}(x,t)\|^2 \dx \lesssim \frac{N^{-\frac{2(s+1)}{d'}}}{\sigma_t^2}.
    \end{align}
    Specifically, $\phi_{\mathrm{score}, 4}\in \Phi(L,W,S,B)$ holds, where
    \begin{align}\label{eq:GeneralizationL1lowdim-2-2}
    L= \Ord(\log^2 N) ), \|W\|_{\infty} = \Ord( \log^3 N), S = \Ord(\log^4 N),\text{ and }B= \exp(\Ord(\log^2 N)).
\end{align}
\end{lemma}
\begin{proof}
    First note that $\nabla\log p_t^{(2)}(x)=-\frac{1}{\sigma_t^2}(I_d-A)(I_d-A^\top)x$.
    We approximate this via the following four steps. 
    \begin{enumerate}
\item $\sigma_t$ is approximated by $\NetworkSigmaA$ from \cref{lemma:SigmaM}. Here we set $\eps \leftarrow (\underline{T}^4\land \eps^4)\eps^4$. 
\item Based on the approximation of $\sigma_t$, $\sigma_t^{-2}$ is approximated by $\NetworkInvA(\cdot;2)$ from \cref{Corollary:ApproxInv}. Here we set $\eps\leftarrow(\underline{T}\land \eps)\eps$.
\item $(I_d-A)(I_d-A^\top)$ is realized by $\ReLU((I_d-A)(I_d-A^\top) \cdot x + 0) - \ReLU(-(I_d-A)(I_d-A^\top) \cdot x + 0) $.
\item According to \cref{Lemma:BaseNN02} with $\eps\leftarrow\eps$ and $C\leftarrow\underline{T}^{-1}\lor \sqrt{\log \eps^{-1}}$, multiplication of $\sigma_t^{-2}$ and $(I_d-A)(I_d-A^\top)$ is constructed.
\end{enumerate}
By concatenating these networks (using \cref{Lemma:ConcateNetwork}), the obtained network size is bounded as
\begin{align}
 &   L= \Ord(\log^2 \epsD^{-1}+\log^2 \underline{T}^{-1}) ), \|W\|_{\infty} = \Ord( \log^3 \epsD^{-1}+\log^3 \underline{T}^{-1}), S = \Ord(\log^4 \epsD^{-1}+\log^4 \underline{T}^{-1}),\\ &\text{ and }B= \exp(\Ord(\log^2 \epsD^{-1}+\log^2 \underline{T}^{-1})).
\end{align}
Then, for $x\in B_{t,\eps}$ with $t\geq \underline{T}$, we have that
\begin{align}
    \|\nabla\log p_t^{(2)}(x)-\phi_{\mathrm{score}, 4}\| \lesssim \eps.
\end{align}
This yields that
\begin{align}
    \int_{B_{t,\eps}}p_t(x) \|\nabla\log p_t^{(2)}(x)-\phi_{\mathrm{score}, 4}\|\dx \lesssim \eps.
\end{align}
Together with \cref{lemma:Intrinsic-Auxiliary}, by taking $\eps=\mathrm{poly}(N^{-1})$, we have the assertion.
\end{proof}

\begin{proof}[Proof of \cref{theorem:GeneralizationL1lowdim}]
    Note that while the error bound \eqref{eq:GeneralizationL1lowdim-2-1} in \cref{lemma:GeneralizationL1lowdim-2} is tighter than the bounds \eqref{eq:GeneralizationL1lowdim-1-1} and \eqref{eq:GeneralizationL1lowdim-1-2} in \cref{lemma:GeneralizationL1lowdim-1}, the required network size \eqref{eq:GeneralizationL1lowdim-2-2} in \cref{lemma:GeneralizationL1lowdim-2} is smaller than the size bounds in \cref{lemma:GeneralizationL1lowdim-1}.
    Also note that the bounds in \cref{lemma:GeneralizationL1lowdim-1} are the same as those in  \cref{theorem:Approximation} and \cref{lemma:ApproximationSmoothArea}, except for that $d$ is replaced by $d'$.
    Therefore, by simply aggregating $\phi_{\mathrm{score}, 3}$ and $\phi_{\mathrm{score}, 4}$, we obtain the counterpart of the approximation theorems \cref{theorem:Approximation} and \cref{lemma:ApproximationSmoothArea}, and the rest of the analysis are the same as that of the $d$-dimensional case.
    Therefore, we obtain the statement.
\end{proof}

\section{Auxiliary lemmas}

This final section summarizes existing results and prepares basic tools for the main parts of the proofs.
A large part of this section (\cref{subsection:Appendix-F1,subsection:Preparation-Multiplicative,subsection:Preparation-Exponential,subsection:Appendix-F4}) is devoted to introduction of basic tools for the function approximation with neural networks, and thus those familiar with such topics \citep{yarotsky2017error,petersen2018optimal,schmidt2019deep} can skip these subsections (although they contain some refinement and extension).
\Cref{Lemma:GaussianBound} is for elementary bounds on the Gaussian distribution and hitting time of the Brownian motion.


In the following we will define constants $C_{\mathrm{f},1}$ and $C_{\mathrm{f},2}$. Other than in this section, they are denoted by $C_{\mathrm{f}}$, and sometimes other constants that comes from this section can be also denoted by $C_{\mathrm{f}}$.

\subsection{Construction of a larger neural network}\label{subsection:Appendix-F1}

Through construction of the desired neural network, we often need to combine sub-networks that approximates simpler functions to realize more complicated functions.
We prepare the following lemmas, whose direct source is \citet{nakada2020adaptive} but similar ideas date back to earlier literature such as \citet{yarotsky2017error,petersen2018optimal}.

First we consider construction of composite functions. 
Although the bound on the sparsity $S$ was not given in the original version, 
we can verify it by carefully checking their proof.
\begin{lemma}[Concatenation of neural networks (Remark 13 of \citet{nakada2020adaptive})]\label{Lemma:ConcateNetwork}
    For any neural networks $\phi^1\colon \R^{d_1}\to \R^{d_2},\phi^2\colon \R^{d_2}\to \R^{d_3},\cdots,\phi^k\colon \R^{d_k}\to \R^{d_{k+1}}$ with $\phi^i\in \Psi(L^i,W^i,S^i,B^i)\ (i=1,2,\cdots,d)$, there exists a neural network $\phi\in \Phi(L,W,S,B)$ satisfying $\phi(x) = \phi^k \circ \phi^{k-1} \cdots \circ \phi^1(x)$ for all $x \in\R^{d_1}$, with
    \begin{align}
        L = \sum_{i=1}^k L^i,\quad
     W  \leq 2\sum_{i=1}^k W^i,\quad
         S  \leq\sum_{i=1}^k S^i +\sum_{i=1}^{k-1} (\|A_{L^i}^i\|_{0} +\|b_{L^i}^i\|_{0}+ \|A_{1}^{i+1}\|_{0})\leq 2\sum_{i=1}^k S^i ,\quad \text{and }
         B  \leq \max_{1\leq i \leq k} B^i.
    \end{align}
    Here $A_{j}^i$ is the parameter matrix and $b_{j}^i$ is the bias vector at the $j$th layer of the $i$th neural network $\phi^i$. 
\end{lemma}

Next we introduce the identity function.
\begin{lemma}[Identity function (p.19 of \citet{nakada2020adaptive})]\label{Lemma:IdentityFunc}
    For $L\geq 2$ and $d\in \N$, there exists a neural network $\NetworkIdentityA^{d,L}\in \Phi(L,W,S,B)$ with parameters $(A_1,b_1) = ((I_d,-I_d)^\top ,0), (A_i,b_i) = (I_{2d} ,0) (i=1,2,\cdots,L-2), (A_L) = ((I_d,-I_d), 0)$, that realize $d$-dimensional identity map.
    Here, 
    \begin{align}
        \|W\|_\infty =2d,\quad
        S=2dL, \quad
        B=1.
    \end{align}
    For $L=1$, a neural network $\NetworkIdentityA^{d,1}\in \Phi(1,(d),d,1)$ with parameters $(A_1,b_1) = (I_d ,0)$ realizes $d$-dimensional identity map.
\end{lemma}

We then consider parallelization of neural networks. 
The following lemmas are Remarks 14 and 15 of \citet{nakada2020adaptive} with a modification to allow sub-networks to have different depths.
\begin{lemma}[Parallelization of neural networks]\label{Lemma:ParallelNetwork}
    For any neural networks $\phi^1,\phi^2, \cdots, \phi^k$ with  $\phi^i\colon\R^{d_i}\to \R^{d_i'}$ and $\phi^i\in \Psi(L^i,W^i,S^i,B^i)\ (i=1,2,\cdots,d)$, there exists a neural network $\phi\in \Phi(L,W,S,B)$ satisfying $\phi(x) = [\phi^1(x^1)^\top\ \phi^2(x^2)^\top\ \cdots\ \phi^k(x^k)^\top]^\top\colon \R^{d_1+d_2+\cdots+d_k}\to\R^{d_1'+d_2'+\cdots+d_k'}$ for all $x = (x_1^\top\ x_2^\top\ \cdots \ x_k^\top)^\top \in\R^{d_1+d_2+\cdots+d_k}$ (here $x_i$ can be shared), with
    \begin{align}
    L = L, \quad
      \|W\|_\infty \leq \sum_{i=1}^k \|W^i\|_\infty,\quad
         S  \leq\sum_{i=1}^k S^i ,\quad \text{and }
         B  \leq \max_{1\leq i \leq k} B^i \quad\text{ (when $L=L_i$ holds for all $i$)},
         \\
    L = \max_{1\leq i \leq k}L^i , \quad
      \|W\|_\infty \leq 2\sum_{i=1}^k \|W^i\|_\infty,\quad
         S  \leq 2\sum_{i=1}^k (S^i + LW^i_L),\quad \text{and }
         B  \leq \max\{\max_{1\leq i \leq k} B^i,1\} \quad\text{ (otherwise)}.         
    \end{align}

    Moreover, there exists a network $\phi_{\rm sum}(x)\in \Phi(L,W,S,B)$ that realizes $ = \sum_{i=1}^k \phi^i(x)$, with
       \begin{align}
    L = \max_{1\leq i \leq k}L^i + 1 , \quad
      \|W\|_\infty \leq 4\sum_{i=1}^k \|W^i\|_\infty,\quad
         S  \leq 4\sum_{i=1}^k (S^i +L W_L) + 2W_L,\quad \text{and }
         B  \leq \max\{\max_{1\leq i \leq k} B^i,1\}.         
    \end{align}
\end{lemma}
\begin{proof}[Proof of \cref{Lemma:ParallelNetwork}]
    Let us consider the first part.
    For the case when $L=L_i$ holds for all $i$, the assertions are exactly the same as Remarks 14 and 15 \citet{nakada2020adaptive}.
    Otherwise, we first prepare a network $\phi'^i$ realizing $\NetworkIdentityA^{d,L-L_i} \circ \phi^i$ for all $i$, so that every network have the same depth without changing outputs of the networks.
    From \cref{Lemma:ConcateNetwork,Lemma:IdentityFunc}, $\phi'^i\in \Phi(L,{W'}^i,S'^i,B'^i)$ holds, with $L=\max_{1\leq i \leq k}L^i, \|W'^i\|_\infty = \max\{\|W^i\|_\infty,2W_L\} \leq 2\|W^i\|_\infty, S'^i \leq 2S^i + 2(L-L_i)W^i_L \leq 2(S^i + LW^i_L)$, and $B'^i = \max\{B^i,1\}$.
    We then apply the results for the case of $L=L_i\ (i=1,2,\cdots,k)$.

    For the second part, since summation of the outputs of $k$ neural networks can be realized by a $1$ layer neural network with the width of $k$, \cref{Lemma:ParallelNetwork} together with \cref{Lemma:ConcateNetwork} gives the bound to realize $\sum_{i=1}^k \phi^i(x)$.
\end{proof}

In the analysis of the score-based diffusion model, we often face unbounded functions.
To resolve difficulty coming from the unboundedness, the clippling operation is often be adopted.
\begin{lemma}[Clipping function]\label{Lemma:ClippingFunc}
    For any $a, b \in \R^d$ with $a_i\leq b_i \ (i=1,2,\cdots,d)$, there exists a clipping function $\NetworkClipA(x;a,b) \in \Phi(2,(d, 2d, d)^\top, 7d, \max_{1\leq i\leq d}\max\{|a_i|,b_i\})$ such that
    \begin{align}
        \NetworkClipA(x;a,b)_i = \min\{b_i,\max\{x_i,a_i\}\} \quad (i=1,2,\cdots,d)
    \end{align}
    holds. When $a_i=c$ and $b_i=C$ for all $i$, we sometimes denote $\NetworkClipA(x;a,b)$ as $\NetworkClipA(x;c,C)$ using scaler values $c$ and $C$.
\end{lemma}
\begin{proof}
    Because, for each coordinate $i$, $\min\{b_i,\max\{x_i,a_i\}\}$ is realized as
    \begin{align}
        \min\{b_i,\max\{x_i,a_i\}\} = \ReLU(x_i - a_i) - \ReLU(x_i - b_i) + a_i \in \Phi(2,(1,2,1), 7, \max\{|a_i|,b_i\}),
    \end{align}
    parallelizing this for all $i$ with \cref{Lemma:ParallelNetwork} yields the assertion.
\end{proof}
With the above clipping function, we prepare switching functions, which gives the way to construct approximation in the combined region when there are two different approximations valid for different regions.
\begin{lemma}[Switching function]\label{Lemma:SwitchingFunc}
    Let $\underline{t}_1<\underline{t}_2<\overline{t}_1<\overline{t}_2$, and $f(x,t)$ be some scaler-valued function (for a vector-valued function, we just apply this  coordinate-wise).
    Assume that $\phi^1(x,t)
    $ and $\phi^2(x,t)
    $ approximate $f(x,t)$ up to an additive error of $\epsilon$ but approximation with $\phi^1(x,t)$ and $\phi^2(x,t)$ are valid for $[\underline{t}_1,\overline{t}_1]$ and $[\underline{t}_2,\overline{t}_2]$, respectively.
    Then, there exist neural networks $\NetworkSwitchA(t;\underline{t}_2,\overline{t}_1),\NetworkSwitchB(t;\underline{t}_2,\overline{t}_1)\in \Phi(3,(1,2,1,1)^\top,8,\max\{\overline{t}_1,(\overline{t}_1-\underline{t}_2)^{-1}\})$, and $\NetworkSwitchA(t;\underline{t}_2,\overline{t}_1)\phi^1(x,t) + \NetworkSwitchB(t;\underline{t}_2,\overline{t}_1)\phi^2(x,t)$ approximates $f(x,t)$ up to an additive error of $\epsilon$ in $[\underline{t}_1,\overline{t}_2]$.
    
\end{lemma}
\begin{proof}
    We define 
    \begin{align}
        \NetworkSwitchA(t;\underline{t}_2,\overline{t}_1) = \frac{1}{\overline{t}_1-\underline{t}_2}\ReLU(\NetworkClipA(t;\underline{t}_2,\overline{t}_1) - \underline{t}_2),\quad\text{and }
        \NetworkSwitchB(t;\underline{t}_2,\overline{t}_1) = \frac{1}{\overline{t}_1-\underline{t}_2}\ReLU(\overline{t}_1 - \NetworkClipA(t;\underline{t}_2,\overline{t}_1))
        .
    \end{align}
    Here $\NetworkSwitchA(t;\underline{t}_2,\overline{t}_1),\NetworkSwitchB(t;\underline{t}_2,\overline{t}_1)\in [0,1]$, $\NetworkSwitchA(t;\underline{t}_2,\overline{t}_1)+\NetworkSwitchB(t;\underline{t}_2,\overline{t}_1) = 1$ for all $t$, $\NetworkSwitchA(t;\underline{t}_2,\overline{t}_1)=0$ for all $t\geq \overline{t}_1$, and $\NetworkSwitchB(t;\underline{t}_2,\overline{t}_1)$ for $t\leq \underline{t}_2$.
    From this construction, the assertion follows. 
\end{proof}

\subsection{Basic neural network structure that approximates rational functions}\label{subsection:Preparation-Multiplicative}
When approximating a function in the Besov space with a neural network, the most basic structure of the network is that of approximating polynomials \citep{suzuki2018adaptivity}. 
In our construction of the diffused B-spline basis, we need to approximate rational functions.

We begin with monomials.
Although the traditional fact that we can approximate monomials with neural networks with an arbitrary additive error of $\epsilon$ using only $\Ord(\log\eps^{-1})$ non-zero parameters has been very famous \citep{yarotsky2017error,petersen2018optimal,schmidt2020nonparametric}, we could not find the result that explicitly states the dependency on parameters including the degree and the range of the input.
Therefore, just to be sure,  we revisit Lemma A.3 of \citet{schmidt2020nonparametric} and here gives the extended version of that lemma.
\begin{lemma}[Approximation of monomials]
\label{Lemma:BaseNN02}
    Let $d\geq 2$, $\ConstMultA \geq 1$,  $0<\epsSensitivityC\leq 1$.
    For any $\epsB>0$, there exists a neural network $\NetworkMultiB(x_1,x_2,\cdots,x_d)\in \Psi(L,W,S,B)$ with
   $L = \Ord(\log d(\log \epsB^{-1}+ d \log \ConstMultA)), \|W\|_\infty = 48d, S = \Ord(d \log \epsB^{-1} + d\log \ConstMultA)), B=\ConstMultA^d$
    such that
    \begin{align}
     &   \left|\NetworkMultiB(x_1',x_2',\cdots,x_d') - \prod_{d'=1}^d x_{d'}\right| \leq \epsB + d \ConstMultA^{d-1} \epsSensitivityC,
        \quad  \text{for all } x\in [-\ConstMultA,\ConstMultA]^d\text{ and } x'\in \R\ \text{with $\|x-x'\|_\infty \leq \epsSensitivityC$},
    \end{align}
    $|\NetworkMultiB(x)|\leq \ConstMultA^d$ for all $x\in \R^d$, and $\NetworkMultiB(x_1',x_2',\cdots,x_d')=0$ if at least one of $x_i'$ is $0$.
    
    We note that some of $x_i,x_j\ (i\ne j)$ can be shared. For $\prod_{i=1}^{I} x_{i}^{\alpha_{i}}$ with $\alpha_{i}\in \Z_+\ (i=1,2,\cdots,I)$ and $\sum_{i=1}^{I} \alpha_{i} = d$,  there exists a neural network satisfying the same bounds as above, and the network is denoted by $\NetworkMultiB(x;\alpha)$.
\end{lemma}
\begin{proof}
    First of all, it is known from \citet{schmidt2020nonparametric} that there exists a neural network $\NetworkMultiA'(x,y)\in \Psi(L,W,S,B)$ with $L=i+5,\|W\|_\infty = 6, 
    B=1$ such that
    \begin{align}
        |\NetworkMultiA'(x,y) - xy| \leq 2^{-i}, \quad \text{for all } (x,y)\in [0,1]^2,
    \end{align}
    and $|\NetworkMultiA'(x,y)|\leq 1$ for all $(x,y)\in \R^2$, and $\NetworkMultiA'(x,y)=0$ if either $x$ or $y$ is $0$.
    With this network, we can see that $|\sign(xy)\NetworkMultiA'(|x|,|y|)-xy| \leq 2^{-i}$ holds for all $(x,y)\in [-1,1]^2$,  $|\NetworkMultiA'(x,y)|\leq 1$ for all $(x,y)\in \R^2$, and $\NetworkMultiA(x,y)=0$ if either $x$ or $y$ is $0$.
    Because
    \begin{align}
        \sign(xy)\NetworkMultiA'(|x|,|y|)
    &
        =\ReLU(\NetworkMultiA'(\ReLU(x),\ReLU(y)) + \NetworkMultiA'(\ReLU(-x),\ReLU(-y))
    \\ & \quad
        -\NetworkMultiA'(\ReLU(-x),\ReLU(y)) -\NetworkMultiA'(\ReLU(x),\ReLU(-y)))
    \\ &  \quad
        - \ReLU(-\NetworkMultiA'(\ReLU(x),\ReLU(y)) - \NetworkMultiA'(\ReLU(-x),\ReLU(-y))
    \\ & \quad
        +\NetworkMultiA'(\ReLU(-x),\ReLU(y)) +\NetworkMultiA'(\ReLU(x),\ReLU(-y)))
        \\ & =:\NetworkMultiA(x,y)
    \end{align}
    holds, we can realize the function $xy$ for $[-1,1]^d$, by a neural network $\NetworkMultiA(x,y)\in \Psi(L,W,S,B)$ with $L=i+7,\|W\|_\infty = 48, S\leq L\|W\|_\infty(\|W\|_\infty+1)=48(i+7), B=1$ with an approximation error up to $2^{-i}$.
    

    
    Then, following \citet{schmidt2020nonparametric}, 
    we recursively construct $\NetworkMultiA(x_1,x_2,\cdots,x_{2^{j+1}})$ using 
    \begin{align}
        \NetworkMultiA(x_1,x_2,\cdots,x_{2^{j+1}}) = \NetworkMultiA(\NetworkMultiA(x_1,x_2,\cdots,x_{2^{j}}), \NetworkMultiA(x_{2^{j}+1},x_{2^{j}+2},\cdots,x_{2^{j+1}})).
    \end{align}
    By filling extra dimensions of $(x_1,x_2,\cdots,x_{2^{j}})$ with $1$,  we obtain the neural network  $\NetworkMultiB(x_1,x_2,\cdots,x_d) \in \Psi(L,W,S,B)$
    for all $d\geq 2$ and 
   $L = \Ord(\log d(\log \epsB^{-1}+ \log d)), \|W\|_\infty = 48d, S = \Ord(d (\log \epsB^{-1} + \log d)), B=1$
    such that
    \begin{align}
     &   \left|\NetworkMultiA(x_1,x_2,\cdots,x_d) - \prod_{d'=1}^d x_{d'}\right| \leq \epsB ,
        \quad  \text{for all } x\in [-1,1]^d.
    \end{align}

    We then construct $\NetworkMultiB$ as follows:
    \begin{align}
      \NetworkMultiB(x) = \ConstMultA^d\NetworkMultiA(\NetworkClipA(x;-\ConstMultA,\ConstMultA)/\ConstMultA)    
      .
    \end{align}
    Here the approximation error over $[-C,C]^d$ is bounded by $\ConstMultA^{-d}\epsB$.
    We reset $\epsB\leftarrow\ConstMultA^{-d}\epsB$ so that the approximation error is smaller than $\eps$, and then we have $\NetworkMultiB\in \Phi(L,W,S,B)$ with $L = \Ord(\log d(\log d + \log \epsB^{-1}+ d \log \ConstMultA)), \|W\|_\infty = 48d, S = \Ord(d(\log d + \log \epsB^{-1} + d\log \ConstMultA)), B=1$.
    Therefore, the bounds on $L,\|W\|_\infty,B,S$ in the assertion follows from \cref{Lemma:ClippingFunc,Lemma:ConcateNetwork}.

    When the input fluctuates, we have
    \begin{align}
      &  \left|\ConstMultA^d\NetworkMultiA(\NetworkClipA(x';-\ConstMultA,\ConstMultA)/\ConstMultA)
        -\prod_{i=1}^d x_i\right|
 \\&    \leq
        \left|\ConstMultA^d\NetworkMultiA(\NetworkClipA(x';-\ConstMultA,\ConstMultA)/\ConstMultA)
        -\prod_{i=1}^d \min\{C,\max\{x_i',-C\}\}\right|
        +
        \left|\prod_{i=1}^d \min\{C,\max\{x_i',-C\}\} - \prod_{i=1}^d x_i\right|
      \\ &   \leq \ConstMultA^d \cdot \ConstMultA^{-d}\epsB + \ConstMultA^{d-1} \sum_{i=1}^d |x_i - \min\{C,\max\{x_i',-C\}\}|
        = \epsB + d \ConstMultA^{d-1} \epsSensitivityC
        ,
    \end{align}
    which yields the first part of the assertion.

    Finally, we note that some of $x_i,x_j\ (i\ne j)$ can be shared because
    all we need is to identify columns in the first layer of $\NetworkMultiA(x_1,\cdots,x_d)$ that correspond to the same coordinate.
\end{proof}
 
We next provide how to approximate the reciprocal function $y = \frac{1}{x}$.
Approximation of rational functions has already investigated in \citep{telgarsky2017neural,boulle2020rational}.
However, we found that their bounds (in Lemma 3.5 of \citet{telgarsky2017neural}) of $L=\Ord(\log^7 \epsD^{-1})$ and $\Ord(\log^4 \epsD^{-1})$ nodes can be improved with careful use of local Taylor expansion up to the order of $\Ord(\log \epsD^{-1})$, so we provide our own proof.

\begin{lemma}[Approximating the reciprocal function]\label{Lemma:ApproxInv}
    For any $0<\epsD <1$, there exists $\NetworkInvA \in \Psi(L,W,S,B)$ with $L\leq \Ord(\log^2 \epsD^{-1}), \|W\|_{\infty} = \Ord(\log^3 \epsD^{-1}), S = \Ord(\log^4 \epsD^{-1})$, and $B= \Ord(\epsD^{-2})$ such that
    \begin{align}
        \left|\NetworkInvA(x') - \frac{1}{x}\right| \leq \epsD + \frac{|x'-x|}{\epsD^2}, \quad \text{for all }x\in [\epsD,\epsD^{-1}] \text{ and }x'\in \R.
    \end{align}
\end{lemma}
\begin{proof}
    We approximate the inverse function $y = \frac1x$ with a piece-wise polynomial function.
    We take $x_i = 1.5^i\cdot \epsD \ (i=0,1,\cdots,i^*:=\lceil 2\log_{1.5}\epsD^{-1} \rceil)$ so that $x_{i^*} \geq \epsD^{-1}$ and approximate $y = \frac1x$ in the following way:
    \begin{align}\label{eq:LemmaApproxInv-1}
         \frac1x \fallingdotseq \sum_{i=1}^{i^*} f_i(\NetworkClipA(x;x_{i-1},x_i)) + \frac{1}{\epsD},
    \end{align}
    where $f_i(x)$ is a function that satisfies $f_i(x) = 0$ for $x\leq x_{i-1}$, $f_i(x) = - \frac{1}{x_{i-1}} + \frac{1}{x_i}$ for $x_{i}\leq x$, and 
    \begin{align}
        \max_{x_{i-1}\leq x \leq x_i} |f_i(x) - 1/x + 1/x_{i-1}| \leq \frac{\epsD}{2}.
    \end{align}
    
    Now we show construction of such functions.
    First, by $\frac{1}{x} =\frac{1}{x_{i-1} }\frac{x_{i-1}}{x}= \frac{1}{x_{i-1} } \sum_{l'=1}^\infty (-\frac{x}{x_{i-1}}+1)^{l'}\ (1\leq \frac{x}{x_{i-1}}\leq 1.5)$, let
    \begin{align}
        \tilde{f}_i (x) = \frac{1}{x_{i-1}} \sum_{l'=1}^{l} (-x/x_{i-1} + 1)^{l'} - \frac{1}{x_{i-1}}
        .
    \end{align}
    The difference between $\tilde{f}_i (x)$ and $\frac{1}{x} - \frac{1}{x_{i-1}}$ is $((x_{i-1}-x)/x_{i-1})^{l+1}/x$, which is bounded by $2^{-l-1}/x$.
    Moreover, by adding $\frac{(\frac{1}{x_i}-\tilde{f}_i (x_i))(x - x_{i-1})}{x_i - x_{i-1}} = \frac{((x_{i-1}-x_i)/x_{i-1})^{l+1}(x - x_{i-1})}{x_i(x_i - x_{i-1})}$ to $\tilde{f}_i (x) $, we have $f_i(x)$, with $f_i(x_{i-1}) = 0$, $f_i(x_{i}) = - \frac{1}{x_{i-1}} + \frac{1}{x_i}$, and 
    \begin{align}
        \max_{x_{i-1}\leq x \leq x_i} |f_i(x) - 1/x + 1/x_{i-1}| \leq 2^{-l}/x \leq 2^{-l}\epsD^{-1}.
    \end{align}
    Thus, we take $l = \lceil \log_2 2\epsD^{-1}\rceil$ so that RHS is smaller than $\frac{\epsD}{2}$.
    Therefore, we finally have the explicit approximation of $y=\frac1x$:
    \begin{align}\label{eq:LemmaApproxInv-2}
       f(x) = &\underbrace{\sum_{i=1}^{i^*} \frac{1}{x_{i-1}}\sum_{l'=1}^l (-\NetworkClipA(x;x_{i-1},x_i))/x_{i-1} + 1)^{l'}}_{\mathrm{(a)}} - \sum_{i=1}^{i^*} \frac{1}{x_{i-1}}\\& + \underbrace{\sum_{i=1}^{i^*}\frac{((x_{i-1}-x_i)/x_{i-1})^{l+1}(\NetworkClipA(x;x_{i-1},x_i)) - x_{i-1})}{x_i(x_i - x_{i-1})}}_{\mathrm{(b)}}+ \frac{1}{\epsD}.
    \end{align}
    
    From \cref{Lemma:BaseNN02}, $(-\NetworkClipA(x;x_{i-1},x_i))/x_{i-1} + 1)^{l'}$ is realized by 
    $L = \Ord((\log \log \epsD^{-1} + \log \epsD^{-1})\log \log \epsD^{-1}), \|W\|_\infty = \Ord(\log \epsD^{-1}), S = \Ord(\log \epsD^{-1}(\log \log \epsD^{-1} + \log \epsD^{-1})), B=1.5^{\lceil \log_2 2\epsD^{-1}\rceil} = \Ord(\epsD^{-1})$ so that approximation error for each is bounded by $\Ord(\epsD^2/li^*)$.
    Because there are $\Ord(li^*)$ terms in $\mathrm{(a)}$ of \eqref{eq:LemmaApproxInv-2}, from \cref{Lemma:ConcateNetwork,Lemma:ParallelNetwork}, the final approximation error of $f(x)$ using a neural network $\NetworkInvA$ is $\frac{\epsD}{2}$, where $\NetworkInvA\in \Phi(L,W,S,B)$ with $L\leq \Ord((\log \log \epsD^{-1} + \log \epsD^{-1})\log \log \epsD^{-1}), \|W\|_{\infty} = \Ord(\log^3 \epsD^{-1}), S = \Ord(\log^3 \epsD^{-1}(\log \log \epsD^{-1} + \log \epsD^{-1}))$, and $B= \Ord(\epsD^{-2})$. (Here $B= \Ord(\epsD^{-2})$ is calculated because in $\mathrm{(b)}$ we need to bound the coefficient $\frac{((x_{i-1}-x_i)/x_{i-1})^{l+1}}{x_i(x_i-x_{i-1})}$ by $\eps^{-2}$.)

    The sensitivity analysis follows from $|\NetworkInvA(x') - \frac{1}{x}| \leq |\NetworkInvA(x') - \frac{1}{\max\{x',\epsD\}}| + |\frac{1}{\max\{x',\epsD\}} - \frac{1}{x}|$.
\end{proof}
Combining \cref{Lemma:ApproxInv,Lemma:BaseNN02}, we have the following corollary.
\begin{corollary}\label{Corollary:ApproxInv}
    For any $0<\epsD <1$, there exists $\NetworkInvA \in \Psi(L,W,S,B)$ with $L\leq \Ord(\log^2 l  + \log^2 \epsD) ), \|W\|_{\infty} = \Ord(l + \log^3 \epsD^{-1}), S = \Ord(l\log l + l \log\epsD^{-1} +\log^4 \epsD^{-1})$, and $B= \Ord(\epsD^{- (2\lor l)})$ such that
    \begin{align}
        \left|\NetworkInvA(x';l) - \frac{1}{x^l}\right| \leq \epsD + l \frac{|x'-x|}{\epsD^{l+1}}, \quad \text{for all }x\in [\epsD,\epsD^{-1}] \text{ and }x'\in \R.
    \end{align}
\end{corollary}
\begin{proof}
    Consider $\NetworkMultiB (\cdot; l) \circ \NetworkInvA$.
    The result directly follows from \cref{Lemma:BaseNN02} and \cref{Lemma:ApproxInv}.
\end{proof}

In the same way, by using Taylor expansion of $\sqrt{1+x}$ at each interval defined in the above proof, we can obtain a similar result for $y=\sqrt{x}$.
\begin{lemma}[Approximating the root function]\label{Lemma:ApproxRoot}
    For any $0<\epsD<1$, there exists $\NetworkRootA \in \Psi(L,W,S,B)$ with $L\leq \Ord(\log^2 \epsD^{-1}), \|W\|_{\infty} = \Ord(\log^3 \epsD^{-1}), S = \Ord(\log^4 \epsD^{-1})$, and $B= \Ord(\epsD^{-1})$ such that
    \begin{align}
        \left|\NetworkRootA(x') - \sqrt{x}\right| \leq \epsD + \frac{|x'-x|}{\sqrt{\epsD}}, \quad \text{for all }x\in [\epsD,\epsD^{-1}] \text{ and }x'\in \R.
    \end{align}
\end{lemma}

\subsection{How to deal with exponential functions}\label{subsection:Preparation-Exponential}

We sometimes need to approximate certain types of integrals where the integrand contains a density function of some Gaussian distribution and the integral interval is $\R^d$.
for example, the diffused B-spline basis is a typical example of them.
To deal with them, we adopt the following two-step argument:
first we clip the integral interval, and next we approximate the integrand with rational functions.
We need rational functions because the density function depends on the inverse of (the squared-root of) the variance, which depends on $t$ and should be approximated.
The first lemma corresponds to the first step, and the second and third correspond to the second step, respectively.
\begin{lemma}[Clipping of integrals]\label{Lemma:ClipInt}
    Let $x\in \R^d$, $0<m_t\leq 1$, $\alpha \in \Z_+^d$ with $\sum_{i=1}^d \alpha_i \leq k$, and $f$ be an any function on $\R^d$ whose absolute value is bounded by $\ConstDensityBoundB$.
    For any $0<\epsF<\frac12$, there exists a constant $\ConstDifBoundE$ that only depends on $k$ and $d$, such that
    \begin{align}
      &  \left|
        \int_{\R^d} \prod_{i=1}^d \left(\frac{m_t y_i-x_i}{\sigma_t}\right)^{\alpha_i}f(y)\frac{1}{\sigma_t^{d}(2\pi)^\frac{d}{2}}\exp\left(-\frac{\|m_ty-x\|^2}{2\sigma_t^2}\right)dy
      \right. \\ &\left. \quad\quad\quad\quad\quad\quad\quad\quad\quad\quad\quad\quad\quad -
        \int_{A^x} \prod_{i=1}^d \left(\frac{m_t y_i-x_i}{\sigma_t}\right)^{\alpha_i}f(y)\frac{1}{\sigma_t^{d}(2\pi)^\frac{d}{2}}\exp\left(-\frac{\|m_ty-x\|^2}{2\sigma_t^2}\right)dy
        \right| \lesssim \eps
        ,
    \end{align}
    where $A^x = \prod_{i=1}^d a^x_i $ with $ a^x_i =  [\frac{x_i}{m_t} - \frac{\sigma_t\ConstDifBoundE}{m_t}\sqrt{\log \epsF^{-1}}, \frac{x_i}{m_t} + \frac{\sigma_t\ConstDifBoundE}{m_t}\sqrt{\log \epsF^{-1}}]$.
\end{lemma}
\begin{proof}
    \begin{align}
        &\frac{1}{\sigma_t^{d}(2\pi)^\frac{d}{2}}\left|
        \int_{\R^d} \prod_{i=1}^d \left(\frac{m_t y_i-x_i}{\sigma_t}\right)^{\alpha_i}f(y)\exp\left(-\frac{\|m_ty-x\|^2}{2\sigma_t^2}\right)dy
        -
        \int_{A^x} \prod_{i=1}^d \left(\frac{m_t y_i-x_i}{\sigma_t}\right)^{\alpha_i}f(y)\exp\left(-\frac{\|m_ty-x\|^2}{2\sigma_t^2}\right)dy
        \right|
        \\ & \leq
         \frac{\ConstDensityBoundB}{\sigma_t^{d}(2\pi)^\frac{d}{2}} \int_{\R^d \setminus A^x} \prod_{i=1}^d \left(\frac{|m_t y_i-x_i|}{\sigma_t}\right)^{\alpha_i}\mathbbm{1}[\|y\|_\infty \leq 1]\exp\left(-\frac{\|m_ty-x\|^2}{2\sigma_t^2}\right)dy  
         \quad  (\text{by $|f(y)|\leq \ConstDensityBoundB$})
        \\ & \leq
          \frac{\ConstDensityBoundB}{\sigma^{d}(2\pi)^\frac{d}{2}}\sum_{i=1}^d \int_{\underbrace{\R \times \cdots \times \R}_{i-1 \text{ times}} \times (\R \setminus a^x_i) \times \underbrace{\R \times \cdots \times \R}_{d-i \text{ times}} } \prod_{j=1}^d \left(\frac{|m_t y_j-x_j|}{\sigma_t}\right)^{\alpha_j}\mathbbm{1}[|y_j|\leq 1]\exp\left(-\frac{\|m_ty-x\|^2}{2\sigma_t^2}\right)dy  
         \\ & = 
        \ConstDensityBoundB \sum_{i=1}^d \prod_{j=1}^d \left(\mathbbm{1}[i\ne j] \int_{\R}\left(\frac{|m_t y_j-x_j|}{\sigma_t}\right)^{\alpha_j} \frac{\mathbbm{1}[|y_j|\leq 1]}{\sigma_t(2\pi)^\frac{1}{2}}\exp\left(-\frac{(m_ty_j-x_j)^2}{2\sigma_t^2}\right)dy_j  
         \right.   \\ & \quad\quad\quad\quad\quad\quad\quad\quad\quad\quad\quad\quad\quad \left.      
        +
        \mathbbm{1}[i= j]\int_{\R\setminus a^x_i}\left(\frac{|m_t y_j-x_j|}{\sigma_t}\right)^{\alpha_j}\frac{\mathbbm{1}[|y_j|\leq 1]}{\sigma_t(2\pi)^\frac{1}{2}}\exp\left(-\frac{(m_ty_j-x_j)^2}{2\sigma_t^2}\right)dy_j
        \right)
        .\label{eq:LemmaClipInt-1}
    \end{align}
    We now bound each term.
    First, 
    \begin{align}
        \int_{\R}\left(\frac{|m_t y_j-x_j|}{\sigma_t}\right)^{\alpha_j}\frac{\mathbbm{1}[|y_j|\leq 1]}{\sigma_t(2\pi)^\frac{1}{2}}\exp\left(-\frac{(m_ty_j-x_j)^2}{2\sigma_t^2}\right)dy_j
     \leq
        \begin{cases}
        \frac{1}{m_t}\int_{\R}|y_j'|^{\alpha_j}\frac{1}{(2\pi)^\frac{1}{2}}\exp\left(-\frac{y_j'^2}{2}\right)dy_j' \quad \left(\frac{m_ty_j-x_j}{\sigma_t} = y_j'\right)
        \\
       \frac{2^{d+\alpha_j}}{\sigma_t^{\alpha_j+1}(2\pi)^\frac{1}{2}}\quad (\text{because of the term of $\mathbbm{1}[|y_j|\leq 1].$})
        \end{cases}
    \end{align}
    Thus, LHS can be bounded by $\lesssim\max\left\{\frac{1}{m_t},\frac{1}{\sigma_t^{\alpha_j+1}}\right\}\lesssim 1$.

    Next, 
        \begin{align}
        &
            \int_{\R\setminus a^x_i}\left(\frac{|m_t y_j-x_j|}{\sigma_t}\right)^{\alpha_j}\frac{\mathbbm{1}[|y_j|\leq 1]}{\sigma_t(2\pi)^\frac{1}{2}}\exp\left(-\frac{(m_ty_j-x_j)^2}{2\sigma_t^2}\right)dy_j
            \label{eq:Mouhitotsunobound}
        \\ &\leq
            \frac{2}{m_t}\int_{\ConstDifBoundE\sqrt{\log \epsF^{-1}}}^\infty 
            |y_j|^{\alpha_j}\exp\left(-\frac{y_j^2}{2}\right) \dy_i \quad \left(\text{by letting $\frac{m_t y_j-x_j}{\sigma_t}\mapsto y_j$}\right)
        \\ &\leq 
        \begin{cases}
        \frac{2}{m_t}\sum_{l=0}^{\frac{\alpha_j-1}{2}}\frac{(\alpha_j-1)!!}{(2l)!!}(\ConstDifBoundE^2\log \epsF^{-1})^{l}\epsC^{\frac{\ConstDifBoundE}{2}}
        & (\text{ if }\alpha_j \text{ is odd})
        \\
        \frac{2}{m_t}\sum_{l=1}^{\frac{\alpha_j}{2}}\frac{(\alpha_j-1)!!}{(2l-1)!!}(\ConstDifBoundE^2\log \epsF^{-1})^{l}\epsC^{\frac{\ConstDifBoundE}{2}} + \frac{2}{m_t}\int_{\ConstDifBoundE\sqrt{\log \epsF^{-1}}}^\infty\exp\left(-\frac{y_j^2}{2}\right) \dy_j& (\text{ if } \alpha_j \text{ is even})
        .
        \end{cases}
    \end{align}
    Therefore, by setting $\ConstDifBoundE$ sufficiently large, in a way that $\ConstDifBoundE$ depends on $\alpha_j (\leq k)$ and $d$, this can be bounded by $\frac{\epsF}{m_t}$.
    Moreover, if $m_t\gtrsim 1$, then the integral interval does not overlap with $-1\leq y_j \leq 1$, and in this case \eqref{eq:Mouhitotsunobound} is alternatively bounded by $0$.

    Therefore, \eqref{eq:LemmaClipInt-1} can further be bounded by 
    \begin{align}
        \eqref{eq:LemmaClipInt-1}
        \lesssim 
      \sum_{i=1}^d \prod_{j=1}^d 1^{d-1} \cdot  \eps \lesssim \eps,
    \end{align}
    which gives the assertion.
\end{proof}

Next we give the ways of Taylor expansion of exponential functions with polynomials (\cref{Lemma:TaylorExp}) and with neural networks (\cref{Lemma:TaylorExp2}), respectively.
\begin{lemma}[Approximating an exponential function with polynomials]\label{Lemma:TaylorExp}
    Let $A>0$ and $0\leq m_t\leq 1$. For $t\geq \max\{4eA^2,\lceil \log_2 \epsG^{-1} \rceil\}$, we have that
    \begin{align}
        \left|\exp\left(-\frac{(x-m_ty)^2}{2\sigma_t^2}\right) -   
        \sum_{s=0}^{t-1} \frac{(-1)^s}{s!}\frac{(x-m_ty)^{2s}}{2^s \sigma_t^{2s}} \right|
        \leq 
         \epsG
    \end{align}
    for all $y\in [\frac{-\sigma_t A+x}{m_t}, \frac{\sigma_t A+x}{m_t}]$.
\end{lemma}
\begin{proof}
   By standard Taylor expansion of $e^z$ up to degree $t-1$, we have
    \begin{align}
        \exp\left(-\frac{(x-m_ty)^2}{2\sigma_t^2}\right) =   
        \sum_{s=0}^{t-1} \frac{(-1)^s}{s!}\frac{(x-m_ty)^{2s}}{2^s \sigma_t^{2s}} + \frac{(-1)^t}{t!}\frac{(\theta(x-m_ty))^{2t}}{2^t \sigma_t^{2t}}
    \end{align}
    with some $\theta\in (0,1)$.
    We bound the second term of the residual. When 
    $y\in [\frac{-\sigma_t A+x}{m_t}, \frac{\sigma_t A+x}{m_t}]$
    and $t$ is the minimum integer satisfying $t\geq \max\{4eA^2,\lceil \log_2 \epsG^{-1} \rceil\}$, we have
    \begin{align}
        \frac{1}{t!}\frac{(\theta(x-m_ty)+(1-\theta)x)^{2t}}{2^t \sigma_t^{2t}}
    \leq
        \frac{(2 \sigma_t A)^{2t}}{t! 2^t \sigma_t^{2t}}
    \leq
        \frac{(2 \sigma_t A)^{2t}}{(t/e)^t \cdot 2^t \sigma_t^{2t}}
     \leq 
        \frac{2^t A^{2t} }{(4A^2)^{t}}
    \leq  
        \frac{1}{2^t } 
    \leq
        \epsG,
    \end{align}
    where we used the fact $t! \geq (t/e)^t$.
\end{proof}

\begin{lemma}[Approximating an exponential function with a neural network]\label{Lemma:TaylorExp2}
    Take $\eps>0$ arbitrarily.
    There exists a neural network $\NetworkExpA\in \Phi(L,W,S,B)$ such that
    \begin{align}
        \sup_{x, x'\geq 0}\left|e^{-x'} - \NetworkExpA(x)\right| \leq \eps + |x-x'|
    \end{align}
    holds, where
    $L = \Ord(\log^2 \eps^{-1}), \|W\|_\infty =\Ord(\log \eps^{-1}),S = \Ord(\log^2 \eps^{-1}), B = \exp(\Ord(\log^2 \eps^{-1}))$.
    Moreover, $|\NetworkExpA(x)| \leq \eps$ for all $x \geq \log 3 \eps^{-1}$.
\end{lemma}
\begin{proof}
    Let us take $A = \log 3\eps^{-1}$.
    From Taylor expansion, for all $x$ in $0\leq x\leq A$, we have
    \begin{align}
        \left| e^{-x} - \sum_{i=0}^{k-1} \frac{(-1)^i}{i!}x^i \right|\leq  \frac{A^k}{k!}
        .
    \end{align}
    Moreover, we can evaluate RHS as $\frac{A^k}{k!} \leq \left(\frac{eA}{k}\right)^k$, so by taking $k = \max\{2eA,\lceil\log_2 3\eps^{-1}\rceil\}$, we can bound the RHS by $\frac{\eps}{3}$.
    Now we approximate each $x^i$ using \cref{Lemma:BaseNN02} with $d=\Ord(A+\log \eps^{-1}), C = \Ord(A), \eps=\frac{\eps}{3k}$ and aggregate them using \cref{Lemma:ParallelNetwork}.
    This gives the neural network with $L = \Ord(A^2 + \log^2 \eps^{-1}), \|W\|_\infty =\Ord(A+\log \eps^{-1}),S = \Ord(A^2 + \log^2 \eps^{-1}), B = \exp(\log A \cdot \Ord(A + \log \eps^{-1}))$.
    Finally, we add two layers $\NetworkClipA(x;0,A)$ before this neural network to limit the input within $x>0$.
    Then, we obtain a neural network $\NetworkExpA$ that approximates $e^{-x}$ with an additive error up to $\frac{2\eps}{3}$ in $[0,A]$.
    Moreover, for $x>A$, we have $|\NetworkExpA(x) - e^{-x}| \leq |e^{-x}-e^{-A}| + |\NetworkExpA(A) - e^{-A}|\leq \frac{\eps}{3} + \frac{2\eps}{3} = \eps$.

    The sensitivity analysis follows from $|\NetworkExpA(x') - e^{-x}| \leq |\NetworkExpA(\max\{x',0\}) - e^{-x}| \leq |\NetworkExpA(\max\{x',0\}) - e^{-\max\{x',0\}}| + |e^{-\max\{x',0\}} - e^{-x}| \leq \eps + |\max\{x',0\}-x|\leq \eps + |x'-x|$.
\end{proof}

\subsection{Existing results for approximation}\label{subsection:Appendix-F4}

Our diffused B-spline basis decomposition (\cref{section:Approximation,section:Appendix-Approximation}) is built on the B-spline basis decomposition of the Besov space \citep{devore1988interpolation,suzuki2018adaptivity}.
The following fact can be found in Lemma 2 of \citet{suzuki2018adaptivity} (although the original version adopts $\Omega = [0,1]^d$, we can easily adjust the difference by dividing the domain into cubes with each side length $1$).
The magnitude of $|\alpha_{k,j}|$ is evaluated in p.17 of \citet{suzuki2018adaptivity}.
\begin{lemma}[Approximability of the Besov space (\citet{suzuki2018adaptivity})]\label{Lemma:SuzukiBesov}
    Let $ C>0$. Under $s>d(1/p - 1/r)_+$ and $0<s<\min\{l,l-1+1/p\}$ where $l\in \N$ is the order of the cardinal B-spline bases, for any $f\in B_{p,q}^s([-C,C]^d)$, there exists $f_N$ that satisfies
    \begin{align}\label{eq:SuzukiBesov-1}
        \|f - f_N\|_{L^r([-C,C]^d)} \lesssim C^s N^{-s/d} \|f\|_{B^s_{p,q}([-C,C]^d)}
    \end{align}
    for $N \gg 1$, and has the following form:
    \begin{align}
        f_N(x) = \sum_{k=0}^K \sum_{j \in J(k)} \alpha_{k,j} M_{k,j}^d (x) +
        \sum_{k=K+1}^{K^*} \sum_{i=1}^{n_k}\alpha_{k,j_i} M_{k,j_i}^d (x)\quad 
        \text{with }\sum_{k=0}^K |J(k)|+\sum_{k=K+1}^{K^*}n_k=N,
    \end{align}
    where $ J(k)=\{-C2^{k}-l, -C2^{k}-l+1,\cdots C2^{k}-1,  C2^{k}\}$, $(j_i)_{i=1}^{n_k} \subseteq J(k)$, $K=\Ord(d^{-1}\log (N/C^d))$, $K^* = (\Ord(1)+\log (N/C^d)) \nu^{-1} + K, n_k = \Ord((N/C^d)2^{-\nu(k-K)})\ (k=K+1,\cdots,K^*)$ for $\delta = d(1/p - 1/r)_+$ and $\nu = (s-\delta)/(2\delta)$.
    Moreover, $|\alpha_{k,j}| \lesssim N^{(\nu^{-1} + d^{-1})(d/p - s)_+}$.
\end{lemma}


\subsection{Elementary bounds for the Gaussian and hitting time}
\begin{lemma}\label{Lemma:GaussianBound}

    Let $0<\eps\ll 1$, $l\in \Z_+^d$, and $p(x)$ be the density funciton of $\mathcal{N}(0,\sigma_t^2I_d)$, i.e., $p(x)= \frac{1}{\sigma_t^d(2\pi)^\frac{d}{2}}\exp\left(-\frac{\|x\|^2}{\sigma_t^2}\right)$.
    Then, the following bound holds:
    \begin{align}
        \int_{\|x\|_\infty \geq \sigma_t \sqrt{4\log dl\eps^{-1}}}\frac{\prod_{i=1}^d x_i^{l_i}}{\sigma^{\sum_{i=1}^d l_i}} p(x) \dx \lesssim \eps .
    \end{align}
        We sometimes write $\sqrt{4\log dl\eps^{-1}}=C_{\mathrm{f},2}\sqrt{\log \eps^{-1}}$.
\end{lemma}
\begin{proof}
 Let us denote $x^l=\prod_{i=1}^d x_i^{l_i}$ and $|l|=\sum_{i=1}^d l_i$ for simple presentation.
    Let $r=\|x\|_\infty$, and we get
    \begin{align}
     &   \int_{\|x\|_\infty \geq \sigma_t \sqrt{4\log \eps^{-1}}}\frac{x^l}{\sigma_t^{|l|}} p(x) \dx \\ 
     &   \int_{\|x\|_1 \geq \sigma_t \sqrt{4\log \eps^{-1}}}\frac{x^l}{\sigma_t^{|l|}} p(x) \dx \\ 
     &\leq
        \int_{r=\sigma_t \sqrt{4\log \eps^{-1}}}^\infty\frac{r^{|l|}}{\sigma_t^{|l|}}   \frac{1}{\sigma_t^d(2\pi)^\frac{d}{2}}\exp\left(-\frac{r^2}{2\sigma^2}\right)(d-1)r^{d-1}\mathrm{d}r
        \\ & = \int_{s=\sqrt{4\log \eps^{-1}}}^\infty  s^{|l|+d-1}  \frac{1}{(2\pi)^\frac{d}{2}}\exp\left(-\frac{s^2}{2}\right)(d-1)\mathrm{d}s
        \quad (\text{by letting $s=r/\sigma_t$})
        \\ & = \frac{(4\log \eps^{-1})^{(|l|+d-1)/2}}{(2\pi)^\frac{d}{2}}\exp\left(-\frac{4\log \eps^{-1}}{2}\right)(d-1)
        +\int_{s=\sqrt{4\log \eps^{-1}}}^\infty  \frac{(|l|+d-1) s^{|l|+d-2}}{{(2\pi)^\frac{d}{2}}}\exp\left(-\frac{s^2}{2}\right)(d-1)\mathrm{d}s
        \\ & = \cdots = 
        \sum_{0\leq i \leq \lfloor\frac{|l|+d-1}{2}\rfloor}\frac{\frac{(|l|+d-1)!!}{(|l|+d-1-2i)!!}(4\log \eps^{-1})^{(|l|+d-1-2i)/2}(d-1)}{(2\pi)^\frac{d}{2}}\eps^2
        \\ &\quad\quad\quad\quad\quad\quad +
        \begin{cases}
           \int_{s=\sqrt{4\log \eps^{-1}}}^\infty  \frac{(|l|+d-1)!! }{{(2\pi)^\frac{d}{2}}}\frac{1}{(2\pi)^\frac{d}{2}}\exp\left(-\frac{s^2}{2}\right)(d-1)\mathrm{d}s
            & (\text{$|l|+d$: even})
        \\
        0
          &( \text{$|l|+d$: odd})
        \end{cases} \quad (\text{by iterating integration by parts})
        \\ & \lesssim \eps^2 \log^\frac{d+|l|-1}{2} \eps^{-1}.\label{eq:Elementary-bound-2}
    \end{align}
    Replacing $\eps$ by $\eps/dl$, RHS of \eqref{eq:Elementary-bound-2} is bounded by
    \begin{align}
        \frac{\eps^2}{d^2l^2} \log^\frac{dn+|l|-1}{2} (\eps/dl)^{-1} \lesssim \eps,
    \end{align}
    which yields the conclusion.
\end{proof}

\begin{lemma}\label{Lemma:HittingTime}
    Let $(B_s)_{[0,t]}$ be the $1$-dimensional Brownian motion and $X_t = \int_0^t \beta_s \mathrm{d}B_s$, with $\beta_s \leq \betahigh$.
    Then, we have that
    \begin{align}
        \mathbb{P}\left[\sup_{s\in [0,t]} |X_t| \geq2\sqrt{\betahigh t\log(2 \eps^{-1})}\right]\leq \eps.
    \end{align}
\end{lemma}
\begin{proof}
    We bound the case $\beta_s\equiv \betahigh$ because it maximize the hitting probability.
 According to \citet{karatzas1991brownian}, for $x>0$,
 \begin{align}
     \mathbb{P}\left[\sup_{s\in [0,t]} |X_t| \geq x\right] = \frac{4}{\sqrt{2\pi}} \int_{\frac{x}{\sqrt{2\betahigh t}}}^\infty e^{-y^2/2}\dy
     = \frac{4}{\sqrt{2\pi}} \int_{\frac{x}{\sqrt{4\betahigh t}}}^\infty e^{-z^2}\sqrt{2}\mathrm{d}z
    \leq
     2e^{-x^2/4\betahigh t}.
 \end{align}
 For the second equality, we simply replaced $y/\sqrt{2}$ with $z$.
 For the last inequality, we used $\frac{4}{\sqrt{2\pi}}\cdot \sqrt{2}\leq 2$ and $\int_x^\infty e^{-y^2}\dy \leq e^{-x^2}$. 
 Therefore, setting $x=2\sqrt{\betahigh t\log(2 \eps^{-1})}$ yields the assertion.
\end{proof}

\end{document}